\documentclass{article} 
\usepackage{iclr2023_conference,times}
\iclrfinaltrue

\usepackage[T1]{fontenc}    
\usepackage{hyperref}
\usepackage{url}
\usepackage{graphicx}       
\usepackage{subfigure}
\usepackage{amsmath}
\usepackage{amssymb}
\usepackage{cleveref}
\usepackage{float}
\usepackage{amsthm}
\usepackage{comment}
\usepackage{bm}
\usepackage{xcolor}
\usepackage{mathrsfs}
\usepackage{bbm}
\usepackage{xfrac}
\usepackage{float}
\usepackage{wrapfig}
\newtheorem{definition}{Definition}[section]
\newtheorem{theorem}{Theorem}[section]
\newtheorem{assumption}{Assumption}[section]

\newtheorem{lemma}{Lemma}[section]
\newtheorem{corollary}{Corollary}[section]

\newtheorem{hypothesis}{Hypothesis}[section]

\newtheorem{finding}{Finding}[section]
\newtheorem{remark}{Remark}[section]

\usepackage{algorithm2e}
\RestyleAlgo{ruled}
\SetKwFor{ForP}{for}{do in parallel} {end}

\SetCommentSty{myalgcommfont}
\SetKwProg{Fn}{Function}{:}{end}
\LinesNumbered
\SetAlgoHangIndent{1.2em}

\newcommand{\longbo}[1]{{\color{blue}Longbo: #1}}

\crefname{finding}{Finding}{Findings}
\crefname{hypothesis}{Hypothesis}{Hypotheses}
\crefname{algocf}{Algorithm}{Algorithms}

\usepackage[shortlabels]{enumitem}
\setlist[enumerate,1]{leftmargin=0.6cm}

\usepackage{mycustom}
\newcommand{\Bloc}{B_{\mathrm{loc}}}

\newcommand{\vmu}{\bm{\mu}}
\newcommand{\D}{\mathrm{D}}
\newcommand{\vp}{\bm{p}}
\newcommand{\E}{\mathbb{E}}
\newcommand{\PP}{\mathbb{P}}

\newcommand{\cD}{\mathcal{D}}
\newcommand{\ctD}{\tilde{\mathcal{D}}}
\newcommand{\cO}{\mathcal{O}}
\newcommand{\ctO}{\mathcal{\tilde{O}}}
\newcommand{\cB}{\mathcal{B}}

\newcommand{\cA}{\mathcal{A}}
\newcommand{\cAs}[1][s]{\mathcal{A}^{(#1)}}
\newcommand{\cR}{\mathcal{R}}
\newcommand{\tcR}{\tilde{\mathcal{R}}}

\newcommand{\cE}{\mathcal{E}}
\newcommand{\cM}{\mathcal{M}}

\newcommand{\cS}{\mathcal{S}}

\newcommand{\cEs}[1][s]{\cE^{(#1)}}
\newcommand{\bcEs}[1][s]{\bar{\cE}^{(#1)}}

\newcommand{\cZ}{\mathcal{Z}}
\newcommand{\cZeps}{\mathcal{Z}^{\epsilon}}
\newcommand{\cL}{\mathcal{L}}
\newcommand{\cT}{\mathcal{T}}

\newcommand{\cV}{\mathcal{V}}

\newcommand{\cPX}{\mathcal{P}_{\vX}}
\newcommand{\cC}{\mathcal{C}}
\newcommand{\cPz}{\mathcal{P}_{\vzeta}}

\newcommand{\cI}{\mathcal{I}}

\newcommand{\R}{\mathbb{R}}
\newcommand{\Rg}{R_{\mathrm{grp}}}
\newcommand{\Hg}{H_{\mathrm{grp}}}

\newcommand{\N}{\mathbb{N}}

\newcommand{\etae}{\eta_{\mathrm{e}}}

\newcommand{\vw}{\bm{w}}

\newcommand{\tPsi}{\tilde{\Psi}}
\newcommand{\as}[1][s]{a^{(#1)}}
\newcommand{\cs}[1][s]{c^{(#1)}}
\newcommand{\xinran}[1]{\textcolor{brown}{[Xinran: #1]}}
\newcommand{\kaifeng}[1]{\textcolor{orange}{[Kaifeng: #1]}}

\newcommand{\ttheta}{\tilde{\theta}}
\newcommand{\Gat}{\Gamma^{\epsilon_2}}
\newcommand{\Gath}{\Gamma^{\epsilon_3}}
\newcommand{\Gaz}{\Gamma^{\epsilon_0}}
\newcommand{\Gao}{\Gamma^{\epsilon_1}}
\newcommand{\vtheta}{{\bm{\theta}}}
\newcommand{\vphi}{{\bm{\phi}}}
\newcommand{\vphinull}{\vphi_{\mathrm{null}}}
\newcommand{\vu}{{\bm{u}}}
\newcommand{\vz}{{\bm{z}}}

\newcommand{\ve}{\bm{e}}
\newcommand{\vg}{\bm{g}}

\newcommand{\vv}{\bm{v}}
\newcommand{\vx}{\bm{x}}

\newcommand{\vX}{\bm{X}}

\newcommand{\dvvs}[1][s]{\dot{\vv}^{(#1)}}
\newcommand{\tvvs}[1][s]{\tilde{\vv}^{(#1)}}
\newcommand{\hvvs}[1][s]{\hat{\vv}^{(#1)}}
\newcommand{\vzero}{\bm{0}}

\newcommand{\tr}{\mathrm{tr}}
\newcommand{\hvths}[1][s]{\hat{\bm{\theta}}^{(#1)}}

\newcommand{\vXs}[1][s]{{\vX}^{(#1)}}
\newcommand{\bhvths}[1][s]{{\hat{\bm{\theta}}}^{(#1)}_{\mathrm{avg}}}
\newcommand{\hves}[1][s]{\hat{\ve}^{(#1)}}

\newcommand{\vzeta}{\bm{\zeta}}
\newcommand{\tvzeta}{\tilde{\bm{\zeta}}}
\newcommand{\hvzeta}{\hat{\vzeta}}
\newcommand{\Beo}{B^{\epsilon_1}}
\newcommand{\scls}{s_{\mathrm{cls}}}
\newcommand{\vthetanull}{\vtheta_{\mathrm{null}}}
\newcommand{\onec}{\mathbbm{1}}
\newcommand{\scrM}{\bm{m}}
\newcommand{\scrMs}[1][s]{\scrM^{(#1)}}

\newcommand{\scrFsg}[1][s]{\mathscr{F}^{(#1)}_{\mathrm{good}}}

\newcommand{\tvths}[1][s]{\tilde{\vtheta}^{(#1)}}
\newcommand{\cTs}[1][s]{\cT^{(#1)}}

\newcommand{\vh}{\bm{h}}
\newcommand{\vxi}{\bm{\xi}}

\newcommand{\xis}[1][s]{\xi^{(#1)}}

\newcommand{\vect}{\mathrm{vec}}

\newcommand{\mA}{\bm{A}}
\newcommand{\mB}{\bm{B}}

\newcommand{\mH}{\bm{H}}
\newcommand{\mI}{\bm{I}}

\newcommand{\mM}{\bm{M}}
\newcommand{\mP}{\bm{P}}

\newcommand{\tmP}{\tilde{\bm{P}}}
\newcommand{\mSig}{\bm{\Sigma}}
\newcommand{\mSigzperp}{\bm{\Sigma}_{0, \perp}}
\newcommand{\mSigzpara}{\bm{\Sigma}_{0, \parallel}}

\newcommand{\mSigzea}{\bm{\Sigma}_{0, \perp, \parallel}}

\newcommand{\mPsi}{\bm{\Psi}}

\newcommand{\vths}[1][s]{\vtheta^{(#1)}}
\newcommand{\tvus}[1][s]{\tilde{\vu}^{(#1)}}
\newcommand{\dvus}[1][s]{\dot{\vu}^{(#1)}}
\newcommand{\vus}[1][s]{\vu^{(#1)}}
\newcommand{\hvu}{\hat{\vu}}
\newcommand{\hvus}[1][s]{\hat{\vu}^{(#1)}}
\newcommand{\bvths}[1][s]{\bar{\vtheta}^{(#1)}}

\newcommand{\hvxs}[1][s]{\hat{\vx}^{(#1)}}

\newcommand{\hvxsT}[1][s]{\hat{{\vx}}^{(#1) \top}}

\newcommand{\bhvxs}[2][s]{{\hat{\vx}}^{(#1)}_{\mathrm{avg}, #2}}
\newcommand{\bhvxsT}[2][s]{{\hat{\vx}}^{(#1)\top}_{\mathrm{avg}, #2}}

\newcommand{\vxs}[1][s]{\vx^{(#1)}}
\newcommand{\vW}{\bm{W}}
\newcommand{\vb}{\bm{b}}
\newcommand{\vbn}[1][n]{\bm{b}^{(#1)}}
\newcommand{\vsign}[1][n]{\bm{\sigma}^{(#1)}}
\newcommand{\sigs}[1][n]{{\sigma}^{(#1)}}

\newcommand{\vsig}{\bm{\sigma}}

\newcommand{\bvxs}[1][s]{\bar{\vx}^{(#1)}}

\newcommand{\vgs}[1][s]{\vg^{(#1)}}
\newcommand{\vzs}[1][s]{\vz^{(#1)}}
\newcommand{\hvzs}[1][s]{{\vz}^{(#1)}}
\newcommand{\vphs}[1][s]{\vphi^{(#1)}}
\newcommand{\hvphs}[1][s]{\hat{\vphi}^{(#1)}}

\newcommand{\hvqs}[1][s]{\hat{\bm{q}}^{(#1)}}

\newcommand{\mAs}[1][s]{\mA^{(#1)}}
\newcommand{\hmAs}[1][s]{\hat{\mA}^{(#1)}}
\newcommand{\bhmAs}[1][s]{\hat{\mA}_{\mathrm{avg}}^{(#1)}}

\newcommand{\As}[1][s]{A^{(#1)}}

\newcommand{\hAs}[1][s]{\hat{A}^{(#1)}}
\newcommand{\bhAs}[3][s]{{\hat{A}_{\mathrm{avg}, #2, #3}}^{(#1)}}

\newcommand{\mMs}[1][s]{\mM^{(#1)}}
\newcommand{\hmMs}[1][s]{\hat{\mM}^{(#1)}}

\newcommand{\hMs}[1][s]{\hat{M}^{(#1)}}

\newcommand{\hmBs}[1][s]{\hat{\mB}^{(#1)}}
\newcommand{\hBs}[1][s]{\hat{B}^{(#1)}}

\newcommand{\bs}[1][s]{b^{(#1)}}
\newcommand{\vDeltas}[1][s]{\bm{\Delta}^{(#1)}}
\newcommand{\dvDeltas}[1][s]{\dot{\bm{\Delta}}^{(#1)}}

\newcommand{\tDelta}{\tilde{\Delta}}
\newcommand{\tvDelta}{\tilde{\bm{\Delta}}}

\newcommand{\hDvphs}[1][s]{\Delta\hat{\vphi}^{(#1)}}

\newcommand{\hDvphsT}[1][s]{\Delta\hat{\vphi}^{(#1)\top}}

\newcommand{\spann}{\mathrm{span}}

\newcommand{\bvDeltas}[1][s]{\bar{\bm{\Delta}}^{(#1)}}
\newcommand{\tvDeltas}[1][s]{\tilde{\bm{\Delta}}^{(#1)}}
\newcommand{\hvDeltas}[1][s]{\hat{\bm{\Delta}}^{(#1)}}

\newcommand{\tDeltas}[1][s]{\tilde{\Delta}^{(#1)}}
\newcommand{\Deltas}[1][s]{{\Delta}^{(#1)}}

\newcommand{\hvZ}{\hat{\bm{Z}}}
\newcommand{\vZs}[1][s]{\bm{Z}^{(#1)}}
\newcommand{\tvZ}{\tilde{\bm{Z}}}
\newcommand{\dvZs}[1][s]{\dot{\bm{Z}}^{(#1)}}
\newcommand{\tvZs}[1][s]{\tilde{\bm{Z}}^{(#1)}}

\newcommand{\Sig}{\Sigma}
\newcommand{\tC}{\tilde{C}}
\newcommand{\hC}{\hat{C}}
\newcommand{\Cg}[1]{C_{g, #1}}
\newcommand{\mPperp}{\mP_{\perp}}
\newcommand{\mPpara}{\mP_{\parallel}}

\newcommand{\stgood}{\{\text{global\  step\ } (s, t) \text{\ is\ } \eta^{100}\text{-good}\}}

\newcommand{\norm}[1]{\|#1\|}
\newcommand{\abs}[1]{\lvert #1 \rvert}
\newcommand{\labs}[1]{\left\lvert #1 \right\rvert}
\newcommand{\normtwo}[1]{\norm{#1}_2}
\newcommand{\lnormtwo}[1]{\left\|#1\right\|_2}

\newcommand{\Cov}{\mathrm{Cov}}
\newcommand{\inner}[2]{\left\langle{#1}, {#2}\right\rangle}
\newcommand{\inne}[2]{\langle{#1}, {#2}\rangle}
\newcommand{\rank}{\mathrm{rank}}
\newcommand{\diag}{\mathrm{diag}}
\newcommand{\poly}{\mathrm{poly}}

\newcommand{\Rtot}{R_{\mathrm{tot}}}

\newcommand{\dd}{\mathrm{d}}
\newcommand{\gradGa}{\nabla_{\Gamma}}
\newcommand{\sigmax}{\sigma_{\max}}

\newcommand{\floor}[1]{\lfloor #1 \rfloor}
\newcommand{\llfloor}[1]{\left\lfloor #1 \right\rfloor}
\newcommand{\ceil}[1]{\lceil #1 \rceil}

\title{
  Why (and When) does Local SGD Generalize Better than SGD?
}

\author{%
Xinran Gu\thanks{Equal contribution} \\
Institute for Interdisciplinary Information Sciences\\
Tsinghua University\\
\texttt{gxr21@mails.tsinghua.edu.cn} \\
\And
Kaifeng Lyu\footnotemark[1] \\
Department of Computer Science \\
Princeton University \\
\texttt{klyu@cs.princeton.edu} \\
\And
Longbo Huang\thanks{Corresponding authors} \\
Institute for Interdisciplinary Information Sciences\\
Tsinghua University\\
\texttt{longbohuang@tsinghua.edu.cn} \\
\And
Sanjeev Arora\footnotemark[2] \\
Department of Computer Science \\
Princeton University \\
\texttt{arora@cs.princeton.edu}
}

\begin{document}

\maketitle

\begin{abstract}
Local SGD is a communication-efficient variant of SGD for large-scale training, where multiple GPUs perform SGD independently and average the model parameters periodically.
It has been recently observed that Local SGD can not only achieve the design goal of reducing the communication overhead but also lead to higher test accuracy than the corresponding SGD baseline~\citep{lin2020dont}, though the training regimes for this to happen are still in debate~\citep{ortiz2021trade}.
This paper aims to understand why (and when) Local SGD generalizes better based on Stochastic Differential Equation (SDE) approximation.
The main contributions of this paper include (i) the derivation of an SDE that captures the long-term behavior of Local SGD in the small learning rate regime, showing how noise drives the iterate to drift and diffuse after it has reached close to the manifold of local minima, (ii) a comparison between the SDEs of Local SGD and SGD, showing that Local SGD induces a stronger drift term that can result in a stronger effect of regularization, e.g., a faster reduction of sharpness, and (iii) empirical evidence validating that having a small learning rate and long enough training time enables the generalization improvement over SGD but removing either of the two conditions leads to no improvement.

\end{abstract}

\section{Introduction}\label{sec:intro}
As deep models have grown larger,  training them with reasonable wall-clock times has led to new distributed environments and new variants of gradient-based training. 
Recall that Stochastic Gradient Descent (SGD) tries to solve 
$\min_{\vtheta \in \R^d} \E_{\xi \sim \ctD}[\ell(\vtheta; \xi)]$, where $\vtheta
\in \R^d$ is the parameter vector of the model, $\ell(\vtheta; \xi)$ is the loss
function for a data sample $\xi$ drawn from the training distribution $\ctD$, e.g., the uniform distribution over the training set. 
SGD with learning rate $\eta$ and batch size $B$ does the following update at each step, using a batch of $B$ independent  $\xi_{t,1}, \dots, \xi_{t,B} \sim \ctD$:
\begin{equation}\label{eq:upd intro sgd}
    \vtheta_{t+1} \gets \vtheta_t - \eta \vg_t, \quad \text{where} \quad \vg_t = \frac{1}{B}\sum_{i=1}^{B} \nabla \ell(\vtheta_t; \xi_{t,i}).
\end{equation}
{\em Parallel SGD} tries to improve wall-clock time when the batch size $B$ is large enough. It distributes the gradient computation to $K \ge 2$ workers, each of whom focuses on a local batch of $\Bloc := B/K$ samples and computes the average gradient over the local batch. Finally, $\vg_t$ is obtained by averaging the local gradients over the $K$ workers.

However, large-batch training leads to a significant test accuracy drop compared
to a small-batch training baseline with the same number of training steps or
epochs~\citep{smith2020generalization,shallue2019measuring,keskar2017on,jastrzkebski2017three}. Reducing this {\em generalization gap} is the goal of much subsequent research. 
It was suggested that the generalization gap arises because  larger batches lead to a reduction in the level of noise in batch gradient (see \Cref{sec:related}
for more discussion).
The {\em Linear Scaling Rule}~\citep{krizhevsky2014one,goyal2017accurate,jastrzkebski2017three} tries to fix this by increasing the learning rate in proportion to batch size. This is found to reduce the generalization gap for (parallel) SGD, but does not entirely eliminate it. 

To reduce the generalization gap further, \citet{lin2020dont} discovered that a variant of SGD, called {\em Local SGD}~\citep{yu2019parallel,
wang2019adaptive, ijcai2018p447}, can be used as a strong component.
Perhaps surprisingly, Local SGD itself is not designed for improving generalization, but for reducing the high communication cost for synchronization among
the workers, which is another important issue that often bottlenecks large-batch training~\citep{DBLP:conf/interspeech/SeideFDLY14,
DBLP:conf/interspeech/Strom15,chen2016revisiting, DBLP:conf/nips/RechtRWN11}.
Instead of averaging the local gradients per step as in parallel SGD, Local SGD allows $K$ workers to train their models locally and averages the local {\em model parameters} whenever they finish $H$ local steps. Here every worker samples a new batch at each local step, and in this paper we focus on the case where all the workers draw samples with or without replacement from the {\em same} training set. See \Cref{sec:pseudocode} for the pseudocode.

More specifically, \citet{lin2020dont} proposed {\em Post-local SGD}, a hybrid method that starts with parallel SGD (equivalent to
Local SGD with $H=1$ in math) and switches to Local SGD with $H > 1$ after a
fixed number of steps $t_0$. They showed through extensive
experiments that Post-local SGD significantly outperforms parallel SGD
in test accuracy when $t_0$ is carefully chosen.
In \Cref{fig:Postlocal}, we reproduce this phenomenon on both CIFAR-10 and ImageNet.

As suggested by the success of Post-local SGD,
Local SGD can improve the generalization of SGD
by merely adding more local steps (while fixing the other hyperparameters),
at least when the training starts from a model pre-trained by SGD. But the underlying mechanism is not very clear, and there is also controversy about when this phenomenon can happen (see \Cref{subsec:debate} for a survey).
The current paper tries to understand: {\em Why does Local SGD generalize better?} {\em Under what general conditions does this generalization benefit arise?}

\begin{figure}[t]
\vspace{-0.5in}
\begin{center}
    \subfigure[CIFAR-10,  $B=4096$, ResNet-56. ]{
    \includegraphics[width=0.4\textwidth]{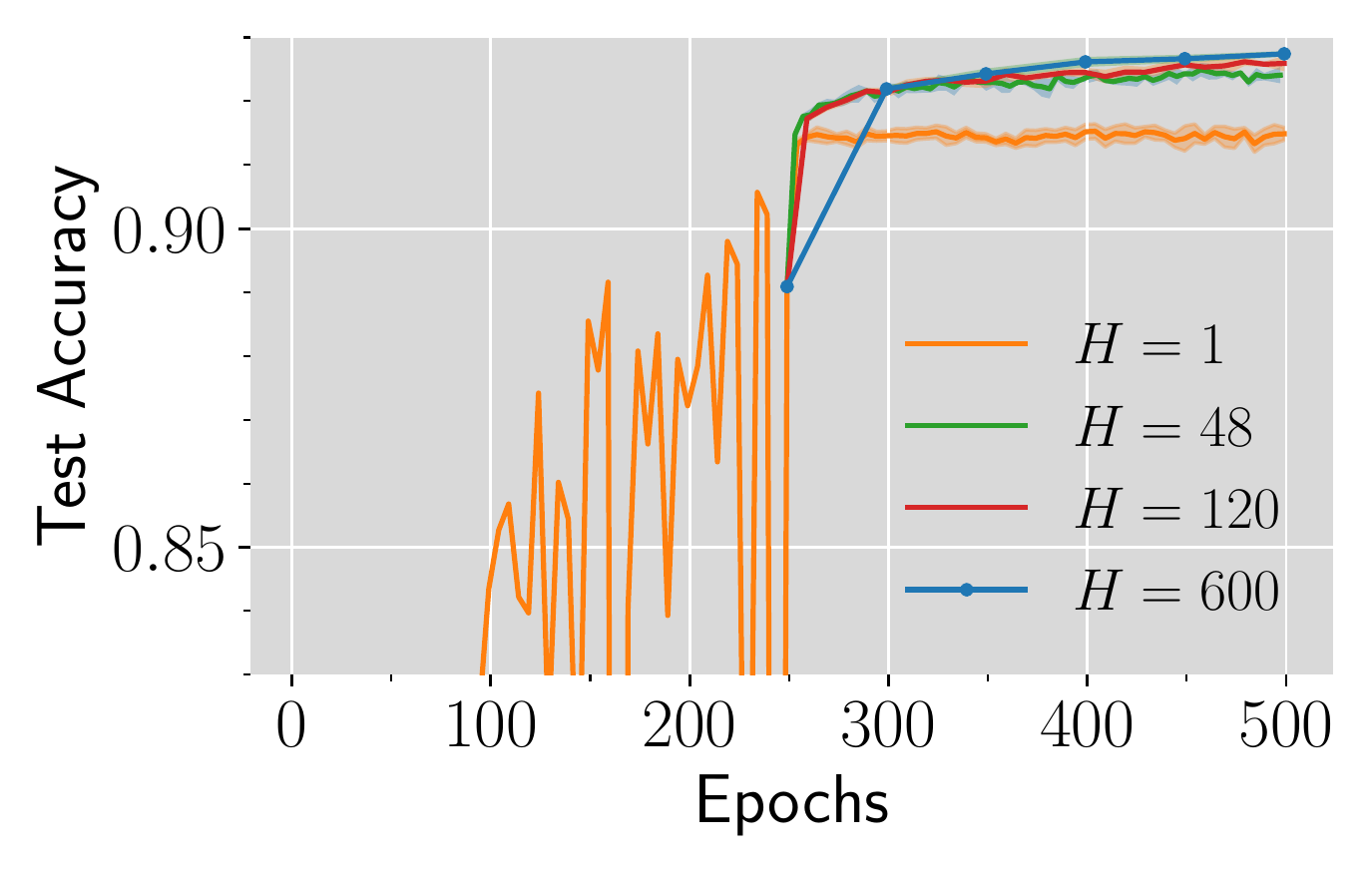}\label{fig:cifar_intro}
    }
    \hspace{0.1 in}
    \subfigure[ImageNet, $B=8192$, ResNet-50. ]{
        \includegraphics[width=0.4\textwidth]{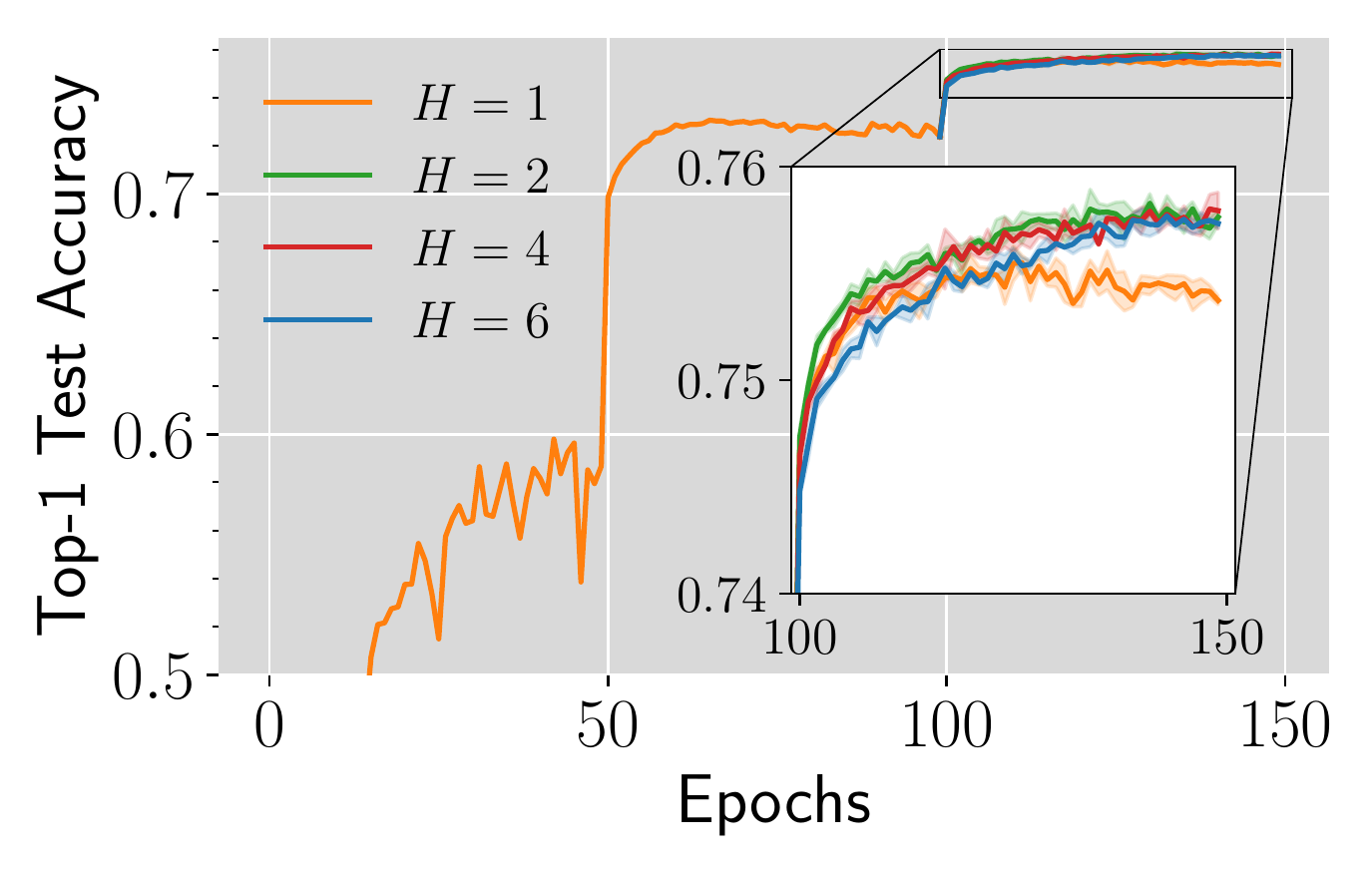}\label{fig:img_intro}
    }
    \vspace{-0.1in}
    \caption{
    Post-Local SGD ($H > 1$) generalizes better than  SGD ($H=1$). We switch to Local SGD at the first learning rate decay (epoch \#250) for CIFAR-10 and at the second learning rate decay (epoch \#100) for ImageNet. 
    See \Cref{sec:intro detail} for training details. }
        \label{fig:Postlocal}
    \vspace{-0.1in}
\end{center}
\end{figure}
Previous theoretical research on Local SGD is mainly restricted to the
convergence rate for minimizing a convex or non-convex objective (see
\Cref{sec:related} for a survey). A related line of works~\citep{stich2018local, yu2019parallel, khaled2020tighter} showed that
Local SGD has a slower convergence rate compared with parallel SGD after
running the same number of steps/epochs. This convergence result suggests that Local SGD may implicitly regularize the model through insufficient
optimization, but this does not explain why parallel SGD with early stopping, which may incur an even higher training loss, still generalizes worse than Post-local SGD.

\myparagraph{Our Contributions.}
In this paper, we provide the first theoretical understanding on why (and when) switching from parallel SGD to Local SGD improves generalization.
\begin{enumerate}
    \item In \Cref{sec:key-factors}, we conduct ablation studies on CIFAR-10 and ImageNet and identify a clean setting where adding local steps to SGD consistently improves generalization: if the learning rate is small and the total number of steps is sufficient, Local SGD eventually generalizes better than the corresponding (parallel) SGD baseline.
    \item In \Cref{subsec:sde-near-man}, we derive a special SDE that characterizes the long-term behavior of Local SGD in the small learning rate regime, as inspired by a previous work~\citep{li2021happens} that proposed this type of SDE for modeling SGD. These SDEs can track the dynamics after the iterate has reached close to a manifold of minima. In this regime, the expected gradient is near zero, but the gradient noise can drive the iterate to wander around. In contrast to the conventional SDE~\eqref{eq:canonical SDE} for SGD, where the drift and diffusion terms are connected respectively to the expected gradient and gradient noise, the SDE we derived for Local SGD has drift and diffusion terms both connected to gradient noise. 
    \item \Cref{subsec:inter} explains the generalization improvement of Local SGD over SGD by comparing the corresponding SDEs: increasing the number of local steps $H$ strengthens the drift term of SDE while keeping the diffusion term untouched. We hypothesize that having a stronger drift term can benefit generalization.
    \item As a by-product, we provide a new proof technique that can give the first quantitative approximation bound for how well \citet{li2021happens}'s SDE approximates SGD.
\end{enumerate}
Back to the discussion on the generalization gap between small- and large-batch training, we remark that this gap can occur early in training when the learning rate is very large~\citep{smith2020generalization} and Local SGD cannot prevent this gap in this phase. Instead, our theory suggests that Local SGD can reduce the gap in late training phases after decaying the learning rate.

\section{When does Local SGD Generalize Better?}\label{sec:regime}

In our motivating example of Post-local SGD, switching from SGD to Local SGD can outperform running SGD alone (i.e., no switching) in test accuracy, but this improvement does not always arise and can depend on the choice of the switching time point. Because of this, a necessary first step for developing a theoretical understanding of Local SGD is to identify {\em under what general conditions} Local SGD can improve the generalization of SGD by merely adding local steps.

\subsection{The Debate on Local SGD}\label{subsec:debate}

We first summarize a debate in the literature regarding {\em when} to switch from SGD to Local SGD in running Post-local SGD, which hints the conditions so that Local SGD can improve upon SGD.




\myparagraph{Local SGD generalizes better than SGD on CIFAR-10.} 
\citet{lin2020dont} empirically observed that Post-local SGD
exhibits a better generalization performance than SGD. Most of their experiments
are conducted on CIFAR-10 and CIFAR-100 with multiple learning rate decays, and the algorithm switches from (parallel) SGD to Local SGD right
after the first learning rate decay. We refer to this particular choice of the
switching time point as the \emph{first-decay switching strategy} for short. To
justify this strategy, they empirically showed that the generalization
improvement can be less significant if starting Local SGD from the beginning or
right after the second learning rate decay. It has also been observed by
\citet{wang2021cooperative} that running Local SGD from the beginning improves
generalization, but  the test accuracy improvement may not be large enough. A
subsequent work by \citet{lin2020extrapolation} showed that adding local steps
to Extrap-SGD, a variant of SGD proposed therein, after the first learning rate
decay also improves generalization, suggesting that the first-decay switching
strategy can also be applied to the post-local variant of other optimizers.

\myparagraph{Does Local SGD exhibit the same generalization benefit on large-scale datasets?}
Going beyond CIFAR-10, \citet{lin2020dont} conducted a few ImageNet experiments
and showed that Post-local SGD with first-decay switching strategy still leads
to better generalization than SGD. However, the improvement is sometimes
marginal, e.g., $0.1\%$ for batch size $8192$. For the general case,
they suggested that the time of switching should be tuned aiming at
``capturing the time when trajectory starts to get into the influence basin of a
local minimum'' in a footnote, but no further discussion or experiments are
provided to justify this guideline. \citet{ortiz2021trade} conducted a more
extensive evaluation on ImageNet (with a different set of
hyperparameters) and concluded with the opposite: the first-decay switching
strategy can hurt the validation accuracy.
Instead, switching at a later time, such as the second learning rate decay,
 leads to a better validation accuracy than SGD.\footnote{This
generalization improvement is not mentioned explicitly in \citep{ortiz2021trade}
but can be clearly seen from Figures 7 and 8 in their paper.} To explain this
phenomenon, they conjecture that switching to Local SGD has a
regularization effect that is beneficial only in the short-term, so it is always
better to switch as late as possible. They further conjecture that this discrepancy between CIFAR-10 and
ImageNet is mainly due to the task scale. On TinyImageNet, which is a spatially
downscaled subset of ImageNet, the first-decay switching strategy indeed leads to better validation
accuracy.

\vspace{-0.05in}

\subsection{Key Factors: Small Learning Rate and Sufficient Training Time} \label{sec:key-factors}

\vspace{-0.05in}
All the above papers agree that Post-local/Local SGD improves upon SGD to some extent. However, 
it is in debate under what conditions the generalization benefit can consistently occur.
We now conduct ablation studies to identify the key factors so that adding
local steps improves the generalization of SGD. We run parallel SGD and Local
SGD with the same learning rate $\eta$, local batch size $\Bloc$, and number of
workers $K$, but Local SGD performs $H > 1$ local steps per round.
We start training from
the same initialization and compare their generalization after the same number
of epochs. As Post-local SGD can be viewed as Local SGD starting from an
SGD-pretrained model, the initial point in our experiments can be 
either random or a checkpoint of SGD training.
For simplicity, we keep the learning rate constant over time.
Post-local SGD that switches the training mode at the last learning rate decay
corresponds to this case, as the learning rate remains constant thereafter.
See \Cref{sec:pseudocode} for the implementation details of parallel SGD and Local SGD and \Cref{sec:detail regime} for more details about the experimental setup.

\begin{figure}[t]
\vspace{-0.5in}
\begin{center}
    \subfigure[\footnotesize CIFAR-10, start from random.]{
        \includegraphics[width=0.3\textwidth]{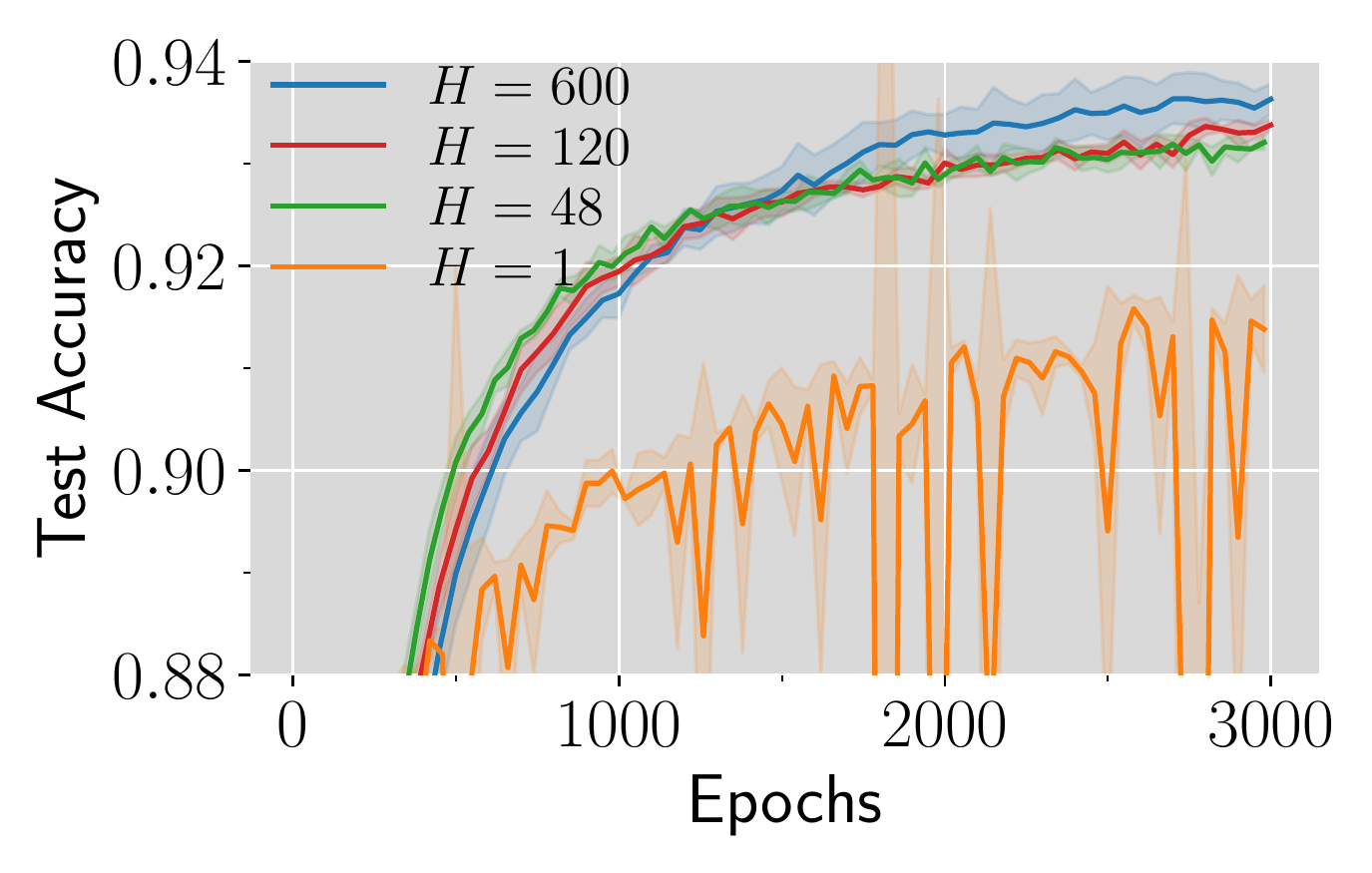}\label{fig:effect-a}
    }
    \hspace{0.1in}
    \subfigure[\footnotesize CIFAR-10, start from \#$250$.]{
        \includegraphics[width=0.3\textwidth]{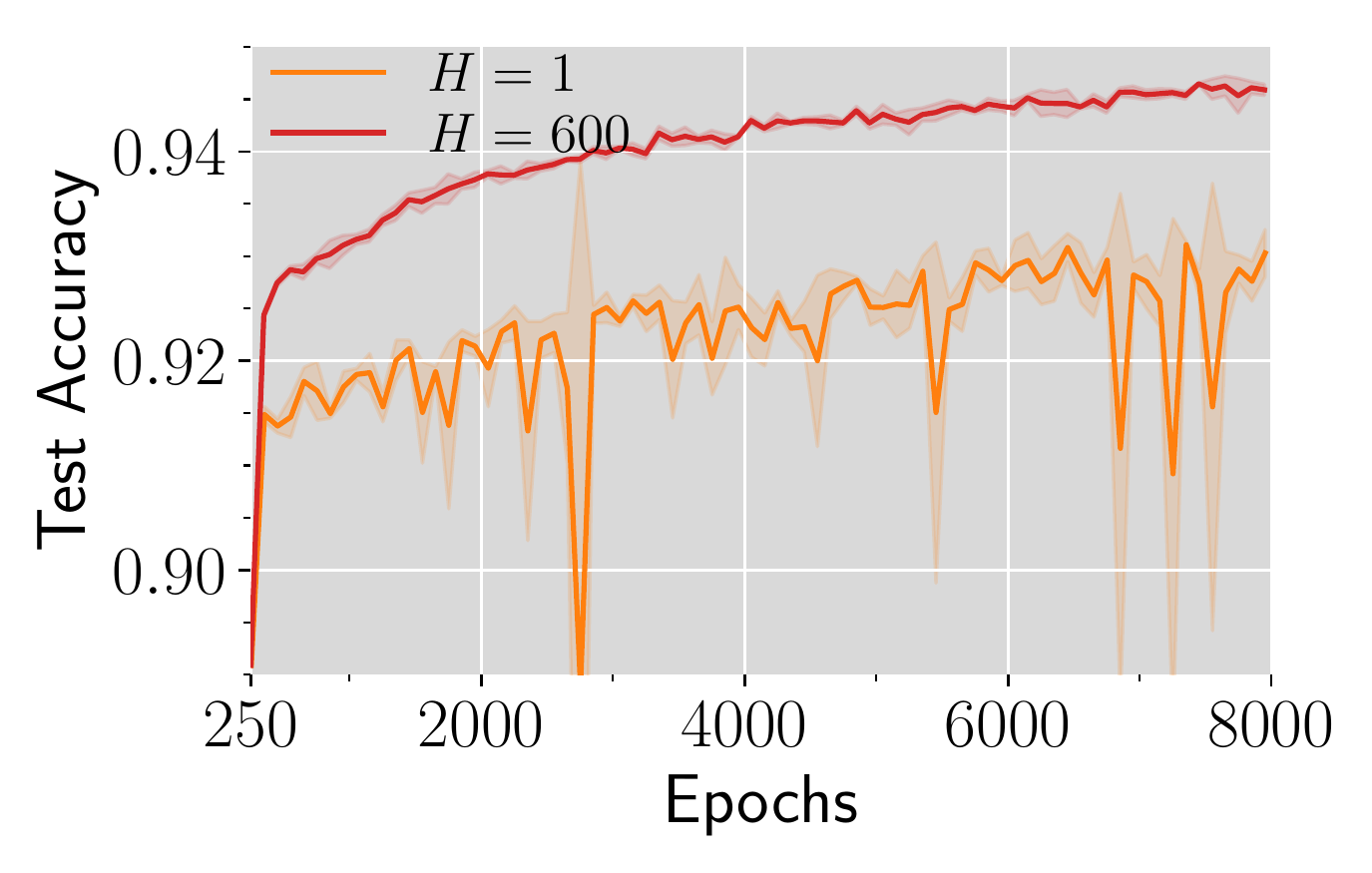}\label{fig:effect-b}
    }
    \hspace{0.1in}
    \subfigure[\footnotesize ImageNet, start from \#$100$.]{
        \includegraphics[width=0.3\textwidth]{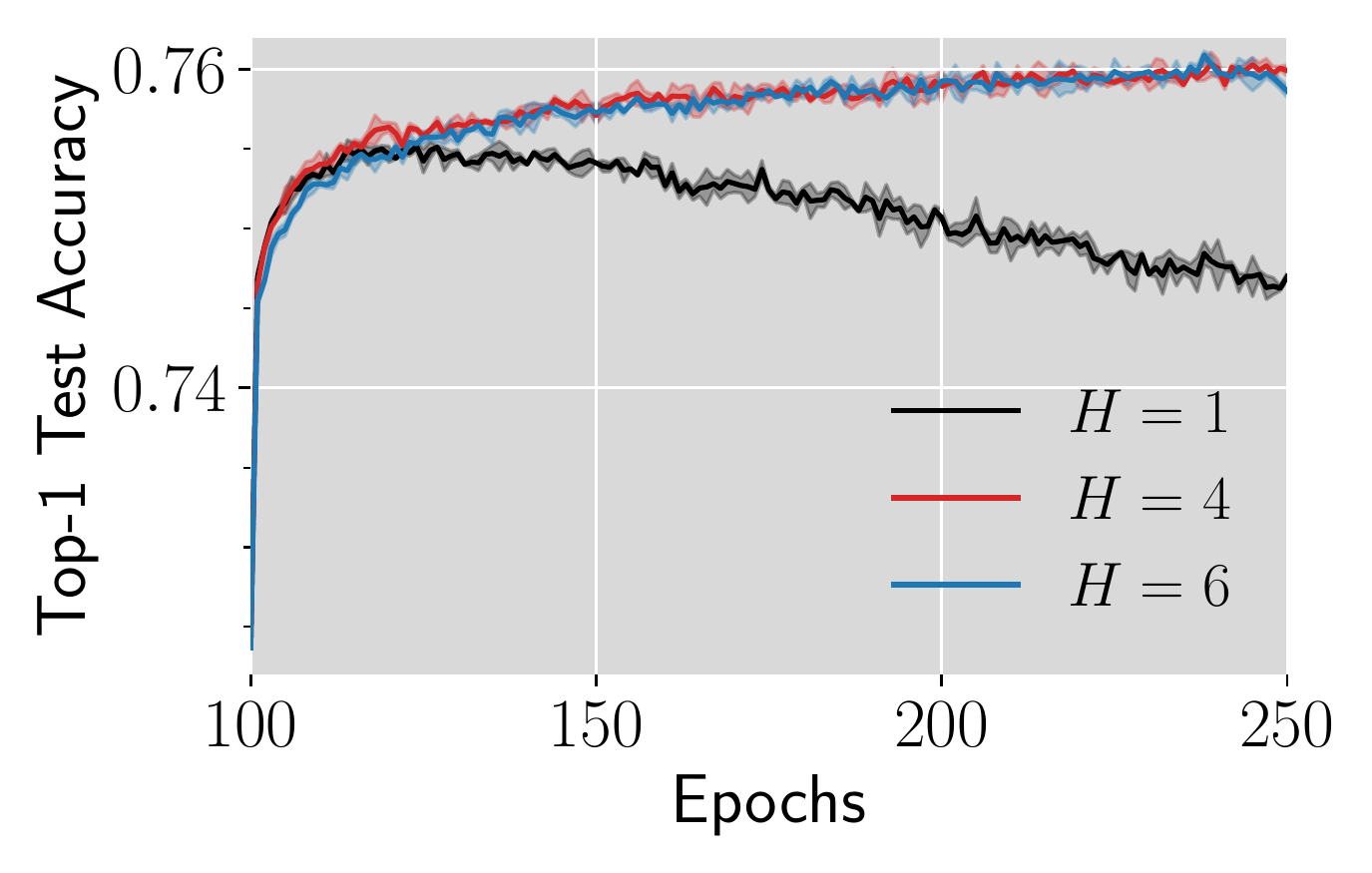}\label{fig:effect-c}
    }

    \vspace{-0.05in}
    \subfigure[\footnotesize ImageNet, first phase $\eta=3.2$.]{
        \includegraphics[width=0.3\textwidth]{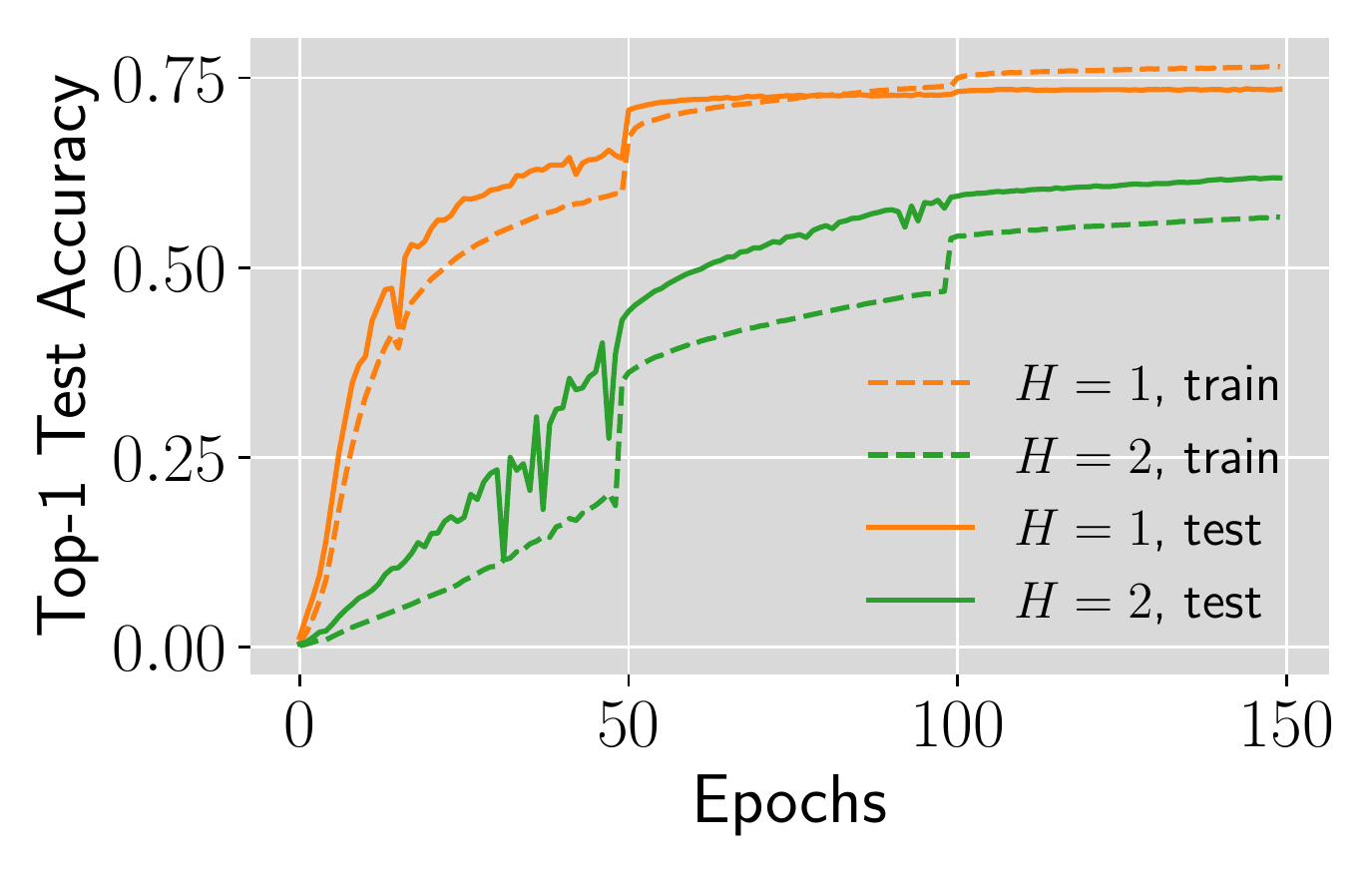}\label{fig:effect-d}
    }
    \hspace{0.1in} 
    \subfigure[CIFAR-10, test acc v.s. $H$.]{
        \includegraphics[width=0.3\textwidth]{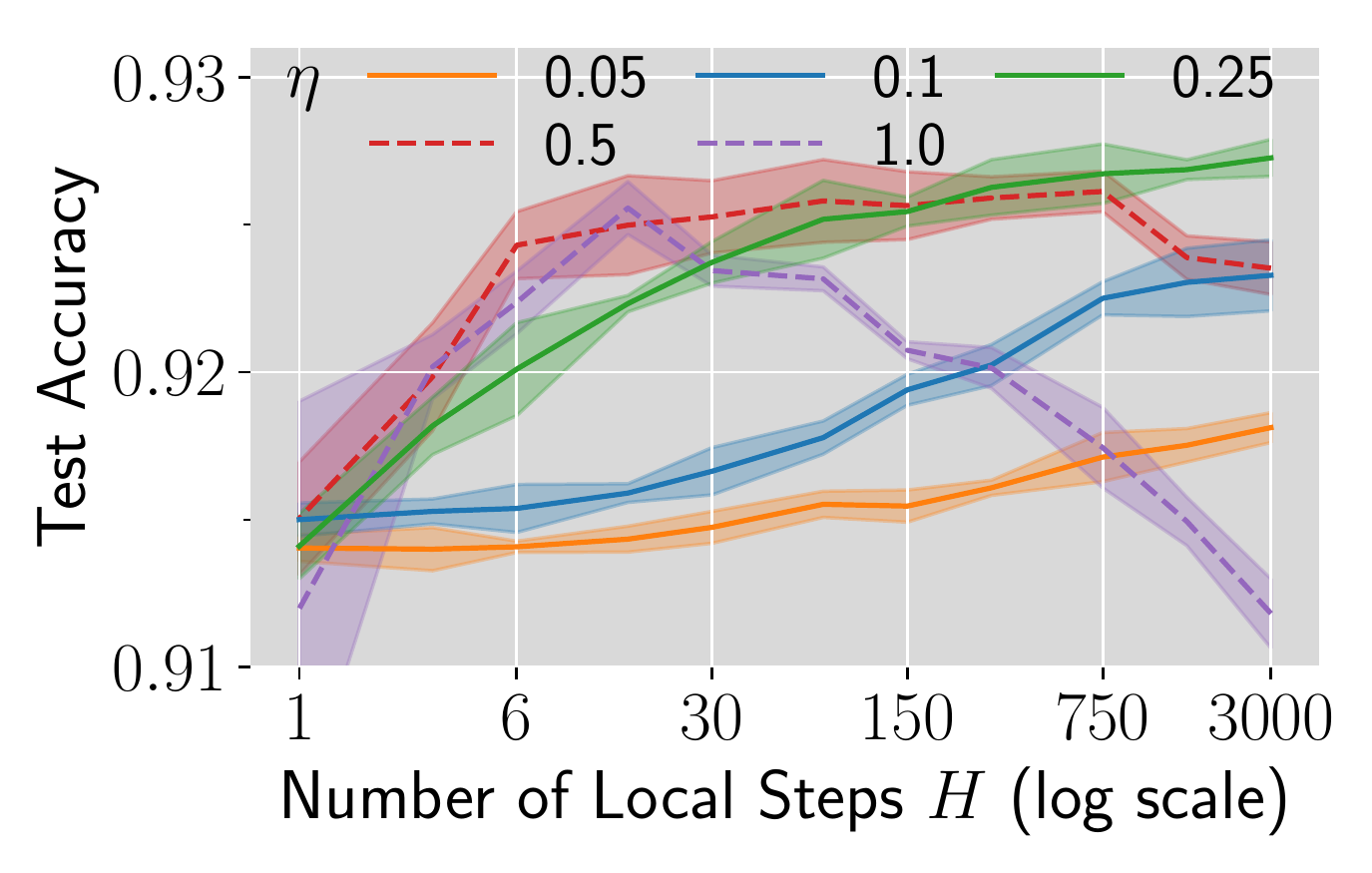}\label{fig:effect-e}
    }
    \hspace{0.1in}
    \subfigure[\footnotesize ImageNet, test acc v.s. $H$.]{
        \includegraphics[width=0.3\textwidth]{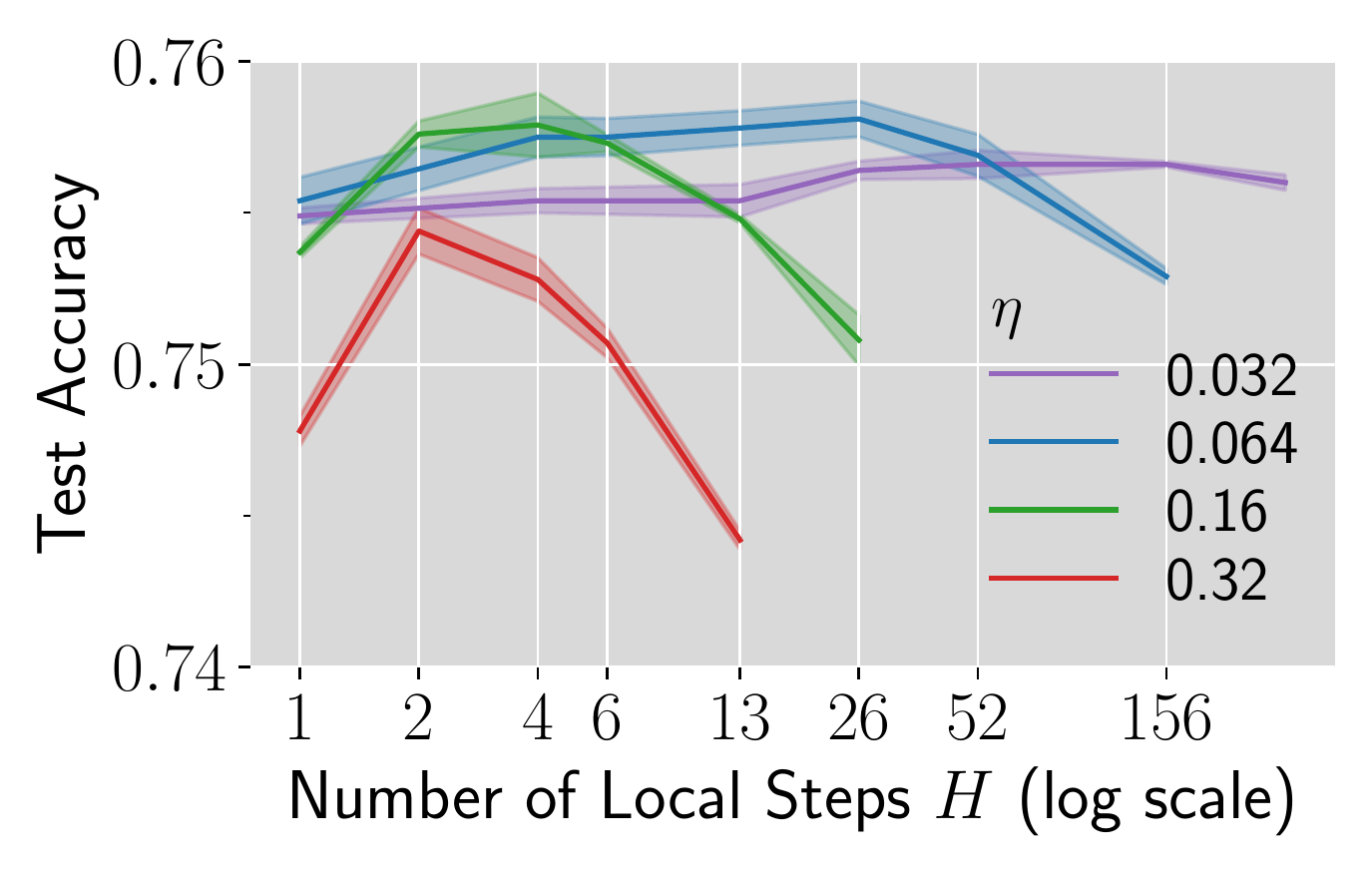}\label{fig:effect-f}
    }
    \vspace{-0.1in}
    \caption{
        Ablation studies on $\eta$, $H$ and training time in the same setting as \Cref{fig:Postlocal}.
        For (a)(d), we train from random initialization.
        For (b)(c)(e)(f), we start training
        from the checkpoints saved at the switching time points in \Cref{fig:Postlocal} (epoch \#250 for CIFAR-10 and epoch \#100 for ImageNet). See \Cref{sec:detail regime} for training details.
    } \label{fig:effect}
    \vspace{-0.1in}
\end{center}
\end{figure}

The first observation we have is that the generalization benefits can be
reproduced on both CIFAR-10 and ImageNet in our setting (see
\Cref{fig:Postlocal}). 
We remark that Post-local SGD and SGD in 
\citet{lin2020dont,ortiz2021trade}
are implemented with
accompanying Nesterov momentum terms. 
The learning rate also decays a couple of times in training with Local SGD.
Nevertheless, our experiments show that the Nesterov momentum and learning rate decay are \emph{not} necessary for Local SGD to generalize better than
SGD.
Our main finding after further ablation studies is summarized below:


\begin{finding} \label{fin:main}
    Given a sufficiently small learning rate and a
    sufficiently long training time, Local SGD exhibits better generalization
    than SGD, if the number of local steps $H$ per round is tuned properly
    according to the learning rate. This holds for both training from random initialization
    and from pre-trained models.
\end{finding}

Now we go through each point of our main finding. See also \Cref{sec:regime add exp} for more plots.

\myparagraph{(1). Pretraining is not necessary.}
In contrast to previous works claiming the benefits of Post-local SGD over
Local SGD~\citep{lin2020dont,ortiz2021trade}, we observe that Local SGD with
random initialization also generalizes significantly better than SGD, as long as
the learning rate is small and the training time is sufficiently long (\Cref{fig:effect-a}). Starting from a pretrained model may shorten the time
to reach this generalization benefit to show up (\Cref{fig:effect-b}),
but it is not necessary.

\myparagraph{(2). Learning rate should be small.}
We experiment with a wide range of learning rates to conclude that setting a
small learning rate is necessary. The learning rate is $0.32$ for
\Cref{fig:effect-a,fig:effect-b} and is $0.16$ for \Cref{fig:effect-c}. As shown in \Cref{fig:effect-d}, Local SGD encounters optimization difficulty in the first phase where $\eta$ is large ($\eta=3.2$), 
resulting in inferior
final test accuracy. Even for training from a pretrained model, the
generalization improvement of Local SGD disappears for large learning rates
(e.g., $\eta=1.6$ in \Cref{fig:add effect-d}). In contrast,
if a longer training time is allowed,
reducing the learning rate of Local SGD 
does not lead to test accuracy drop (\Cref{fig:add effect-c}).


\myparagraph{(3). Training time should be long enough.}
To investigate the effect of training time, in \Cref{fig:effect-b,fig:effect-c},
we extend the training budget for the Post-local SGD experiments in
\Cref{fig:Postlocal} and observe that a longer training time leads to greater
generalization improvement upon SGD. On the other hand, Local SGD generalizes
worse than SGD in the first few epochs of \Cref{fig:effect-a,fig:effect-c}; see
\Cref{fig:add effect-a,fig:add effect-b} for an enlarged view.



\myparagraph{(4). The number of local steps $H$ should be tuned carefully.} 
The number of local steps $H$
has a complex interplay with the learning rate $\eta$, 
but generally speaking,
a smaller $\eta$ needs
a higher $H$ to achieve consistent generalization improvement.
For CIFAR-10 with
a post-local training budget of 250 epochs (see \Cref{fig:effect-e}), the test accuracy
first rises as $H$ increases, and begins to fall as $H$ exceeds some
threshold for relatively large $\eta$ (e.g., $\eta\geq 0.5$) while keeps growing
for smaller $\eta$ (e.g., $\eta<0.5$). For ImageNet with a post-local training
budget of 50 epochs (see \Cref{fig:effect-f}), the test
accuracy first increases and then decreases in $H$ for all learning rates.


\myparagraph{Reconciling previous works.} Our finding can help to settle
the debate presented in \Cref{subsec:debate} to a large extent. Simultaneously
requiring  a small learning rate and sufficient training time poses a trade-off
when learning rate decay is used with a limited training budget: switching to
Local SGD earlier may lead to a large learning rate, while switching later makes the generalization improvement of Local SGD less noticeable due to fewer update steps.
It is thus unsurprising that first-decay switching strategy is not
always the best when the dataset and learning rate schedule change.

The need for sufficient training time does not contradict with 
\citet{ortiz2021trade}'s conjecture that Local SGD only has a ``short-term''
generalization benefit. In their experiments,  the generalization improvement usually disappears right after the next
learning rate decay (instead of after a fixed amount of time). We suspect that
the real reason why the improvement vanishes is that the number
of local steps $H$ was kept as a constant, but our finding 
suggests tuning $H$ after $\eta$ changes. 
In \Cref{fig:add effect-e}, we reproduce
this phenomenon and show that increasing $H$
after learning rate decay retains the improvement. 

\myparagraph{Generalization performances at the optimal learning rate of SGD.}
    In practice, the learning rate of SGD is usually tuned to achieve the best
    training loss/validation accuracy within a fixed training budget. Our
    finding suggests that when the tuned learning rate is small and the training
    time is sufficient, Local SGD can offer generalization improvement over SGD. As an example,
    in our experiments on training from an SGD-pretrained model,
    the optimal learning rate for SGD is $0.5$ on CIFAR-10 (\Cref{fig:effect-e}) and $0.064$ on ImageNet (\Cref{fig:effect-f}).
    With the same learning rate as SGD,
    the test accuracy is improved by $1.1\%$ on CIFAR-10 and $0.3\%$ on ImageNet when using Local SGD with $H=750$ and $H=26$ respectively.
    The improvement could become even higher if the learning rate of Local SGD is
    carefully tuned.

\section{Theoretical Analysis of Local SGD: The Slow SDE} \label{sec:theory}

In this section, we adopt an SDE-based
approach to rigorously establish the generalization benefit of Local SGD in a general setting. Below, we first identify the difficulty
of adapting the SDE framework to Local SGD. Then, we present our novel SDE
characterization of Local SGD around the manifold of minimizers
and explain the generalization benefit of Local SGD with our SDE. 

\myparagraph{Notations.} 
We follow the notations in \Cref{sec:intro}.
We denote by $\eta$ the learning rate, $K$ the number of workers, $B$ the (global) batch size, $\Bloc := B / K$ the local batch size, $H$ the number of local steps,
$\ell(\vtheta; \zeta)$ the loss function for a data sample $\zeta$, and $\ctD$ the training distribution.
Furthermore, we define $\cL(\vtheta) := \E_{\xi \sim \ctD}[\ell(\vtheta; \xi)]$
as the expected loss, $\mSig(\vtheta) := \Cov_{\xi \sim \ctD}[\nabla \ell(\vtheta; \xi)]$ as
the noise covariance of gradients at $\vtheta$. 
Let $\{\vW_t\}_{t\geq 0}$ denote the standard Wiener process.
For a mapping $F:\R^d\to \R^d$, denote by $\partial F(\vtheta)$ the Jacobian at
$\vtheta$ and $\partial^2 F(\vtheta)$ the second order derivative at $\vtheta$.
Furthermore,  for any matrix $\mM\in \R^{d\times d}$, $\partial^2
F(\vtheta)[\mM]=\sum_{i\in[d]}\inne{ \frac{\partial^2 F_i}{\partial
\vtheta^2}}{\mM}\ve_i$ where $\ve_i$ is the $i$-th vector of the standard
basis. We write $\partial^2 (\nabla \cL)(\vtheta)[\mM]$ as $\nabla^3 \cL(\vtheta)[\mM]$ for short.

\myparagraph{Local SGD.} 
We use the following formulation of Local SGD for theoretical analysis. See also \Cref{sec:pseudocode} for the pseudocode. Local SGD proceeds in multiple rounds of model averaging, where each round produces a global iterate $\bvths$.
In the $(s+1)$-th round, every worker $k \in [K]$
starts with its local copy of the global iterate $\vths_{k,0} \gets \bvths$
and does $H$ steps of SGD with local batches. In the $t$-th local step of the $k$-th worker, it draws a local batch of $\Bloc := B / K$
independent samples $\xis_{k,t,1}, \dots, \xis_{k,t,\Bloc}$ from a
shared training distribution $\ctD$ and updates as follows:
\begin{align}\label{eq:upd intro localsgd}
    \vths_{k,t+1} \gets \vths_{k, t} - \eta \vgs_{k, t}, \quad \text{where} \quad
    \vgs_{k, t} = \frac{1}{\Bloc}\sum_{i=1}^{\Bloc} \nabla \ell(\vths_{k,t}; \xis_{k,t,i}), \quad t = 0, \dots, H - 1.
\end{align}
The local updates on different workers are independent of each other as there is no communication.
After finishing the $H$ local steps, the workers aggregate the resulting local iterates $\vths_{k,H}$ and assign the average to the next global iterate: $\bvths[s+1] \gets \frac{1}{K}\sum_{k=1}^K \vths_{k, H}$.


\vspace{-0.03in}
\subsection{Difficulty of Adapting the SDE Framework to Local SGD} \label{subsec:con-sde}

\vspace{-0.03in}




A widely-adopted approach to understanding the dynamics of SGD is to approximate it
from a continuous perspective with the following SDE~\eqref{eq:canonical SDE}, which we call the {\em
conventional SDE approximation}. Below, we discuss why it cannot be directly adopted 
to characterize the behavior of Local SGD. 
\begin{align}
    \dd\vX(t)=-\nabla \cL(\vX)\dd t+\sqrt{\tfrac{\eta}{B}}\mSig^{\sfrac{1}{2}}(\vX)\dd \vW_t. \label{eq:canonical SDE}
\end{align}
It is proved by \citet{li2019stochastic} that this SDE is a
first-order approximation to SGD, 
where each discrete step corresponds to a continuous time interval of $\eta$. Several previous works adopt this
SDE approximation and connect good generalization to having a large diffusion
term $\sqrt{\frac{\eta}{B}} \mSig^{\sfrac{1}{2}} \dd \vW_t$ in the
SDE~\citep{jastrzkebski2017three,smith2020generalization}, because  
a suitable amount of noise can be necessary
for large-batch training to generalize well (see also \Cref{sec:related,sec:addrelated}).

According to \Cref{fin:main},
it is tempting to consider
the limit $\eta \to 0$
and see if Local SGD can also be modeled via a variant of the conventional SDE.
In this case the typical time length that guarantees a good SDE approximation
error is $\cO(\eta^{-1})$ discrete
steps~\citep{li2019stochastic,li2021validity}. However, this time scaling is too
short for the difference to appear between Local SGD and SGD. Indeed,
\Cref{thm:small thm} below shows that they closely track each other for
$\cO(\eta^{-1})$ steps.
\begin{theorem}\label{thm:small thm}
    Assume that the loss function $\cL$ is
    $\cC^3$-smooth with bounded second and third order derivatives and that
    $\nabla \ell(\vtheta; \xi)$ is bounded. Let $T > 0$ be a constant,
    $\bvths$ be the $s$-th global iterate of Local SGD and $\vw_t$ be the $t$-th
    iterate of SGD with the same initialization $\vw_0=\bvths[0]$ and same
    $\eta, \Bloc, K$. Then for any $H\leq \frac{T}{\eta}$ and $\delta = \cO(\poly(\eta))$, it holds with
    probability at least $1-\delta$ that for all $s \le \frac{T}{\eta H}$,
    $\normtwo{\bvths-\vw_{sH}}=\cO(\sqrt{\eta \log \frac{1}{\eta\delta}})$.
\end{theorem}
We defer the proof for the above theorem to \Cref{sec:small thm proof}. 
See also \Cref{sec:multisde} for \citet{lin2020dont}'s attempt to model Local
SGD with multiple conventional SDEs and for our discussion on why it does not
give much insight.

\subsection{SDE Approximation near the Minimizer Manifold} \label{subsec:sde-near-man}

Inspired by a recent paper~\citep{li2021happens}, our strategy to overcome the shortcomings of the conventional SDE is to design a new SDE that can guarantee a good approximation for $\cO(\eta^{-2})$ discrete steps, much longer than the $\cO(\eta^{-1})$ discrete steps for the conventional SDE.
Following their setting, we assume the existence of a manifold $\Gamma$ consisting only of local minimizers and track the global iterate $\bvths$ around the manifold after it takes $\ctO(\eta^{-1})$ steps to approach the manifold.
Although the expected gradient $\nabla \cL$ is near zero around the manifold, the dynamics are still non-trivial because the noise can drive the iterate to move a significant distance in $\cO(\eta^{-2})$ steps.
\begin{assumption}\label{a:smooth} The loss function $\cL(\cdot)$ and the matrix
square root of the noise covariance   $\mSig^{\sfrac{1}{2}}(\cdot)$  are $\cC^{\infty}$-smooth.
Besides, we assume that $\normtwo{\nabla \ell(\vtheta; \xi)}$ is bounded by a constant for all $\vtheta$ and $\xi$. 
\end{assumption}
\begin{assumption}\label{a:Gamma} $\Gamma$ is a $\cC^{\infty}$-smooth,
$(d-m)$-dimensional submanifold of $\R^d$, where any $\vzeta\in \Gamma$ is a local
minimizer of $\cL$. For all $\vzeta\in \Gamma$,
$\rank(\nabla^2 \cL(\vzeta))=m$. Additionally,
there exists an open neighborhood of $\Gamma$, denoted as $U$, such that $\Gamma
=  \arg\min_{\vtheta \in U}\cL(\vtheta)$.
\end{assumption}
\begin{assumption}\label{a:compact}
$\Gamma$ is a compact manifold.
\end{assumption}
The smoothness assumption on $\cL$ is generally satisfied when we use smooth
activation functions, such as Swish \citep{ramachandran2017searching},  softplus and
GeLU \citep{hendrycks2016gaussian}, which work equally well as ReLU in many
circumstances. The existence of a minimizer manifold with $\rank(\nabla^2
\cL(\vzeta))=m$ has also been made as a key assumption in
\citet{fehrman2020convergence,li2021happens,lyu2022understanding},
where $\rank(\nabla^2 \cL(\vzeta))=m$ ensures that the Hessian is maximally
non-degenerate on the manifold and implies that the tangent space at $\vzeta \in
\Gamma$ equals the null space of $\nabla^2\cL(\vzeta)$. The last assumption
is made to prevent the analysis from being too technically involved.

Our SDE for Local SGD characterizes the training dynamics near $\Gamma$. 
For ease of presentation, we define the following projection operators $\Phi, P_{\vzeta}$
for points and differential forms respectively.
\begin{definition}[Gradient Flow Projection]
    Fix a point $\vthetanull \notin \Gamma$. 
For $\vx \in \R^d$, consider the gradient flow $\frac{\dd \vx(t)}{\dd t} = -\nabla
\cL(\vx(t))$ with $\vx(0)=\vx$. 
We denote the gradient flow projection of $\vx$ as $\Phi(\vx)$.
$\Phi(\vx) := \lim_{t\to +\infty} \vx(t)$ if the limit exists and
belongs to $\Gamma$; otherwise, $\Phi(\vx) = \vthetanull$.
\end{definition}
\begin{definition}
    For any $\vzeta \in \Gamma$ and any differential form $\mA \dd \vW_t + \vb
    \dd t$ in Itô calculus, where $\mA$ is a matrix and $\vb$ is a vector, we
    use $P_{\vzeta}(\mA \dd \vW_t + \vb \dd t)$ as a shorthand for the
    differential form $\partial \Phi(\vzeta) \mA \dd \vW_t + \left(\partial
    \Phi(\vzeta) \vb + \frac{1}{2}\partial^2 \Phi(\vzeta)[\mA \mA^\top]\right)
    \dd t$.
\end{definition}
See \citet{oksendal2013stochastic} for an introduction to Itô calculus.
Here $P_{\vzeta}$ equals
$\Phi(\vzeta + \mA \dd \vW_t + \vb \dd t) - \Phi(\vzeta)$ by Itô calculus, which
means that $P_{\vzeta}$ projects an infinitesimal step from $\vzeta$, so that $\vzeta$
after taking the projected step does not leave the manifold $\Gamma$.
It can be shown by simple calculus that $\partial \Phi(\vzeta)$ equals the
projection matrix onto the tangent space of $\Gamma$ at $\vzeta$.
We decompose the noise covariance $\mSig(\vzeta)$ for $\vzeta
\in \Gamma$ into two parts: the noise in the tangent space
$\mSig_{\parallel}(\vzeta) := \partial \Phi(\vzeta) \mSig(\vzeta) \partial
\Phi(\vzeta)$ and the noise in the rest $\mSig_{\Diamond}(\vzeta) :=
\mSig(\vzeta) - \mSig_{\parallel}(\vzeta)$. Now we are ready to state our SDE for Local SGD.
\begin{definition}[Slow SDE for Local SGD]
Given $\eta, H > 0$ and $\vzeta_0 \in \Gamma$, define $\vzeta(t)$ as the
solution of the following SDE with initial condition $\vzeta(0) = \vzeta_0$:
\begin{align}\label{eq:zeta}
    \dd \vzeta(t)&=P_{\vzeta}\Big(\underbrace{\tfrac{1}{\sqrt{B}} \mSig_{\parallel}^{\sfrac{1}{2}}(\vzeta)\dd \vW_t}_{\text{(a)\ diffusion}}
    \underbrace{-\tfrac{1}{2B} \nabla^3 \cL(\vzeta)[\widehat{\mSig}_{\Diamond}(\vzeta)] \dd t}_{\text{(b)\ drift-I}}
    \underbrace{-\tfrac{K-1}{2B} \nabla^3 \cL(\vzeta)[\widehat{\mPsi}(\vzeta)] \dd t}_{\text{(c)\ drift-II}}
    \Big).
\end{align}
Here $\widehat{\mSig}_{\Diamond}(\vzeta)$, $\widehat{\mPsi}(\vzeta) \in \R^{d \times d}$ are defined as
\begin{align}
      \widehat{\mSig}_{\Diamond}(\vzeta) &:= {\textstyle \sum}_{i, j: (\lambda_i \ne 0) \lor (\lambda_j \ne 0)} \,\tfrac{1}{\lambda_i + \lambda_j} \left\langle\mSig_{\Diamond}(\vzeta), \vv_i \vv_j^\top \right\rangle \vv_i \vv_j^\top, \label{eq:hatmsig} \\
      \widehat{\mPsi}(\vzeta) &:= {\textstyle\sum}_{i, j: (\lambda_i \ne 0) \lor (\lambda_j \ne 0)} \,\tfrac{\psi(\eta H\cdot (\lambda_i + \lambda_j))}{\lambda_i + \lambda_j} \left\langle\mSig_{\Diamond}(\vzeta), \vv_i \vv_j^\top \right\rangle \vv_i \vv_j^\top, \label{eq:hatpsi}
\end{align}
where 
$\{\vv_i\}_{i=1}^{d}$ is a set of eigenvectors of $\nabla^2 \cL(\vzeta)$
that forms an orthonormal eigenbasis, and
$\lambda_1,\dots,\lambda_d$ are the corresponding eigenvalues.
Additionally,
$\psi(x):=\frac{e^{-x}-1+x}{x}$ for $x \ne 0$
and  $\psi(0) = 0$.
\end{definition}
The use of $P_{\vzeta}$ keeps $\vzeta(t)$ on the manifold $\Gamma$ through
projection. $\mSig_{\parallel}^{\frac{1}{2}}(\vzeta)$ introduces a diffusion
term to the SDE in the tangent space. The two drift terms involve
$\widehat{\mSig}_{\Diamond}(\cdot)$ and  $\widehat{\mPsi}(\cdot)$, which can be
intuitively understood as rescaling the entries of the noise covariance in the
eigenbasis of Hessian. In the special case where
$\nabla^2\cL=\diag(\lambda_1, \cdots, \lambda_d) \in \R^{d \times d}$,
we have $\widehat{\Sigma}_{\Diamond,i, j} = \frac{1}{\lambda_i +
\lambda_j} \Sigma_{0, i, j}$. $\widehat{\Psi}_{i, j} = \frac{\psi(\eta
H (\lambda_i + \lambda_j))}{\lambda_i + \lambda_j} \Sigma_{0, i, j}$. $\psi(x)$
is a monotonically increasing function, which goes 
from $0$ to 1 as $x$ goes from $0$ to infinity (see \Cref{fig:psi})

We name this SDE as the {\em Slow SDE for Local SGD} because we will show that
each discrete step of Local SGD corresponds to a continuous time interval of
$\eta^2$ instead of an interval of $\eta$ in the conventional SDE. In this
sense, our SDE is ``slower'' than the conventional SDE (and hence can
track a longer horizon). This Slow SDE is inspired by \citet{li2021happens}.
Under nearly the same set of assumptions, they proved that SGD
can be tracked by an SDE that is essentially equivalent to~\eqref{eq:zeta} with
$K=1$, namely, without the drift-II term.
\begin{align}\label{eq:sgd-zeta}
    \dd \vzeta(t)&=P_{\vzeta}\Big(\underbrace{\tfrac{1}{\sqrt{B}} \mSig_{\parallel}^{\sfrac{1}{2}}(\vzeta)\dd \vW_t}_{\text{(a)\ diffusion}}
    \underbrace{-\tfrac{1}{2B} \nabla^3 \cL(\vzeta)[\widehat{\mSig}_{\Diamond}(\vzeta)] \dd t}_{\text{(b)\ drift-I}}
    \Big),
\end{align}
We refer to~\eqref{eq:sgd-zeta} as the {\em Slow SDE for SGD}. We remark that
the drfit-II term in~\eqref{eq:zeta} is novel and is the key  to separate the
generalization behaviors of Local SGD and SGD in theory. We will discuss this
point later in \Cref{subsec:inter}.
Now we present our SDE approximation theorem for Local SGD.
\begin{theorem}\label{main thm: flow} Let Assumptions~\ref{a:smooth} to
 \ref{a:compact} hold. Let $T>0$ be a constant and $\vzeta(t)$ be the solution
 to \eqref{eq:zeta} with the initial condition $\vzeta(0)=\Phi(\bvths[0])\in
 \Gamma$. If $H$ is set to $\tfrac{\alpha}{\eta}$ for some constant $\alpha >
 0$, then for any $\cC^3$-smooth function $g(\vtheta)$,
 $\max_{0\leq s \leq \frac{T}{H\eta^2}}\left \lvert\E[g(\Phi(\bvths)] - \E[g(\vzeta(sH\eta^2)]\right\rvert=\ctO(\eta^{0.25})$, where $\ctO(\cdot)$ hides log factors and constants
 that are independent of $\eta$ but can depend on $g(\vtheta)$.
\end{theorem}

\begin{theorem}
\label{coro:alpha}
    For $\delta=\cO(\poly(\eta))$, with probability at least $1-\delta$,  
    it holds 
for all $\cO(\frac{1}{\alpha}\log\frac{1}{\eta})\leq s\leq\frac{T}{\alpha\eta} $
    that
    $\Phi(\bvths)\in \Gamma$ and $\normtwo{\bvths-\Phi(\bvths)}=\cO(\sqrt{\alpha\eta \log \frac{\alpha}{\eta\delta}})$,
    where $\cO(\cdot)$ hides constants independent of $\eta$, $\alpha$ and $\delta$.
\end{theorem}

\Cref{main thm: flow} suggests that the trajectories of the manifold projection
and the solution to the Slow SDE~\eqref{eq:zeta} are close to each other in the
weak approximation sense. 
That is, $\{\Phi(\bvths)\}$ and $\{\vzeta(t)\}$ cannot be distinguished by evaluating test functions from a wide function class, including all polynomials. This measurement of closeness between
the iterates of stochastic gradient algorithms and their SDE approximations is
also adopted by~\citet{li2019stochastic,li2021validity,malladi2022sdes}, but
their analyses are for conventional SDEs. \Cref{coro:alpha} further states that
the iterate $\bvths$ keeps close to its manifold projection after the first few
rounds.

\begin{remark}
To connect to \Cref{fin:main}, we remark that 
our theorems (1) do not require
the model to be pre-trained (as long as the gradient flow starting with
$\vths[0]$ converges to $\Gamma$); (2) give better bounds for smaller $\eta$;
(3) characterize a long training horizon $\sim \eta^{-2}$. The need for tuning $H$ will
be discussed in \Cref{sec:local-sgd-tuning}.
\end{remark}

\myparagraph{Technical Contribution.} The proof technique for \Cref{main thm: flow} is novel and
significantly different from the Slow SDE analysis of SGD in~\citet{li2021validity}. Their analysis uses advanced stochastic calculus and invokes Katzenberger's theorem~\citep{kat} to show that SGD converges to the Slow SDE in distribution, but no quantitative error bounds are provided. Also, due to the local updates and multiple aggregation steps in Local SGD, it is unclear how to extend Katzenberger's theorem to our case. To overcome this difficulty, we develop a new approach to analyze the Slow SDEs, which is not only based on relatively simpler mathematics but can also provide the quantitative error bound $\ctO(\eta^{0.25})$ in weak approximation. Specifically, we adopt the general framework proposed by~\citet{li2019stochastic}, which uses the method of moments to bound the closeness between the trajectories of discrete methods and SDE solutions, namely $\Phi(\bvths)$ and $\vzeta(t)$ in our case. Their framework can provide approximation guarantees for $\cO(\eta^{-1})$ steps of a discrete algorithm with learning rate $\eta$, but it is not directly applicable to our case because we want to capture $\cO(\eta^{-2})$ steps of Local SGD. Instead, we treat $\cO(\eta^{-\beta})$ rounds as a ``giant step'' of Local SGD with an ``effective'' learning rate $\eta^{1-\beta}$, where $\beta$ is a constant in $(0, 1)$, and we develop a detailed dynamical analysis to derive the recursive formulas of the moments for the change in every step, every round, and every $\cO(\eta^{-\beta})$ rounds. We then apply the framework of \citet{li2019stochastic} to translate Local SGD to the Slow SDE and optimize the choice of $\beta$ to minimize the approximation error bound, settling on $\beta = 0.25$. See \Cref{sec: proof outline} for our proof outline.
A by-product of our result is the first quantitative approximation bound for the Slow SDE approximation for SGD, which can be easily obtained by setting $K=1$.

\subsection{Interpretation of the Slow SDEs} \label{subsec:inter}

In this subsection, we compare the Slow SDEs for SGD and Local SGD
and provide an important insight into why Local SGD generalizes better
than SGD: Local SGD strengthens the drift term in the Slow SDE
which makes the implicit regularization of stochastic gradient noise more effective.

\subsubsection{Interpretation of the Slow SDE for SGD.}

The Slow SDE for SGD~\eqref{eq:sgd-zeta} consists of the diffusion and drift-I
terms. The former injects noise into the dynamics in the tangent space; the latter one
drives the dynamics to move along the negative gradient of
$\frac{1}{2B}\inne{\nabla^2 \cL(\vzeta)}{\widehat{\mSig}_{\Diamond}(\vzeta)}$
projected onto the tangent space, but ignoring the dependency of
$\widehat{\mSig}_{\Diamond}(\vzeta)$ on $\vzeta$. This can be connected to the
class of semi-gradient methods which only computes a part of the gradient
\citep{mnih2015human, DBLP:books/lib/SuttonB98,
DBLP:conf/iclr/BrandfonbrenerB20}. In this view, the long-term behavior of SGD
is similar to a stochastic semi-gradient method minimizing 
the implicit regularizer
$\frac{1}{2B}\inne{\nabla^2 \cL(\vzeta)}{\widehat{\mSig}_{\Diamond}(\vzeta)}$
on the minimizer manifold of the original loss $\cL$.

Though 
the semi-gradient method may not perfectly optimize its objective,
the above argument reveals that SGD has a deterministic trend
toward the region with a smaller magnitude of Hessian, which is commonly believed
to correlate with better generalization \citep{hochreiter1997flat, keskar2017on,
neyshabur2017exploring,jiang2020fantastic} (see \Cref{sec:related} for more discussions). In
contrast, the diffusion term can be regarded as  a random perturbation to this
trend, which can impede optimization when the drift-I term is not strong enough.

Based on this view, we conjecture that
{\bf strengthening the drift term} 
of the Slow SDE  can help SGD to better regularize the model,
yielding a better generalization performance.
More specifically, we propose the following hypothesis,
which 
compares the generalization performances of 
the following generalized Slow SDEs.
Note that $(\frac{1}{B}, \frac{1}{2B})$-Slow SDE corresponds to the Slow SDE for SGD~\eqref{eq:sgd-zeta}.
\begin{definition}
    For $\kappa_1, \kappa_2 \ge 0$, define $(\kappa_1, \kappa_2)$-Slow SDE to be the following:
    \begin{align}\label{eq:zeta-kappa}
    \dd \vzeta(t)&=
    P_{\vzeta}\Big(
        \sqrt{\kappa_1} \mSig_{\parallel}^{\sfrac{1}{2}}(\vzeta)\dd \vW_t
        -\kappa_2 \nabla^3 \cL(\vzeta)[\widehat{\mSig}_{\Diamond}(\vzeta)] \dd t
    \Big).
    \end{align}
\end{definition}

\begin{hypothesis} \label{hyp:drift}
    Starting at a minimizer $\vzeta_0 \in \Gamma$,
    run $(\kappa_1, \kappa_2)$-Slow SDE and $(\kappa_1, \kappa_2')$-Slow SDE
    respectively for the same amount of time $T > 0$ and obtain $\vzeta(T), \vzeta'(T)$.
    If $\kappa_2 > \kappa_2'$, then the expected test accuracy at $\vzeta(T)$
    is better than that at $\vzeta'(T)$.
\end{hypothesis}
Due to the No Free Lunch Theorem, we do not claim that our hypothesis is always
true, but we do believe that the hypothesis holds when training usual neural
networks (e.g., ResNets, VGGNets) on standard benchmarks (e.g., CIFAR-10, ImageNet).

\myparagraph{Example: Training with Label Noise Regularization.}
To exemplify the generalization benefit of having a larger drift term, we follow
a line of theoretical works~\citep{li2021happens, blanc2020implicit,
damian2021label} to study the case of training over-parameterized neural nets
with label noise regularization. For a $C$-class classification task, the label
noise regularization is as follows: every time we draw a sample from the
training set, we make the true label as it is with probability $1-p$, and
replace it with any other label with equal probability $\frac{p}{C-1}$. When we
use cross-entropy loss, the Slow SDE for SGD turns out to be a simple
deterministic gradient flow on $\Gamma$ (instead of a semi-gradient
method) for minimizing the trace of Hessian: $\dd \vzeta(t) = -\frac{1}{4B}
\gradGa \tr(\nabla^2 \cL(\vzeta)) \dd t$,
where $\gradGa f$ stands for the gradient of the function $f$ projected to the
tangent space of $\Gamma$. Checking the validity of our hypothesis reduces to
the following question: {\em Is minimizing the trace of Hessian beneficial to
generalization?} Many previous works provide positive answers, including the
line of works we just mentioned. \citet{blanc2020implicit} and \citet{li2021happens}
connect minimizing the trace of Hessian to finding sparse or low-rank
solutions for training two-layer linear nets. \citet{damian2021label}
empirically showed that good generalization correlates with a smaller trace
of Hessian in training ResNets with label noise. Besides, \citet{ma2021on}
connect the trace of Hessian to the smoothness of the function represented by a
deep neural net. We refer the readers to \Cref{sec:thm label noise} for further discussion on the Slow
SDEs in this case. 

\subsubsection{Local SGD Strengthens the Drift Term in Slow SDE.} \label{sec:local-sgd-sde-interpretation}

Based on \Cref{hyp:drift}, now we argue that 
{\bf Local SGD improves generalization by strengthening the drift term of the Slow SDE}.

First, it can be seen from~\eqref{eq:zeta} that the Slow SDE for Local SGD has an additional drfit-II term.
Similar to the drift-I term of the Slow SDE for SGD, this drift-II term drives the dynamics to move along the
negative semi-gradient of $\frac{K-1}{2B}\inne{\nabla^2
\cL(\vzeta)}{\widehat{\mPsi}(\vzeta)}$ (with the dependency of
$\widehat{\mPsi}(\vzeta)$ on $\vzeta$ ignored). Combining it with the implicit
regularizer induced by the drift-I term, we can see that the long-term behavior
of Local SGD is similar to a stochastic semi-gradient method minimizing the
implicit regularizer $\frac{1}{2B}\inne{\nabla^2
\cL(\vzeta)}{\widehat{\mSig}_{\Diamond}(\vzeta)} + \frac{K-1}{2B}\inne{\nabla^2
\cL(\vzeta)}{\widehat{\mPsi}(\vzeta)}$ on the minimizer manifold of $\cL$. 

Comparing the definitions of
$\widehat{\mSig}_{\diamond}(\vzeta)$~\eqref{eq:hatmsig} and
$\widehat{\mPsi}(\vzeta)$~\eqref{eq:hatpsi}, we can see that
$\widehat{\mPsi}(\vzeta)$ is basically a rescaling of the entries of
$\widehat{\mSig}_{\diamond}(\vzeta)$ in the eigenbasis of Hessian, where the
rescaling factor $\psi(\eta H \cdot (\lambda_i + \lambda_j))$ for each entry is
between $0$ and $1$ (see \Cref{fig:psi} for the plot of $\psi$). When $\eta H$
is small, the rescaling factors should be close to $\psi(0) = 0$, then
$\widehat{\mPsi}(\vzeta) \approx \vzero$, leading to almost no additional
regularization. On the other hand, when $\eta H$ is large, the rescaling factors
should be close to $\psi(+\infty) = 1$, so $\widehat{\mPsi}(\vzeta) \approx
\widehat{\mSig}_{\diamond}(\vzeta)$. We can then merge the two implicit
regularizers as $\frac{K}{2B}\inne{\nabla^2
\cL(\vzeta)}{\widehat{\mSig}_{\Diamond}(\vzeta)}$,
and \eqref{eq:zeta} becomes
the $(\frac{1}{B}, \frac{K}{2B})$-Slow SDE, which is restated below:
\begin{align}\label{eq:zeta-inf}
    \dd \vzeta(t)&=P_{\vzeta}\Big(\tfrac{1}{\sqrt{B}} \mSig_{\parallel}^{\sfrac{1}{2}}(\vzeta)\dd \vW_t
    -\tfrac{K}{2B} \nabla^3 \cL(\vzeta)[\widehat{\mSig}_{\Diamond}(\vzeta)] \dd t \Big).
\end{align}
From the above argument we know how the Slow SDE of Local SGD~\eqref{eq:zeta} changes 
as $\eta H$
transitions from $0$ to $+\infty$.
Initially, when $\eta H = 0$, 
\eqref{eq:zeta} is the same as the $(\frac{1}{B}, \frac{1}{2B})$-Slow SDE for SGD.
Then increasing $\eta H$ strengthens the drift term of~\eqref{eq:zeta}.
As $\eta H \to +\infty$, \eqref{eq:zeta} transitions to 
the $(\frac{1}{B}, \frac{K}{2B})$-Slow SDE, where the drift term becomes $K$ times larger.

According to \Cref{hyp:drift}, the $(\frac{1}{B}, \frac{K}{2B})$-Slow SDE
generalizes better than the $(\frac{1}{B}, \frac{1}{2B})$-Slow SDE,
so Local SGD with $\eta H = +\infty$ should generalize better than SGD.
When $\eta H$ is chosen realistically as a finite value,
the generalization performance of Local SGD interpolates
between these two cases,
which results in a worse generalization than $\eta H = +\infty$
but should still be better than SGD.

\subsubsection{Theoretical Insights into Tuning the Number of Local Steps}
\label{sec:local-sgd-tuning}


Based on our Slow SDE approximations,
we now discuss how the number of local steps $H$ affects the generalization of Local SGD.
When $\eta$ is small but finite,
tuning $H$ offers a trade-off between regularization strength
and SDE approximation quality.
Larger $\alpha := \eta H$ makes the regularization stronger in the SDE (as discussed in \Cref{sec:local-sgd-sde-interpretation}),
but the SDE itself may lose track of Local SGD,
which can be seen from the error bound $\cO(\sqrt{\alpha \eta \log \frac{\alpha}{\eta \delta}})$ in~\Cref{coro:alpha}. Therefore, we expect the test accuracy to first increase and then decrease as we gradually increase $H$. Indeed, we observe in \Cref{fig:effect-e,fig:effect-f} that the plot of test accuracy versus $H$ is unimodal for each $\eta$. 

It is thus necessary to tune $H$ for the best generalization.
When $H$ is tuned together with other hyperparameters, such as learning rate $\eta$,
our Slow SDE approximation recommends setting $H$ to be at least $\Omega(\eta^{-1})$
so that $\alpha := \eta H$ does not vanish in the Slow SDE.
Since larger $\alpha$ gives a stronger regularization effect,
the optimal $H$ should be set to the largest value so that the Slow SDE does not lose track of Local SGD.
Indeed, we empirically observed that when $H$ is tuned optimally,
$\alpha$ increases as $\eta$ decreases,
suggesting that the optimal $H$ grows faster than $\Omega(\eta^{-1})$. See \Cref{fig:add effect-f}.

\subsubsection{Understanding the Diffusion Term in the Slow SDE} \label{sec:effect-diffusion}

\begin{figure}[t]
    \vspace{-0.5in}
\begin{center}
 \subfigure[CIFAR-10, $H=600$ for $K>1$.]{
        \includegraphics[width=0.4\textwidth]{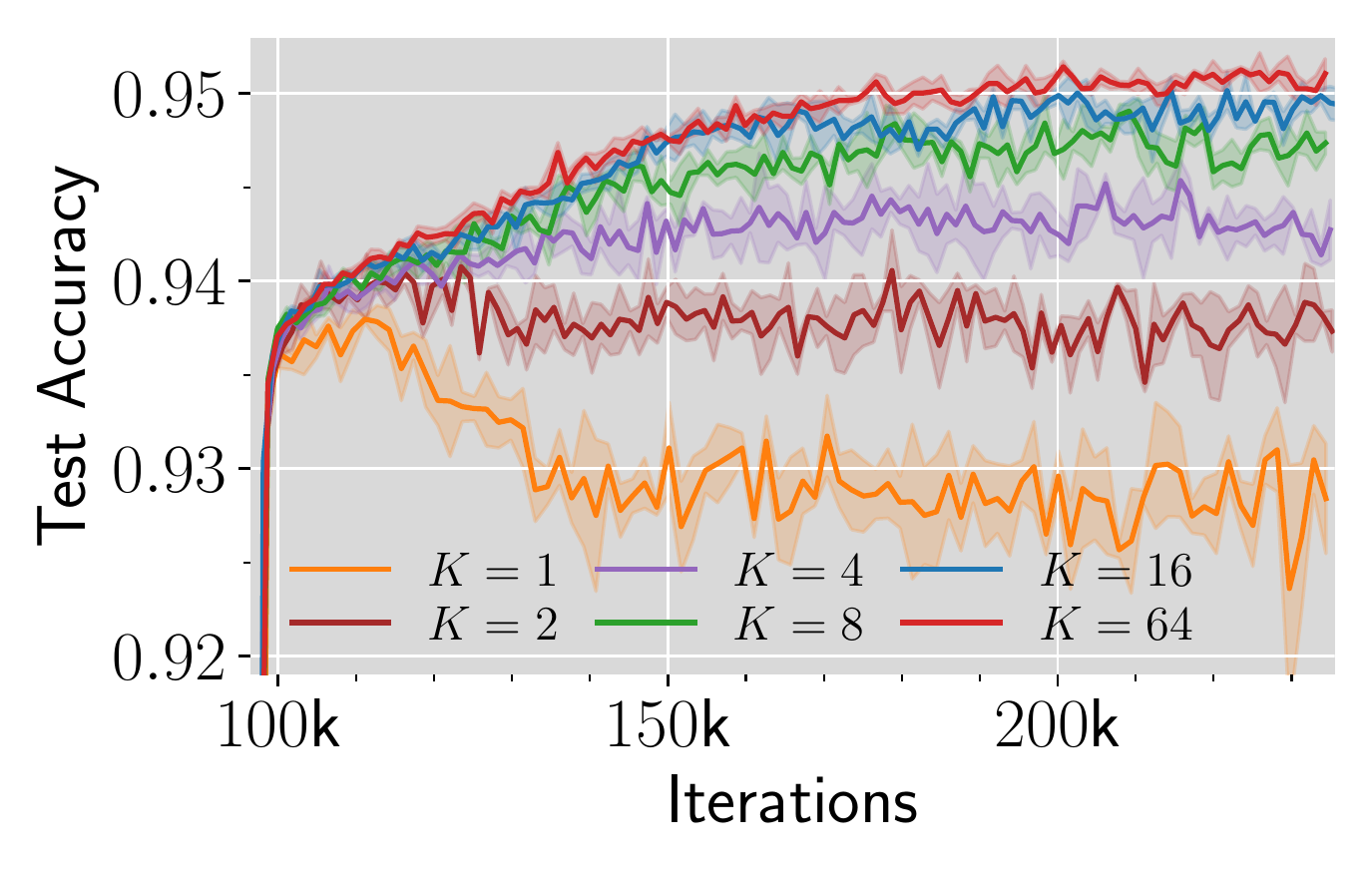}\label{fig:cifar diffusion}}
    \hspace{0.1in}
         \subfigure[ImageNet, $H=78$ for $K>1$.]{
        \includegraphics[width=0.4\textwidth]{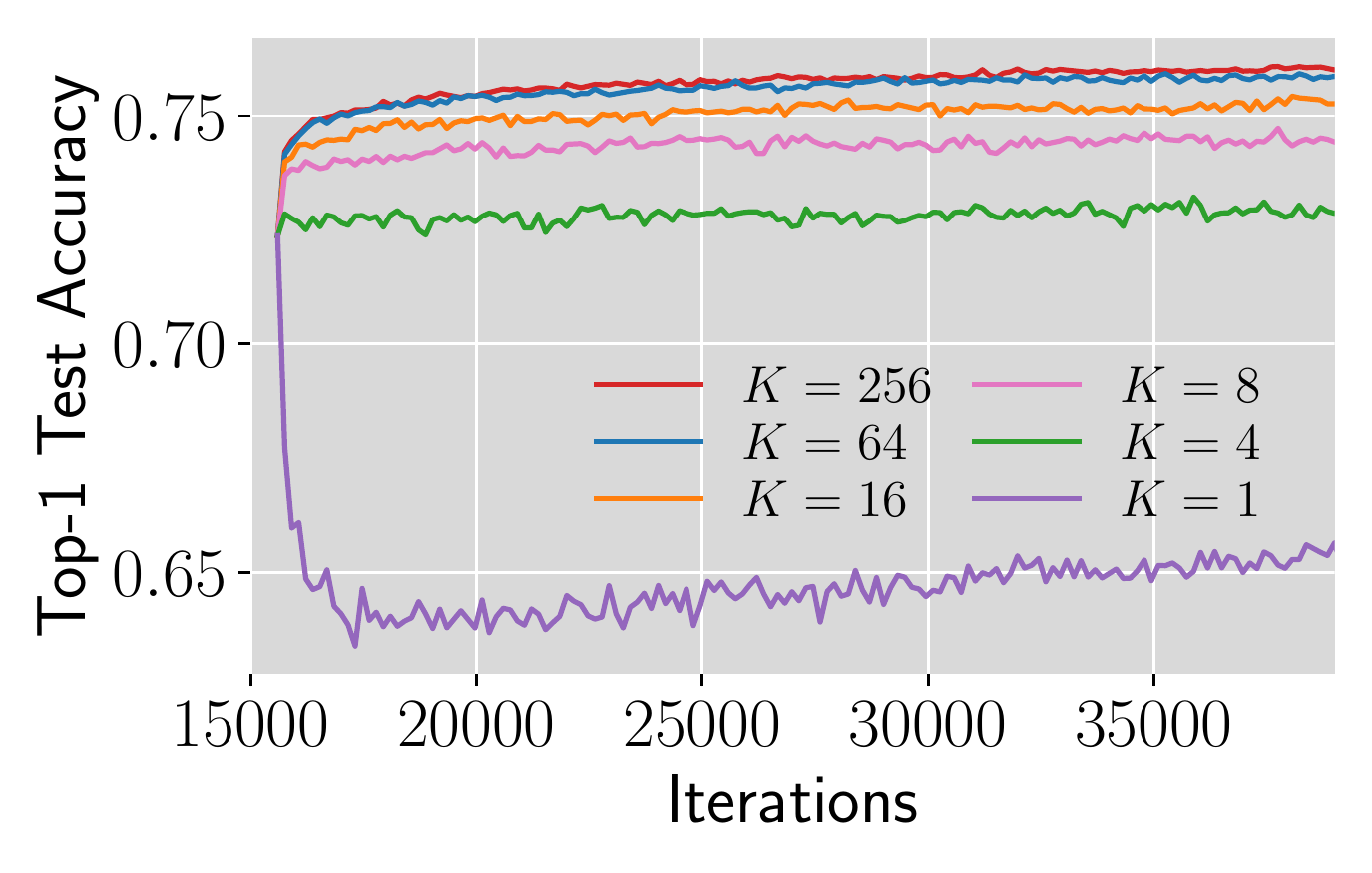}\label{fig:img diffusion}
    }
    \vspace{-0.1in}
        \caption{Reducing the diffusion term of the Slow SDE for Local SGD leads
        to better generalization. Test accuracy improves as we increase $K$ with
        fixed $\eta$ and $H$ to reduce the diffusion term while keeping the
        drift term untouched. See \Cref{sec: diffusion details} for
        details. }
  \label{fig:diffusion}
     \vspace{-0.1in}
\end{center}
\end{figure}

So far, we have discussed why adding local steps enlarges the drift term in the Slow SDE and why enlarging the drift term can benefit generalization.
Besides this, here we remark that another way
to accelerate the corresponding semi-gradient method for minimizing the implicit
regularizer is to reduce the diffusion term,
so that the trajectory more closely follows the drift term.
More formally, we propose the following:
\begin{hypothesis} \label{hyp:diffusion}
    Starting at a minimizer $\vzeta_0 \in \Gamma$,
    run $(\kappa_1, \kappa_2)$-Slow SDE and $(\kappa_1, \kappa_2')$-Slow SDE
    respectively for the same amount of time $T > 0$ and obtain $\vzeta(T), \vzeta'(T)$.
    If $\mSig_{\parallel} \not\equiv \vzero$ and $\kappa_1 < \kappa_1'$, then the expected test accuracy at $\vzeta(T)$
    is better than that at $\vzeta'(T)$.
\end{hypothesis}
Here we exclude the case of $\mSig_{\parallel} \equiv \vzero$ because in this
case the diffusion term in the Slow SDE is always zero. To verify
\Cref{hyp:diffusion}, we set the product $\alpha := \eta H$ large, keep $H,\eta$ fixed,
increase the number of workers $K$, and compare the generalization performances
after a fixed amount of training steps (but after different numbers of epochs). This case
corresponds to the $(\frac{1}{K\Bloc}, \frac{1}{2\Bloc})$-Slow SDE, so adding
more workers should reduce the diffusion term. As shown in \Cref{fig:diffusion},
a higher test accuracy is indeed achieved for larger $K$.

\myparagraph{Implication: Enlarging the learning rate is not equally effective as adding local steps.}
Given that Local SGD improves generalization by strengthening the drift
term, it is natural to wonder if enlarging the learning rate of SGD would also
lead to similar improvements.
While it is true that enlarging the learning rate
effectively increases the drift term, it also increases the diffusion term
simultaneously, which can hinder the implicit regularization by \Cref{hyp:diffusion}. In contrast,
adding local steps does not change the diffusion term. As shown in \Cref{fig:add
exp-a}, even when the learning rate of SGD is increased, SGD still underperforms
Local SGD by about $2\%$ in test accuracy.

On the other hand, in the special case of
where $\mSig_{\parallel} \equiv \vzero$,
\Cref{hyp:diffusion}
does not hold,
and enlarging the learning rate by
$\sqrt{K}$ results in the same Slow SDE as adding local steps
(see \Cref{sec:thm label noise} for derivation).
Then these two actions should produce the same generalization improvement,
unless the learning rate is so large that Slow SDE loses track of the training dynamics.
As an example of such a special case, an experiment with label noise regularization is
presented in \Cref{fig:K=4 label noise}.

\section{The Effect of Global Batch Size on Generalization}\label{sec:effect of batch size}

In this section, we discuss the effect of global batch size on the
generalization of Local SGD. Given that the computation power of a single worker
is limited, we consider the case where the local batch size $\Bloc$ is fixed and
the global batch size $B = K \Bloc$ is tuned by adding or removing the
workers. This scenario is relevant to the practice because one may want to know
the maximum parallelism possible to train the neural net without causing
generalization degradation.


For SGD, previous works have proposed the Linear Scaling
Rule~(LSR)~\citep{krizhevsky2014one,goyal2017accurate,jastrzkebski2017three}:
scaling the learning rate $\eta \mapsto \kappa \eta$ linearly with the global
batch size $B \mapsto \kappa B$ yields the same conventional
SDE~\eqref{eq:canonical SDE} under a constant epoch budget, hence leading to almost the same generalization performance as long as the SDE approximation does not fail.





\begin{figure}[t]
    \vspace{-0.5in}
\begin{center}
 \subfigure[CIFAR-10, start from \#$250$.]{
        \includegraphics[width=0.4\textwidth]{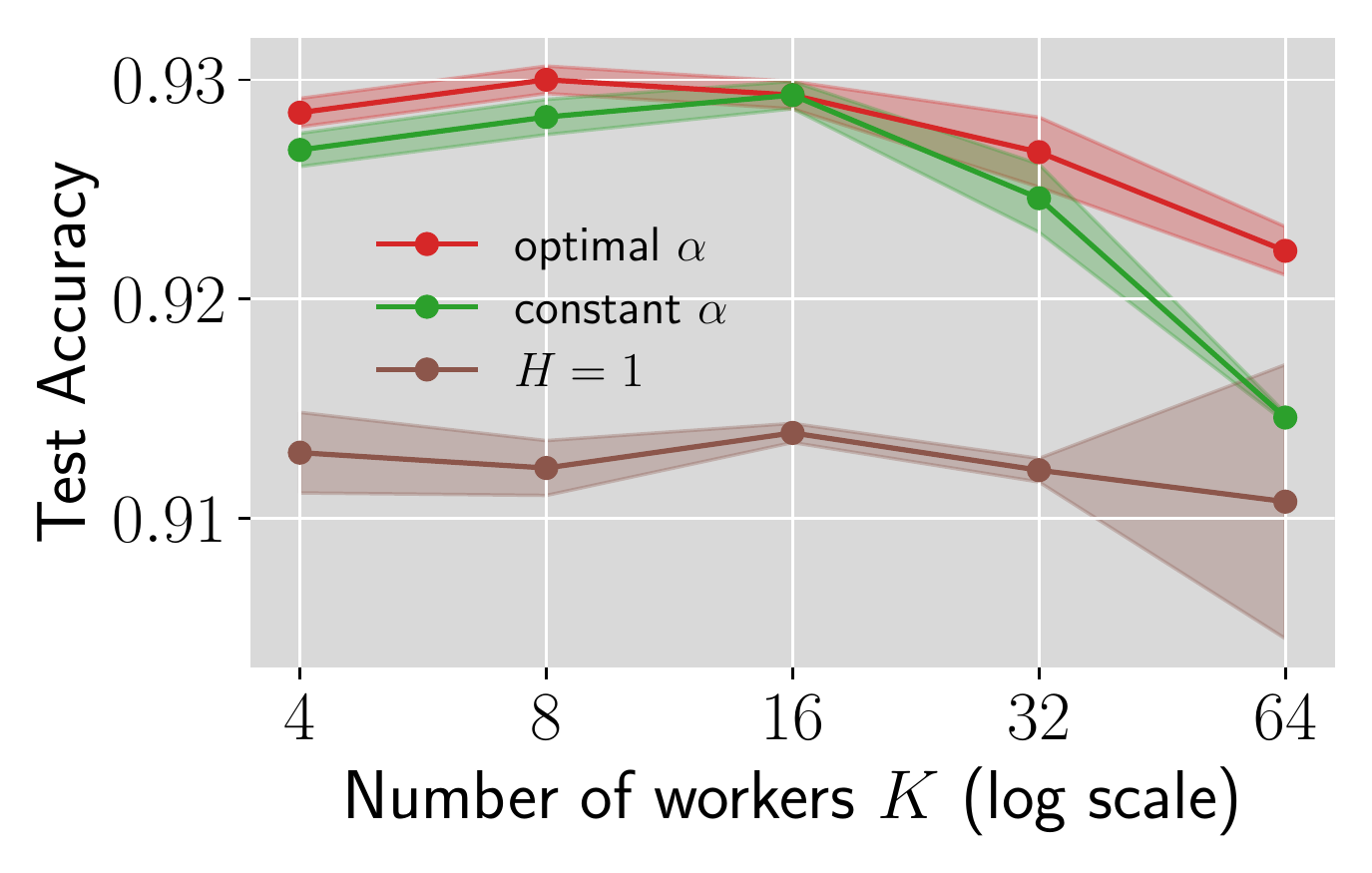}\label{fig:cifar lsr}   }
         \subfigure[ImageNet, start from \#$100$.]{
        \includegraphics[width=0.4\textwidth]{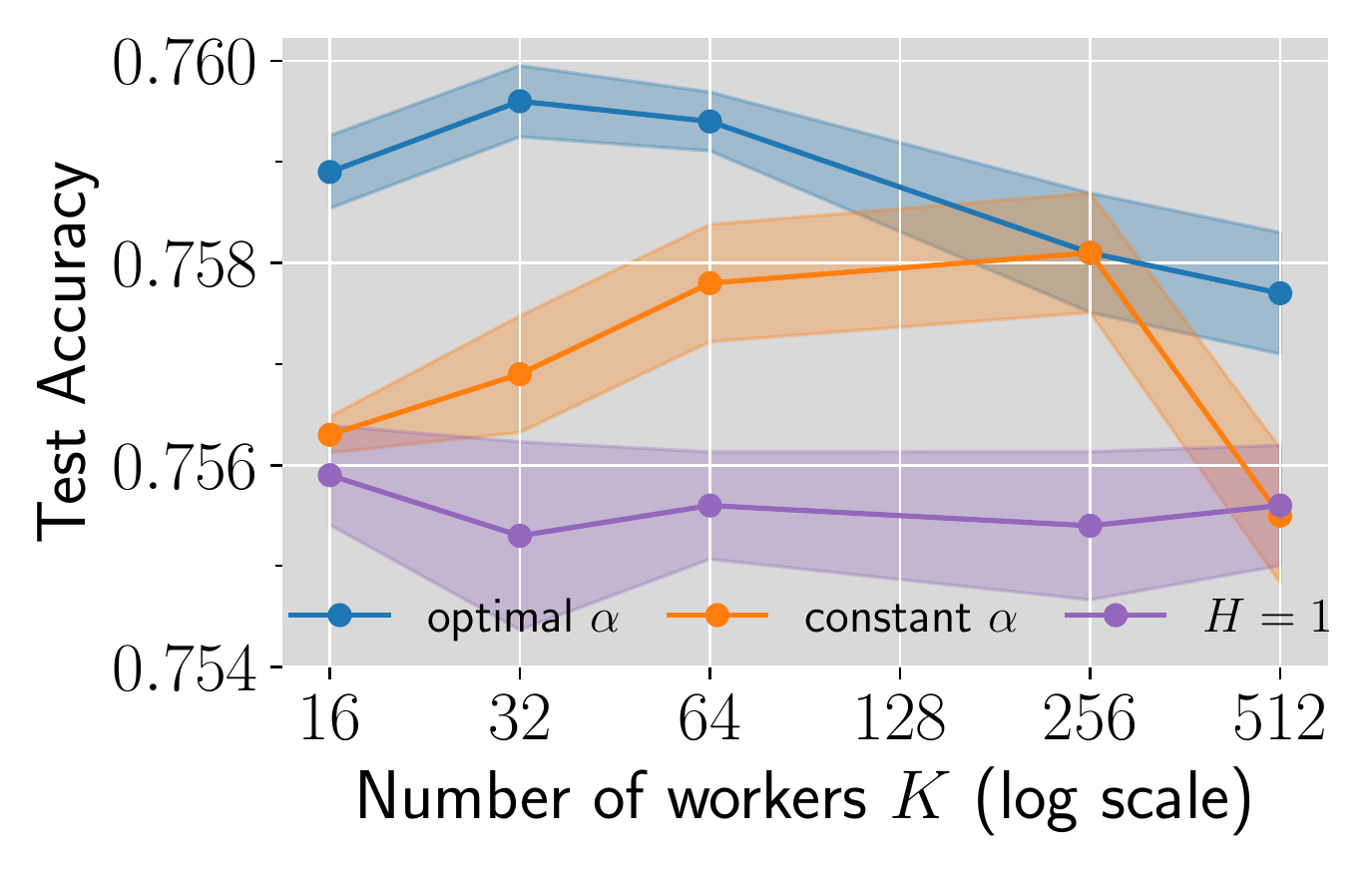}\label{fig:img lsr}
    }
    
    \vspace{-0.1in}
\subfigure[CIFAR-10, start from \#$250$.]{
        \includegraphics[width=0.4\textwidth]{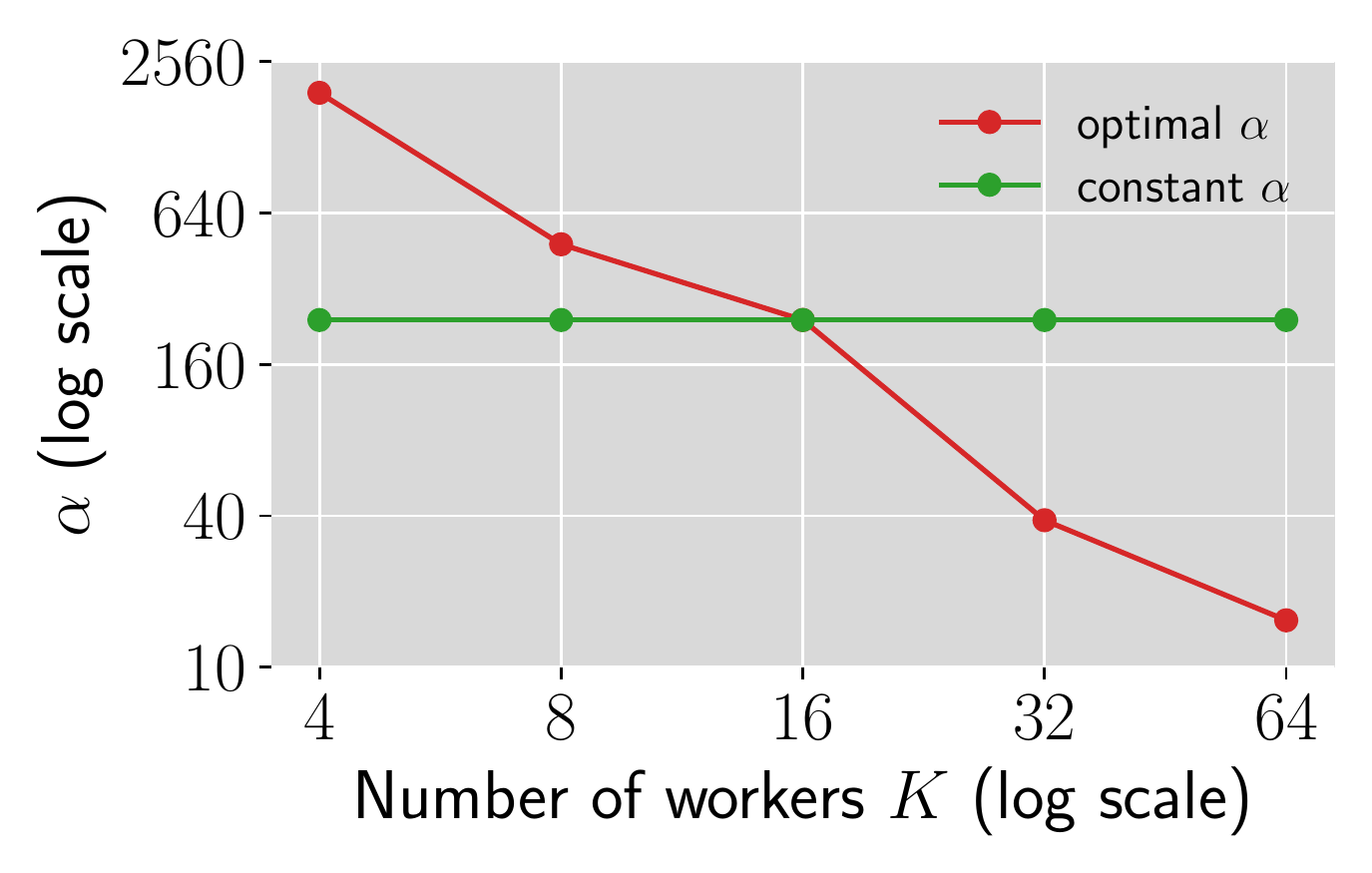}\label{fig:cifar lsr alpha}   }
         \subfigure[ImageNet, start from \#$100$.]{
        \includegraphics[width=0.4\textwidth]{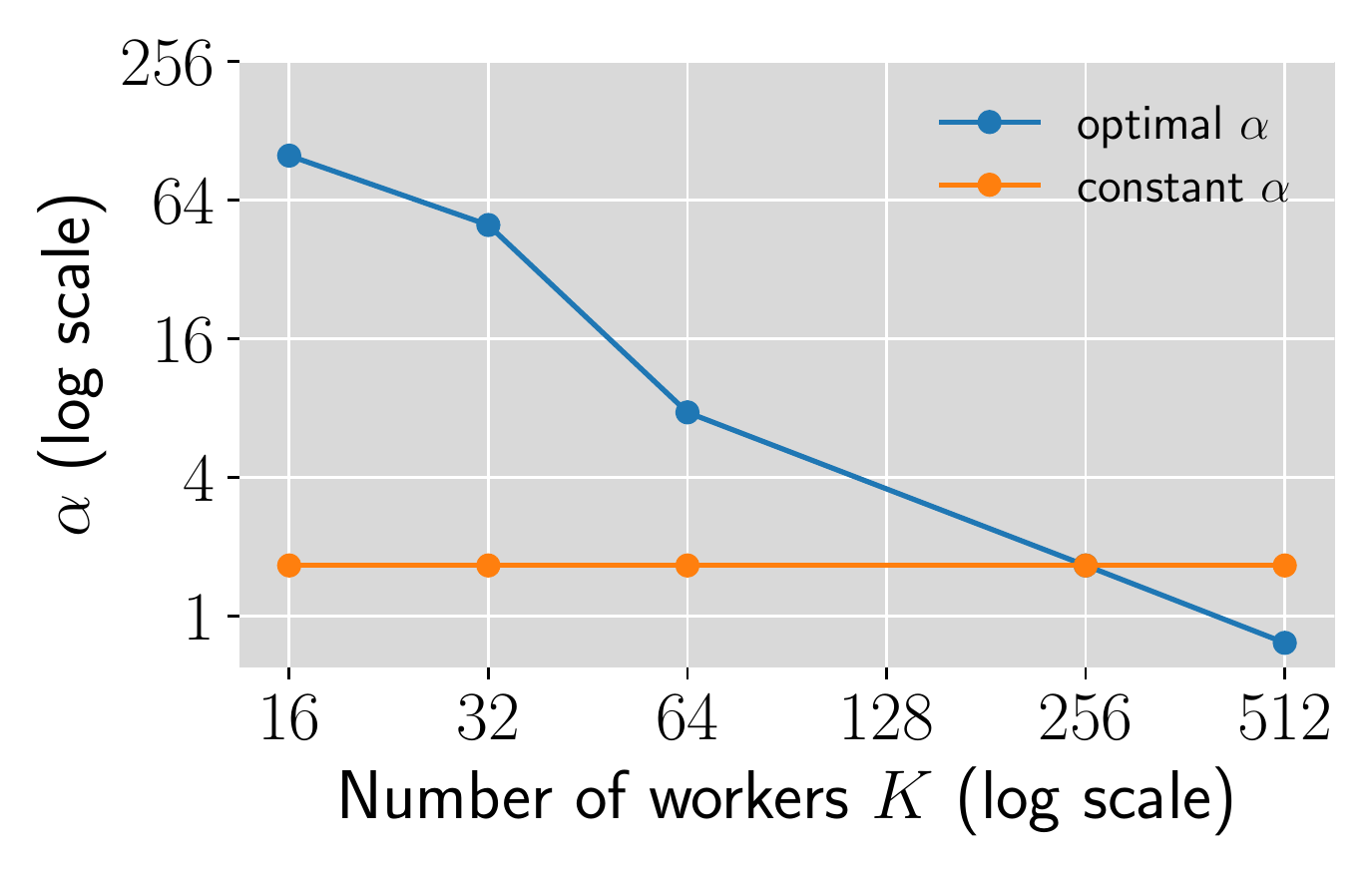}\label{fig:img lsr alpha}
    }
    \vspace{-0.1in}
    \caption{
    For training from CIFAR-10 and ImageNet checkpoints, Local SGD consistently outperforms SGD ($H=1$)
    across different batch sizes $B$ (fixing $\Bloc$ and varying $K$), where the learning rate is scaled by the LSR $\eta \propto B$. Two possible ways of tuning the number of local steps $H$ are considered: \textbf{(1).}~Tune $H$ for the best test accuracy for $K=16$ and $K = 256$ respectively on CIFAR-10 and ImageNet, then scale $H$ as $H \propto 1/B$ so that $\alpha := \eta H$ is constant; \textbf{(2).}~Tune $H$ specifically for each $K$.
    See \Cref{sec: bs details} for training details.
    }
  \label{fig:lsr}
     \vspace{-0.2in}
\end{center}
\end{figure}


We show in \Cref{thm:lsr sgd} that the LSR does not change the Slow SDE of SGD either. Experiments in \Cref{fig:lsr} show that the LSR indeed holds nicely when we continue training with small learning rates from the same CIFAR-10 and ImageNet checkpoints as in~\Cref{fig:effect}. Here we choose $K=16$ and $K=256$ as the base settings for CIFAR-10 and ImageNet, respectively,
and then tune the learning rate to maximize the test accuracy.
As shown in \Cref{fig:cifar lsr,fig:img lsr}, the optimal learning rate turns out to be small enough that the LSR can be applied to scale the global batch size with only a minor change in test accuracy.

Now, assuming the learning rate is scaled as LSR, we study how to tune the number of local steps $H$ for Local SGD for better generalization. A natural choice is to tune $H$ in the base settings and keep $\alpha$ unchanged
via scaling $H \mapsto H / \kappa$. Then the following SDE can be derived (see \Cref{thm:lsr localsgd}):
\begin{align}\label{eq:zeta-lsr}
    \dd \vzeta(t)&=P_{\vzeta}\Big(\underbrace{\tfrac{1}{\sqrt{B}} \mSig_{\parallel}^{\sfrac{1}{2}}(\vzeta)\dd \vW_t}_{\text{(a)\ diffusion (unchanged)}}
    \underbrace{-\tfrac{1}{2B} \nabla^3 \cL(\vzeta)[\widehat{\mSig}_{\Diamond}(\vzeta)] \dd t}_{\text{(b)\ drift-I (unchanged)}}
    \underbrace{-\tfrac{\kappa K-1}{2B} \nabla^3 \cL(\vzeta)[\widehat{\mPsi}(\vzeta)] \dd t}_{\text{(c)\ drift-II (rescaled)}}
    \Big).
\end{align}
Compared with \eqref{eq:zeta}, the drift-II term here is rescaled by a positive factor.
Again, when $\alpha$ is large, we can follow the argument in \Cref{sec:local-sgd-sde-interpretation}
to approximate $\widehat{\mPsi}(\vzeta) \approx \widehat{\mSig}_{\Diamond}(\vzeta)$
and obtain the following $(\frac{1}{B}, \frac{\kappa K}{B})$-Slow SDE:
\begin{align}\label{eq:zeta-inf2}
    \dd \vzeta(t)&=P_{\vzeta}\Big(\tfrac{1}{\sqrt{B}} \mSig_{\parallel}^{\sfrac{1}{2}}(\vzeta)\dd \vW(t)
    -\tfrac{\kappa K}{2B} \nabla^3 \cL(\vzeta)[\widehat{\mSig}_{\Diamond}(\vzeta)] \dd t \Big).
\end{align}
The drift term of the above SDE is always stronger than SGD \eqref{eq:sgd-zeta}, as long as there exists more than one worker after the scaling (i.e., $\kappa K > 1$).
As expected from \Cref{hyp:drift}, we observed in the experiments that the generalization performance of Local SGD is always better than or at least comparable to SGD across different batch sizes (see \Cref{fig:cifar lsr,fig:img lsr}).

Taking a closer look into the drift term in the Slow SDE~\eqref{eq:zeta-inf2},
we can find that it scales linearly with $\kappa$. According to \Cref{hyp:drift}, the SDE is expected to generalize better when adding more workers ($\kappa > 1$)
and to generalize worse when removing some workers ($\kappa < 1$).
For the latter case, we indeed observed that the test accuracy of Local SGD drops when removing workers.
For the case of adding workers, however, we also need to take into account that the LSR specifies a larger learning rate and causes a larger SDE approximation error for the same $\alpha$, which may cancel the generalization improvement brought by strengthening the drift term. In the experiments, we observed that the test accuracy does not rise when adding more workers to the base settings.

Since $\alpha$ also controls the regularization strength~(\Cref{sec:local-sgd-tuning}), it would be beneficial to decrease $\alpha$ for large batch size so as to better trade-off between regularization strength and approximation quality.
In \Cref{fig:cifar lsr alpha,fig:img lsr alpha},
we plot the optimal value of $\alpha$ for each batch size,
and we indeed observed that the optimal $\alpha$ drops as we scale up $K$.
Conversely, a smaller batch size (and hence a smaller learning rate)
allows for using a larger $\alpha$ to enhance regularization while still keeping a low approximation error (\Cref{coro:alpha}).
The test accuracy curves in~\Cref{fig:cifar lsr,fig:img lsr} indeed show that 
setting a larger $\alpha$ can compensate for the accuracy drop when reducing the batch size.






\section{Discussion}\label{sec:discussion}

\myparagraph{Connection to the conventional wisdom that the diffusion term matters more.}
As mentioned in \Cref{subsec:con-sde}, it is believed in the literature is that
a large diffusion term in the conventional SDE leads to good generalization. One
may think that the diffusion term in the Slow SDE corresponds to that in the
conventional SDE, and thus enlarging the diffusion term rather than the drift
term should lead to better generalization. However, we note that both the
diffusion and drift terms in the Slow SDEs are resulted from the long-term
effects of the diffusion term in the conventional SDE (Slow SDEs become
stationary if $\mSig = \vzero$). This means our view characterizes the role of
gradient noise in more detail, and therefore, goes one step further on the
conventional wisdom.

\myparagraph{Slow SDEs for neural nets with modern training techniques.}
In modern neural net training, it is common to add normalization layers and
weight decay ($L^2$-regularization) for better optimization and generalization.
However, these techniques lead to violations of our assumptions, e.g., no fixed
point exists in the regularized
loss~\citep{li2020reconciling,ahn2022understanding}. Still, a minimizer manifold
can be expected to exist for the unregularized loss. \citet{li2022fast} noted
that the drift and diffusion around the manifold proceeds faster in this case,
and derived a Slow SDE for SGD that captures $\cO(\frac{1}{\eta} \log
\frac{1}{\eta})$ discrete steps instead of $\cO(\frac{1}{\eta^2})$. We believe
that our analysis can also be extended to this case, and that adding local steps
still results in the effect of strengthening the drift term.

\section{Conclusions}

In this paper, we provide a theoretical analysis for Local SGD that captures its
long-term generalization benefit in the small learning rate regime. We derive
the Slow SDE for Local SGD as a generalization of the Slow SDE for
SGD~\citep{li2021happens}, and attribute the generalization improvement over SGD
to the larger drift term in the SDE for Local SGD. Our empirical validation shows that Local SGD indeed induces generalization benefits with small learning rate and long enough training time. The main limitation of our work is that our analysis does not imply any direct theoretical separation between SGD and Local SGD in terms of test accuracy, which requires a much deeper understanding of the loss landscape and the Slow SDEs and is left for future work. Another direction for future work is to design distributed training methods that provably generalize better than SGD based on the theoretical insights obtained from Slow SDEs.
\section*{Acknowledgement and Disclosure of Funding}
The work of Xinran Gu and Longbo Huang is supported by the Technology and Innovation Major Project of the Ministry of Science and Technology of China under Grant 2020AAA0108400 and 2020AAA0108403, the Tsinghua University Initiative Scientific Research Program, and Tsinghua Precision Medicine Foundation 10001020109.
The work of Kaifeng Lyu and Sanjeev Arora is supported by funding from NSF, ONR, Simons Foundation, DARPA and SRC.

\bibliographystyle{iclr2023_conference}
\bibliography{references} 

\begin{thebibliography}{82}
\providecommand{\natexlab}[1]{#1}
\providecommand{\url}[1]{\texttt{#1}}
\expandafter\ifx\csname urlstyle\endcsname\relax
  \providecommand{\doi}[1]{doi: #1}\else
  \providecommand{\doi}{doi: \begingroup \urlstyle{rm}\Url}\fi

\bibitem[Ahn et~al.(2022)Ahn, Zhang, and Sra]{ahn2022understanding}
Kwangjun Ahn, Jingzhao Zhang, and Suvrit Sra.
\newblock Understanding the unstable convergence of gradient descent.
\newblock In Kamalika Chaudhuri, Stefanie Jegelka, Le~Song, Csaba Szepesvari,
  Gang Niu, and Sivan Sabato (eds.), \emph{Proceedings of the 39th
  International Conference on Machine Learning}, volume 162 of
  \emph{Proceedings of Machine Learning Research}, pp.\  247--257. PMLR, 17--23
  Jul 2022.

\bibitem[Basu et~al.(2019)Basu, Data, Karakus, and
  Diggavi]{NEURIPS2019_Qsparse}
Debraj Basu, Deepesh Data, Can Karakus, and Suhas Diggavi.
\newblock Qsparse-local-{SGD}: Distributed {SGD} with quantization,
  sparsification and local computations.
\newblock In H.~Wallach, H.~Larochelle, A.~Beygelzimer, F.~d\textquotesingle
  Alch\'{e}-Buc, E.~Fox, and R.~Garnett (eds.), \emph{Advances in Neural
  Information Processing Systems}, volume~32. Curran Associates, Inc., 2019.

\bibitem[Bengio(2012)]{bengio2012practical}
Yoshua Bengio.
\newblock \emph{Practical Recommendations for Gradient-Based Training of Deep
  Architectures}, pp.\  437--478.
\newblock Springer Berlin Heidelberg, Berlin, Heidelberg, 2012.
\newblock ISBN 978-3-642-35289-8.
\newblock \doi{10.1007/978-3-642-35289-8_26}.

\bibitem[Blanc et~al.(2020)Blanc, Gupta, Valiant, and
  Valiant]{blanc2020implicit}
Guy Blanc, Neha Gupta, Gregory Valiant, and Paul Valiant.
\newblock Implicit regularization for deep neural networks driven by an
  ornstein-uhlenbeck like process.
\newblock In Jacob Abernethy and Shivani Agarwal (eds.), \emph{Proceedings of
  Thirty Third Conference on Learning Theory}, volume 125 of \emph{Proceedings
  of Machine Learning Research}, pp.\  483--513. PMLR, 09--12 Jul 2020.

\bibitem[Brandfonbrener \& Bruna(2020)Brandfonbrener and
  Bruna]{DBLP:conf/iclr/BrandfonbrenerB20}
David Brandfonbrener and Joan Bruna.
\newblock Geometric insights into the convergence of nonlinear {TD} learning.
\newblock In \emph{8th International Conference on Learning Representations,
  {ICLR} 2020, Addis Ababa, Ethiopia, April 26-30, 2020}. OpenReview.net, 2020.

\bibitem[Chen et~al.(2016)Chen, Pan, Monga, Bengio, and
  Jozefowicz]{chen2016revisiting}
Jianmin Chen, Xinghao Pan, Rajat Monga, Samy Bengio, and Rafal Jozefowicz.
\newblock Revisiting distributed synchronous {SGD}.
\newblock \emph{arXiv preprint arXiv:1604.00981}, 2016.

\bibitem[Chen \& Huo(2016)Chen and Huo]{chen2016}
Kai Chen and Qiang Huo.
\newblock Scalable training of deep learning machines by incremental block
  training with intra-block parallel optimization and blockwise model-update
  filtering.
\newblock In \emph{2016 IEEE International Conference on Acoustics, Speech and
  Signal Processing (ICASSP)}, pp.\  5880--5884, 2016.
\newblock \doi{10.1109/ICASSP.2016.7472805}.

\bibitem[Damian et~al.(2021)Damian, Ma, and Lee]{damian2021label}
Alex Damian, Tengyu Ma, and Jason~D. Lee.
\newblock Label noise {SGD} provably prefers flat global minimizers.
\newblock In A.~Beygelzimer, Y.~Dauphin, P.~Liang, and J.~Wortman Vaughan
  (eds.), \emph{Advances in Neural Information Processing Systems}, 2021.

\bibitem[Dinh et~al.(2017)Dinh, Pascanu, Bengio, and Bengio]{dinh2017sharp}
Laurent Dinh, Razvan Pascanu, Samy Bengio, and Yoshua Bengio.
\newblock Sharp minima can generalize for deep nets.
\newblock In Doina Precup and Yee~Whye Teh (eds.), \emph{Proceedings of the
  34th International Conference on Machine Learning}, volume~70 of
  \emph{Proceedings of Machine Learning Research}, pp.\  1019--1028. PMLR,
  06--11 Aug 2017.

\bibitem[Du \& Duan(2007)Du and Duan]{du2007invariant}
Aijun Du and JinQiao Duan.
\newblock Invariant manifold reduction for stochastic dynamical systems.
\newblock \emph{Dynamic Systems and Applications}, 16:\penalty0 681--696, 2007.

\bibitem[Falconer(1983)]{falconer1983differentiation}
KJ~Falconer.
\newblock Differentiation of the limit mapping in a dynamical system.
\newblock \emph{Journal of the London Mathematical Society}, 2\penalty0
  (2):\penalty0 356--372, 1983.

\bibitem[Fehrman et~al.(2020)Fehrman, Gess, and
  Jentzen]{fehrman2020convergence}
Benjamin Fehrman, Benjamin Gess, and Arnulf Jentzen.
\newblock Convergence rates for the stochastic gradient descent method for
  non-convex objective functions.
\newblock \emph{Journal of Machine Learning Research}, 21:\penalty0 136, 2020.

\bibitem[Filipovi{\'c}(2000)]{filipovic2000invariant}
Damir Filipovi{\'c}.
\newblock Invariant manifolds for weak solutions to stochastic equations.
\newblock \emph{Probability theory and related fields}, 118\penalty0
  (3):\penalty0 323--341, 2000.

\bibitem[Foret et~al.(2021)Foret, Kleiner, Mobahi, and
  Neyshabur]{foret2021sharpnessaware}
Pierre Foret, Ariel Kleiner, Hossein Mobahi, and Behnam Neyshabur.
\newblock Sharpness-aware minimization for efficiently improving
  generalization.
\newblock In \emph{International Conference on Learning Representations}, 2021.

\bibitem[Glasgow et~al.(2022)Glasgow, Yuan, and Ma]{glasgow2022sharp}
Margalit~R Glasgow, Honglin Yuan, and Tengyu Ma.
\newblock Sharp bounds for federated averaging ({Local SGD}) and continuous
  perspective.
\newblock In \emph{International Conference on Artificial Intelligence and
  Statistics}, pp.\  9050--9090. PMLR, 2022.

\bibitem[Goyal et~al.(2017)Goyal, Doll{\'a}r, Girshick, Noordhuis, Wesolowski,
  Kyrola, Tulloch, Jia, and He]{goyal2017accurate}
Priya Goyal, Piotr Doll{\'a}r, Ross Girshick, Pieter Noordhuis, Lukasz
  Wesolowski, Aapo Kyrola, Andrew Tulloch, Yangqing Jia, and Kaiming He.
\newblock Accurate, large minibatch {SGD}: Training imagenet in 1 hour.
\newblock \emph{arXiv preprint arXiv:1706.02677}, 2017.

\bibitem[Haddadpour et~al.(2019)Haddadpour, Kamani, Mahdavi, and
  Cadambe]{haddadpour2019local}
Farzin Haddadpour, Mohammad~Mahdi Kamani, Mehrdad Mahdavi, and Viveck Cadambe.
\newblock Local {SGD} with periodic averaging: Tighter analysis and adaptive
  synchronization.
\newblock \emph{Advances in Neural Information Processing Systems}, 32, 2019.

\bibitem[He et~al.(2015)He, Zhang, Ren, and Sun]{he2015delving}
Kaiming He, Xiangyu Zhang, Shaoqing Ren, and Jian Sun.
\newblock Delving deep into rectifiers: Surpassing human-level performance on
  imagenet classification.
\newblock In \emph{Proceedings of the IEEE international conference on computer
  vision}, pp.\  1026--1034, 2015.

\bibitem[He et~al.(2016)He, Zhang, Ren, and Sun]{he2016deep}
Kaiming He, Xiangyu Zhang, Shaoqing Ren, and Jian Sun.
\newblock Deep residual learning for image recognition.
\newblock In \emph{Proceedings of the IEEE conference on computer vision and
  pattern recognition}, pp.\  770--778, 2016.

\bibitem[Hendrycks \& Gimpel(2016)Hendrycks and Gimpel]{hendrycks2016gaussian}
Dan Hendrycks and Kevin Gimpel.
\newblock Gaussian error linear units (gelus).
\newblock \emph{arXiv preprint arXiv:1606.08415}, 2016.

\bibitem[Hochreiter \& Schmidhuber(1997)Hochreiter and
  Schmidhuber]{hochreiter1997flat}
Sepp Hochreiter and J{\"u}rgen Schmidhuber.
\newblock Flat minima.
\newblock \emph{Neural computation}, 9\penalty0 (1):\penalty0 1--42, 1997.

\bibitem[Hoffer et~al.(2017)Hoffer, Hubara, and Soudry]{hoffer2017train}
Elad Hoffer, Itay Hubara, and Daniel Soudry.
\newblock Train longer, generalize better: closing the generalization gap in
  large batch training of neural networks.
\newblock \emph{Advances in neural information processing systems}, 30, 2017.

\bibitem[Hu et~al.(2017)Hu, Li, Li, and Liu]{hu2017diffusion}
Wenqing Hu, Chris~Junchi Li, Lei Li, and Jian-Guo Liu.
\newblock On the diffusion approximation of nonconvex stochastic gradient
  descent.
\newblock \emph{arXiv preprint arXiv:1705.07562}, 2017.

\bibitem[Ibayashi \& Imaizumi(2021)Ibayashi and
  Imaizumi]{ibayashi2022expescape}
Hikaru Ibayashi and Masaaki Imaizumi.
\newblock Exponential escape efficiency of {SGD} from sharp minima in
  non-stationary regime.
\newblock \emph{arXiv preprint arXiv:2111.04004}, 2021.

\bibitem[Jastrz{\k{e}}bski et~al.(2017)Jastrz{\k{e}}bski, Kenton, Arpit,
  Ballas, Fischer, Bengio, and Storkey]{jastrzkebski2017three}
Stanis{\l}aw Jastrz{\k{e}}bski, Zachary Kenton, Devansh Arpit, Nicolas Ballas,
  Asja Fischer, Yoshua Bengio, and Amos Storkey.
\newblock Three factors influencing minima in {SGD}.
\newblock \emph{arXiv preprint arXiv:1711.04623}, 2017.

\bibitem[Jia et~al.(2018)Jia, Song, He, Wang, Rong, Zhou, Xie, Guo, Yang, Yu,
  et~al.]{jia2018highly}
Xianyan Jia, Shutao Song, Wei He, Yangzihao Wang, Haidong Rong, Feihu Zhou,
  Liqiang Xie, Zhenyu Guo, Yuanzhou Yang, Liwei Yu, et~al.
\newblock Highly scalable deep learning training system with mixed-precision:
  Training imagenet in four minutes.
\newblock \emph{Advances in Neural Information Processing Systems}, 2018.

\bibitem[Jiang et~al.(2020)Jiang, Neyshabur, Mobahi, Krishnan, and
  Bengio]{jiang2020fantastic}
Yiding Jiang, Behnam Neyshabur, Hossein Mobahi, Dilip Krishnan, and Samy
  Bengio.
\newblock Fantastic generalization measures and where to find them.
\newblock In \emph{International Conference on Learning Representations}, 2020.

\bibitem[Kairouz et~al.(2021)Kairouz, McMahan, Avent, Bellet, Bennis, Bhagoji,
  Bonawitz, Charles, Cormode, Cummings, et~al.]{kairouz2021advances}
Peter Kairouz, H~Brendan McMahan, Brendan Avent, Aur{\'e}lien Bellet, Mehdi
  Bennis, Arjun~Nitin Bhagoji, Kallista Bonawitz, Zachary Charles, Graham
  Cormode, Rachel Cummings, et~al.
\newblock Advances and open problems in federated learning.
\newblock \emph{Foundations and Trends{\textregistered} in Machine Learning},
  14\penalty0 (1--2):\penalty0 1--210, 2021.

\bibitem[Karimireddy et~al.(2020)Karimireddy, Kale, Mohri, Reddi, Stich, and
  Suresh]{karimireddy2020scaffold}
Sai~Praneeth Karimireddy, Satyen Kale, Mehryar Mohri, Sashank Reddi, Sebastian
  Stich, and Ananda~Theertha Suresh.
\newblock Scaffold: Stochastic controlled averaging for federated learning.
\newblock In \emph{International Conference on Machine Learning}, pp.\
  5132--5143. PMLR, 2020.

\bibitem[Katzenberger(1991)]{kat}
G.~S. Katzenberger.
\newblock Solutions of a stochastic differential equation forced onto a
  manifold by a large drift.
\newblock \emph{The Annals of Probability}, 19\penalty0 (4):\penalty0 1587 --
  1628, 1991.

\bibitem[Keskar et~al.(2017)Keskar, Mudigere, Nocedal, Smelyanskiy, and
  Tang]{keskar2017on}
Nitish~Shirish Keskar, Dheevatsa Mudigere, Jorge Nocedal, Mikhail Smelyanskiy,
  and Ping Tak~Peter Tang.
\newblock On large-batch training for deep learning: Generalization gap and
  sharp minima.
\newblock In \emph{International Conference on Learning Representations}, 2017.

\bibitem[Khaled et~al.(2020)Khaled, Mishchenko, and
  Richt{\'a}rik]{khaled2020tighter}
Ahmed Khaled, Konstantin Mishchenko, and Peter Richt{\'a}rik.
\newblock Tighter theory for local {SGD} on identical and heterogeneous data.
\newblock In \emph{International Conference on Artificial Intelligence and
  Statistics}, pp.\  4519--4529. PMLR, 2020.

\bibitem[Kleinberg et~al.(2018)Kleinberg, Li, and
  Yuan]{kleinberg2018alternative}
Bobby Kleinberg, Yuanzhi Li, and Yang Yuan.
\newblock An alternative view: When does {SGD} escape local minima?
\newblock In Jennifer Dy and Andreas Krause (eds.), \emph{Proceedings of the
  35th International Conference on Machine Learning}, volume~80 of
  \emph{Proceedings of Machine Learning Research}, pp.\  2698--2707. PMLR,
  10--15 Jul 2018.

\bibitem[Krizhevsky(2014)]{krizhevsky2014one}
Alex Krizhevsky.
\newblock One weird trick for parallelizing convolutional neural networks.
\newblock \emph{arXiv preprint arXiv:1404.5997}, 2014.

\bibitem[Krizhevsky et~al.(2009)]{krizhevsky2009learning}
Alex Krizhevsky et~al.
\newblock Learning multiple layers of features from tiny images.
\newblock 2009.

\bibitem[Leclerc et~al.(2022)Leclerc, Ilyas, Engstrom, Park, Salman, and
  Madry]{leclerc2022ffcv}
Guillaume Leclerc, Andrew Ilyas, Logan Engstrom, Sung~Min Park, Hadi Salman,
  and Aleksander Madry.
\newblock ffcv.
\newblock \url{https://github.com/libffcv/ffcv/}, 2022.

\bibitem[LeCun et~al.(2012)LeCun, Bottou, Orr, and
  M{\"u}ller]{lecun2012efficient}
Yann~A. LeCun, L{\'e}on Bottou, Genevieve~B. Orr, and Klaus-Robert M{\"u}ller.
\newblock \emph{Efficient {BackProp}}, pp.\  9--48.
\newblock Springer Berlin Heidelberg, Berlin, Heidelberg, 2012.
\newblock ISBN 978-3-642-35289-8.
\newblock \doi{10.1007/978-3-642-35289-8_3}.

\bibitem[Li et~al.(2019{\natexlab{a}})Li, Tai, and E]{li2019stochastic}
Qianxiao Li, Cheng Tai, and Weinan E.
\newblock Stochastic modified equations and dynamics of stochastic gradient
  algorithms i: Mathematical foundations.
\newblock \emph{Journal of Machine Learning Research}, 20\penalty0
  (40):\penalty0 1--47, 2019{\natexlab{a}}.

\bibitem[Li et~al.(2019{\natexlab{b}})Li, Huang, Yang, Wang, and
  Zhang]{li2019convergence}
Xiang Li, Kaixuan Huang, Wenhao Yang, Shusen Wang, and Zhihua Zhang.
\newblock On the convergence of fedavg on non-iid data.
\newblock In \emph{International Conference on Learning Representations},
  2019{\natexlab{b}}.

\bibitem[Li et~al.(2020)Li, Lyu, and Arora]{li2020reconciling}
Zhiyuan Li, Kaifeng Lyu, and Sanjeev Arora.
\newblock Reconciling modern deep learning with traditional optimization
  analyses: The intrinsic learning rate.
\newblock \emph{Advances in Neural Information Processing Systems},
  33:\penalty0 14544--14555, 2020.

\bibitem[Li et~al.(2021{\natexlab{a}})Li, Malladi, and Arora]{li2021validity}
Zhiyuan Li, Sadhika Malladi, and Sanjeev Arora.
\newblock On the validity of modeling {SGD} with stochastic differential
  equations (sdes).
\newblock \emph{Advances in Neural Information Processing Systems},
  34:\penalty0 12712--12725, 2021{\natexlab{a}}.

\bibitem[Li et~al.(2021{\natexlab{b}})Li, Wang, and Arora]{li2021happens}
Zhiyuan Li, Tianhao Wang, and Sanjeev Arora.
\newblock What happens after {SGD} reaches zero loss?--a mathematical
  framework.
\newblock In \emph{International Conference on Learning Representations},
  2021{\natexlab{b}}.

\bibitem[Li et~al.(2022)Li, Wang, and Yu]{li2022fast}
Zhiyuan Li, Tianhao Wang, and Dingli Yu.
\newblock Fast mixing of stochastic gradient descent with normalization and
  weight decay.
\newblock In Alice~H. Oh, Alekh Agarwal, Danielle Belgrave, and Kyunghyun Cho
  (eds.), \emph{Advances in Neural Information Processing Systems}, 2022.

\bibitem[Lin et~al.(2020{\natexlab{a}})Lin, Kong, Stich, and
  Jaggi]{lin2020extrapolation}
Tao Lin, Lingjing Kong, Sebastian Stich, and Martin Jaggi.
\newblock Extrapolation for large-batch training in deep learning.
\newblock In Hal~Daumé III and Aarti Singh (eds.), \emph{Proceedings of the
  37th International Conference on Machine Learning}, volume 119 of
  \emph{Proceedings of Machine Learning Research}, pp.\  6094--6104. PMLR,
  13--18 Jul 2020{\natexlab{a}}.

\bibitem[Lin et~al.(2020{\natexlab{b}})Lin, Stich, Patel, and
  Jaggi]{lin2020dont}
Tao Lin, Sebastian~U. Stich, Kumar~Kshitij Patel, and Martin Jaggi.
\newblock Don't use large mini-batches, use {Local SGD}.
\newblock In \emph{International Conference on Learning Representations},
  2020{\natexlab{b}}.

\bibitem[Lyu et~al.(2022)Lyu, Li, and Arora]{lyu2022understanding}
Kaifeng Lyu, Zhiyuan Li, and Sanjeev Arora.
\newblock Understanding the generalization benefit of normalization layers:
  Sharpness reduction, 2022.

\bibitem[Ma \& Ying(2021)Ma and Ying]{ma2021on}
Chao Ma and Lexing Ying.
\newblock On linear stability of {SGD} and input-smoothness of neural networks.
\newblock In M.~Ranzato, A.~Beygelzimer, Y.~Dauphin, P.S. Liang, and J.~Wortman
  Vaughan (eds.), \emph{Advances in Neural Information Processing Systems},
  volume~34, pp.\  16805--16817. Curran Associates, Inc., 2021.

\bibitem[Malladi et~al.(2022)Malladi, Lyu, Panigrahi, and
  Arora]{malladi2022sdes}
Sadhika Malladi, Kaifeng Lyu, Abhishek Panigrahi, and Sanjeev Arora.
\newblock On the {SDE}s and scaling rules for adaptive gradient algorithms.
\newblock In Alice~H. Oh, Alekh Agarwal, Danielle Belgrave, and Kyunghyun Cho
  (eds.), \emph{Advances in Neural Information Processing Systems}, 2022.

\bibitem[Mann et~al.(2009)Mann, McDonald, Mohri, Silberman, and
  Walker]{DBLP:conf/nips/MannMMSW09}
Gideon Mann, Ryan~T. McDonald, Mehryar Mohri, Nathan Silberman, and Dan Walker.
\newblock Efficient large-scale distributed training of conditional maximum
  entropy models.
\newblock In \emph{Advances in Neural Information Processing Systems 22}, pp.\
  1231--1239, 2009.

\bibitem[McMahan et~al.(2017)McMahan, Moore, Ramage, Hampson, and
  y~Arcas]{mcmahan2017communication}
Brendan McMahan, Eider Moore, Daniel Ramage, Seth Hampson, and Blaise~Aguera
  y~Arcas.
\newblock Communication-efficient learning of deep networks from decentralized
  data.
\newblock In \emph{Artificial intelligence and statistics}, pp.\  1273--1282.
  PMLR, 2017.

\bibitem[Mnih et~al.(2015)Mnih, Kavukcuoglu, Silver, Rusu, Veness, Bellemare,
  Graves, Riedmiller, Fidjeland, Ostrovski, et~al.]{mnih2015human}
Volodymyr Mnih, Koray Kavukcuoglu, David Silver, Andrei~A Rusu, Joel Veness,
  Marc~G Bellemare, Alex Graves, Martin Riedmiller, Andreas~K Fidjeland, Georg
  Ostrovski, et~al.
\newblock Human-level control through deep reinforcement learning.
\newblock \emph{nature}, 518\penalty0 (7540):\penalty0 529--533, 2015.

\bibitem[Neyshabur et~al.(2017)Neyshabur, Bhojanapalli, Mcallester, and
  Srebro]{neyshabur2017exploring}
Behnam Neyshabur, Srinadh Bhojanapalli, David Mcallester, and Nati Srebro.
\newblock Exploring generalization in deep learning.
\newblock In I.~Guyon, U.~Von Luxburg, S.~Bengio, H.~Wallach, R.~Fergus,
  S.~Vishwanathan, and R.~Garnett (eds.), \emph{Advances in Neural Information
  Processing Systems}, volume~30. Curran Associates, Inc., 2017.

\bibitem[Ortiz et~al.(2021)Ortiz, Frankle, Rabbat, Morcos, and
  Ballas]{ortiz2021trade}
Jose Javier~Gonzalez Ortiz, Jonathan Frankle, Mike Rabbat, Ari Morcos, and
  Nicolas Ballas.
\newblock Trade-offs of {Local SGD} at scale: An empirical study.
\newblock \emph{arXiv preprint arXiv:2110.08133}, 2021.

\bibitem[Povey et~al.(2014)Povey, Zhang, and Khudanpur]{povey2014parallel}
Daniel Povey, Xiaohui Zhang, and Sanjeev Khudanpur.
\newblock Parallel training of dnns with natural gradient and parameter
  averaging.
\newblock \emph{arXiv preprint arXiv:1410.7455}, 2014.

\bibitem[Ramachandran et~al.(2017)Ramachandran, Zoph, and
  Le]{ramachandran2017searching}
Prajit Ramachandran, Barret Zoph, and Quoc~V Le.
\newblock Searching for activation functions.
\newblock \emph{arXiv preprint arXiv:1710.05941}, 2017.

\bibitem[Recht et~al.(2011)Recht, R{\'{e}}, Wright, and
  Niu]{DBLP:conf/nips/RechtRWN11}
Benjamin Recht, Christopher R{\'{e}}, Stephen~J. Wright, and Feng Niu.
\newblock Hogwild: {A} lock-free approach to parallelizing stochastic gradient
  descent.
\newblock In \emph{Advances in Neural Information Processing Systems 24}, pp.\
  693--701, 2011.

\bibitem[Russakovsky et~al.(2015)Russakovsky, Deng, Su, Krause, Satheesh, Ma,
  Huang, Karpathy, Khosla, Bernstein, Berg, and Fei-Fei]{ILSVRC15}
Olga Russakovsky, Jia Deng, Hao Su, Jonathan Krause, Sanjeev Satheesh, Sean Ma,
  Zhiheng Huang, Andrej Karpathy, Aditya Khosla, Michael Bernstein,
  Alexander~C. Berg, and Li~Fei-Fei.
\newblock {ImageNet Large Scale Visual Recognition Challenge}.
\newblock \emph{International Journal of Computer Vision (IJCV)}, 115\penalty0
  (3):\penalty0 211--252, 2015.
\newblock \doi{10.1007/s11263-015-0816-y}.

\bibitem[Seide et~al.(2014)Seide, Fu, Droppo, Li, and
  Yu]{DBLP:conf/interspeech/SeideFDLY14}
Frank Seide, Hao Fu, Jasha Droppo, Gang Li, and Dong Yu.
\newblock 1-bit stochastic gradient descent and its application to
  data-parallel distributed training of speech dnns.
\newblock In Haizhou Li, Helen~M. Meng, Bin Ma, Engsiong Chng, and Lei Xie
  (eds.), \emph{{INTERSPEECH} 2014, 15th Annual Conference of the International
  Speech Communication Association, Singapore, September 14-18, 2014}, pp.\
  1058--1062. {ISCA}, 2014.
\newblock URL
  \url{http://www.isca-speech.org/archive/interspeech\_2014/i14\_1058.html}.

\bibitem[Shallue et~al.(2019)Shallue, Lee, Antognini, Sohl-Dickstein, Frostig,
  and Dahl]{shallue2019measuring}
Christopher~J. Shallue, Jaehoon Lee, Joseph Antognini, Jascha Sohl-Dickstein,
  Roy Frostig, and George~E. Dahl.
\newblock Measuring the effects of data parallelism on neural network training.
\newblock \emph{Journal of Machine Learning Research}, 20\penalty0
  (112):\penalty0 1--49, 2019.

\bibitem[Simonyan \& Zisserman(2015)Simonyan and Zisserman]{vgg}
K.~Simonyan and A.~Zisserman.
\newblock Very deep convolutional networks for large-scale image recognition.
\newblock In \emph{International Conference on Learning Representations}, May
  2015.

\bibitem[Smith et~al.(2020)Smith, Elsen, and De]{smith2020generalization}
Samuel Smith, Erich Elsen, and Soham De.
\newblock On the generalization benefit of noise in stochastic gradient
  descent.
\newblock In Hal~Daumé III and Aarti Singh (eds.), \emph{Proceedings of the
  37th International Conference on Machine Learning}, volume 119 of
  \emph{Proceedings of Machine Learning Research}, pp.\  9058--9067. PMLR,
  13--18 Jul 2020.

\bibitem[Smith et~al.(2021)Smith, Dherin, Barrett, and De]{smith2021on}
Samuel~L Smith, Benoit Dherin, David Barrett, and Soham De.
\newblock On the origin of implicit regularization in stochastic gradient
  descent.
\newblock In \emph{International Conference on Learning Representations}, 2021.

\bibitem[Stich(2018)]{stich2018local}
Sebastian~U Stich.
\newblock {Local SGD} converges fast and communicates little.
\newblock In \emph{International Conference on Learning Representations}, 2018.

\bibitem[Strom(2015)]{DBLP:conf/interspeech/Strom15}
Nikko Strom.
\newblock Scalable distributed {DNN} training using commodity {GPU} cloud
  computing.
\newblock In \emph{{INTERSPEECH} 2015, 16th Annual Conference of the
  International Speech Communication Association, Dresden, Germany, September
  6-10, 2015}, pp.\  1488--1492. {ISCA}, 2015.

\bibitem[Su \& Chen(2015)Su and Chen]{su2015experiments}
Hang Su and Haoyu Chen.
\newblock Experiments on parallel training of deep neural network using model
  averaging.
\newblock \emph{arXiv preprint arXiv:1507.01239}, 2015.

\bibitem[Sutton \& Barto(1998)Sutton and Barto]{DBLP:books/lib/SuttonB98}
Richard~S. Sutton and Andrew~G. Barto.
\newblock \emph{Reinforcement learning - an introduction}.
\newblock Adaptive computation and machine learning. {MIT} Press, 1998.
\newblock ISBN 978-0-262-19398-6.

\bibitem[Wang \& Joshi(2019)Wang and Joshi]{wang2019adaptive}
Jianyu Wang and Gauri Joshi.
\newblock Adaptive communication strategies to achieve the best error-runtime
  trade-off in local-update {SGD}.
\newblock \emph{Proceedings of Machine Learning and Systems}, 1:\penalty0
  212--229, 2019.

\bibitem[Wang \& Joshi(2021)Wang and Joshi]{wang2021cooperative}
Jianyu Wang and Gauri Joshi.
\newblock Cooperative {SGD}: A unified framework for the design and analysis of
  local-update {SGD} algorithms.
\newblock \emph{Journal of Machine Learning Research}, 22\penalty0
  (213):\penalty0 1--50, 2021.

\bibitem[Wang et~al.(2022)Wang, Das, Joshi, Kale, Xu, and
  Zhang]{wang2022unreasonable}
Jianyu Wang, Rudrajit Das, Gauri Joshi, Satyen Kale, Zheng Xu, and Tong Zhang.
\newblock On the unreasonable effectiveness of federated averaging with
  heterogeneous data.
\newblock \emph{arXiv preprint arXiv:2206.04723}, 2022.

\bibitem[Woodworth et~al.(2020{\natexlab{a}})Woodworth, Patel, Stich, Dai,
  Bullins, Mcmahan, Shamir, and Srebro]{woodworth2020local}
Blake Woodworth, Kumar~Kshitij Patel, Sebastian Stich, Zhen Dai, Brian Bullins,
  Brendan Mcmahan, Ohad Shamir, and Nathan Srebro.
\newblock Is local sgd better than minibatch sgd?
\newblock In \emph{International Conference on Machine Learning}, pp.\
  10334--10343. PMLR, 2020{\natexlab{a}}.

\bibitem[Woodworth et~al.(2020{\natexlab{b}})Woodworth, Patel, and
  Srebro]{woodworth2020minibatch}
Blake~E Woodworth, Kumar~Kshitij Patel, and Nati Srebro.
\newblock Minibatch vs {Local SGD} for heterogeneous distributed learning.
\newblock \emph{Advances in Neural Information Processing Systems},
  33:\penalty0 6281--6292, 2020{\natexlab{b}}.

\bibitem[Wu et~al.(2018)Wu, Ma, and E]{wu2018how}
Lei Wu, Chao Ma, and Weinan E.
\newblock How sgd selects the global minima in over-parameterized learning: A
  dynamical stability perspective.
\newblock In S.~Bengio, H.~Wallach, H.~Larochelle, K.~Grauman, N.~Cesa-Bianchi,
  and R.~Garnett (eds.), \emph{Advances in Neural Information Processing
  Systems}, volume~31. Curran Associates, Inc., 2018.

\bibitem[Xie et~al.(2021)Xie, Sato, and Sugiyama]{xie2021a}
Zeke Xie, Issei Sato, and Masashi Sugiyama.
\newblock A diffusion theory for deep learning dynamics: Stochastic gradient
  descent exponentially favors flat minima.
\newblock In \emph{International Conference on Learning Representations}, 2021.

\bibitem[You et~al.(2018)You, Zhang, Hsieh, Demmel, and
  Keutzer]{you2018imagenet}
Yang You, Zhao Zhang, Cho-Jui Hsieh, James Demmel, and Kurt Keutzer.
\newblock Imagenet training in minutes.
\newblock In \emph{Proceedings of the 47th International Conference on Parallel
  Processing}, pp.\  1--10, 2018.

\bibitem[You et~al.(2020)You, Li, Reddi, Hseu, Kumar, Bhojanapalli, Song,
  Demmel, Keutzer, and Hsieh]{you2020large}
Yang You, Jing Li, Sashank Reddi, Jonathan Hseu, Sanjiv Kumar, Srinadh
  Bhojanapalli, Xiaodan Song, James Demmel, Kurt Keutzer, and Cho-Jui Hsieh.
\newblock Large batch optimization for deep learning: Training {BERT} in 76
  minutes.
\newblock In \emph{International Conference on Learning Representations}, 2020.

\bibitem[Yu et~al.(2019)Yu, Yang, and Zhu]{yu2019parallel}
Hao Yu, Sen Yang, and Shenghuo Zhu.
\newblock Parallel restarted {SGD} with faster convergence and less
  communication: Demystifying why model averaging works for deep learning.
\newblock In \emph{Proceedings of the AAAI Conference on Artificial
  Intelligence}, volume~33, pp.\  5693--5700, 2019.

\bibitem[Zhang et~al.(2020)Zhang, Karimireddy, Veit, Kim, Reddi, Kumar, and
  Sra]{zhang2020adaptive}
Jingzhao Zhang, Sai~Praneeth Karimireddy, Andreas Veit, Seungyeon Kim, Sashank
  Reddi, Sanjiv Kumar, and Suvrit Sra.
\newblock Why are adaptive methods good for attention models?
\newblock \emph{Advances in Neural Information Processing Systems},
  33:\penalty0 15383--15393, 2020.

\bibitem[Zhang et~al.(2014)Zhang, Trmal, Povey, and Khudanpur]{acoustic}
Xiaohui Zhang, Jan Trmal, Daniel Povey, and Sanjeev Khudanpur.
\newblock Improving deep neural network acoustic models using generalized
  maxout networks.
\newblock In \emph{2014 IEEE International Conference on Acoustics, Speech and
  Signal Processing (ICASSP)}, pp.\  215--219, 2014.
\newblock \doi{10.1109/ICASSP.2014.6853589}.

\bibitem[Zhou \& Cong(2018)Zhou and Cong]{ijcai2018p447}
Fan Zhou and Guojing Cong.
\newblock On the convergence properties of a k-step averaging stochastic
  gradient descent algorithm for nonconvex optimization.
\newblock In \emph{Proceedings of the Twenty-Seventh International Joint
  Conference on Artificial Intelligence, {IJCAI-18}}, pp.\  3219--3227.
  International Joint Conferences on Artificial Intelligence Organization, 7
  2018.
\newblock \doi{10.24963/ijcai.2018/447}.
\newblock URL \url{https://doi.org/10.24963/ijcai.2018/447}.

\bibitem[Zhu et~al.(2018)Zhu, Wu, Yu, Wu, and Ma]{zhu2018anisotropic}
Zhanxing Zhu, Jingfeng Wu, Bing Yu, Lei Wu, and Jinwen Ma.
\newblock The anisotropic noise in stochastic gradient descent: Its behavior of
  escaping from sharp minima and regularization effects.
\newblock \emph{arXiv preprint arXiv:1803.00195}, 2018.

\bibitem[Zinkevich et~al.(2010)Zinkevich, Weimer, Li, and
  Smola]{NIPS2010_abea47ba}
Martin Zinkevich, Markus Weimer, Lihong Li, and Alex Smola.
\newblock Parallelized stochastic gradient descent.
\newblock In J.~Lafferty, C.~Williams, J.~Shawe-Taylor, R.~Zemel, and
  A.~Culotta (eds.), \emph{Advances in Neural Information Processing Systems},
  volume~23. Curran Associates, Inc., 2010.

\bibitem[Øksendal(2013)]{oksendal2013stochastic}
Bernt Øksendal.
\newblock \emph{Stochastic differential equations: an introduction with
  applications}.
\newblock Springer Science \& Business Media, 2013.

\end{thebibliography}
\newpage
\tableofcontents
\newpage

\appendix

\section{Additional Related Works} \label{sec:addrelated}
\label{sec:related}
\myparagraph{Optimization aspect of Local SGD.} Local SGD is a
communication-efficient variant of parallel SGD, where multiple workers perform
SGD independently and average the model parameters periodically. Dating back
to~\citet{DBLP:conf/nips/MannMMSW09} and~\citet{NIPS2010_abea47ba}, this
strategy has been widely adopted to reduce the communication cost and speed up
training in both scenarios of data center distributed training~\citep{chen2016,
acoustic, povey2014parallel, su2015experiments} and Federated Learning
\citep{mcmahan2017communication,kairouz2021advances}.  To further accelerate
training, \citet{wang2019adaptive} and \citet{haddadpour2019local} proposed
adaptive schemes for the averaging frequency, and \citet{NEURIPS2019_Qsparse}
combined Local SGD with gradient compression. Motivated to theoretically
understand the empirical success of Local SGD, a lot of researchers analyzed the
convergence rate of Local SGD under various settings, e.g.,
homogeneous/heterogeneous data and convex/non-convex objective functions.
Among them, \citet{yu2019parallel, stich2018local, khaled2020tighter,
woodworth2020local} focus on the homogeneous setting where data for each worker
are independent and identically distributed (IID). \citet{li2019convergence,
karimireddy2020scaffold, glasgow2022sharp, woodworth2020minibatch,
wang2022unreasonable} study the heterogeneous setting, where workers have
non-IID data and local updates may induce ``client drift''
\citep{karimireddy2020scaffold} and hurt optimization.
The error bound of Local SGD obtained by these works is typically inferior to
that of SGD with the same global batch size for fixed number of iterations/epochs and becomes worse as the
number of local steps increases, revealing a trade-off between less
communication and better optimization. In this paper, we are interested in the
generalization aspect of Local SGD in the homogeneous setting, assuming the training loss can be optimized
to a small value.
\myparagraph{Gradient noise and generalization.} 
The effect of stochastic gradient noise on generalization has been studied from
different aspects, e.g., changing the order of learning different
patterns~\citet{li2019stochastic}, inducing an implicit regularizer in the
second-order SDE approximation~\citet{smith2021on,li2019stochastic}. Our work
follows a line of works studying the effect of noise in the lens of sharpness,
which is long believed to be related to
generalization~\citet{hochreiter1997flat,neyshabur2017exploring}.
\citet{keskar2017on} empirically observed that large-batch training leads to
worse generalization and sharper minima than small-batch training.
\citet{wu2018how,hu2017diffusion,ma2021on} showed that gradient noise
destabilizes the training around sharp minima,
and
\citet{kleinberg2018alternative,zhu2018anisotropic,xie2021a,ibayashi2022expescape}
quantitatively characterized how SGD escapes sharp minima. The most related
papers are \citet{blanc2020implicit,damian2021label,li2021happens}, which focus
on the training dynamics near a manifold of minima and study the effect of noise
on sharpness (see also \Cref{subsec:sde-near-man}). Though the mathematical
definition of sharpness may be vulnerable to the various symmetries in deep
neural nets~\citep{dinh2017sharp}, sharpness still appears to be one of the most
promising tools for predicting
generalization~\citep{jiang2020fantastic,foret2021sharpnessaware}.

\myparagraph{Improving generalization in large-batch training.}
The generalization issue of the large-batch (or full-batch) training has been
observed as early as \citep{bengio2012practical,lecun2012efficient}. As
mentioned in \Cref{sec:intro}, the generalization issue of large-batch
training could be due to the lack of a sufficient amount of stochastic noise. To
make up the noise in large-batch training,
\citet{krizhevsky2014one,goyal2017accurate} empirically discovered the {\em
Linear Scaling Rule} for SGD, which suggests enlarging the learning rate
proportionally to the batch size.
\citet{jastrzkebski2017three}
adopted an SDE-based analysis to justify that this scaling rule indeed retains
the same amount of noise as small-batch training~(see also
\Cref{subsec:con-sde}). However, the SDE approximation may fail if the learning
rate is too large~\citep{li2021validity}, especially in the early phase of
training before the first learning rate decay~\citep{smith2020generalization}.
\citet{shallue2019measuring} demonstrated that generalization gap between small- and large-batch training
can also depend on many other training hyperparameters. Besides
enlarging the learning rate, other approaches have also been proposed to reduce the
gap, including training longer~\citep{hoffer2017train},
learning rate warmup~\citep{goyal2017accurate}, LARS~\citep{you2018imagenet},
LAMB~\citep{you2020large}. In this paper, we focus on using Local SGD to improve
generalization, but adding local steps is a generic training trick that can also
be combined with others, e.g., Local LARS~\citep{lin2020dont}, Local
Extrap-SGD~\citep{lin2020extrapolation}.
\newpage
\section{Implementation Details of Parallel SGD, Local SGD and Post-local SGD} \label{sec:pseudocode}

In this section, we present the formal procedures for Parallel SGD, Local SGD
and Post-local SGD. Given a training dataset and a data augmentation function,
\Cref{algo:sampler-wp,algo:sampler-wop} show the implementations of distributed
samplers for sampling local batches with and without replacement.
Then \Cref{algo:psgd,algo:lsgd,algo:plsgd} show the implementations of parallel SGD,
Local SGD and Post-local SGD that can run with either of the samplers.
\myparagraph{Sampling with replacement.}
Our theory analyzes parallel SGD, Local SGD and Post-local SGD when local
batches are sampled with replacement (\Cref{algo:sampler-wp}). That is, local
batches consist of IID samples from the same training distribution $\ctD$, where
$\ctD$ serves as an abstraction of the distribution of an augmented sample drawn from the training dataset.
The mathematical
formulations are given in \Cref{sec:intro}.
\myparagraph{Sampling without replacement.}
Slightly different from our theory, we use the sampling without replacement
(\Cref{algo:sampler-wop}) in our experiments unless otherwise stated.
This sampling scheme is standard in practice: 
it is used by \citet{goyal2017accurate} for parallel SGD
and by \citet{lin2020dont,ortiz2021trade} for Post-local/Local SGD.
This sampling scheme works as follows. At the
beginning of every epoch, the whole training dataset is shuffled and evenly
partitioned into $K$ shards. Each worker takes one shard and samples batches
without replacement. When all workers pass their own shard, the next epoch
begins and the whole dataset is reshuffled.
An alternative view is that the
workers always share the same dataset. For each epoch, they perform local steps
by sampling batches of data without replacement until the dataset contains too
few data to form a batch. Then another epoch starts with the dataset reloaded to
the initial state.
\myparagraph{Discrepancy in Sampling Schemes.} 
We argue that this discrepancy
between theory and experiments on sample schemes is minor.
Though sampling without replacement is
standard in practice, most previous works, e.g., \citet{wang2019adaptive,
li2021validity, zhang2020adaptive}, analyze sampling with replacement for
technical simplicity and yields meaningful results.

Moreover, even if we change the
sampling scheme to with replacement, Local SGD can still improve the generalization of SGD
(by merely adding local steps). See \Cref{sec:regime add exp} for the experiments.
We believe that the reasons
for better generalization of Local SGD with either sampling scheme are similar
and leave the analysis for sampling without replacement for future work.

\SetKwFunction{fnSample}{Sample}

\newpage

\begin{algorithm}[t]
    \caption{Distributed Sampler on $K$ Workers (Sampling with Replacement)}
    \label{algo:sampler-wp}

    \textbf{Require}: shared training dataset $\cD$, data augmentation function $\cA(\hat{\xi})$ \\
    \textbf{Hyperparameters}: local batch size $\Bloc$ \\
    \BlankLine
    \BlankLine
    \Fn{\fnSample{} {\bf on} worker $k$}{
        Draw $\Bloc$ IID samples $\hat{\xi}_1, \dots, \hat{\xi}_{\Bloc}$ from $\cD$ with replacement \;
        $\xi_b \gets \cA(\hat{\xi}_b)$ for all $1 \le b \le \Bloc$ \tcp*[r]{apply data augmentation}
        \Return{$(\xi_1, \dots, \xi_{\Bloc})$} \;
    }
    \BlankLine
    \BlankLine
\end{algorithm}

\begin{algorithm}[htbp]
    \caption{Distributed Sampler on $K$ Workers (Sampling without Replacement)}\label{algo:sampler-wop}

    \textbf{Require}: shared training dataset $\cD$, data augmentation function $\cA(\hat{\xi})$ \\
    \textbf{Hyperparameters}: local batch size $\Bloc$ \\
    \textbf{Constant}: $N_{\mathrm{loc}} := \left\lfloor \frac{\abs{\cD}}{K\Bloc} \right\rfloor$
    \tcp*[f]{number of local batches per worker per epoch} \\
    \textbf{Local Variables}: $c^{(k)} \gets N_{\mathrm{loc}} \Bloc$ for worker $k$ \tcp*[f]{number of samples drawn in this epoch} \\
    \BlankLine
    \BlankLine
    \Fn{\fnSample{} {\bf on} worker $k$}{
        \If {$c^{(k)} = N_{\mathrm{loc}} \Bloc$}{
            \tcp{Now start a new epoch}
            Wait until all the other workers reach this line \tcp*[r]{synchronize}
            Draw a random permutation $P$ of $1, \dots, \abs{D}$ jointly with other workers
            so that the same permutation is shared among all workers \tcp*[r]{reshuffle the dataset}
            $Q^{(k)}_{j} \gets P_{(k-1)N_{\mathrm{loc}} \Bloc + j}$ for all $1 \le j \le N_{\mathrm{loc}}$ \tcp*[r]{partition the dataset}
            $c^{(k)} \gets 0$ \;
        }
        \For{$i=1,\dots,\Bloc$}{
            $\hat{\xi}_i \gets$ the $Q^{(k)}_{c^{(k)} + i}$-th data point of $\cD$ \tcp*[r]{sample without replacement}
            $\xi_i \gets \cA(\hat{\xi}_i)$ \tcp*[r]{apply data augmentation}
        }
        $c^{(k)} \gets c^{(k)} + \Bloc$ \;
        \Return{$(\xi_1, \dots, \xi_{\Bloc})$} \;
    }
    \BlankLine
    \BlankLine
\end{algorithm}

\begin{algorithm}[t]
    \caption{Parallel SGD on $K$ Workers}\label{algo:psgd}
    \textbf{Input}: loss function $\ell(\vtheta; \xi)$, initial parameter $\vtheta_0$ \\
    \textbf{Hyperparameters}: total number of iterations $T$, learning rate $\eta$, local batch size $\Bloc$ \\
    \BlankLine
    \BlankLine
    \For{$t=0, \cdots, T-1$}{
        \ForP{each worker $k$} {
            $(\xi_{k,t,1}, \dots, \xi_{k,t,\Bloc}) \gets \fnSample{}$  \tcp*[r]{sample a local batch}
            $\vg_{k, t}\gets \frac{1}{\Bloc}\sum_{i=1}^{\Bloc} \nabla \ell (\vtheta_t; \xi_{k,t, i})$ \tcp*[r]{computing the local gradient}
        }
        $\vg_t \gets \frac{1}{K} \sum_{k=1}^K \vg_{k, t}$ \tcp*[r]{all-Reduce aggregation of local gradients} 
        $\vtheta_{t+1}\gets\vtheta_t - \eta_t \vg_t$ \tcp*[r]{update the model} 
    }
    \BlankLine
    \BlankLine
\end{algorithm}

\begin{algorithm}[H]
    \caption{Local SGD on $K$ Workers}\label{algo:lsgd}
    \textbf{Input}: loss function $\ell(\vtheta; \xi)$, initial parameter $\bvths[0]$ \\
    \textbf{Hyperparameters}: total number of rounds $R$, number of local steps $H$ per round \\
    \textbf{Hyperparameters}: learning rate~$\eta$, local batch size $\Bloc$ \\
    \BlankLine
    \BlankLine
    \For{$s=0, \dots, R-1$} {
        \ForP{each worker $k$} {
            $\vths_{k,0} \gets \bvths[0]$ \tcp*[r]{maintain a local copy of the global iterate}
            \For{$t = 0, \dots, H - 1$} {
                $(\xis_{k,t,1}, \dots, \xis_{k,t,\Bloc}) \gets \fnSample{}$  \tcp*[r]{sample a local batch}
                $\vgs_{k, t}\gets \frac{1}{\Bloc}\sum_{i=1}^{\Bloc} \nabla \ell (\vths_{k,t}; \xis_{k,t, i})$ \tcp*[r]{computing the local gradient}
                $\vths_{k,t+1} \gets \vths_{k,t} - \eta \vgs_{k,t}$ \tcp*[r]{update the local model}
            }
        }
        $\bvths[s+1] \gets \frac{1}{K} \sum_{k=1}^K \vths_{k,H}$ \tcp*[r]{all-Reduce aggregation of local iterates} 
    }
    \BlankLine
    \BlankLine
\end{algorithm}

\begin{algorithm}[H]
    \caption{Post-local SGD on $K$ Workers}\label{algo:plsgd}

    \textbf{Input}: loss function $\ell(\vtheta; \xi)$, initial parameter $\vtheta_0$ \\
    \textbf{Hyperparameters}: total number of iterations $T$, learning rate $\eta$, local batch size $\Bloc$ \\
    \textbf{Hyperparameters}: switching time point $t_0$, number of local steps $H$ per round \\ 
    \textbf{Ensure}: $T - t_0$ is a multiple of $H$ \\

    \BlankLine
    \BlankLine
    Starting from $\vtheta_0$, run Parallel SGD for $t_0$ iterations and obtain $\vtheta_{t_0}$ \;
    Starting from $\vtheta_{t_0}$, run Local SGD for $\frac{1}{H}(T-t_0)$ rounds with $H$ local steps per round \; 
    \Return{the final global iterate of Local SGD} \;
    \BlankLine
    \BlankLine
\end{algorithm}

\newpage
\section{Modeling Local SGD with Multiple Conventional SDEs} \label{sec:multisde}

\citet{lin2020dont} 
tried to informally explain the success of Local SGD by adopting the argument
that larger diffusion term in the conventional SDE leads to better
generalization (see \Cref{subsec:con-sde,sec:related}). Basically, they attempted to write multiple SDEs, each of which
describes the $H$-step local training process of each worker in each round (from
$\vths_{k,0}$ to $\vths_{k,H}$). The key
difference between each of these SDEs and the SDE for SGD~\eqref{eq:canonical SDE} is that the
former one has a larger diffusion term because the workers use batch size
$\Bloc$ instead of $B$:
\begin{align}
    \dd\vX(t)=-\nabla \cL(\vX)\dd t+\sqrt{\frac{\eta}{\Bloc}}\mSig^{\sfrac{1}{2}}(\vX)\dd \vW_t. \label{eq:multi-local-sde}
\end{align}
\citet{lin2020dont} then argue that the total amount of ``noise'' in the
training dynamics of Local SGD is larger than that of SGD. However, it is hard
to see whether it is indeed larger, since
the model averaging step at the end of each round can reduce the variance in
training and may cancel the effect of having larger diffusion terms.

More formally, a complete modeling of Local SGD following this idea should view the
sequence of global iterates $\{\bvths\}$ as a Markov
process $\{\vXs\}$.
Let $\cPX(\vx, B, t)$ the distribution of $\vX(t)$ in~\eqref{eq:canonical SDE}
with initial condition $\vX(0)=\vx$.
Then the Markov transition 
should be 
$\vXs[s+1]=\frac{1}{K}\sum_{k=1}^K\vXs_{k,H}$,
where $\vXs_{1,H}, \dots, \vXs_{K,H}$
are $K$ independent samples from 
$\cPX(\vXs,\Bloc, H\eta)$, i.e., sampling from~\eqref{eq:multi-local-sde}.

Consider one round of model averaging. It is true that $\cPX(\vXs, \Bloc,
H\eta)$ may have a larger variance than the corresponding SGD baseline
$\cPX(\vXs, B, H\eta)$ because the former one has a smaller batch size. However,
it is unclear whether $\vXs[s+1]$ also has a larger variance than $\cPX(\vXs, B,
H\eta)$. This is because $\vXs[s+1]$ is the average of $K$ samples,
which means we have to compare $\frac{1}{K}$ times the variance of 
$\cPX(\vXs, \Bloc, H\eta)$ with the variance of 
$\cPX(\vXs, B, H\eta)$. Then it is unclear which one is larger.

In the special case where $H \eta$ is small,
$\cPX(\vXs, \Bloc, H\eta)$ is approximately equal to the following Gaussian distribution:
\begin{align}
    \mathcal{N}\left(\vXs - \eta H \nabla \cL(\vXs), \frac{\eta^2 H}{\Bloc} \mSig(\vXs)\right)
\end{align}
Then averaging over $K$ samples gives
\begin{align}
    \mathcal{N}\left(\vXs - \eta H \nabla \cL(\vXs), \frac{\eta^2 H}{B} \mSig(\vXs)\right),
\end{align}
which is exactly the same as the Gaussian approximation of the SGD baseline.
This means there do exist certain cases where \citet{lin2020dont}'s argument
does not give a good separation between Local SGD and SGD.

Moreover, we do not gain any further insights from this modeling
since it is hard to see how model averaging interacts with the SDEs.

\section{Additional Experimental Results
}\label{sec:regime
add exp} In this section, we present additional experimental results to further verify our finding.

\myparagraph{Supplementary Plot: Training time should be long enough.} \Cref{fig:add
effect-a,fig:add effect-b} show enlarged views for
\Cref{fig:effect-a,fig:effect-c} respectively, showing that Local SGD can
generalize worse than SGD in the first few epochs.

\myparagraph{Supplementary Plot: Learning rate should be small.}
\Cref{fig:add effect-c} shows that reducing the learning rate from $0.32$ to
$0.064$ does not lead to test accuracy drop for Local SGD on CIFAR-10, if the training time
is allowed to be longer and the number of local steps $H$ is set properly.
\Cref{fig:add effect-d} presents the case where, with a large learning rate, the
generalization improvement of Local SGD disappears even starting from a
pre-trained model.

\myparagraph{Supplementary Plot: Reconciling our main finding with \citet{ortiz2021trade}.}
In \Cref{fig:add effect-e},
the generalization benefit of Local SGD with $H=24$ becomes less significant
after the learning rate decay at epoch $226$, which is consistent with the
observation by \citet{ortiz2021trade} that the generalization benefit of Local
SGD usually disappears after the learning rate decay. But we can preserve the
improvement by increasing $H$ to $900$. Here, we use Local SGD with momentum.

\myparagraph{Supplementary Plot: Optimal $\alpha$ gets larger for smaller $\eta$.} In \Cref{fig:add effect-f}, we summarize the optimal $\alpha:=\eta H$ that enables the highest test accuracy for each learning rate in \Cref{fig:effect-f}. We can see that the optimal $\alpha$ increases as we decrease the learning rate. The reason is that the approximation error bound $\cO(\sqrt{\alpha \eta \log
\frac{\alpha}{\eta \delta}})$ in~\Cref{coro:alpha} decreases with $\eta$, allowing for a larger value of $\alpha$ to better regularize the model.

\begin{figure}[htbp]
\begin{center}
 \subfigure[CIFAR-10, start from random.]{
        \includegraphics[width=0.3\textwidth]{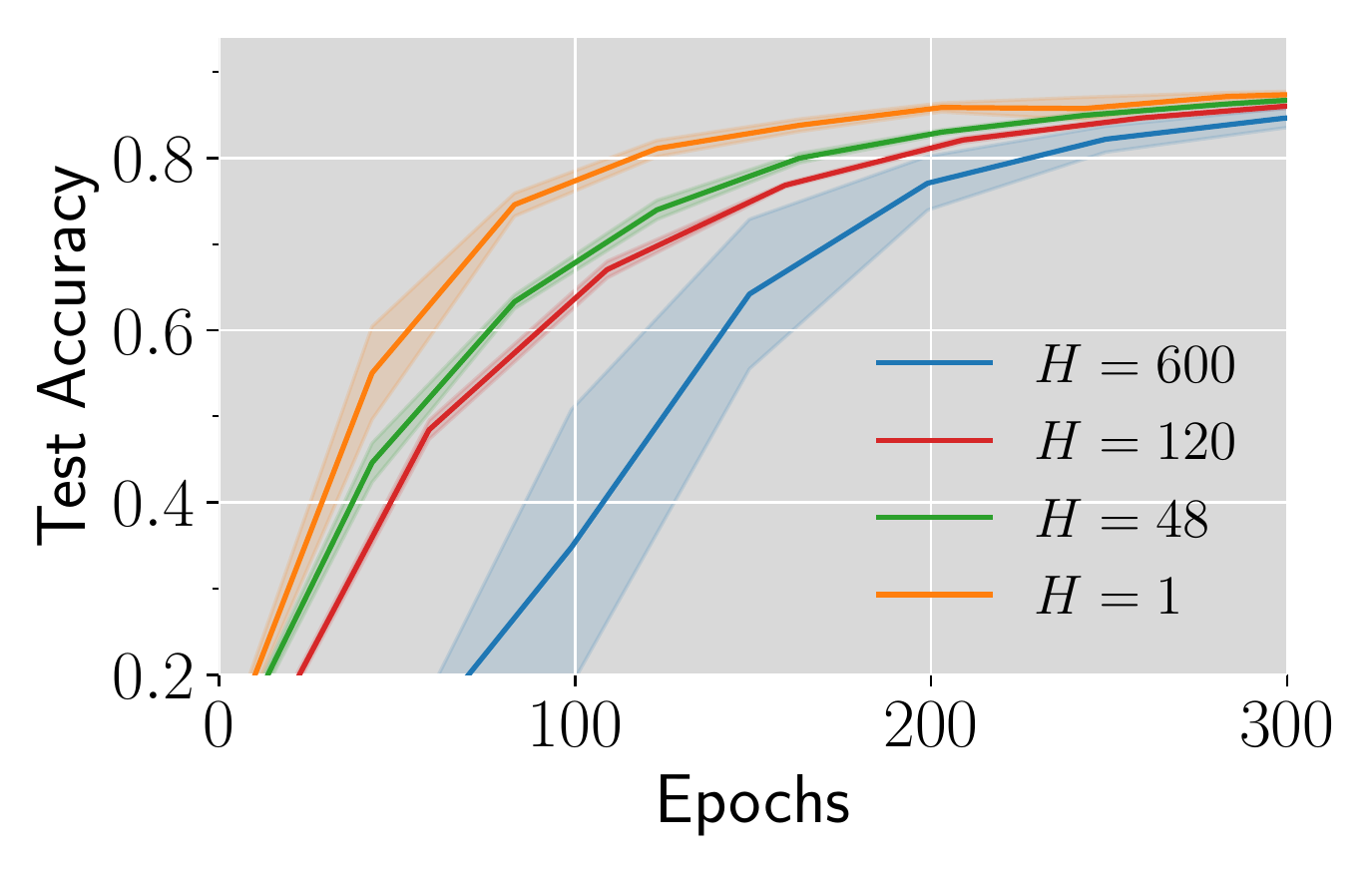}\label{fig:add effect-a}
    }
    \hspace{0.1in}
    \subfigure[ImageNet, start from \#$250$.]{
        \includegraphics[width=0.3\textwidth]{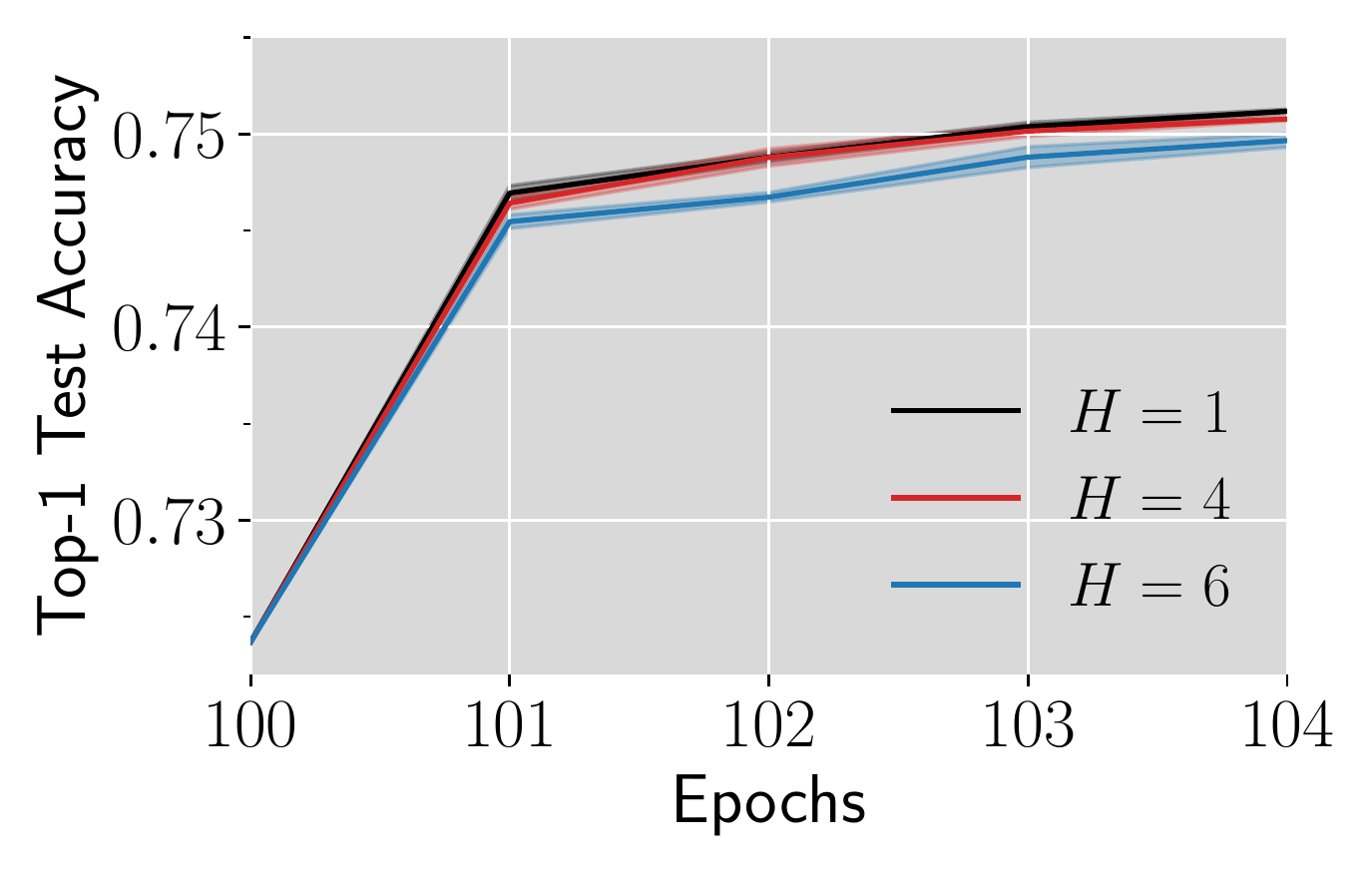}\label{fig:add effect-b}
        \vspace{-0.2in}
    }
    \hspace{0.1in}
    \subfigure[CIFAR-10, start from \#$100$.]{
        \includegraphics[width=0.3\textwidth]{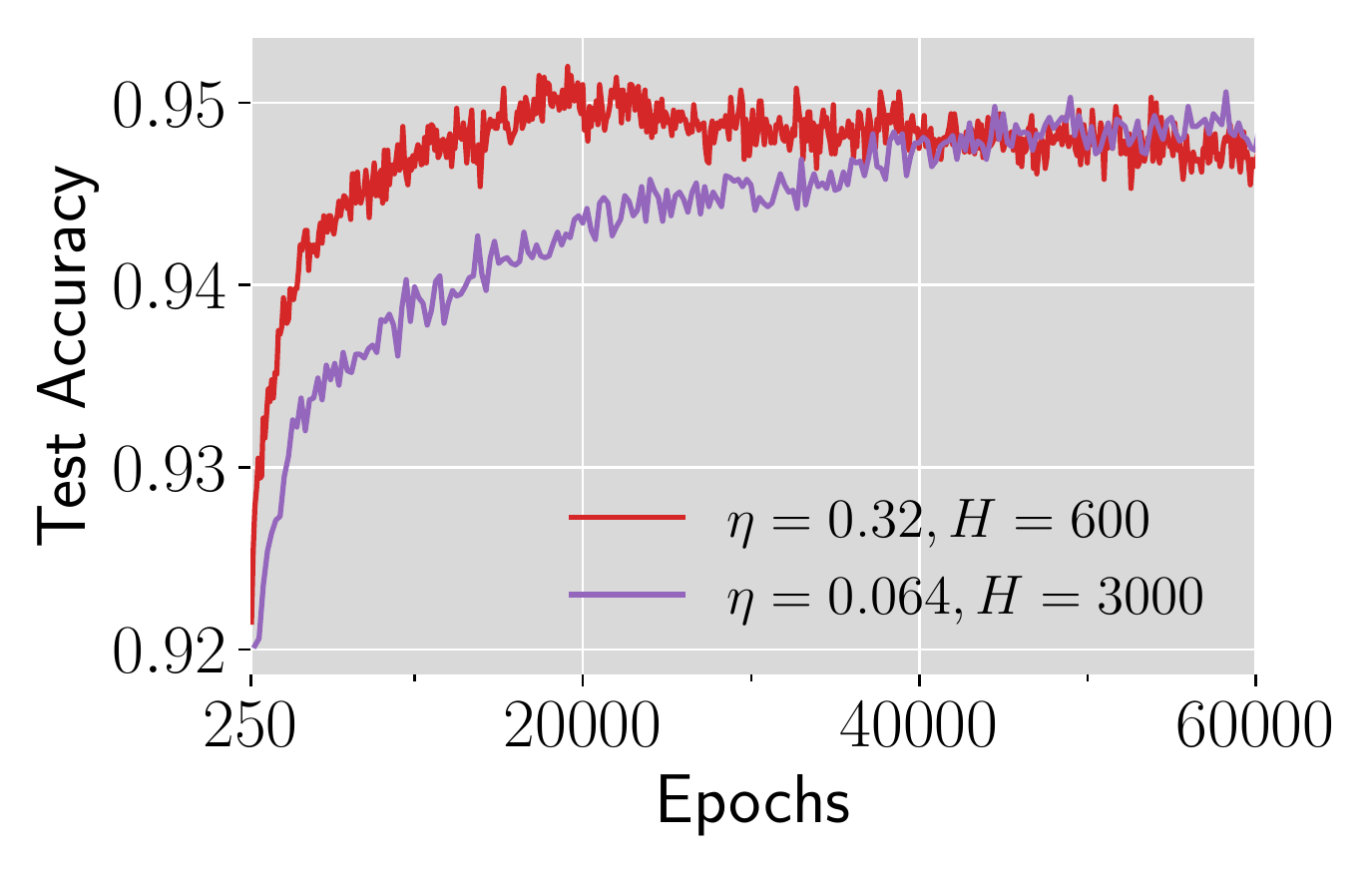}\label{fig:add effect-c}
        \vspace{-0.125in}
    }
    \subfigure[ImageNet, start from \#$100$.]{
        \includegraphics[width=0.3\textwidth]{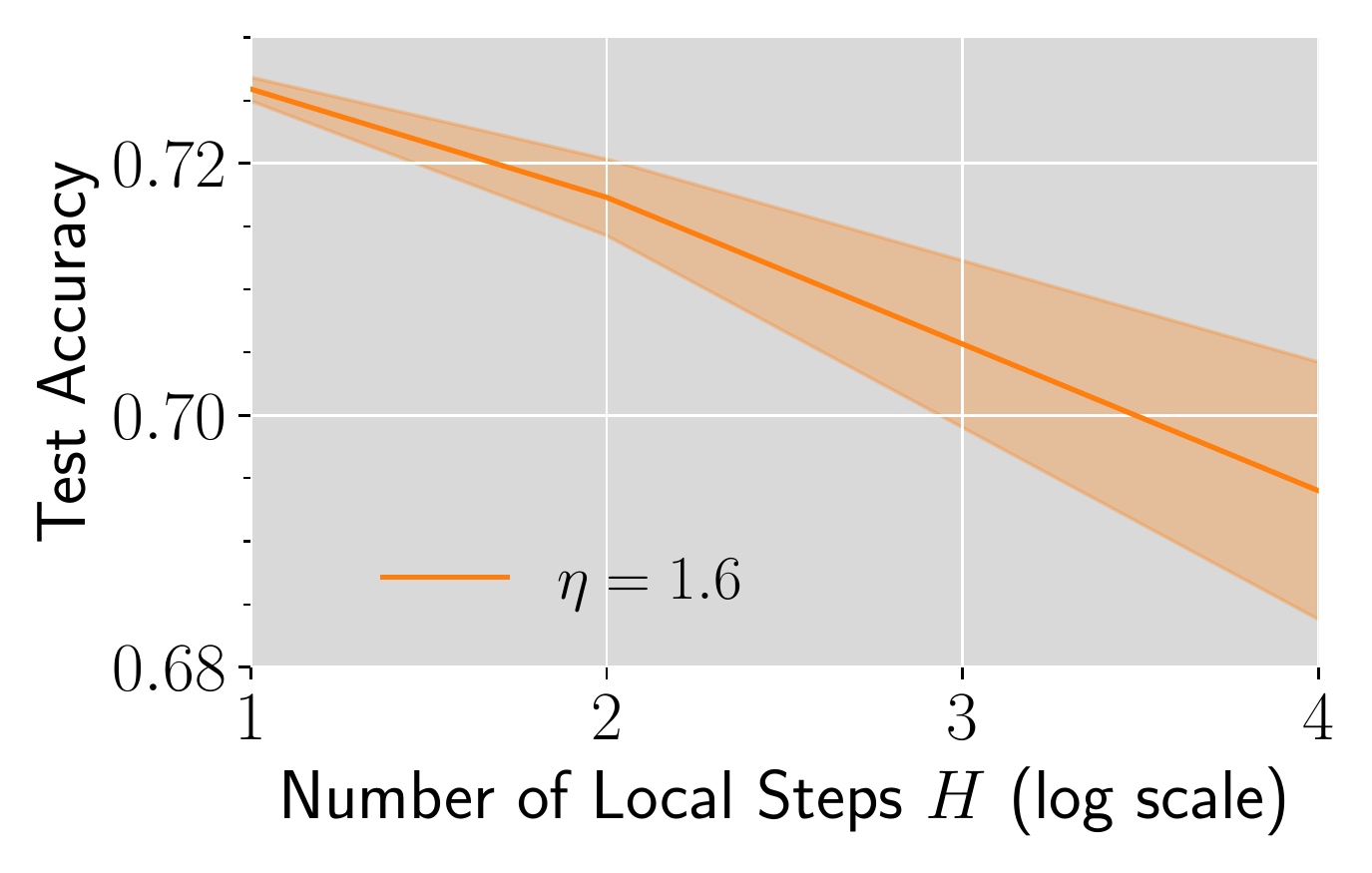}\label{fig:add effect-d}
    }
    \hspace{0.1in}
    \subfigure[CIFAR-10, start from \#$150$.]{
        \includegraphics[width=0.3\textwidth]{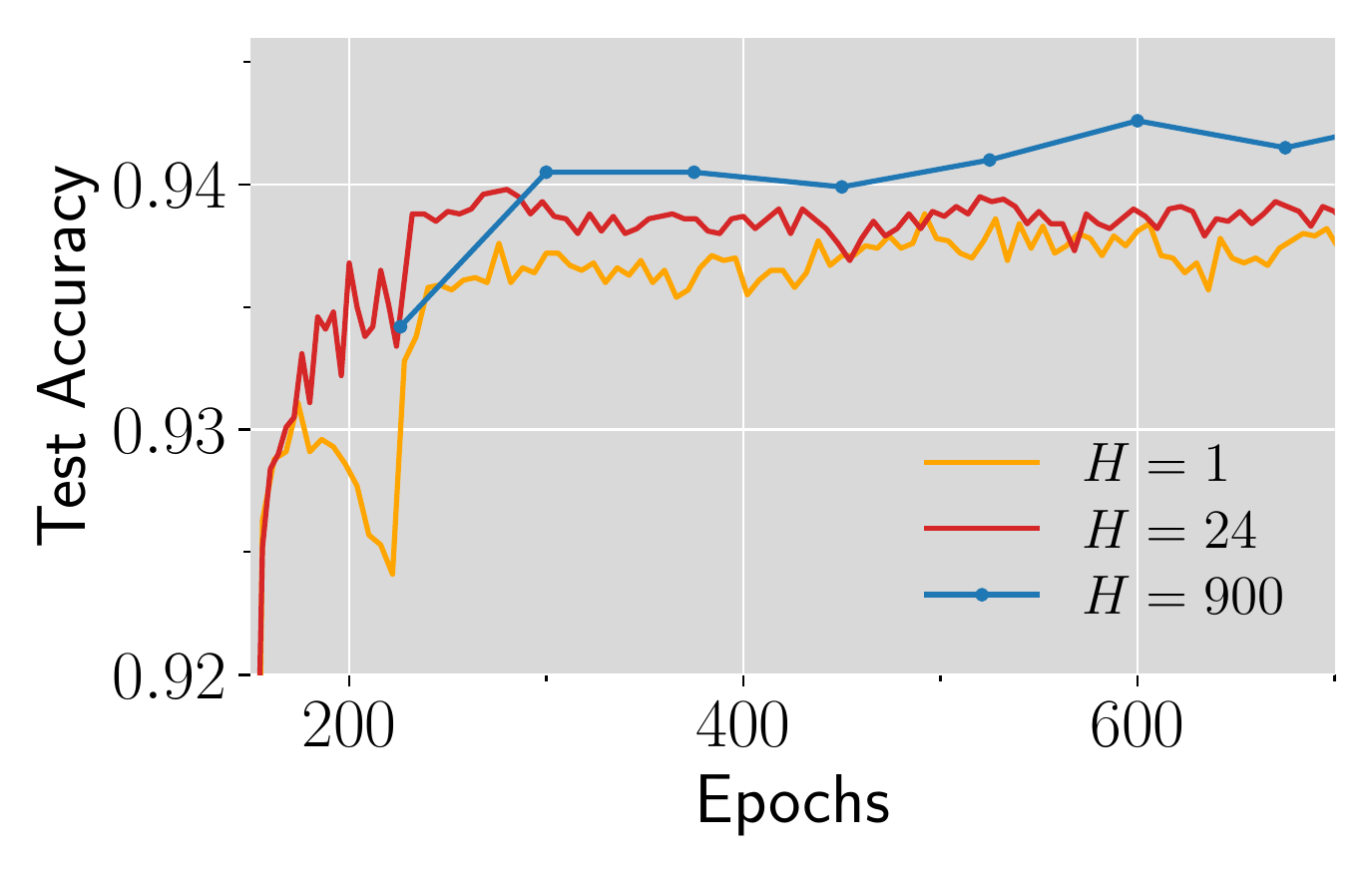}\label{fig:add effect-e}
        \vspace{-0.125in}
    }
    \hspace{0.1in}
        \subfigure[ImageNet, optimal $\alpha$ v.s. $\eta$.]{
        \includegraphics[width=0.3\textwidth]{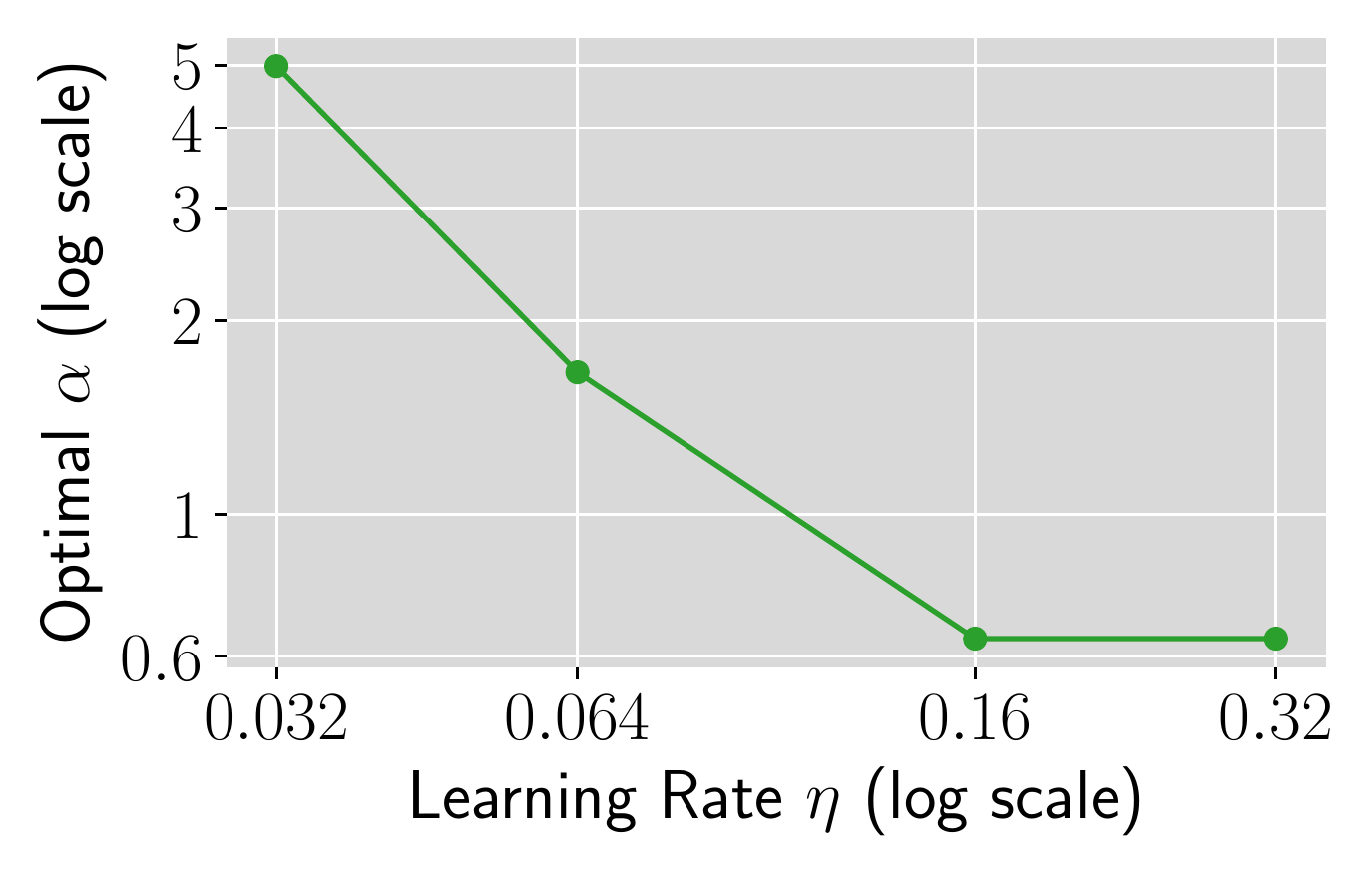}\label{fig:add effect-f}
        \vspace{-0.125in}
    }
    \caption{Additional experimental results about the effect of the learning rate, training time and the number of local steps. See \Cref{sec:detail regime} for details.}
        \label{fig:add effect}
     \vspace{-0.2in}
\end{center}
\end{figure}
\begin{figure}[htbp]
\begin{center}
 \subfigure[SGD with various $\eta$.]{
        \includegraphics[width=0.3\textwidth]{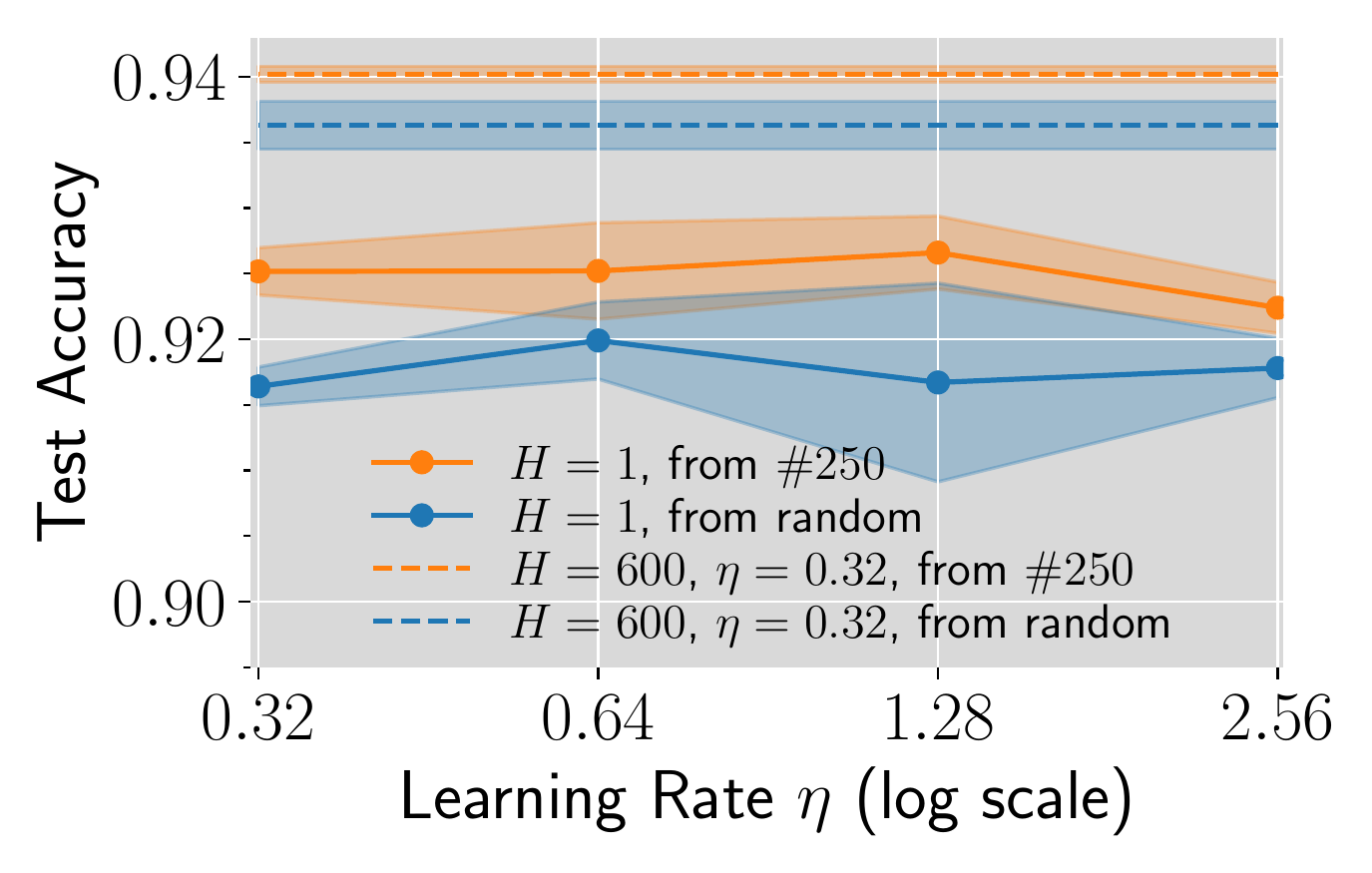}\label{fig:add exp-a}
    }
        \hspace{0.1in}
    \subfigure[SGD with larger batch sizes.]{
        \includegraphics[width=0.3\textwidth]{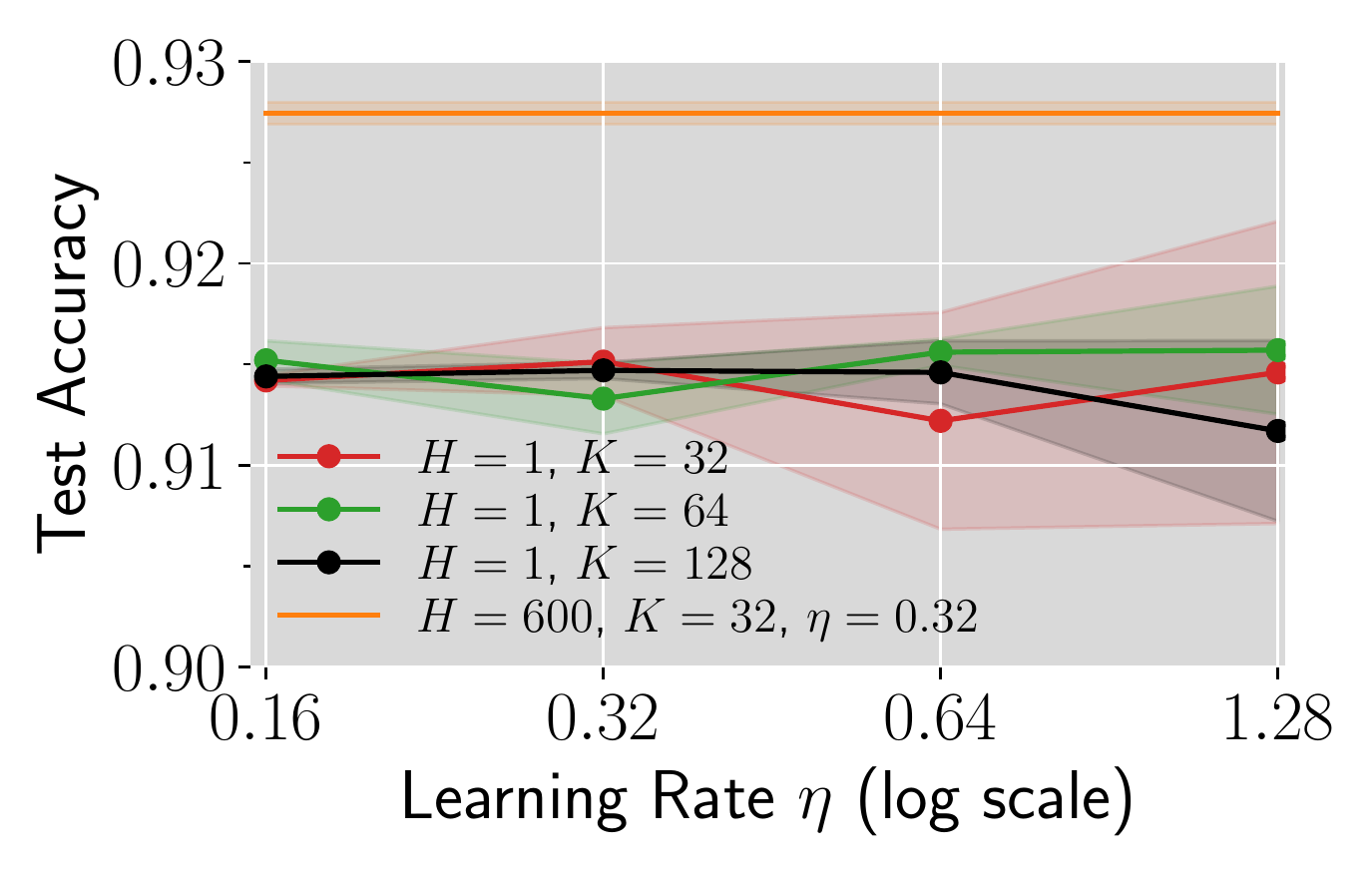}\label{fig:add exp-b}
        \vspace{-0.125in}
    }
    \hspace{0.1in}
    \subfigure[Post-local SGD, sampling with replacement.]{
        \includegraphics[width=0.3\textwidth]{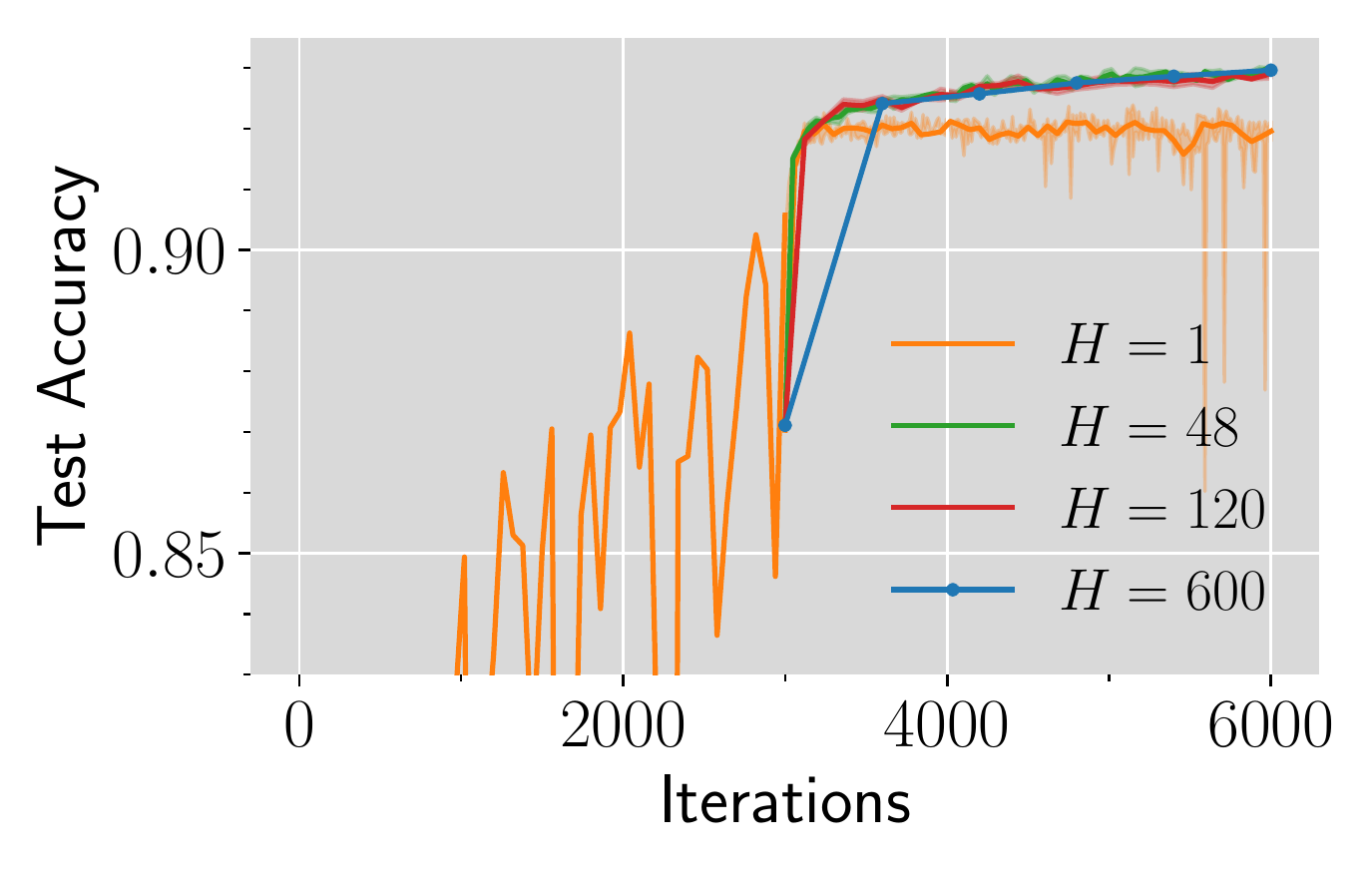}\label{fig:add exp-c}
        \vspace{-0.125in}
    }
    \vspace{-0.2in}
    \caption{Additional experimental results on CIFAR-10. See \Cref{sec:add exp detail} for details.}
     \vspace{-0.2in}
        \label{fig:add exp}
\end{center}
\end{figure}

\myparagraph{SGD generalizes worse even with extensively tuned learning rates.} In \Cref{fig:add
exp-a}, we run SGD from both random initialization and the pre-trained model for another $3,000$ epochs
with various learning rates and report the test accuracy. We can see that none
of the SGD runs beat Local SGD with the fixed learning rate $\eta=0.32$.
Therefore, the inferior performance of SGD in \Cref{fig:effect-a,fig:effect-b}
is not due to the improper learning rate and Local SGD indeed generalizes
better. 
\myparagraph{SGD with larger batch sizes performs no better.} In
\Cref{fig:add exp-b}, we enlarge the batch size of SGD and report the test
accuracy for various learning rates. We can see that SGD with larger batch sizes
performs no better and none of the SGD runs outperform Local SGD with the fixed
learning rate $\eta=0.32$. This result is unsurprising since it is well
established in the literature \citep{jastrzkebski2017three,
smith2020generalization, keskar2017on} that larger batch size typically leads to
worse generalization. See \Cref{sec:addrelated} for a  survey of empirical and
theoretical works on understanding and resolving this phenomenon.
\myparagraph{Sampling with or without replacement does not matter. } Note that there is a
slight discrepancy in sampling schemes between our theoretical and experimental
setup: the update rules \eqref{eq:upd intro sgd} and \eqref{eq:upd intro
localsgd} assume that data are sampled with replacement while most experiments
use sampling without replacement (\Cref{sec:pseudocode}). To eliminate the effect of this discrepancy,
we conduct additional experiments on Post-local SGD using sampling with
replacement (see \Cref{fig:add exp-c}) and Post-local SGD significantly outperforms SGD.

\section{Discussions on Local SGD with Label Noise Regularization}\label{sec:thm label noise}
\subsection{The Slow SDE for Local SGD with Label Noise Regularization}
In this subsection, we present the Slow SDE for Local SGD in the case of label noise regularization and show that Local SGD indeed induces a stronger regularization term, which presumably leads to better generalization.

\begin{theorem}[Slow SDE for Local SGD with label noise regularization]\label{thm:local sgd label noise}
For a $C$-class classification task with cross-entropy loss, the slow SDE of Local SGD with label noise has the following form:
\begin{align}
    \dd \vzeta(t) = -\frac{1}{4B} \gradGa \left(\tr(\nabla^2 \cL(\vzeta)) + (K-1) \cdot \frac{\tr (F(2H\eta \nabla^2\cL(\vzeta) ))}{2H\eta} \right)\dd t,
    \label{eq:flow-label-noise-localsgd}
    \end{align}
    where $F(x):=\int_0^x \psi(y)\dd y$ and is interpreted as a matrix function. Additionally, $\gradGa f$ stands for the gradient of a function $f$ projected to the tangent space of $\Gamma$.
\end{theorem}
\begin{proof}
    See \Cref{sec:app label noise}. 
\end{proof}

Note that the magnitude of the RHS in \eqref{eq:flow-label-noise-localsgd} becomes larger as $H$ increases. By letting $H$ to go to infinity, we further have the following theorem.
\begin{theorem}
 \label{thm:label-noise-K-times}
    As the number of local steps $H$ goes to infinity, the slow SDE of Local SGD with label noise \eqref{eq:flow-label-noise-localsgd}can be simplified as:
    \begin{align}
        \dd \vzeta(t) = -\frac{K}{4B} \gradGa \tr(\nabla^2 \cL(\vzeta)) \dd t.
        \label{eq:flow-label-noise-localsgd-H-large}
    \end{align}
\end{theorem}
\begin{proof}
    We obtain the corollary by simply taking the limit. By L'Hospital's rule,
\begin{align*}
    \lim_{x\to +\infty}\frac{F(ax)}{x}&=\lim_{x\to +\infty}\frac{\dd F(ax)}{\dd x}=\lim_{x\to +\infty}a\psi(ax)=a.
\end{align*}
Therefore, 
\begin{align}
    \lim _{x\to +\infty}\frac{\tr (F(2H\eta \nabla^2\cL(\vzeta) ))}{2H\eta}=\tr(\nabla^2\cL(\vzeta)). \label{eq:infty tr}
\end{align}
Substituting \eqref{eq:infty tr} into \eqref{eq:flow-label-noise-localsgd} yields \eqref{eq:flow-label-noise-localsgd-H-large}.
\end{proof}


As introduced in \Cref{subsec:inter}, the Slow SDE for SGD with label noise regularization has the following form: 
\begin{align}
    \dd \vzeta(t) = -\frac{1}{4B} \gradGa \tr(\nabla^2 \cL(\vzeta)) \dd t,
    \label{eq:flow-label-noise-std}
\end{align}
which is a deterministic flow that keeps reducing the trace of Hessian. As the trace of Hessian can be seen as a measure for the sharpness of the local loss landscape,
 \eqref{eq:flow-label-noise-std} 
indicates that SGD  with label noise regularization has an implicit bias toward flatter minima,
which presumably promotes generalization \citep{hochreiter1997flat,keskar2017on,neyshabur2017exploring}. From \Cref{thm:local sgd label noise,thm:label-noise-K-times}, we can conclude that Local SGD accelerates the process of sharpness reduction, thereby leading to better generalization. Furthermore, the regularization effect gets stronger for larger $H$ and is approximately $K$ times that of SGD.  We also conduct experiments on non-augmented CIFAR-10 with label noise regularization to verify our conclusion. As shown in \Cref{fig:label_noise}, increasing the number of local steps indeed gives better generalization performance.


\subsection{The Equivalence of Enlarging the Learning Rate and Adding Local Steps}
In this subsection, we explain in detail why training with label noise regularization is a special case where  enlarging the learning rate of SGD can bring the same generalization benefit as adding local steps. TWhen we scale up the learning rate of SGD $\eta \mapsto \kappa \eta$ (while keeping other hyperparameters unchanged), the corresponding Slow SDE is \eqref{eq:flow-label-noise-std} with time horizon $\kappa^2 T$ instead of $T$,
where SGD tracks a continuous interval of $\kappa^2\eta^2$  per step instead of $\eta^2$.
After rescaling the time horizon to $T$ so that SGD tracks a continuous interval of $\eta^2$ per step, we obtain
\begin{align}
      \dd \vzeta(t) = -\frac{\kappa^2}{4B} \gradGa \tr(\nabla^2 \cL(\vzeta)) \dd t.\label{eq:flow-label-noise-scale}
\end{align}
Let $\kappa=\sqrt{K}$ in \eqref{eq:flow-label-noise-scale}
 and we obtain the same Slow SDE as \eqref{eq:flow-label-noise-localsgd-H-large}, which is for Local SGD with a large number of local steps.  In \Cref{fig:K=4 label noise}, we conduct experiments to verify that SGD indeed achieves comparable test accuracy to that of Local SGD with a large $H$ if its learning rate is scaled up by $\sqrt{K}$ that of Local SGD.

\begin{figure}[t]
\begin{center}
 \subfigure[ResNet-56 + GroupNorm.]{
        \includegraphics[width=0.34\textwidth]{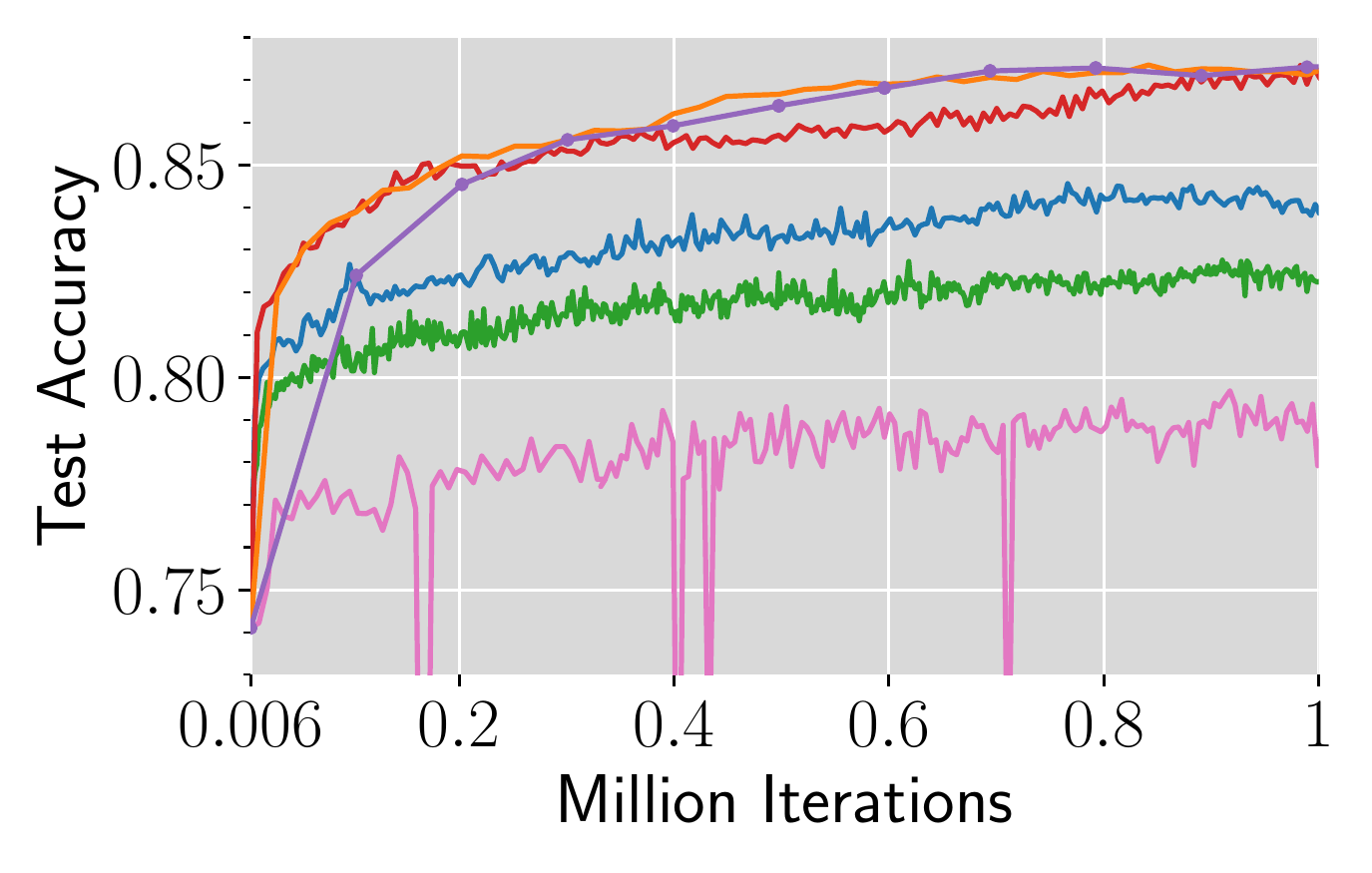}  \label{fig:label noise resnet}  }
    \hspace{0.1in}
         \subfigure[VGG-16 w/o normalization.]{
        \includegraphics[width=0.52\textwidth]{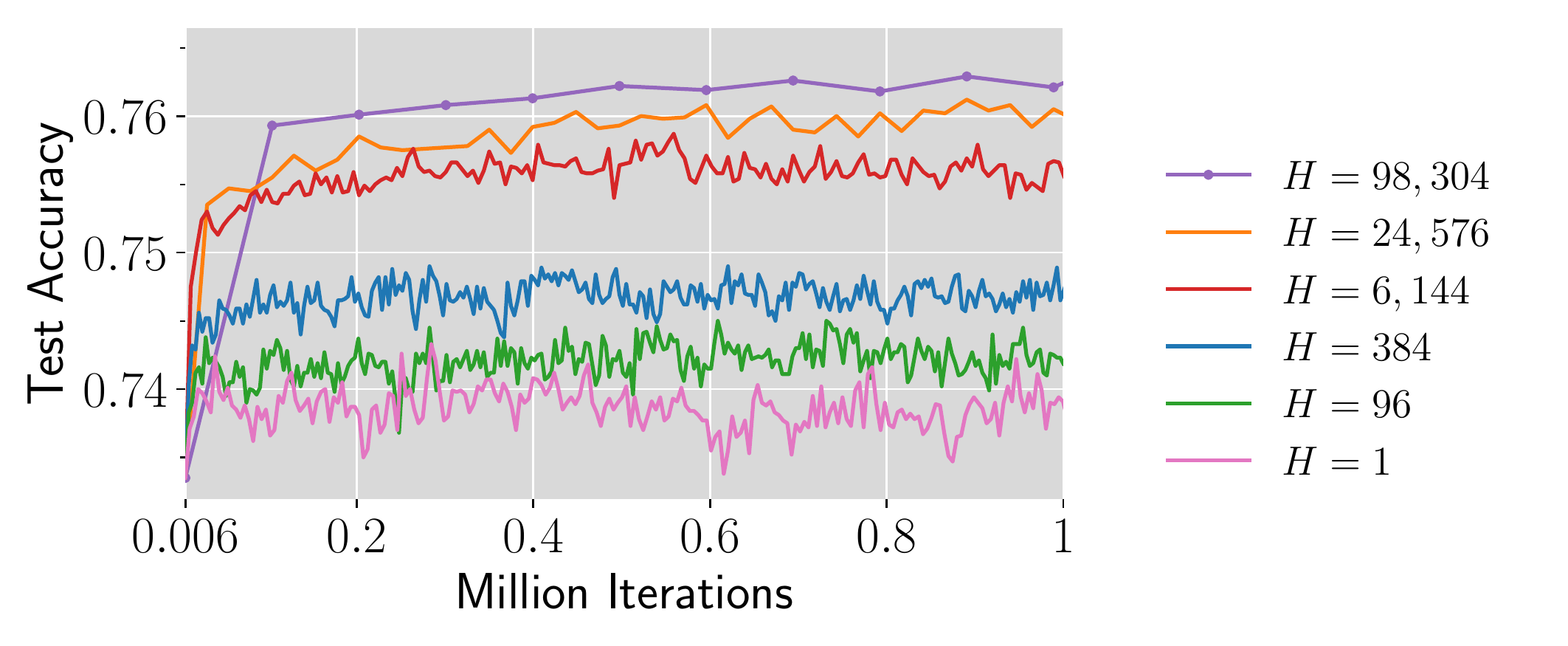}
        \label{fig:label noise vgg}
    }
    \vspace{-0.1in}
        \caption{Local SGD with label noise regularization on CIFAR-10 without data augmentation using $K=32$ ,$\Bloc=128$. A larger number of local steps indeed enables higher test accuracy. For both 
 architectures, we replace ReLU with Swish. See \Cref{sec:label noise detail} for training details.}
 \label{fig:label_noise}
\end{center}
\end{figure}
\begin{figure}[htbp]
\begin{center}

\includegraphics[width=0.42\textwidth]{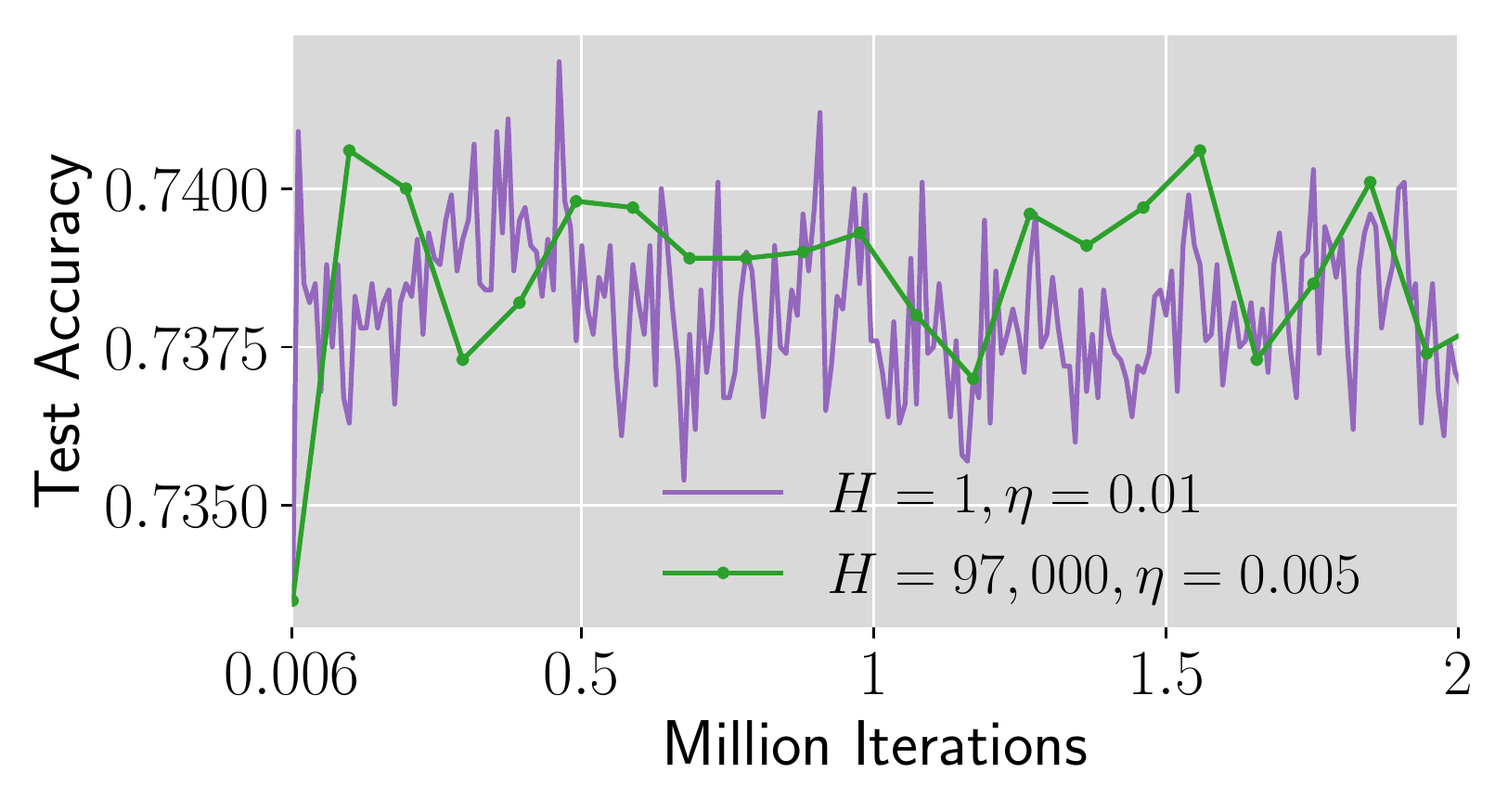}
   \vspace{-0.1in}
    \caption{Local SGD with label noise regularization on CIFAR-10 without data augmentation using $K=4$, $\Bloc=128$. SGD ($H=1$) indeed achieves comparable test accuracy as Local SGD with  a large $H$ when we scale up its learning rate to $\sqrt{K}$ times that of Local SGD. See \Cref{sec:label noise detail} for training details. }
    \label{fig:K=4 label noise}
\end{center}
\end{figure}

\section{Deriving the Slow SDE after Applying the LSR} \label{sec:change-K}

In this section, we derive the Slow SDEs for SGD and Local SGD after applying the LSR in
 \Cref{sec:effect of batch size}. The results are formally summarized in the following theorems.
 \begin{theorem}[Slow SDE for SGD after applying the LSR]\label{thm:lsr sgd} Let Assumptions~\ref{a:smooth} to
 \ref{a:compact} hold. Assume that we run SGD with learning rate $\eta'=\kappa \eta$ and the number of workers $K'=\kappa K$ for some constant $\kappa>0$.
 Let $T>0$ be a constant and $\vzeta(t)$ be the solution
 to \eqref{eq:sgd-zeta} with the initial condition $\vzeta(0)=\Phi(\vtheta_0)\in
 \Gamma$. Then for any $\cC^3$-smooth function $g(\vtheta)$,
 $\max_{0\leq s \leq \frac{\kappa T}{\eta'^2}}\left \lvert\E[g(\Phi(\vtheta_s)] - \E[g(\vzeta(s \eta'^2/\kappa)]\right\rvert=\ctO(\eta'^{0.25})$, where $\ctO(\cdot)$ hides log factors and constants
 that are independent of $\eta'$ but can depend on $g(\vtheta)$. 

 \end{theorem}

\begin{proof}
 Replacing $B$ with $\kappa B$ in the
original Slow SDE for Local SGD \eqref{eq:sgd-zeta} gives the following Slow SDE:
\begin{align}\label{eq:sgd-zeta-2}
    \dd \vzeta(t)&=P_{\vzeta}\Big(\underbrace{\tfrac{1}{\sqrt{\kappa B}} \mSig_{\parallel}^{\sfrac{1}{2}}(\vzeta)\dd \vW_t}_{\text{(a)\ diffusion}}
    \underbrace{-\tfrac{1}{2\kappa B} \nabla^3 \cL(\vzeta)[\widehat{\mSig}_{\Diamond}(\vzeta)] \dd t}_{\text{(b)\ drift-I}}
    \Big).
\end{align}
Note that the continuous time horizon for \eqref{eq:sgd-zeta-2} is $\kappa T$ instead of $T$ since after applying the LSR, SGD tracks a continuous interval of $\kappa^2 \eta^2$ 
per step
instead of $\eta^2$ while the total number of steps is scaled down by $\kappa$.
We can then rescale the time scaling to obtain \eqref{eq:sgd-zeta}
that holds for $T$.
\end{proof}

 \begin{theorem}[Slow SDE for Local SGD after applying the LSR]\label{thm:lsr localsgd} Let Assumptions~\ref{a:smooth} to
 \ref{a:compact} hold. Assume that we run Local SGD with learning rate $\eta'=\kappa \eta$, the number of workers $K'=\kappa K$, and the number of local steps $H'=\frac{\alpha}{\kappa\eta}$ for some constants $\alpha, \kappa>0$.
 Let $T>0$ be a constant and $\vzeta(t)$ be the solution
 to \eqref{eq:zeta-lsr2} with the initial condition $\vzeta(0)=\Phi(\bvths[0])\in
 \Gamma$. Then for any $\cC^3$-smooth function $g(\vtheta)$,
 $\max_{0\leq s \leq \frac{\kappa T}{H'\eta'^2}}\left \lvert\E[g(\Phi(\bvths)] - \E[g(\vzeta(s H'\eta'^2/\kappa)]\right\rvert=\ctO(\eta'^{0.25})$, where $\ctO(\cdot)$ hides log factors and constants
 that are independent of $\eta'$ but can depend on $g(\vtheta)$. 
 \begin{align}\label{eq:zeta-lsr2}
    \dd \vzeta(t)&=P_{\vzeta}\Big(\underbrace{\tfrac{1}{\sqrt{B}} \mSig_{\parallel}^{\sfrac{1}{2}}(\vzeta)\dd \vW_t}_{\text{(a)\ diffusion (unchanged)}}
    \underbrace{-\tfrac{1}{2B} \nabla^3 \cL(\vzeta)[\widehat{\mSig}_{\Diamond}(\vzeta)] \dd t}_{\text{(b)\ drift-I (unchanged)}}
    \underbrace{-\tfrac{\kappa K-1}{2B} \nabla^3 \cL(\vzeta)[\widehat{\mPsi}(\vzeta)] \dd t}_{\text{(c)\ drift-II (rescaled)}}
    \Big).
\end{align}

 \end{theorem}
 \begin{proof}
 Replacing $B$ with $\kappa B$ in the
original Slow SDE for Local SGD \eqref{eq:zeta} gives the following Slow SDE:
\begin{align}\label{eq:zeta-lsr-raw}
    \dd \vzeta(t)&=P_{\vzeta}\Big(\underbrace{\tfrac{1}{\sqrt{\kappa B}} \mSig_{\parallel}^{\sfrac{1}{2}}(\vzeta)\dd \vW_t}_{\text{(a)\ diffusion}}
    \underbrace{-\tfrac{1}{2 \kappa B} \nabla^3 \cL(\vzeta)[\widehat{\mSig}_{\Diamond}(\vzeta)] \dd t}_{\text{(b)\ drift-I}}
    \underbrace{-\tfrac{\kappa K-1}{2 \kappa B} \nabla^3 \cL(\vzeta)[\widehat{\mPsi}(\vzeta)] \dd t}_{\text{(c)\ drift-II}}
    \Big).
\end{align}
Note that the continuous time horizon for \eqref{eq:zeta-lsr-raw} is $\kappa T$ instead of $T$ since after applying the LSR, Local SGD tracks a continuous interval of $\kappa^2 \eta^2$ 
per step
instead of $\eta^2$ while the total number of steps is scaled down by $\kappa$.
We can then rescale the time scaling to obtain \eqref{eq:zeta-lsr2}
that holds for $T$.
\end{proof}


\newpage
\section{Proof of \Cref{thm:small thm}}\label{sec:small thm proof}
This section presents the proof for \Cref{thm:small thm}. First, we introduce some notations 
that will be used throughout this section. 
For the sequence of Local SGD iterates $\{\vths_{k, t}: k\in [K], 0\leq t \leq H,  s\geq 0\}$, 
we introduce an auxiliary sequence $\{\hvu_{t}\}_{t\in \N}$, which consists of GD iterates from $\bvths[0]$:
\[
    \hvu_0 = \bvths[0], \qquad \hvu_{t+1} \gets \hvu_t -\eta \nabla \cL(\hvu_t).
\] 
For convenience, let $\hvus_t:=\hvu_{sH+t}$ and $\vz_{k,sH+t}
:=\vzs_{k,t}$. We will use $\hvus_t$ and $\hvu_{sH+t}$, $\vzs_{k,t}$ and $\vz_{k,sH+t}$ interchangeably.
Recall that we have assumed that $\cL$ is $\cC^3$-smooth with bounded second and third order derivatives. 
Let $\nu_2:=\sup_{\vtheta \in \R^d}\normtwo{\nabla^2 \cL(\vtheta)}$ and $\nu_3:= \sup_{\vtheta \in \R^d}\normtwo{\nabla^3 \cL(\vtheta)}$.  Since $\nabla \ell(\vtheta;\vzeta)$ is bounded, the gradient noise $\vzs_{k,t}$ is also bounded. We denote by $\sigmax$ an upper bound such that $\normtwo{\vzs_{k,t}}\leq \sigmax$ holds for all $s, k, t$.

To prove \Cref{thm:small thm}, 
we will show that both Local SGD iterates $\bvths$ and SGD iterates $\vw_{sH}$ 
 track GD iterates $\hvu_{sH}$ closely with high probability. 
For each client $k$, define the following sequence $\{\hvZ_{k,t}:t\geq 0\}$, which will be used in the proof for bounding the overall effect of noise.
\begin{align*}
    \hvZ_{k,t}&=
    \sum_{\tau=0}^{t-1}\left[\prod_{l=\tau+1}^{t-1}(\mI - \eta \nabla^2 \cL(\hvu_{l}))\right]
    \vz_{k,\tau},
    \qquad \hvZ_{k,0}=\vzero, \qquad \forall k \in [K].
\end{align*}
The following lemma shows that $\hvZ_{k,t}$ is concentrated around the origin.
\begin{lemma}[Concentration property of $\{\hvZ_{k,t}\}$]\label{lemma: concen hvZ}
With probability at least $1-\delta$, the following holds simultaneously
for all $k\in[K]$, $0\leq t < \floor{\frac{T}{\eta}}$:
$$\normtwo{\hvZ_{k, t}} \leq \hC_1\sigmax\sqrt{\frac{2T}{\eta}
    \log \frac{2 T K}{\delta\eta}},$$
where $\hC_1: =\exp(T \nu_2) $.
\end{lemma}
\begin{proof}
For each $\hvZ_{k,t}$, construct a sequence $\{\hvZ_{k,t,t'}\}_{t'=0}^{t}$:
$$\hvZ_{k,t,t'}:=\sum_{\tau=0}^{t'-1}
\left(\prod_{l=\tau+1}^{t-1}(\mI-\eta \nabla^2 \cL(\hvu_{l}))\right)\vzs_{k, \tau},
\qquad \tvZs_{k, t,0}=\vzero.$$
Since $\normtwo{\nabla^2 \cL(\hvu_l)}\leq \nu_2$ for all $l\geq0$,
the following holds for all $0\leq\tau< t-1$ and $0<t<\floor{\frac{T}{\eta}}$:
\[
    \lnormtwo{\prod_{l=\tau+1}^{t-1}
    (\mI - \eta \nabla^2 \cL(\hvu_{l}))}
    \leq (1+\rho_2\eta)^t\leq\exp(T \nu_2)=\hC_1.
\]
So $\{\hvZ_{k, t,t'}\}_{t'=0}^t$ is a martingale with
 $\normtwo{\hvZ_{k,t, t'}-\hvZ_{k, t, t'-1}}\leq \hC_1 \sigma_{\max}$. 
Since $\hvZ_{k,t}=\hvZ_{k,t,t}$, by Azuma-Hoeffding's inequality, 
\begin{align*}
    \PP(\normtwo{\hvZ_{k,t}}\geq \epsilon') 
    \leq 2\exp{\left(\frac{-\epsilon'^2}{2t \left(\hC_1\sigma_{\max}\right)^2}\right)}.
\end{align*}
Taking union bound on all $k\in[K]$ and 
$0\leq t\leq \floor{\frac{T}{\eta}}$, we can conclude that
 with probability at least $1-\delta$, 
\begin{align*}
    \normtwo{\hvZ_{k, t}} \leq \hC_1\sigmax\sqrt{\frac{2T}{\eta}
    \log \frac{2 T K}{\delta\eta}}, \qquad \forall 0\leq t < \llfloor{\frac{T}{\eta}}, k\in[K].
\end{align*}
\end{proof}
The following lemma states that, with high probability, Local SGD iterates $\vths_{k,t}$ and $\bvths$ closely
track the gradient descent iterates $\hvu_{sH}$ for $\floor{\frac{T}{H\eta}}$
rounds. 
\begin{lemma}\label{lemma:small lemma local}
For $\delta=\cO(\poly(\eta))$,  the following inequalities hold with probability at least $1-\delta$:
\begin{align*}
    \normtwo{\vths_{k,t}-\hvu_{sH+t}}\leq \hC_3 \sqrt{\eta \log \frac{1}{\eta\delta}},\qquad \forall k\in[K], 0\leq s < \llfloor{\frac{T}{H\eta}}, 0\leq t \leq H,
\end{align*}
and
\begin{align*}
    \normtwo{\bvths-\hvu_{sH}}\leq \hC_3\sqrt{\eta \log \frac{1}{\eta \delta}}, \qquad \forall 0 \leq s \leq\llfloor{ \frac{T}{H\eta}},
\end{align*}
where $\hC_3$ is a constant independent of $\eta$ and $H$.
\end{lemma}
\begin{proof}
Let $\hvDeltas_{k,t}:=\vths_{k,t}-\hvus_t$ and
$\bvDeltas:=\bvths-\hvus_0$ be the differences between the Local SGD and GD iterates. 
According to the update rule for $\vths_{k, t}$ and $\hvus_t$, 
\begin{align}
    \vths_{k, t+1}&=\vths_{k, t}-\eta \nabla \cL(\vths_{k, t})-\eta \vzs_{k, t}\label{eq:theta upd}\\
    \hvus_{t+1}&=\hvus_t-\eta \nabla \cL(\hvus_t)\label{eq:u upd}.
\end{align}
Subtracting \eqref{eq:theta upd} by \eqref{eq:u upd} gives
\begin{align}
    \hvDeltas_{k, t+1}&=\hvDeltas_{k, t}-\eta (\nabla \cL(\vths_{k, t})-
    \nabla \cL(\hvus_t))-\eta \vzs_{k, t}\notag\\
    &=(\mI - \eta \nabla^2 \cL(\hvus_t))\hvDeltas_{k, t}
    -\eta \vzs_{k, t}+\eta \hvvs_{k, t}, \label{eq:expand2}
\end{align}
where $\hvvs_{k, t}$ is a remainder term with norm $\normtwo{\hvvs_{k, t}}\leq \frac{\nu_3}{2} \normtwo{\hvDeltas_{k, t}}^2$.
For the $s$-th round of Local SGD, we can apply \eqref{eq:expand2} $t$ times to obtain the following:
\begin{align}
\begin{aligned}\label{eq:recur hdelta}
  \hvDeltas_{k, t}
 &=\left[\prod_{\tau =0}^{t-1}(\mI - \eta \nabla^2 \cL(\hvus_{\tau}))\right]
 \hvDeltas_{k, 0}
 -\eta \underbrace{\sum_{\tau=0}^{t-1}
 \left[\prod_{l=\tau+1}^{t-1}(\mI-\eta \nabla^2 \cL(\hvus_{l}))\right]
 \vzs_{k, \tau}}_{\cT}\\
 & \quad 
 + \eta \sum_{\tau=0}^{t-1}\prod_{l=\tau+1}^{t-1}
 (\mI-\eta \nabla^2 \cL(\hvus_{l}))\hvvs_{k, \tau}.
\end{aligned}
\end{align}
Here, $\cT$ can be expressed in the following form:
\begin{align*}
    \cT=\hvZ_{k,sH+t}-\left[\prod_{l=sH}^{sH+t-1}
    (\mI - \eta \nabla^2 \cL(\hvu_l))\right]\hvZ_{k,sH}.
\end{align*}
    Substituting in $t=H$ and taking the average, we derive the following
recursion:
\begin{align}
\bvDeltas[s+1]&=\frac{1}{K}\sum_{k\in [K]}\hvDeltas_{k, H}\notag\\
&=\left[\prod_{\tau =0}^{H-1}(\mI 
- \eta \nabla^2 \cL(\hvus_{\tau}))\right]
\bvDeltas \notag \\
& \quad -\frac{\eta}{K}\sum_{k\in[K]} 
\hvZ_{k,(s+1)H}+\frac{\eta}{K}\sum_{k\in[K]} \left[\prod_{l=sH}^{(s+1)H-1} (\mI - \eta \nabla^2 \cL(\hvu_l))\right]\hvZ_{k,sH}
\notag\\
& \quad 
+ \frac{\eta}K\sum_{k\in[K]} \sum_{\tau=0}^{H-1}\prod_{l=\tau+1}^{H-1}
(\mI-\eta \nabla^2 \cL(\hvus_{l}))\hvvs_{k, \tau}.\label{eq:recur delta}
\end{align}
Applying~\eqref{eq:recur delta} $s$ times yields
\begin{align}
    \bvDeltas&=-\frac{\eta}{K}\sum_{k\in[K]}
   \hvZ_{k,sH}+\frac{\eta}{K}\sum_{r=0}^{s-1} \sum_{\tau=0}^{H-1}\sum_{k\in[K]} 
   \left[\prod_{l=rH+\tau+1}^{sH}
    (\mI-\eta \nabla^2 \cL(\hvu_{l}))\right]\hvvs[r]_{k, \tau}\label{eq:deltas}.
\end{align}
Substitute \eqref{eq:deltas} into \eqref{eq:recur hdelta}
and we have
\begin{align*}
    \hvDeltas_{k,t}&=-\frac{\eta}{K}\sum_{k'\in[K]}
    \hvZ_{k',sH}-\eta \hvZ_{k,sH+t}+\eta\left[\prod_{l=sH}^{sH+t-1}(\mI-
    \eta \nabla^2\cL(\hvu_l))\right]\hvZ_{k,sH}\\
    &\quad+\frac{\eta}{K}\sum_{r=0}^{s-1} \sum_{\tau=0}^{H-1}\sum_{k'\in[K]}
     \left[\prod_{l=rH+\tau+1}^{sH+t-1}
     (\mI-\eta \nabla^2 \cL(\hvu_{l}))\right]\hvvs[r]_{k', \tau}\\
     &\quad + \eta \sum_{\tau=0}^{t-1}
     \left[\prod_{l=sH+\tau+1}^{sH+t-1}(\mI-\eta \nabla^2 \cL(\hvu_l))\right]
    \hvvs_{k,\tau}.
    \end{align*}
By Cauchy-Schwartz inequality and  triangle inequality,  we have
\begin{align}
    \begin{aligned}
    \normtwo{\hvDeltas_{k,t}}&\leq \frac{\eta}{K}
    \left(\sum_{k'\in [K]}\normtwo{\hvZ_{k',sH}}\right)+\eta \normtwo{\hvZ_{k,sH+t}}
    +\eta \hC_1\normtwo{\hvZ_{k, sH}}\\
    &\quad + \frac{\eta \hC_1\nu_3}{2 K}\sum_{r=0}^{s-1}\sum_{\tau=0}^{H-1}\sum_{k'\in [K]}\normtwo{{\hvDeltas[r]} _{k', \tau}}^2
    +\frac{\eta\hC_1\nu_3}{2}\sum_{\tau=0}^{t-1}\normtwo{{\hvDeltas[r]}_{k, \tau}}^2,
    \end{aligned}\label{eq:hvdelta induction}
\end{align}
where $\hC_1=\exp(\nu_2T)$.

Below we prove by induction that for $\delta=\cO(\poly(\eta))$, if 
\begin{align}
    \normtwo{\hvZ_{k,t}}\leq \hC_1
 \sigmax\sqrt{\frac{2T}{\eta}\log \frac{2TK}{\eta\delta}}, \quad \forall 0\leq t < \llfloor{\frac{T}{\eta}}, k\in[K],\label{eq:induc condi}
\end{align}
then there exists a constant $\hC_2$
 such that for all $k\in[K], 0\leq s < \floor{\frac{T}{\eta H}}$ and $  0\leq t \leq H$, 
 \begin{align}
    \normtwo{\hvDeltas_{k,t}}\leq \hC_2 \sqrt{\eta \log \frac{2TK}{\eta \delta}}. \label{eq: prop hdelta}
 \end{align}
 First, for all $k\in [K]$, $\normtwo{{\hvDeltas[0]}_{k,0}}=0$ and hence~\eqref{eq: prop hdelta} 
 holds. Assuming that \eqref{eq: prop hdelta} holds for all ${\hvDeltas[r]}_{k',\tau}$ where $k'\in[K], 0\leq r < s$, $0\leq \tau \leq H$
 and $r=s$, $0\leq \tau<t$, then by~\eqref{eq:hvdelta induction}, for all $k\in [K]$, the following holds:
 \begin{align*}
    \normtwo{\hvDeltas_{k,t}}\leq 3\hC_1^2\sigmax \sqrt{2T\eta \log \frac{2TK}{\eta\delta}}+
    \hC_1\hC_2^2T\eta\nu_3 \log \frac{2TK}{\eta\delta}.
 \end{align*}
 Let $\hC_2 \geq 6\hC_1^2\sigmax \sqrt{2T}$. Then for sufficiently 
 small $\eta$, \eqref{eq: prop hdelta} holds. By \Cref{lemma: concen hvZ}, 
 \eqref{eq:induc condi} holds with probability at least $1-\delta$. Furthermore,
 notice that $\bvths-\hvu_{sH}=\frac{1}{K}\sum_{k\in [K]}{\hvDeltas[s-1]}_{k,H}$. 
 Hence we have the lemma.
\end{proof}
The iterates of standard SGD can be viewed as the local iterates on
a single client with the number of local steps $\floor{\frac{T}{\eta}}$.
Therefore, we can directly apply
 \Cref{lemma:small lemma local} and obtain the
following lemma about the SGD iterates $\vw_t$.
\begin{corollary}\label{coro: small lemma SGD}
    For $\delta=\cO(\poly(\eta))$, the following holds with probability at least $1-\delta$:
    \begin{align*}
        \normtwo{\vw_{sH}-\hvu_{sH}}\leq \hC_3\sqrt{\eta\log\frac{1}{\eta\delta}}, \qquad \forall 0 \leq s \leq \frac{T}{H\eta},
    \end{align*}
    where $\hC_3$ is the same constant as in \Cref{lemma:small lemma local}.
\end{corollary}
Applying \Cref{lemma:small lemma local} and \Cref{coro: small lemma SGD} and taking the union bound, we have \Cref{thm:small thm}.

\section{Proof Outline of Main Theorems}\label{sec: proof outline}


 We adopt the general framework proposed by \citet{li2019stochastic} to bound the closeness of discrete algorithms and SDE solutions via the method of moments. However, their framework is not directly applicable to our case since they
 provide approximation guarantees for discrete algorithms with learning rate $\eta$ for $\cO(\eta^{-1})$ steps while we want to capture Local SGD for $\cO(\eta^{-2})$ steps. To overcome this difficulty, we treat $\Rg:=\floor{\frac{1}{\alpha\eta^{\beta}}}$ rounds as a ``giant step'' of Local SGD with an ``effective'' learning rate $\eta^{1-\beta}$, where $\beta$ is a constant in $(0, 1)$, and derive the recursive formulas to compute the moments for the change in every step, every round, and every $\Rg$ rounds. The formulation of the recursions requires a detailed analysis of the limiting dynamics of the iterate and careful control of approximation errors. 

 The dynamics of the iterate can be divided into two phases: the approaching phase (Phase 1) and the drift phase (Phase 2). The approaching phase only lasts for $\cO(\log\frac{1}{\eta})$ rounds, during which the iterate is quickly driven to the minimizer manifold by the negative gradient and ends up within only $\ctO(\sqrt{\eta})$ from $\Gamma$ (see \Cref{subsec: phase 1}). After that, the iterate enters the drifting phase and moves in the tangent space of $\Gamma$ while staying close  to $\Gamma$ (see \Cref{subsec:phase2}). The closeness of the iterates (local and global) and $\Gamma$ is summarized in the following theorem. 
 \begin{theorem}[Closeness of the iterates and $\Gamma$]\label{thm:closeness}
 For $\delta=\cO(\poly(\eta))$, with probability at least $1-\delta$, for all $\cO(\log\frac{1}{\eta})\leq s\leq\floor{T/(H\eta^2)} $,
 \begin{align*}
    \Phi(\bvths)\in \Gamma, \qquad \normtwo{\bvths-\Phi(\bvths)}=\cO\left(\sqrt{\eta \log \frac{1}{\eta\delta}}\right).
 \end{align*}
 Also, for all $\cO(\log\frac{1}{\eta})\leq s<\floor{T/(H\eta^2)} $, $k \in [K]$ and $0 \leq t \leq H$, 
 \begin{align*}
     \quad \normtwo{\vths_{k,t}-\Phi(\bvths)}=\cO\left(\sqrt{\eta \log \frac{1}{\eta\delta}}\right).
 \end{align*}
 Here, $\cO(\cdot)$ hides constants independent of $\eta$ and $\delta$.
 \end{theorem}
To control the approximation errors, we also provide a high probability bound for the change of  the manifold projection within $\Rg$ rounds.
\begin{theorem}[High probability bound for the change of manifold projection]\label{thm:change}
 For $\delta=\cO(\poly(\eta))$, with probability at least $1-\delta$, for all $0\leq s\leq\floor{T/(H\eta^2)} - \Rg $ and $0 \leq r\leq \Rg$, 
 \begin{align*}
    \Phi(\bvths),\Phi(\bvths[s+r])\in \Gamma, \qquad \normtwo{\Phi(\bvths[s+r]) - \Phi(\bvths)} = \cO\left(\eta^{0.5 - 0.5\beta}\sqrt{\log \frac{1}{\eta \delta}}\right), 
 \end{align*}
 where $\cO(\cdot)$ hides constants independent of $\eta$ and $\delta$.
 \end{theorem}
 The proof of \Cref{thm:closeness,thm:change} is based on the analysis of the dynamics of the iterate and presented in \Cref{sec:summary of high prob}. 
 
 Utilizing \Cref{thm:closeness,thm:change}, we move on to estimate the first and second moments of the change of the manifold projection every $\Rg$ rounds. However, the randomness during training might drive the iterate far from the manifold (with a low probability, though), making the dynamics intractable. To tackle this issue, we construct a well-behaved auxiliary sequence $\{\hvths_{k, t}\}$, which is constrained to the neighborhood of $\Gamma$ and equals the original sequence $\{\vths_{k, t}\}$ with high probability (see \Cref{def:htheta}). Then we can formulate recursions for the change of manifold projection of the auxiliary sequence  using the nice properties near $\Gamma$. The estimate of moments  is summarized in \Cref{thm: one step moment new}. 

 Finally, based on the moment estimates, we apply the framework in \cite{li2019stochastic} to show that the manifold projection and the SDE solution are weak approximations of each other in \Cref{sec:weak approx}.

\section{Proof Details of Main Theorems}\label{sec: proof details}
The detailed proof is organized as follows. In \Cref{sec:add notation}, we introduce the notations that will be used throughout the proof. To establish preliminary knowledge, 
\Cref{sec:app derivative} provides explicit expression for the projection operator $\Phi(\cdot)$, and \Cref{subsec: preliminary for GD} presents lemmas about gradient descent (GD) and gradient flow (GF). Based on the preliminary knowledge, we construct a nested working zone to characterize the closeness of the iterate and $\Gamma$ in \Cref{subsec: working zone}. Appendices~\ref{subsec: phase 1} to \ref{sec:weak approx} make up the main body of the proof. Specifically, Appendices~\ref{subsec: phase 1} and~\ref{subsec:phase2} analyze the dynamics of Local SGD iterates for phases 1 and 2, respectively. Utilizing these analyses, we provide the proof of \Cref{thm:closeness,thm:change} in \Cref{sec:summary of high prob} and the proof of \Cref{coro:alpha} in \Cref{sec:alpha proof}. Then we derive the estimation for the first and second moments of one ``giant step  '' $\Phi(\bvths[s+\Rg])-\Phi(\bvths)$ in \Cref{sec:moments phase2}. Finally, we prove the approximation theorem \ref{main thm: flow} in \Cref{sec:weak approx}. 
\subsection{Additional Notations}\label{sec:add notation}
Let $\Rtot:=\floor{\frac{T}{H\eta^2}}$ be the total number of rounds. Denote by $\vphs$ the manifold projection of the global iterate at the beginning of round $s$. Let $\vxs_{k,t}:=\vths_{k,t}-\vphs$ be the difference between the local iterate and the manifold projection of the global iterate. Also define $\bvxs_{H}:=\frac{1}{K}\sum_{k\in [K]}\vxs_{k,H}$ and $\bvxs_{0}:=\frac{1}{K}\sum_{k\in [K]}\vxs_{k,0}$ which is the average of $\vxs_{k,t}$ among $K$ workers at step $0$ and $H$. Then for all $k\in[K]$, $\vxs_{k,0}=\bvxs_{0}=\bvths - \vphs$. Finally, Since $\nabla \ell(\vtheta;\vzeta)$ is bounded, the gradient noise $\vzs_{k,t}$ is also bounded and we denote by $\sigmax$ the upper bound such that $\normtwo{\vzs_{k,t}}\leq \sigmax, \forall s, k, t$.

We first introduce the notion of $\mu$-PL. We will later show that there exists a neighborhood of the minimizer manifold $\Gamma$ where $\cL$ satisfies $\mu$-PL.
\begin{definition}[Polyak-Łojasiewicz Condition]
For $\mu>0$, we say a function $\cL(\cdot)$ satisfies $\mu$-Polyak-Łojasiewicz condition (abbreviated as $\mu$-PL) on set $U$ if 
\[
  \frac{1}{2}\normtwo{\nabla \cL(\vtheta)}^2 \geq \mu( \cL(\vtheta)-\inf_{\vtheta'\in U}\cL(\vtheta')).
\]
\end{definition}
We then introduce the definitions of the $\epsilon$-ball at a point and the $\epsilon$-neighborhood of a set.
For $\vtheta \in \R^d$ and $\epsilon>0$,  $B^{\epsilon}(\vtheta):=\{\vtheta' : \normtwo{\vtheta'-\vtheta}< \epsilon \}$ is the open $\epsilon$-ball centered at $\vtheta$. For a set $\cZ\subseteq \R^d$, $\cZ^{\epsilon}:=\bigcup_{\vtheta \in \cZ}B^{\epsilon}(\vtheta)$ is the $\epsilon$-neighborhood of $\cZ$.







\subsection{Computing the Derivatives of the Limiting Mapping}\label{sec:app derivative}
In subsection, we present lemmas that relate the derivatives
of the limiting mapping $\Phi(\cdot)$ to the derivatives of the loss function $\cL(\cdot)$.
We first introduce the operator $\cV_{\mH}$.
\begin{definition}
    For a semi-definite symmetric matrix $\mH\in\R^{d\times d}$, let $\lambda_j$, $\vv_j$ be the $j$-th
    eigenvalue and eigenvector and $\vv_j$'s form an orthonormal basis
    of $\R^d$. 
    Then, define the operator $\cV_{\mH}:\R^{d\times d}\to \R^{d\times d}$
    as 
    \[\cV_{\mH}(\mM):=\sum_{i, j: \lambda_i \ne 0 \lor \lambda_j \ne 0}\frac{1}{\lambda_i+\lambda_j}\inner{\mM}{\vv_i \vv_j^{\top}}\vv_i \vv_j^{\top}, \forall \mM \in \R^{d\times d}.\]
Intuitively, this operator projects $\mM$ to the base matrix $\vv_i\vv_j^{\top}$
and sums up the projections with weights $\frac{1}{\lambda_i+\lambda_j}$.
\end{definition}
Additionally, for $\vtheta \in \Gamma$, denote by $T_{\vtheta}$ and $T_{\vtheta}^{\perp}$ the tangent and normal space of $\Gamma$ at $\vtheta$ respectively.
Lemmas~\ref{lemma: mani grad} to~\ref{lemma:mani inner perp perp} are from \citet{li2021happens}. We include them to make the paper self-contained.
\begin{lemma}[Lemma C.1 of \citet{li2021happens}]\label{lemma: mani grad}
For any $\vtheta \in \Gamma$ and any $\vv \in T_{\vtheta}(\Gamma)$, it
holds that $\nabla^2 \cL(\vtheta)\vv=\vzero$.
\end{lemma}
\begin{lemma}[Lemma 4.3 of \citet{li2021happens}]\label{lemma:mani projection}
For any $\vtheta\in \Gamma$, $\partial \Phi(\vtheta)\in \R^{d\times d}$
is the projection matrix onto the tangent space $T_{\vtheta}(\Gamma)$.
\end{lemma}
\begin{lemma}[Lemma C.4 of~\cite{li2021happens}]\label{lemma:mani inner para perp}
For any $\vtheta \in \Gamma$, $\vu \in \R^d$ and $\vv \in T_{\vtheta}(\Gamma)$,
it holds that 
\[\partial^2 \Phi(\vtheta)[\vv, \vu]=-\partial \Phi(\vtheta)
\nabla^3\cL(\vtheta)[\vv, \nabla^2 {\cL(\vtheta)}^+\vu]
-\nabla^2{ \cL(\vtheta)}^+\nabla^3\cL(\vtheta)[\vv, \partial \Phi(\vtheta)\vu].\]
\end{lemma}
\begin{lemma}[Lemma C.6 of~\cite{li2021happens}]\label{lemma:mani inner perp perp}
For any $\vtheta \in \Gamma$ and $\mSig \in \spann\{\vu \vu^{\top}\mid \vu \in T_{\vtheta}^{\perp}(\Gamma)\}$,
\begin{align*}
    \inner{\partial^2 \Phi(\vtheta)}{\mSig} =-\partial \Phi(\vtheta)\nabla^3 \cL(\vtheta)[\cV_{\nabla^2 \cL(\vtheta)}(\mSig)].
\end{align*}
\end{lemma}
\begin{lemma}\label{lemma:mani}
For all $\vtheta\in \Gamma$, $\vu, \vv \in T_{\vtheta}(\Gamma)$, it holds that
\begin{align}\label{eq:manifold}
    \partial\Phi(\vtheta)\nabla^3\cL[\vv\vu^{\top}]=\vzero.
\end{align}
\end{lemma}
\begin{proof}
This proof is inspired by Lemma C.4 of \cite{li2021happens}. For any $\vtheta \in \Gamma$, consider a parameterized smooth curve $\vv(t), t\geq 0$ on $\Gamma$ such that $\vv(0)=\vtheta$ and $\vv'(0)=\vv$. Let $\mPpara(t)=\partial\Phi(\vv(t))$,   $\mPperp(t)=\mI-\partial\Phi(\vv(t))$ and 
$\mH(t)=\nabla^2\cL(\vv(t))$. By Lemma C.1 and 4.3 in \cite{li2021happens}, 
\begin{align*}
    \mH(t)=\mPperp(t)\mH(t).
\end{align*}
Take the derivative with respect to $t$ on both sides, 
\begin{align*}
    \mH'(t)&=\mPperp(t)\mH'(t)+\mPperp'(t)\mH(t)\\
    \Rightarrow \mPpara(t)\mH'(t)&=\mPperp'(t)\mH(t)=-\mPpara'(t)\mH(t).
\end{align*}
At $t=0$, we have
\begin{align}\label{eq:ph}
    \mPpara(0)\mH'(0)&=-\mPpara'(0)\mH(0).
\end{align}
WLOG let $\mH(0)=\diag(\lambda_1, \cdots, \lambda_d), \in \R^{d \times d}$, where $\lambda_i = 0$ for all $m < i \le d$.
Therefore $\mPperp(0)=\begin{bmatrix}\mI_{m} & \vzero \\ \vzero & \vzero\end{bmatrix}$, $\mPpara(0)=\begin{bmatrix}\vzero & \vzero \\ \vzero & \mI_{d-m}\end{bmatrix}$.  Decompose $\mPpara'(0)$, $\mH(0)$ and $\mH'(0) $as follows.
\begin{align*}
\mPpara'(0)=
\begin{bmatrix}
\mP'_{\parallel, 11 }(0)& \mP'_{\parallel, 12}(0)\\
\mP'_{\parallel, 21 }(0)& \mP'_{\parallel, 22}(0)
\end{bmatrix}, 
\mH(0)=\begin{bmatrix}
\mH_{11}(0)& \vzero\\
\vzero& \vzero
\end{bmatrix},
\mH'(0)=\begin{bmatrix}
\mH'_{11 }(0)& \mH'_{12}(0)\\
\mH'_{21 }(0)& \mH'_{22}(0)
\end{bmatrix}.
\end{align*} 
Substituting the decomposition into \eqref{eq:ph}, we have
\begin{align*}
    \begin{bmatrix}
\vzero& \vzero\\
\mH'_{21 }(0)& \mH'_{22}(0)
\end{bmatrix}=
-\begin{bmatrix}
\mP'_{\parallel, 11 }(0)\mH_{11}(0)& \vzero\\
\mP'_{\parallel, 21 }(0)\mH_{11}(0)& \vzero
\end{bmatrix}.
\end{align*}
Therefore, $\mH'_{22}(0)=\vzero$ and
\begin{align*}
    \mPpara(0)\mH'(0)=-\mPpara'(0)\mH(0)=-\begin{bmatrix}
\vzero& \vzero\\
\mH'_{21 }(0)& \vzero
\end{bmatrix}.
\end{align*}
Any $\vu\in T_{\vtheta}(\Gamma)$ can be decomposed as $\vu=\begin{bmatrix}\vzero,  \vu_2 \end{bmatrix}^{\top}$ where $ \vu_2 \in \R^{d-m}$. 
With this decomposition, we have $\mPpara(0)\mH'(0)\vu=\vzero$. 
Also, note that $\mH'(0)=\nabla^3\cL(\vtheta)[\vv]$. Hence, 
\begin{align*}
    \partial \Phi(\vtheta)\nabla^3\cL(\vtheta)[\vv\vu^T]=\vzero.
\end{align*}
\end{proof}
\subsection{Preliminary Lemmas for GD and GF}\label{subsec: preliminary for GD}
In this subsection, we introduce a few useful preliminary lemmas about gradient descent and gradient flow. Before presenting the lemmas, we  introduce some notations and assumptions that will be used in this subsection.

Assume that the loss function $\cL(\vtheta)$ is $\rho$-smooth and $\mu$-PL in an open, convex neighborhood $U$ of a local minimizer $\vtheta^*$. Denote by $\cL^*:=\cL(\vtheta^*)$ the minimum value for simplicity. Let $\epsilon'$ be the radius of the open $\epsilon'$-ball centered at $\vtheta^*$ such that $B^{\epsilon'}(\vtheta^*)\subseteq U$. We also define a potential function $\tPsi(\vtheta):=\sqrt{\cL(\vtheta)-\cL^*}$.

Consider gradient descent iterates $\{\hvu_t\}_{t\in \N}$ following the update rule  $\hvu_{t+1} = \hvu_t - \eta \nabla \cL(\hvu_t)$. We first introduce the descent lemma for gradient descent.
\begin{lemma}[Descent lemma for GD]\label{lemma:gd2}
 If $\hvu_{t}\in U$ and $\eta \leq \frac{1}{\rho}$, then
 \begin{align*}
     \frac{\eta}{2}\normtwo{\nabla \cL(\hvu_t)}^2 \leq \cL(\hvu_t)-\cL(\hvu_{t+1}),
 \end{align*}
 and
 \begin{align*}
     \cL(\hvu_{t+1})-\cL^* \leq (1-\mu\eta) (\cL(\hvu_{t})-\cL^*).
 \end{align*}
\end{lemma}
\begin{proof}
By $\rho$-smoothness, 
\begin{align*}
    \cL(\hvu_{t+1})& \leq \cL(\hvu_t) + \inner{\nabla \cL(\hvu_t)}{\hvu_{t+1}-\hvu_{t}}+\frac{\rho\eta^2}{2}\normtwo{\hvu_{t+1}-\hvu_{t}}^2\\
    & = \cL(\hvu_{t}) - \eta (1 - \frac{\rho\eta}{2})\normtwo{\nabla \cL(\hvu_t)}^2\\
    & \leq \cL(\hvu_{t}) -\frac{\eta}{2} \normtwo{\nabla \cL(\hvu_t)}^2
\end{align*}
By the definition of $\mu$-PL, we have
\begin{align*}
     \cL(\hvu_{t+1})-\cL^* \leq (1-\mu\eta) (\cL(\hvu_{t})-\cL^*).
\end{align*}
\end{proof}

Then we prove the Lipschitzness of $\tPsi(\vtheta)$.
\begin{lemma}[Lipschitzness of $\tPsi(\vtheta)$]\label{lemma: tpsi lipschitz}
 $\tPsi(\vtheta)$ is $\sqrt{2\rho}$-Lipschitz for $\vtheta\in U$. That is, for any $\vtheta_1$, $\vtheta_2\in U$, 
\begin{align*}
    \abs{\tPsi(\vtheta_1)-\tPsi(\vtheta_2)}\leq \sqrt{2\rho}\normtwo{\vtheta_1 - \vtheta_2}.
\end{align*}
\end{lemma}
\begin{proof}
Fix $\vtheta_1$ and $\vtheta_2$. Denote by  $\vtheta(t):=(1-t)\vtheta_1 + t \vtheta_2$ the convex combination of $\vtheta_1$ and $\vtheta_2$  where $t\in [0, 1]$. Further define $f(t):=\tPsi(\vtheta(t))$.
Below we consider two cases. 
\paragraph{Case 1.} If $\forall t \in (0, 1)$, $f(t) > 0$, then $f(t)$ is differentiable on $(0, 1)$.
\begin{align*}
    \abs{\tPsi(\vtheta_2)- \tPsi(\vtheta_1)} & = \abs{f(1)-f(0)}\\
    & = \labs{\int_{0}^1 f'(t)\dd t}\\
    & =\labs{\int_0^1 \inner{\nabla \tPsi(\vtheta(t))}{\vtheta_2 - \vtheta_1}\dd t}\\
    & = \labs{\int_{0}^1\frac{\inner{\nabla \cL(\vtheta(t))}{\vtheta_2-\vtheta_1}}{\sqrt{\cL(\vtheta(t))-\cL^*}}\dd t}\\
    & \leq \normtwo{\vtheta_2-\vtheta_1}\int_0^1 \frac{\normtwo{\nabla \cL(\vtheta(t))}}{\sqrt{\cL(\vtheta(t))-\cL^*}}\dd t. 
\end{align*}
By $\rho$-smoothness of $\cL$, for all $\vtheta \in U$,
$$\normtwo{\nabla \cL(\vtheta)}^2\leq 2\rho \left(\cL(\vtheta)-\cL^*\right).$$
Since $\sqrt{\cL(\vtheta(t))-\cL^*} >0$ for all $t \in (0, 1)$,   $\tfrac{\normtwo{\nabla \cL(\vtheta(t))}}{\sqrt{\cL(\vtheta(t))-\cL^*}} \leq \sqrt{2\rho}$. Therefore, 
$$ \abs{\tPsi(\vtheta_2)- \tPsi(\vtheta_1)} \leq \sqrt{2\rho_2}\normtwo{\vtheta_2-\vtheta_1}.$$
\paragraph{Case 2.} If $\exists t'\in (0, 1)$ such that $f(t')=0$, then
\begin{align*}
    \abs{\tPsi(\vtheta_2)- \tPsi(\vtheta_1)}& = \abs{f(1)-f(0)}\\
    & = \labs{(1-t')\frac{f(1)-f(t')}{1-t'}+t'\left(\frac{f(t')-f(0)}{t'}\right)}\\
    & \leq \max \left(\frac{f(1)}{1-t'}, \frac{f(0)}{t'}\right).
\end{align*}
Since $\vtheta(t')$ minimizes $\cL$ in an open set, $\nabla \cL(\vtheta(t'))=\vzero$. By $\rho$-smoothness of $\cL$, for all $\vtheta \in U$, 
\begin{align*}
  \cL(\vtheta) \leq\cL^* + \frac{\rho}{2}\normtwo{\vtheta - \vtheta(t')}^2 \quad\Rightarrow \quad  \tPsi(\vtheta) \leq \sqrt{\frac{\rho}{2}}\normtwo{\vtheta - \vtheta(t')}.
\end{align*}
Therefore, 
\begin{align*}
  f(1)&\leq \sqrt{\frac{\rho}{2}}\normtwo{\vtheta_2-\vtheta(t')}=(1-t')\sqrt{\frac{\rho}{2}}\normtwo{\vtheta_2-\vtheta_1}\\
   f(0)&\leq \sqrt{\frac{\rho}{2}}\normtwo{\vtheta_1-\vtheta(t')}=t'\sqrt{\frac{\rho}{2}}\normtwo{\vtheta_2-\vtheta_1}.
\end{align*}
Then we have
\begin{align*}
    \abs{\tPsi(\vtheta_2)- \tPsi(\vtheta_1)}\leq \sqrt{\frac{\rho}{2}}\normtwo{\vtheta_2 - \vtheta_1}.
\end{align*}
Combining case 1 and case 2, we conclude the proof.
\end{proof}

Below we introduce a lemma that relates the movement of one step gradient descent to the change of the potential function.
\begin{lemma}[Lemma G.1 in \cite{lyu2022understanding}]\label{lemma:potential}
 If $\hvu_t \in U$ and $\eta \leq 1/\rho_2$ then 
\begin{align*}
    \tPsi(\hvu_{t})-\tPsi(\hvu_{t+1}) \geq \frac{\sqrt{2\mu}}{4}\eta \normtwo{\nabla \cL(\hvu_{t})}.
\end{align*}
\end{lemma}

\begin{proof}
\begin{align}
     \tPsi(\hvu_{t})-\tPsi(\hvu_{t+1}) &= \frac{\cL(\hvu_{ t})-\cL(\hvu_{ t+1})}{\tPsi(\hvu_{t})+\tPsi(\hvu_{t+1})}\notag\\
     & \geq \frac{\cL(\hvu_{ t+1})-\cL(\hvu_{ t})}{2\tPsi(\hvu_{ t})}\notag\\
     & \geq \frac{\eta(1-\rho_2\eta/2)\normtwo{\nabla \cL(\hvu_{ t})}^2}{2\tPsi(\hvu_{t})},\notag
\end{align}
where the two inequalities uses \Cref{lemma:gd2}. By $\mu$-PL, $\tPsi(\hvu_{t})\leq \frac{1}{\sqrt{2\mu}}\normtwo{\nabla \cL(\hvu_{t})}$. Therefore, we have $\tPsi(\hvu_{t})-\tPsi(\hvu_{t+1}) \geq \frac{\sqrt{2\mu}}{2}(1-\eta \rho/2)\eta \normtwo{\nabla \cL(\hvu_{t})}\geq \frac{\sqrt{2\mu}}{4}\eta \normtwo{\nabla \cL(\hvu_t)}$.
\end{proof}
Based on \Cref{lemma:potential}, we have the following lemma that bounds the movement of GD over multiple steps.
\begin{lemma}[Bounding the movement of GD]\label{lemma:movement} If $\hvu_0$ is initialized such that $\normtwo{\hvu_0-\vtheta^*} \leq \frac{1}{4}\sqrt{\frac{\mu}{\rho}}\epsilon'$, then for all $t\geq 0$, $\hvu_t\in B^{\epsilon'}(\vtheta^*)$ and 
\begin{align*}
\normtwo{\hvu_t - \hvu_0} \leq\sqrt{\frac{8}{\mu}}\tPsi(\hvu_0).
\end{align*}
\end{lemma}
\begin{proof}
We prove the proposition by induction. When $t=0$, it trivially holds. Assume that the proposition holds for $\hvu_{\tau}$, $0 \leq \tau < t$.  For step $t$, since $\hvu_{\tau}\in B^{\epsilon'}(\vtheta^*)$, we apply \Cref{lemma:potential} and obtain
\begin{align*}
\normtwo{\hvu_t - \hvu_0}\leq \eta \sum_{\tau=0}^{t-1}\normtwo{\nabla \cL(\hvu_{\tau})}\leq \sqrt{\frac{8}{\mu}}\left(\tPsi(\hvu_0)-\tPsi(\hvu_t)\right) \leq\sqrt{\frac{8}{\mu}}\tPsi(\hvu_0).
\end{align*}
Further by $\rho$-smoothness of $\cL(\cdot)$, 
\begin{align*}
    \normtwo{\hvu_t-\hvu_0}\leq \sqrt{\frac{8}{\mu}}\tPsi(\hvu_0)\leq 2\sqrt{\frac{\rho}{\mu}}\normtwo{\hvu_0-\vtheta^*}\leq \frac{1}{2}\epsilon'.
\end{align*}
Therefore, $\normtwo{\hvu_t-\vtheta^*}\leq \normtwo{\hvu_t-\hvu_0}+\normtwo{\hvu_0-\vtheta^*}<\epsilon'$, which concludes the proof.
\end{proof}

Finally, we introduce a lemma adapted from Thm. D.4 of which bounds the movement of GF. \cite{lyu2022understanding}.
\begin{lemma}\label{lemma:pre equiv} 
Assume that $\normtwo{\vtheta_0-\vtheta^*}< \sqrt{\frac{\mu}{\rho}}\epsilon'$. The gradient flow $\vtheta(t)=-\frac{\dd \cL(\vtheta(t))}{\dd t}$ starting at $\vtheta_0$ converges to a point in $U$ and
\begin{align*}
    \lnormtwo{\vtheta_0 -\lim_{t\to +\infty}\vtheta(t)}&\leq \sqrt{\frac{2}{\mu}}\sqrt{\cL(\vtheta_0)-\cL^*}\leq \sqrt{\frac{\rho}{\mu}}\normtwo{\vtheta_0-\vtheta^*}
\end{align*}
\end{lemma}
\begin{proof}
Let $T:=\inf \{t: \vtheta \notin U\}$. Then for all $t<T$, 
\begin{align*}
    \frac{\mathrm{d}}{\mathrm{d} t}\left(\mathcal{L}(\boldsymbol{\theta})-\cL^*\right)^{1 / 2}&=\frac{1}{2}\left(\mathcal{L}(\boldsymbol{\theta})-\cL^*\right)^{-1 / 2} \cdot\left\langle\nabla \mathcal{L}(\boldsymbol{\theta}), \frac{\mathrm{d} \boldsymbol{\theta}}{\mathrm{d} t}\right\rangle\\
    & = -\frac{1}{2}(\cL(\vtheta)-\cL^*)^{-1/2}\normtwo{\nabla \cL(\vtheta)}\normtwo{\frac{\dd \vtheta}{\dd t}}.
\end{align*}
By $\mu$-PL, $\normtwo{\nabla \cL(\vtheta)}\geq \sqrt{2\mu(\cL(\vtheta)-\cL^*)}$. Hence, 
\begin{align*}
     \frac{\mathrm{d}}{\mathrm{d} t}\left(\mathcal{L}(\boldsymbol{\theta})-\mathcal{L}^*\right)^{1 / 2}\leq -\frac{\sqrt{2\mu}}{2}\normtwo{\frac{\dd \vtheta}{\dd t}}.
\end{align*}
Integrating both sides, we have
\begin{align*}
    \int_0^{T}\|\frac{\dd \vtheta(\tau)}{\dd \tau}\|\dd \tau \leq \frac{2}{\sqrt{2\mu}}(\cL(\vtheta_0)-\cL^*)^{1/2}\leq \sqrt{\frac{\rho}{\mu}}\normtwo{\vtheta_0-\vtheta^*}<\epsilon',
\end{align*}
where the second inequality uses $\rho$-smoothness of $\cL$. Therefore, $T=+\infty$ and $\vtheta(t)$ converges to some point in $U$.
\end{proof}

\subsection{Construction of working zones}\label{subsec: working zone}

We construct four nested working zones $(\Gaz, \Gao, \Gat, \Gath)$ in the neighborhood of $\Gamma$. Later we will show that the local iterates $\vths_{k,t}\in \Gat$ and the global iterates $\bvths \in \Gaz$ with high probability after $\cO(\log \frac{1}{\eta})$ rounds. The following lemma illustrates the properties the working zones should satisfy.
\begin{lemma}[Working zone lemma]\label{lemma:workingzone}
There exists constants $\epsilon_0 < \epsilon_1 < \epsilon_2 < \epsilon_3$ such that $(\Gaz, \Gao, \Gat, \Gath)$ satisfy the following properties:
\begin{enumerate}
    \item $\cL$ satisfies $\mu$-PL  in $\Gath$ for some $\mu > 0$.
    \item Any gradient flow starting in $\Gat$ converges to some point in $\Gamma$. Then, by \cite{falconer1983differentiation}, $\Phi(\cdot)$ is $\cC^{\infty}$ in $\Gat$.
    \item Any $\vtheta \in \Gao$ has an $\epsilon_1$-neighborhood $\Beo(\vtheta)$ such that $\Beo(\vtheta) \subseteq \Gat$.
    \item Any gradient descent starting in $\Gaz$ with sufficiently small learning rate will stay in $\Gao$.
\end{enumerate}
\end{lemma}
 
\begin{proof}
Let $\bvths[0]$ be initialized such that $\Phi(\bvths[0])\in \Gamma$. Let $\cZ$ be the  set of all points on the gradient flow trajectory starting from $\bvths[0]$ and $\cZeps$ be the $\epsilon$-neighborhood of $\cZ$, where $\epsilon$ is a positive constant. Since the gradient flow converges to $\vphs[0]$, $\cZ$ and $\cZeps$ are bounded.

We construct four nested working zones.
By Lemma H.3 in \cite{lyu2022understanding}, there exists an $\epsilon_3$-neighborhood of $\Gamma$, $\Gath$, such that $\cL$ satisfies $\mu$-PL for some $\mu > 0$. 
Let $\cM$ be the convex hull of $\Gath \cup \cZeps$ and $\cM^{\epsilon_4}$ be the $\epsilon_4$-neighborhood of $\cM$ where $\epsilon_4$ is a positive constant. Then $\cM^{\epsilon_4}$ is bounded. 

Define  $\rho_2 = \sup_{\vtheta \in \cM^{\epsilon_4}}\normtwo{\nabla^2 \cL(\vtheta)}$ and $\rho_3 = \sup_{\cM^{\epsilon_4}}\normtwo{\nabla^3 \cL(\vtheta)}$. By \Cref{lemma:pre equiv}, we can construct an $\epsilon_2$-neighborhood of $\Gamma$ where $\epsilon_2 < \sqrt{\frac{\mu}{\rho_2}}\epsilon_3$ such that all GF starting in $\Gat$ converges to $\Gamma$.  By \citet{falconer1983differentiation}, $\Phi(\cdot)$ is $\cC^2$ in $\Gath$. Define $\nu_{1}=\sup_{\vtheta \in \Gath}\normtwo{\partial \Phi(\vtheta)}$ and $\nu_{2}=\sup_{\vtheta \in \Gath}\normtwo{\partial^2 \Phi(\vtheta)}$. We  also construct an $\epsilon_1$ neighborhood of $\Gamma$, $\Gao$, where $\epsilon_1\leq \frac{1}{2}\epsilon_2 < \frac{1}{2}\sqrt{\frac{\mu}{\rho_2}}\epsilon_3$ such that all $\vtheta \in \Gao$ has an $\epsilon_1$ neighborhood where $\Phi$ is well defined. Finally, by \Cref{lemma:movement}, there exists an $\epsilon_0$-neighborhood of $\Gamma$ where $\epsilon_0\leq \frac{1}{4}\sqrt{\frac{\mu}{\rho_2}}\epsilon_1$ such that all gradient descent iterates starting in $\Gaz$ with $\eta \leq \frac{1}{\rho_2}$ will stay in $\Gao$. 
\end{proof}
Note that the notions of $\cZ^{\epsilon}$, $\cM^{\epsilon_4}$, $\rho
_2$, $\rho_3$, $\nu_1$, and $\nu_2$ defined in the proof will be useful in the remaining part of this section.
When analyzing the limiting dynamics of Local SGD, we will show that all $\vths_{k, t}$ stays in $\Gat$, $\tvus_t \in \Gao$, $\bvths\in \Gaz$ with high probability after $\cO(\log \frac{1}{\eta})$ rounds.


\subsection{Phase 1: Iterate Approaching the Manifold}\label{subsec: phase 1}
The approaching phase can be further divided into two subphases. In the first subphase, $\bvths[0]$ is initialized such that $\vphs[0]\in \Gamma$. We will show that after a constant number of rounds $s_0$, $\bvths[s_0]$ goes to the 
inner part of $\Gaz$ such that $\normtwo{\bvths[s_0] - \vphs[0]}\leq c \epsilon_0$ with high probability, where $0<c<1$ and the constants will be specified later (see \Cref{sec: subphase 1}). In the second subphase, we show that the iterate can reach within $\ctO(\sqrt{\eta})$ distance from $\Gamma$ after $\cO(\log \frac{1}{\eta})$ rounds with high probability (see \Cref{subsec: subphase 2}).

 \subsubsection{Additional notations}
 Consider an auxiliary sequence $\{\tvus_t\}$ where $\tvus_0=\bvths$ and $\tvus_{t+1}=\tvus_t - \eta \nabla \cL(\tvus_t), 0\leq t \leq H-1$. Define $\tvDeltas_{k, t}:=\vths_{k, t}-\tvus_t$ to be the difference between the local iterate and the gradient descent iterate. Notice that $\tvDeltas_{k, 0}=0$, for all $k$ and $s$. 
 
 Consider a gradient flow $\{\vu(t)\}_{t\geq 0}$ with the initial condition $\vu(0)=\bvths[0]$ and converges to $\vphs[0]\in \Gamma$. For simplicity, let $\vus_t:=\vu(s\alpha + t\eta)$ be the gradient flow after $s$ rounds plus $t$ steps. Let $s_0$ be the smallest number such that $\normtwo{\vus[s_0]_0-\vphs[0]}\leq \frac{1}{4}\sqrt{\frac{\mu}{\rho_2}}\epsilon_0$ . Note that $s_0$ is a constant independent of $\eta$.
 
In this subsection, the minimum value of the loss in \Cref{subsec: preliminary for GD}  corresponds to the loss value on $\Gamma$, i.e.,  $\cL^*=\cL(\vphi), \forall \vphi \in \Gamma$.  
 
We also define the following sequence $\{\tvZs_{k, t}\}^{H}_{t=0}$ that will be used in the proof. Define  
$$\tvZs_{k, t}:=\sum_{\tau=0}^{t-1}\left(\prod_{l=\tau+1}^{t-1}(\mI-\eta \nabla^2 \cL(\tvus_{l}))\right)\vzs_{k, \tau},\qquad \tvZs_{k, 0}=\vzero.$$

\subsubsection{Proof for Subphase 1}\label{sec: subphase 1}
 First, we have the following lemma about the concentration of $\tvZs_{k,t}$.
 \begin{lemma}[Concentration property of $\{\tvZs_{k,t}\}_{t=0}^H$]\label{lemma: concen tvZ}
 Given $\bvths$ such that $\tvus_t \in \Gath \cup \cZeps$ for all $0 \leq t \leq H$, then with probability at least $1-\delta$, 
 $$\normtwo{\tvZs[s]_{k, t}} \leq \tC_1\sigmax\sqrt{2H\log \frac{2 H K}{\delta}}, \qquad \forall 0\leq t\leq H, k\in[K],$$
 where $\tC_1: =\exp(\alpha \rho_2) $.
 \end{lemma}
 \begin{proof}
For each $\tvZs_{k,t}$, construct a sequence $\{\tvZs_{k,t,t'}\}_{t'=0}^{t}$:
$$\tvZs_{k,t,t'}:=\sum_{\tau=0}^{t'-1}\left(\prod_{l=\tau+1}^{t-1}(\mI-\eta \nabla^2 \cL(\tvus_{l}))\right)\vzs_{k, \tau},\qquad \tvZs_{k, t,0}=\vzero.$$
 Since $\tvus_t\in \Gath \cup \cZeps$, we have $\normtwo{\nabla^2 \cL(\tvus_t)}\leq \rho_2$ for all $0 \leq t \leq H$. Then, for all $\tau$ and $t$, 
$$\lnormtwo{\prod_{l=\tau+1}^{t-1}(\mI - \eta \nabla^2 \cL(\tvus_{l}))}\leq (1+\rho_2\eta)^H\leq\exp(\alpha \rho_2)=\tC_1.$$  
Notice that for all $0\leq t \leq H$, $\{\tvZs_{k, t,t'}\}_{t'=0}^t$ is a martingale with $\normtwo{\tvZs_{k,t, t'}-\tvZs_{k, t, t'-1}}\leq \tC_1 \sigma_{\max}$. 
 By Azuma-Hoeffding's inequality, 
\begin{align*}
   \PP(\normtwo{\tvZs_{k,t}}\geq \epsilon') \leq 2\exp{\left(\frac{-\epsilon'^2}{2t \left(\tC_1\sigma_{\max}\right)^2}\right)}
   \leq 2\exp{\left(\frac{-\epsilon'^2}{2H \left(\tC_1\sigma_{\max}\right)^2}\right)}.
\end{align*}
Taking a union bound on all $k\in[K]$ and $0\leq t \leq H$, we can conclude that with probability at least $1-\delta$, 
\begin{align*}
    \normtwo{\tvZs[s]_{k, t}} \leq \tC_1\sigmax\sqrt{2H\log \frac{2 H K}{\delta}}, \qquad \forall 0\leq t\leq H, k\in[K].
\end{align*}
 \end{proof}
 
 The following lemma states that the gradient descent iterates will closely track the gradient flow with the same initial point.
 \begin{lemma}\label{lemma: track gf}
Denote $G:=\sup_{t\geq 0}\normtwo{\nabla\cL( \vu(t))}$ as the upper bound of the gradient on the gradient flow trajectory. If $\normtwo{\tvus_t-\vus_t}=\cO(\sqrt{\eta})$, then for all $0\leq t \leq H$,  the closeness of $\tvus_t$ and $\vus_t$ is bounded by
\begin{align*}
     \normtwo{\tvus_{t}-\vus_{t}}
     & \leq \tC_1\normtwo{\tvus_0 - \vus_0} + \tC_1\eta G,
\end{align*}
where $\tC_1=\exp(\alpha\rho_2)$.
 \end{lemma}
 \begin{proof}
 We prove by induction that 
 \begin{align}\label{eq:uflow}
     \normtwo{\tvus_{t}-\vus_{t}}&\leq (1+\rho_2\eta)^{t}\normtwo{\tvus_0 - \vus_0}+\rho_2 \eta^2 G\sum_{\tau=0}^{t-1}(1+\rho_2\eta)^{\tau}.
 \end{align}
 When $t=0$, \eqref{eq:uflow} holds trivially. Assume that \eqref{eq:uflow} holds for $0\leq \tau \leq t$, then 
 \begin{align*}
     \tvus_{t+1}-\vus_{t+1} &= \tvus_t - \eta \nabla \cL(\tvus_t)-\left(\vu_t-\int_{s\alpha+t\eta}^{s\alpha+(t+1)\eta}\nabla \cL(\vu(v))dv\right)\\
     & = \tvus_t - \vu_t-\eta \left(\nabla \cL(\tvus_t) - \nabla \cL(\vus_t)\right)\\
     & \quad-\int_{s\alpha+t\eta}^{s\alpha+(t+1)\eta}\left(\nabla \cL(\vus_t)-\nabla \cL(\vu(v))\right)dv.
 \end{align*}
By smoothness of $\cL$, 
 \begin{align*}
     \normtwo{\nabla \cL(\vus_t)-\nabla \cL(\vu(v)) }&\leq \rho_2  \normtwo{\vus_t -\vu(v) }\\
     & \leq \rho_2 \int^v_{s\alpha+t\eta}\normtwo{\nabla \cL (\vu(w))}d w\\
     & \leq \rho_2 \eta G.
 \end{align*}
 Since $\rho_2^2\eta^2 G\sum_{\tau=0}^{t-1}(1+\rho_2\eta)^{\tau}\leq\eta G(1+\rho_2\eta)^t\leq \exp(\alpha \rho_2)\eta G$, then $\normtwo{\tvus_t-\vus_t}=\cO(\sqrt{\eta})$, which implies that $\tvus_t \in \cM^{\epsilon_4}$. Hence,   $\normtwo{\nabla \cL(\tvus_{t})-\cL(\vus_{t})}\leq \rho_2 \normtwo{\tvus_t-\vus_t}$.
 
By triangle inequality, 
 \begin{align*}
     \normtwo{\tvus_{t+1}-\vus_{t+1}}&\leq (1+\rho_2 \eta)\normtwo{\tvus_t - \vus_t}+\rho_2 \eta^2 G\\
     &\leq (1+\rho_2 \eta)^{t+1}\normtwo{\tvus_t - \vus_t}+\rho_2 \eta^2 G\sum_{\tau=0}^t(1+\rho_2\eta)^{\tau},
 \end{align*}
 which concludes the induction step.
Appling $1+\rho_2\eta \leq \exp(\rho_2\eta)$, we have the lemma.
 \end{proof}
 Utilizing the concentration probability of $\{\tvZs_{k,t}\}$, we can obtain the following lemma which implies that the Local SGD iterates will closely track the gradient descent iterates with high probability.
\begin{lemma}\label{lemma:tdelta}
Given $\bvths$ such that $\tvus_t \in \Gath \cup \cZeps$ for all $0\leq t \leq H$, then  for $\delta = \cO(\poly(\eta))$, with probability at least $1-\delta$, there exists a constant $\tC_3$ such that
\begin{align*}
    \normtwo{\vths_{k, t}-\tvus_t}\leq \tC_3\sqrt{\eta \log \frac{1}{\eta \delta}}, \quad \forall 0 \leq t \leq H, k\in [K],
\end{align*}
and 
\begin{align*}
    \normtwo{\bvths[s+1]-\tvus_{H}}\leq \tC_3\sqrt{\eta \log \frac{1}{\eta \delta}}.
\end{align*}
\end{lemma}
\begin{proof}
 Since $\tvus_t \in \Gath \cup \cZeps$ for all $0\leq t \leq H$, we have $\normtwo{\nabla^2 \cL(\tvus_t)}\leq \rho_2$.
 According to the update rule for $\vths_{k, t}$ and $\tvus_t$, 
\begin{align}
    \vths_{k, t+1}&=\vths_{k, t}-\eta \nabla \cL(\vths_{k, t})-\eta \vzs_{k, t}\label{eq:thupd2},\\
    \tvus_{t+1}&=\tvus_t-\eta \nabla \cL(\tvus_t)\label{eq:uupd2}.
\end{align}
Subtracting \eqref{eq:uupd2} from \eqref{eq:thupd2}  gives
\begin{align}
    \tvDeltas_{k, t+1}&=\tvDeltas_{k, t}-\eta (\nabla \cL(\vths_{k, t})-\nabla \cL(\tvus_t))-\eta \vzs_{k, t}\notag\\
    &=(\mI - \eta \nabla^2 \cL(\tvus_t))\tvDeltas_{k, t}-\eta \vzs_{k, t}+\eta \tvvs_{k, t}. \label{eq:expand}
\end{align}
Here, $\tvvs_{k,t}=(1-\beta^{(s)}_{k,t})\vths_{k,t}+\beta^{(s)}_{k,t}\tvus_{k,t}$, where $\beta^{(s)}_{k,t}\in(0,1)$ depends on $\vths_{k,t}$ and $\tvus_t$. Therefore, $\normtwo{\tvvs_{k, t}}\leq \frac{\rho_3}{2} \normtwo{\tvDeltas_{k, t}}^2$ if $\vths_{k,t}\in \cM^{\epsilon_4}$. Applying $\eqref{eq:expand}$ $t$ times, we have
\begin{align*}
  \tvDeltas_{k, t}
 &=\left[\prod_{\tau =0}^{t-1}(\mI - \eta \nabla^2 \cL(\tvus_{\tau}))\right]\tvDeltas_{k, 0}-\eta \sum_{\tau=0}^{t-1}\prod_{l=\tau+1}^{t-1}(\mI-\eta \nabla^2 \cL(\tvus_{l}))\vzs_{k, \tau}\\
 & \quad + \eta \sum_{\tau=0}^{t-1}\prod_{l=\tau+1}^{t-1}(\mI-\eta \nabla^2 \cL(\tvus_{l}))\tvvs_{k, \tau}.
\end{align*}
  By Cauchy-Schwartz inequality,  triangle inequality and the definition of $\tvZs_{k, t}$, if for all $0\leq \tau\leq t-1$ and $k\in[K]$,  $\vths_{k,\tau}\in \cM^{\epsilon_4}$, then we have
\begin{align}\label{eq:delta apply}
   \normtwo{\tvDeltas_{k, t}}\leq \eta\normtwo{\tvZs_{k, t}}+\frac{1}{2}\eta\rho_3 \sum_{\tau=0}^{t-1}\tC_1\normtwo{\tvDeltas_{k, \tau}}^2.
\end{align}
Applying \Cref{lemma: concen tvZ}
 and substituting in the value of $H$, we have that with probability at least $1-\delta$,
 \begin{align}\label{eq:tvZ}
 \normtwo{\tvZs_{k, t}}\leq \tC_1 \sigmax\sqrt{\frac{2\alpha}{\eta}\log \frac{2\alpha K}{\eta \delta}}, \qquad \forall k \in K, 0\leq t \leq H.
  \end{align}
Now we show by induction that for $\delta=\cO(\poly(\eta))$, when \eqref{eq:tvZ} holds, there exists a constant $\tC_2>2\sigmax \sqrt{2\alpha}\tC_1$ such that $\normtwo{\tvDeltas_{k, t} }\leq \tC_2 \sqrt{{\eta} \log \frac{2\alpha K }{\eta \delta}}$. 

When $t=0$, $\tvDeltas_{k, 0}=0$. Assume that $\normtwo{\tvDeltas_{k, \tau} }\leq \tC_2 \sqrt{{\eta} \log \frac{2\alpha K }{\eta \delta}}$, for all $k\in[K], 0\leq \tau \leq t-1$. Then for all $0\leq\tau \leq t-1$, $\vths_{k,\tau}\in \cM^{\epsilon_4}$. Therefore, we can apply \eqref{eq:delta apply} and obtain
\begin{align*}
    \normtwo{\tvDeltas_{k, t}}&\leq \eta \normtwo{\tvZs_{k,t}} +\frac{1}{2}\eta\rho_3 \sum_{\tau=0}^{t-1}\tC_1\normtwo{\tvDeltas_{k, \tau}}^2\\
    & \leq \tC_1\sigmax\sqrt{2\alpha\eta\log \frac{2\alpha K}{\eta \delta}}+\frac{1}{2}\tC_1\tC_2^2 \sigmax^2 \alpha\rho_3 \eta\log \frac{2\alpha K}{\eta \delta}.
\end{align*}
Given that $\tC_2 \geq 2\sigmax \sqrt{2\alpha}\tC_1$ and $\delta = \cO(\poly(\eta))$, when $\eta$ is sufficiently small, $\normtwo{\tvDeltas_{k, t}}\leq \tC_2\sqrt{\eta \log\frac{2\alpha K}{\eta \delta}}$.

To sum up, for $\delta = \cO(\poly(\eta))$, with probability at least $1-\delta$, $\normtwo{\tvDeltas_{k, t}}\leq \tC_2 \sqrt{\eta \log\frac{2\alpha K}{\eta \delta}}$ for all $k\in [K]$, $0\leq t \leq H$. By triangle inequality, 
\begin{align*}
    \normtwo{\bvths[s+1]-\tvus_H}\leq \frac{1}{K}\sum_{k\in [K]}\normtwo{\tvDeltas_{k, H}} \leq \tC_2 \sqrt{\eta \log\frac{2\alpha K}{\eta \delta}}.
\end{align*}
\end{proof}
The combination of \Cref{lemma: track gf} and \Cref{lemma:tdelta} leads to the following lemma, which states that the Local SGD iterate will enter $\Gao$ after $s_0$ rounds with high probability.
\begin{lemma}\label{lemma:s0}
Given $\bvths[0]$ such that $\Phi(\bvths[0])\in \Gamma$, then for $\delta = \cO(\poly(\eta))$, there exists a positive constant $\tC_4$ such that with probability at least $1-\delta$, 
\begin{align*}
    \normtwo{\bvths[s_0]-\vphs[0]} \leq \frac{1}{4}\sqrt{\frac{\mu}{\rho_2}}\epsilon_0 + \tC_4 \sqrt{\eta \log \frac{1}{\eta \delta}}. 
\end{align*}
\end{lemma}
\begin{proof}
First, we prove by induction that for $\delta = \cO(\poly(\eta))$, when 
\begin{align}\label{eq:condition on Z}
    \normtwo{\tvZs[s]_{k, t}} \leq \tC_1\sigmax\sqrt{2H\log \frac{2 H K s_0}{\delta}}, \qquad \forall 0\leq t\leq H, k\in[K], 0 \leq s < s_0, 
\end{align} 
 the closeness of $\bvths$ and $\vus_0$ is bounded by 
\begin{align}\label{eq:closeness theta u}
    \normtwo{\bvths - \vus_0} \leq \sum_{l=1}^s\tC_1^l\left(\eta G +  \tC_3 \sqrt{\eta \log \frac{s_0}{\eta\delta}}\right), \qquad \forall 0 \leq s \leq s_0.
\end{align}
When $s=0$, $\bvths[0]=\vus[0]_0$. Assume that \eqref{eq:closeness theta u} holds for round $s$. Then by \Cref{lemma: track gf}, for all $0\leq t \leq H$,
\begin{align*}
    \normtwo{\tvus[s]_{t}-\vus[s]_{t}}
     & \leq \tC_1\normtwo{\tvus_0 - \vus_0} + \tC_1\eta G\\
     &= \tC_1\normtwo{\bvths_0 - \vus_0} + \tC_1\eta G\\
     & \leq \sum_{l=1}^{s}\tC_1^{l+1}\left(\eta G +  \tC_3 \sqrt{\eta \log \frac{s_0}{\eta\delta}}\right)+\tC_1\eta G.
\end{align*}
Therefore, for sufficiently small $\eta$, $\tvus_t \in \cZeps$, $\forall 0 \leq t \leq H$. Combing the above inequality with \Cref{lemma:tdelta}, we have
\begin{align*}
    \normtwo{\bvths[s+1]-\vus[s+1]_0}&=\normtwo{\bvths[s+1]-\vus[s]_H}\\
    & \leq \normtwo{\bvths[s+1]-\tvus_H}+\normtwo{\tvus_H- \vus_H}\\
    & \leq \sum_{l=1}^{s+1}\tC_1^{l+1}\left(\eta G +  \tC_3 \sqrt{\eta \log \frac{s_0}{\eta\delta}}\right),
\end{align*}
which concludes the induction.

Therefore, when \eqref{eq:condition on Z} holds, there exists a positive constant $\tC_4 $ such that $$\normtwo{\bvths[s_0]-\vus[s_0]_0}\leq \tC_4\sqrt{\eta\log \frac{1}{\eta\delta}}.$$
By definition of $\vus[s_0]_0$, 
\begin{align*}
    \normtwo{\bvths[s_0]-\vphs[0]} \leq \frac{1}{4}\sqrt{\frac{\mu}{\rho_2}}\epsilon_0 + \tC_4 \sqrt{\eta \log \frac{1}{\eta \delta}}. 
\end{align*}
Finally, according to \Cref{lemma: concen tvZ}, \eqref{eq:condition on Z} holds with probability at least $1-\delta$.
\end{proof}

\subsubsection{Proof for Subphase 2}\label{subsec: subphase 2}
In subphase 2, we show that the iterate can reach within $\ctO(\sqrt{\eta})$ distance from $\Gamma$ after $\cO(\log \frac{1}{\eta})$ rounds with high probability. The following lemma manifests how the potential function $\tPsi(\bvths)$ evolves after one round.
\begin{lemma}\label{lemma: evolve tpsi}
Given $\bvths \in \Gaz$, for $\delta=\cO(\poly(\eta))$, with probability at least $1-\delta$, 
  \begin{align*}
    \vths_{k, t}\in \Gat, \quad \tPsi(\vths_{k, t})
    &\leq \tPsi(\bvths)+\tC_5\sqrt{\eta \log \frac{1}{\eta\delta}} , \quad \forall k \in [K], 0\leq t \leq H
\end{align*}
and 
\begin{align*}
   \bvths[s+1]\in \Gat, \quad  \tPsi(\bvths[s+1])
    &\leq \exp(-\alpha\mu/2 )\tPsi(\bvths)+\tC_5\sqrt{\eta \log \frac{1}{\eta\delta}}, 
\end{align*}
where $\tC_5$ is a positive constant.
\begin{proof}
Since $\bvths \in \Gaz$, then for all $0\leq t \leq H$, $\tvus_t \in \Gao$ by the definition of the working zone. By \Cref{lemma:gd2}, for $\eta \leq \frac{1}{\rho_2}$, 
\begin{align*}
    \cL(\tvus_t)-\cL^*& \leq (1-\mu \eta)^{t}\left(\cL(\bvths)-\cL^*\right)\leq \cL(\bvths)-\cL^*, \quad \forall 0 \leq t \leq H.
\end{align*}
Specially, for $t=H$, 
\begin{align*}
    \cL(\tvus_H)-\cL^*& \leq (1-\mu \eta)^{{\frac{\alpha}{\eta }}}\left(\cL(\bvths)-\cL^*\right)\leq \exp(-\alpha\mu)(\cL(\bvths)-\cL^*).
\end{align*}
Therefore, 
\begin{align*}
   \tPsi(\tvus_{H})&\leq \exp(-\alpha\mu/2)\tPsi(\bvths).
\end{align*}
According to the proof of \Cref{lemma:tdelta}, for $\delta=\cO(\poly(\eta))$,  when 
\begin{align}\label{eq:tvZ2}
 \normtwo{\tvZs_{k, t}}\leq \tC_1 \sigmax\sqrt{\frac{2\alpha}{\eta}\log \frac{2\alpha K}{\eta \delta}}, \qquad \forall k \in [K], 0\leq t \leq H,
  \end{align}
 there exists a constant $\tC_3$ such that 
 \begin{align*}
    \normtwo{\vths_{k, t}-\tvus_t}\leq \tC_3\sqrt{\eta \log \frac{1}{\eta \delta}}, \quad \forall 0 \leq t \leq H, k\in [K],
\end{align*}
and 
\begin{align*}
    \normtwo{\bvths[s+1]-\tvus_{H}}\leq \tC_3\sqrt{\eta \log \frac{1}{\eta \delta}}.
\end{align*}
Since $\tvus_t \in \Gao$, $\forall 0 \leq t \leq H$, $\bvths[s+1]\in \Gat$ and $\bvths_{k, t}\in \Gat$, $\forall 0 \leq t \leq H$, $k\in [K]$.

By \Cref{lemma: tpsi lipschitz}, $\tPsi(\cdot)$ is $\sqrt{2\rho_2}$-Lipschitz in $\cM^{\epsilon_4}$. Therefore, when \eqref{eq:tvZ2} holds, there exists a constant $\tC_5:=\sqrt{2\rho_2}\tC_3$ such that
\begin{align*}
    \tPsi(\vths_{k, t}) &\leq \tPsi(\tvus_{t})+\sqrt{2\rho_2}\normtwo{\vths_{k, t}-\tvus_t}\\
    &\leq \tPsi(\bvths)+\tC_5\sqrt{\eta \log \frac{1}{\eta\delta}}, 
\end{align*}
and 
\begin{align*}
    \tPsi(\bvths[s+1])&\leq \tPsi(\tvus_H)+\sqrt{2\rho_2}\normtwo{\bvths[s+1]-\tvus_H} \\
    &\leq \exp(-\alpha\mu/2 )\tPsi(\bvths)+\tC_5\sqrt{\eta \log \frac{1}{\eta\delta}}.
\end{align*}
Finally, by \Cref{lemma: concen tvZ}, \eqref{eq:tvZ2} holds with probability at least $1-\delta$.
\end{proof}
\end{lemma}
We are thus led to the following lemma which characterizes the evolution of the potential $\tPsi(\bvths)$ and $\tPsi(\vths_{k, t})$ over multiple rounds. 
\begin{lemma}\label{lemma:tpsi whole}
Given $\normtwo{\bvths[0]-\vphs[0]}\leq\frac{1}{2}\sqrt{\frac{\mu}{\rho_2}}\epsilon_0$, for $\delta=\cO(\poly(\eta))$ and any integer $1\leq R \leq \Rtot$,  with probability at least $1-\delta$, 
\begin{align}\label{eq:tpsi descent}
   \bvths \in \Gaz, \tPsi(\bvths)\leq \exp(-\alpha \mu s/2)\tPsi(\bvths[0])+\frac{1}{1-\exp(-\alpha\mu/2)} \tC_5\sqrt{\eta \log \frac{R}{\eta\delta}}
,  \forall 0 \leq s \leq R.
\end{align}
Furthermore,
\begin{align}\label{eq:tpsi individual}
   \bvths_{k, t}\in \Gat, \quad \tPsi(\vths_{k, t})&\leq  \tPsi(\bvths)+ \tC_5\sqrt{\eta \log \frac{R}{\eta \delta}}, \quad \forall 0\leq t \leq H, 0\leq s < R, k \in [K].
\end{align}
\end{lemma}
\begin{proof}
We prove induction that for $\delta=\cO(\poly(\eta))$, when \begin{align}\label{eq:tvZR}
 \normtwo{\tvZs_{k, t}}\leq \tC_1 \sigmax\sqrt{\frac{2 \alpha}{\eta}\log \frac{2 R \alpha K}{\eta \delta}}, \qquad \forall k \in [K], 0\leq t \leq H, 0 \leq s <R,
  \end{align}
then for all $0\leq s \leq R$, \eqref{eq:tpsi descent} and \eqref{eq:tpsi individual} hold.

When $s=0$, $\bvths[0]\in \Gaz$ and \eqref{eq:tpsi descent} trivially holds. By \Cref{lemma: evolve tpsi}, \eqref{eq:tpsi individual} holds. Assume that  \eqref{eq:tpsi descent} and \eqref{eq:tpsi individual} hold for round $s-1$. Then for round $s$, by \Cref{lemma: evolve tpsi}, $\bvths[s]\in \Gat$ and 
\begin{align*}
     \Psi(\bvths[s])
    &\leq \exp(-\alpha\mu/2 )\tPsi(\bvths[s-1])+\tC_5\sqrt{\eta \log \frac{R}{\eta\delta}}\\
    & \leq  \exp(-\alpha\mu s/2 )\tPsi(\bvths[0])+\frac{1}{1-\exp(-\alpha \mu /2)}\tC_5\sqrt{\eta \log \frac{R}{\eta\delta}},
\end{align*} 
where the second inequality comes from the induction hypothesis. By \Cref{lemma:pre equiv}, 
\begin{align*}
    \normtwo{\bvths - \vphs}& \leq \frac{2}{\sqrt{2\mu }}\tPsi(\bvths)\\
    &\leq \frac{2}{\sqrt{2\mu}}\tPsi(\bvths[0])+ \frac{2}{\sqrt{2\mu}(1-\exp(-\alpha \mu /2))}\tC_5\sqrt{\eta \log \frac{R}{\eta \delta}}\\
    &\leq \frac{1}{2}\epsilon_0+ \frac{2}{\sqrt{2\mu}(1-\exp(-\alpha \mu /2))}\tC_5\sqrt{\eta \log \frac{R}{\eta \delta}}.
\end{align*}
Here, the last inequality uses $\tPsi(\bvths[0])\leq \sqrt{\frac{\rho_2}{2}}\normtwo{\bvths-\vphs[0]}\leq \frac{1}{2}\sqrt{\frac{\mu}{2}}\epsilon_0$. Hence, when $\eta$ is sufficiently small, $\bvths \in \Gaz$. Still by \Cref{lemma: evolve tpsi}, $\bvths_{k, t}\in \Gat$ and 
\begin{align*}
    \tPsi(\vths_{k, t})\leq \tPsi(\bvths)+\tC_5 \sqrt{\eta \log \frac{R}{\eta\delta}}.
\end{align*}
Finally, according to \Cref{lemma: concen tvZ}, \eqref{eq:tvZR} holds with probability at least $1-\delta$. 

\end{proof}

The following corollary is a direct consequence of \Cref{lemma:tpsi whole} and \Cref{lemma:pre equiv}.
\begin{corollary}\label{coro:tpsi greater}
Let $s_1: = \ceil{\frac{20}{\alpha\mu}\log \frac{1}{\eta}}$. Given $\normtwo{\bvths[0]-\vphs[0]}\leq \frac{1}{2}\sqrt{\frac{\mu}{\rho_2}}\epsilon_0$, for $\delta=\cO(\poly(\eta))$, with probability at least $1-\delta$,
\begin{align}\label{eq:tpsi desce}
    \tPsi(\bvths[s_1])
    & \leq \tC_6\sqrt{\eta \log \frac{1}{\eta \delta}}, \quad \normtwo{\bvths[s_1]-\vphs[s_1]}\leq \tC_6 \sqrt{\eta \log \frac{1}{\eta\delta}},
\end{align}
where $\tC_6$ is a constant.
\end{corollary}
\begin{proof}
Substituting in $R=s_1$ to \Cref{lemma:tpsi whole} and applying $\normtwo{\bvths[s_1]-\vphs}\leq \sqrt{\frac{2}{\mu}}\tPsi(\bvths[s_1])$ for $\bvths[s_1]\in \Gaz$,  we have the lemma.
\end{proof}

Finally, we provide a high probability bound for the change  of the projection on the manifold after $s_1$ rounds $\normtwo{\vphs[s_1]-\vphs[0]}$.
\begin{lemma}\label{lemma:dvphs R0}
Let $s_1: = \ceil{\frac{20}{\alpha\mu}\log \frac{1}{\eta}}$. Given $\normtwo{\bvths[0]-\vphs[0]}\leq \frac{1}{2}\sqrt{\frac{\mu}{\rho_2}}\epsilon_0$. For $\delta=\cO(\poly(\eta))$, with probability at least $1-\delta$,
\begin{align*}
    \normtwo{\vphs[s_1]-\vphs[0]}\leq \tC_{8}\log\frac{1}{\eta}\sqrt{\eta \log \frac{1}{\eta\delta}}.
\end{align*}
\end{lemma}
\begin{proof}
From \Cref{lemma:tpsi whole}, for $\delta = \cO(\poly(\eta))$, when
\begin{align}\label{eq:tvZs1}
 \normtwo{\tvZs_{k, t}}\leq \tC_1 \sigmax\sqrt{\frac{2 \alpha}{\eta}\log \frac{2 s_1 \alpha K}{\eta \delta}}, \qquad \forall k \in [K], 0\leq t \leq H, 0 \leq s <s_1,
  \end{align}
then $\bvths \in \Gaz$, for all $ 0\leq s \leq s_1$. By the definition of $\Gaz$,  $\tvus_t \in \Gao$ , for all $0\leq t \leq H, 0 \leq s \leq s_1$. By triangle inequality, $\normtwo{\vphs[s_1]-\vphs[0]}$ can be decomposed as follows.
\begin{align}
\normtwo{\vphs[s_1]-\vphs[0]}&\leq \sum_{s=0}^{s_1-1}\normtwo{\vphs[s+1]-\vphs}\notag\\
& \leq \sum_{s=0}^{s_1-1} \normtwo{\Phi(\tvus_H)-\Phi(\tvus_0)}+ \sum_{s=0}^{s_1-1} \normtwo{\Phi(\bvths[s+1])-\Phi(\tvus_H)}
\label{eq:dvphs}.
\end{align}
By \Cref{lemma:tdelta}, when  \eqref{eq:tvZs1} hold , then for all $0\leq s <s_1-1$,
\begin{align*}
    \normtwo{\bvths[s+1]-\tvus_H}\leq \tC_3 \sqrt{\eta \log \frac{s_1}{\eta\delta}}.
\end{align*}
This implies that $\bvths[s+1]\in B^{\epsilon_1}(\tvus_H)$. Since for all $\vtheta\in \Gat$,  $\normtwo{\partial \Phi(\vtheta)}\leq \nu_1$,  then $\Phi(\cdot)$ is $\nu_1$-Lipschitz in $B^{\epsilon_1}(\tvus_H)$. This gives
\begin{align}
    \normtwo{\Phi(\bvths[s+1])-\Phi(\tvus_H)}&\leq \nu_1  \normtwo{\bvths[s+1]-\tvus_{H}}\notag\\
    & \leq \nu_1 \tC_3 \sqrt{\eta \log \frac{s_1}{\eta \delta}}.\label{eq:term1}
\end{align}
Then we analyze $\normtwo{\bvths[s+1]-\tvus_H}$. By \Cref{lemma:movement} and the definition of $\Gaz$ and $\Gao$, there exists $\vphi\in \Gamma$ such that $\tvus_t \in B^{\epsilon_1}(\vphi)$, $\forall 0 \leq t \leq H$. Therefore, we can expand $\Phi(\tvus_{t+1})$ as follows:
\begin{align*}
    \Phi(\tvus_{t+1})&=\Phi(\tvus_t - \eta \nabla \cL(\tvus_t))\\
    & = \Phi(\tvus_t)-\eta\partial\Phi(\tvus)\nabla \cL(\vus_t)+\frac{\eta^2}{2}\partial^2\Phi(\hvus_t)[\nabla \cL(\tvus_t), \nabla \cL(\tvus_t)]\\
    &=\Phi(\tvus_t) + \frac{\eta^2}{2}\partial^2\Phi\left( \cs_t \tvus_t + (1-\cs_t)\tvus_{t+1}\right)[\nabla \cL(\tvus_t), \nabla \cL(\tvus_t)],
\end{align*}
where $\cs_t \in (0, 1)$. Then we have
\begin{align*}
    \normtwo{\Phi(\tvus_{H})-\Phi(\tvus_0)}&\leq \frac{\eta^2}{2}\sum_{t=0}^{H-1}\normtwo{\partial^2\Phi(\left( \cs_t \tvus_t + (1-\cs_t)\tvus_{t+1}\right))[\nabla \cL(\tvus), \nabla \cL(\tvus_t)]}\\
    & \leq \frac{\eta^2}{2}\nu_2 \sum_{t=0}^{H-1}\normtwo{\nabla \cL(\tvus_t)}^2.
\end{align*}
By \Cref{lemma:gd2}, $\frac{\eta}{2}\normtwo{\nabla \cL(\tvus_t)}^2 \leq \cL(\tvus_t)-\cL(\tvus_{t+1})$. Therefore, 
\begin{align}
     \normtwo{\Phi(\tvus_{H})-\Phi(\tvus_0)}&\leq \eta \nu_2 (\cL(\tvus_0)-\cL(\tvus_H))\notag\\
     &\leq \eta \nu_2 [\tPsi(\bvths)]^2\notag\\
     & \leq \nu_2 \eta \left[2 \exp(-\alpha s\mu)\tPsi(\bvths[0])+ \frac{\tC_5^2\eta}{(1-\exp(-\alpha \mu/2))^2} \log\frac{s_1}{\eta\delta}\right],\label{eq:term2}
\end{align}
where the last inequality uses Cauchy-Schwartz inequality and \Cref{lemma:tpsi whole}. Summing up   \eqref{eq:term2} , we obtain
\begin{align}
    \sum_{s=0}^{s_1-1}\normtwo{\Phi(\tvus_H)-\Phi(\tvus_0)}&\leq \nu_2 \eta\left[ 2\tPsi(\bvths[0])\sum_{s=0}^{s_1-1}\exp(-\alpha \mu s)+\frac{s_1\tC_5^2\eta}{(1-\exp(-\alpha \mu/2))^2} \log\frac{s_1}{\eta\delta}\right]\notag\\
    & \leq \tC_{7}\eta \log \frac{1}{\eta}\log \frac{1}{\eta\delta}\label{eq:term22},
\end{align}
where $\tC_{7}$ is a constant.
Substituting \eqref{eq:term1} and \eqref{eq:term22} into \eqref{eq:dvphs}, for sufficiently small $\eta$, we have
\begin{align*}
    \normtwo{\vphs[s_1]-\vphs[0]}&\leq \nu_1  \tC_3 s_1 \sqrt{\eta \log \frac{s_1}{\eta \delta}}+\tC_{7}\eta \log \frac{1}{\eta}\log \frac{1}{\eta\delta}\\
    & \leq \tC_{8}\log \frac{1}{\eta}\sqrt{\eta \log \frac{1}{\eta\delta}},
\end{align*}
where $\tC_{8}$ is a constant. Finally, according to \Cref{lemma: concen tvZ}, \eqref{eq:tvZs1} holds with probability at least $1-\delta$. 
\end{proof}

\subsection{Phase 2: Iterates Staying Close to Manifold}\label{subsec:phase2}
In this subsection, we show that $\normtwo{\vxs_{k, t}}=\ctO(\sqrt{\eta})$  and $\normtwo{\bvths[s+r]-\bvths}=\ctO(\eta^{0.5-0.5\beta})$, $\forall 0\leq r \leq \Rg$ with high probability. 
\subsubsection{Additional notations}
Before presenting the lemmas, 
we define the following martingale $\{\scrMs_{k, t}\}^{H}_{t=0}$ that will be useful in the proof:
$$\scrMs_{k, t}:=\sum_{\tau=0}^{t-1}\vzs_{k,\tau}, \quad \scrM_{k,0}=\vzero.$$ 
We also define $\tmP:\R^d \to \R^{d\times d}$ as an extension of $\partial \Phi$:
\begin{align*}
    \tmP(\vtheta):=\begin{cases} \partial \Phi(\vtheta), & \text{if\ } \vtheta \in \Gat ,\\
\vzero, & \text{otherwise}.
\end{cases}
\end{align*}
Finally, we define a martingale $\{\vZs_t: s\geq 0, 0\leq t \leq H\}$:
\begin{align*}
    \vZs_t:=\frac{1}{K}\sum_{k\in [K]}\sum_{r=0}^{s-1}\sum_{\tau=0}^{H-1}\tmP(\bvths[r])\vzs[r]_{k,t}+\frac{1}{K}\sum_{k\in [K]}\sum_{\tau=0}^{t-1}\tmP(\bvths)\vzs_{k,t},\quad \vZs[0]_0=\vzero.
\end{align*}

\subsubsection{Proof for the High Probability Bounds}
A direct application of Azuma-Hoeffding's inequality yields the following lemma.

\begin{lemma}[Concentration property of $\scrMs_{k, t}$]\label{lemma:concen m}
With probability at least $1-\delta$, the following holds:
\begin{align*}
    \normtwo{\scrMs_{k,t}} \leq \tC_9 \sqrt{\frac{1}{\eta} \log \frac{1}{\eta\delta}}, \quad \forall 0 \leq t \leq H, k\in[K], 0\leq s <\Rg,
\end{align*}
where $\tC_9$ is a constant.
\end{lemma}
\begin{proof}
Notice that $\normtwo{\scrMs_{k,t+1}-\scrMs_{k,t}}\leq \sigmax$. Then by Azuma-Hoeffdings inequality,
\begin{align*}
    \PP(\normtwo{\scrMs_{k,t}}\geq \epsilon') \leq 2\exp\left(-\frac{\epsilon'^2}{2t \sigmax^2}\right).
\end{align*}
Taking union bound on $K$ clients, $H$ local steps and $\Rg$ rounds, we obtain that the following inequality holds with probability at least $1-\delta$:
\begin{align*}
    \normtwo{\scrMs_{k,t}} \leq \sigmax\sqrt{2 H\log \frac{2K H\Rg }{\delta}}, \quad \forall 0 \leq t \leq H, k\in[K], 0\leq s <\Rg.
\end{align*}
Substituting in  $H=\frac{\alpha}{\eta}$ and $\Rg=\floor{\frac{1}{\alpha \eta^{\beta}}}$ yields the lemma.
\end{proof}
Again applying Azuma-Hoeffding's inequality, we have the following lemma about the concentration property of $\vZs_t$.
\begin{lemma}[Concentration property of $\vZs_t$]\label{lemma:concen Z}
With probability at least $1-\delta$, the following inequality holds:
\begin{align*}
    \normtwo{\vZs_{H}}\leq \tC_{12}\eta^{-0.5-0.5\beta}\sqrt{\log \frac{1}{\eta\delta}}, \quad \forall 0 \leq s < \Rg.
\end{align*}
\end{lemma}
\begin{proof}
Notice that $\normtwo{\vZs_{t+1} - \vZs_t}\leq \nu_2 \sigmax, \forall 0 \leq t \leq H-1$ and $\normtwo{\vZs[s+1]_0-\vZs[s]_H}\leq \nu_2 \sigmax$. By Azuma-Hoeffding's inequality, 
\begin{align*}
    \PP(\normtwo{\vZs_t}\geq \epsilon')\leq 2 \exp\left(-\frac{\epsilon'^2}{2(sH+t)\nu_2^2\sigmax^2}\right).
\end{align*}
Taking union bound on $\Rg$ rounds, we obtain that the following inequality holds with probability at least $1-\delta$:
\begin{align*}
    \normtwo{\vZs_H}\leq \sigmax \nu_2 \sqrt{2 H \Rg \log \frac{2\Rg}{\delta}}, \quad \forall 0 \leq s <\Rg.
\end{align*}
Substituting in $H=\frac{\alpha}{\eta}$ and $\Rg=\floor{\frac{1}{\alpha\eta^{\beta}}}$ yields the lemma.
\end{proof}

We proceed to present a direct corollary of \Cref{lemma:tpsi whole} which provides a bound for the potential function over $\Rg$ rounds.
\begin{lemma}\label{lemma: tpsi grp}
Given $\normtwo{\bvths[0]-\vphs[0]}\leq C_0 \sqrt{\eta \log \frac{1}{\eta}}$ where $C_0$ is a constant, then for $\delta=\cO(\poly(\eta))$, with probability at least $1-\delta$, 
\begin{align}\label{eq:avg1}
  \bvths\in \Gaz, \quad  \tPsi(\bvths)\leq C_1\sqrt{\eta \log \frac{1}{\eta\delta}}, \quad \forall 0 \leq s < \Rg,
\end{align}
and 
\begin{align}\label{eq:indi1}
    \bvths_{k,t}\in \Gat,\quad  \tPsi(\bvths_{k,t})\leq C_1 \sqrt{\eta \log\frac{1}{\eta\delta}}, \quad \forall 0 \leq s < \Rg, 0\leq t\leq H, k \in [K],
\end{align}
where $C_1$ is a constant that can depend on $C_0$.
\end{lemma}
Furthermore, 
\begin{align*}
    \tPsi(\bvths[\Rg])\leq \tC_{10} \sqrt{\eta \log \frac{1}{\eta\delta}},
\end{align*}
where $\tC_9$ is a constant independent of $C_0$.
\begin{proof}
By $\rho_2$-smoothness of $\cL$, $\tPsi(\bvths[0]) \leq C_0\sqrt{\tfrac{\eta\rho_2}{2}\log \frac{1}{\eta}}$.
Substituting $\Rg=\floor{\tfrac{1}{\alpha\eta^{\beta}}}$ and  $\tPsi(\bvths[0]) \leq C_0\sqrt{\tfrac{\eta\rho_2}{2}\log \frac{1}{\eta}}$ into \Cref{lemma:tpsi whole}, for $\delta=\cO(\poly(\eta))$, with probability at least $1-\delta$, \eqref{eq:avg1} and \eqref{eq:indi1} where $C_1$ is a constant that can depend on $C_0$. 

Furthermore, for round $\bvths[\Rg]$, 
\begin{align*}
    \tPsi(\bvths[\Rg])\leq \exp(-\cO(\eta^{-\beta}))+\frac{1}{1-\exp(-\alpha\mu /2)}\tC_5\sqrt{\eta\log\frac{\Rg}{\eta\delta}}
   \leq \tC_{10} \sqrt{\eta \log\frac{1}{\eta\delta}},
\end{align*}
where $\tC_9$ is a constant independent of $C_0$.
\end{proof}

\begin{lemma}\label{lemma:bound x}
Given $\normtwo{\bvths[0]-\vphs[0]}\leq C_0 \sqrt{\eta \log \frac{1}{\eta}}$ where $C_0$ is a constant, then for $\delta=\cO(\poly(\eta))$, with probability at least $1-\delta$, for all $ 0\leq s_0 < \Rg, 0\leq t \leq H$, $k\in[K]$,
\begin{align*}
    \normtwo{\vxs_{k, t}}&
    \leq C_{2}\sqrt{\eta \log \frac{1}{\eta\delta}}, \quad  \normtwo{\bvxs_{H}} \leq C_{2}\sqrt{\eta \log \frac{1}{\eta\delta}},\\
    \normtwo{\bvths_{k,t}-\bvths}&\leq C_2\sqrt{\eta \log\frac{1}{\eta\delta}},  \quad \normtwo{\bvths[s+1]-\bvths}\leq C_2\sqrt{\eta\log\frac{1}{\eta\delta}}. 
\end{align*}
where $C_2$ is  a constant that can depend $C_0$.
Furthermore, 
\begin{align*}
    \normtwo{\bvths[\Rg]-\vphs[\Rg]}\leq \tC_{11}\sqrt{\eta \log \frac{1}{\eta\delta}},
\end{align*}
where $\tC_{11}$ is a constant independent of $C_0$.
\end{lemma}

\begin{proof}
Decomposing $\vxs_{k,t}$ by triangle inequality, we have
\begin{align*}
    \normtwo{\vxs_{k,t}}&\leq \normtwo{\vths_{k,t}-\bvths}+\normtwo{\bvths - \vphs}.
\end{align*}
We first bound $\normtwo{\bvths-\vphs}$. By \Cref{lemma: tpsi grp}, for $\delta=\cO(\poly(\eta))$, with probability at least $1-\frac{\delta}2$, 
\begin{align}
 \tPsi(\bvths)&\leq C_1\sqrt{\eta \log \frac{2}{\eta\delta}}, \forall 0 \leq s < \Rg, \label{eq:1potential}\\
    \tPsi(\vths_{k,t})&\leq C_1 \sqrt{\eta \log \frac{2}{\eta\delta}},\quad  \forall 0 \leq s < \Rg, 0\leq t \leq H, \label{eq:2potential}
\end{align}
and 
\begin{align}\label{eq:3potential}
    \tPsi(\bvths[\Rg])\leq \tC_{10}\sqrt{\eta \log \frac{2}{\eta \delta}}, 
\end{align}
where $C_2$ is a constant that may depend on $C_0$ and $\tC_{10}$ is a constant independent of $C_0$. When \eqref{eq:1potential} and \eqref{eq:3potential} hold, by \Cref{lemma:pre equiv},
\begin{align}
    \normtwo{\bvths-\vphs}&\leq \sqrt{\frac{2}{\mu}}\tPsi(\bvths)\leq C_1 \sqrt{\frac{2\eta}{\mu}\log \frac{2}{\eta\delta}},\label{eq:vths-vphs}\\
    \normtwo{\bvths[\Rg]-\vphs[\Rg]}&\leq \sqrt{\frac{2}{\mu}}\tPsi(\bvths[\Rg])\leq \tC_{10}\sqrt{\frac{2\eta}{\mu}\log \frac{2}{\eta\delta}}.\label{eq:vths-vphsgrp}
\end{align}
Then we bound $\normtwo{\vths_{k,t}-\bvths}$. By the update rule, we have
\begin{align*}
    \vths_{k, t}=\bvths - \eta \sum_{\tau=0}^{t-1}\nabla \cL(\vths_{k, \tau})-\eta \sum_{\tau=0}^{t-1}\vzs_{k, \tau}=\bvths - \eta \sum_{\tau=0}^{t-1}\nabla \cL(\vths_{k, \tau})-\eta \scrMs_{k,t}.
\end{align*}
Still by triangle inequality, we have
\begin{align*}
    \normtwo{\vths_{k,t}-\bvths}& \leq \eta \sum_{\tau=0}^{t-1}\normtwo{\nabla \cL(\vths_{k,\tau})}+\eta \normtwo{\scrMs_{k,t}}.
\end{align*}
Due to $\rho_2$-smoothness of $\cL$, when \eqref{eq:2potential} holds, 
\begin{align}\label{eq:gradient}
\normtwo{\nabla \cL(\vths_{k,\tau})}\leq \sqrt{2\rho_2}\tPsi(\vths_{k,\tau})\leq C_1\sqrt{2\rho_2\eta \log \frac{2}{\eta\delta}}.
\end{align}
By \Cref{lemma:concen m}, with probability at least $1-\frac{\delta}{2}$, 
\begin{align}\label{eq:m}
    \normtwo{\scrMs_{k,t}} \leq \tC_9 \sqrt{\frac{1}{\eta} \log \frac{2}{\eta\delta}}, \quad \forall 0 \leq t \leq H, k\in[K], 0\leq s <\Rg.
\end{align}
Combining \eqref{eq:gradient} and \eqref{eq:m}, when \eqref{eq:2potential} and \eqref{eq:3potential} hold simultaneously, there exists a constant $C_3$ which can depend on $C_0$ such that
\begin{align}\label{eq:vthskt}
    \normtwo{\vths_{k,t}-\bvths}\leq C_3\sqrt{\eta \log \frac{1}{\eta\delta}}, \quad \forall k \in [K], 0 \leq t \leq H.
\end{align}
By triangle inequality, 
\begin{align*}
    \normtwo{\bvths[s+1]-\bvths}\leq C_3\sqrt{\eta \log \frac{1}{\eta\delta}}.
\end{align*}
Combining \eqref{eq:vths-vphs}, \eqref{eq:vths-vphsgrp} and \eqref{eq:vthskt}, we complete the proof.
\end{proof}

Then we provide high probability bounds for the movement of $\vphs[s]$ within $\Rg$ rounds.
\begin{lemma}\label{lemma:delta phi bound}
Given $\normtwo{\bvths[0]-\vphs[0]}\leq C_0 \sqrt{\eta \log \frac{1}{\eta}}$ where $C_0$ is a constant, then for $\delta=\cO(\poly(\eta))$, with probability at least $1-\delta$,
\begin{align*}
    \normtwo{\vphs[s]-\vphs[0]}\leq C_4\eta^{0.5-0.5\beta}\sqrt{\log \frac{1}{\eta \delta}}, \quad \forall 1 \leq s \leq \Rg.
\end{align*}
where $C_4$ is a constant that can depend on $C_0$.
\end{lemma}
\begin{proof}
By the update rule of Local SGD, 
\begin{align*}
    \vths_{k, H}&=\bvths - \eta \sum_{t=0}^{H-1}\nabla \cL(\vths_{k, t})-\eta \sum_{t=0}^{H-1}\vzs_{k, t}
 \end{align*}   
Averaging among $K$ clients gives
\begin{align*}
\bvths[s+1]&=\bvths -  \frac{\eta}{K}\sum_{t=0}^{H-1}\sum_{k\in[K]}\nabla \cL(\vths_{k, t})-\frac{\eta}{K}\sum_{t=0}^{H-1}\sum_{k\in[K]}\vzs_{k, t}.
\end{align*}
By \Cref{lemma:bound x}, for $\delta=\cO(\poly(\eta))$, the following holds with probability at least $1-\delta/3$, 
\begin{align}
    \normtwo{\vths_{k,t}-\bvths}&\leq C_2 \sqrt{\eta \log \frac{3}{\eta \delta}}, \  \vths_{k,t} \in B^{\epsilon_0}(\vphs), \  \forall 0 \leq s <\Rg , 0 \leq t \leq H, k\in[K],\label{eq:delta theta1}\\
   \normtwo{\bvths[s+1]-\bvths}&\leq C_2 \sqrt{\eta \log \frac{3}{\eta \delta}}, \quad \bvths, \bvths[s+1] \in B^{\epsilon_0}(\vphs),\quad  \forall 0\leq s < \Rg \label{eq:delta theta2}.
\end{align}

When \eqref{eq:delta theta1} and \eqref{eq:delta theta2} hold,  we can expand $\Phi(\bvths[s+1])$ as follows: 
\begin{align*}
    \vphs[s+1]& = \vphs +\partial \Phi(\bvths)(\bvths[s+1]-\bvths)
    +\frac{1}{2}\partial^2 \Phi(\tvths)[\bvths[s+1]-\bvths, \bvths[s+1]-\bvths]\\
    &=\vphs \underbrace{-\frac{\eta}{K} \sum_{t=0}^{H-1}\sum_{k\in[K]}\partial \Phi(\bvths) \nabla \cL(\vths_{k, t})}_{\cTs_1}  \underbrace{-
  \frac{\eta}{K}\partial \Phi(\bvths) \sum_{t=0}^{H-1}\sum_{k\in[K]} \vzs_{k, t}}_{\cTs_{2}}\\
    &\quad+\underbrace{\frac{1}{2}\partial^2 \Phi(\as\bvths+(1-\as)\bvths[s+1])[\vths[s+1]-\vths, \vths[s+1]-\vths]}_{\cTs_3},
\end{align*}
where $\as\in (0, 1)$. Telescoping from round 0 to $s-1$,   we have
\begin{align*}
   \normtwo{\vphs - \vphs[0]} &= \sum_{r=0}^{s-1}\cTs[r]_1+\sum_{r=0}^{s-1}\cTs[r]_2 + \sum_{r=0}^{s-1}\cTs[r]_3.
\end{align*}

From \eqref{eq:delta theta2}, we can bound $\normtwo{\cTs_3}$ by $\normtwo{\cTs_3}\leq \frac{1}{2}\nu_2 C_2 ^2 \eta \log \frac{3}{\eta\delta}$. We proceed to bound $\normtwo{\cTs_1}$. When \eqref{eq:delta theta1} and \eqref{eq:delta theta2} hold,  we have
\begin{align*}
    \partial \Phi(\bvths) \nabla \cL(\vths_{k, t})&=\partial\Phi(\vths_{k, t})\nabla \cL(\vths_{k,t})+\partial^2 \Phi(\hvths_{k, t})[\vths_{k, t}-\bvths,\nabla \cL(\vths_{k,t})]\\
    &=\partial^2 \Phi(\bs_{k, t} \bvths + (1-\bs_{k, t})\hvths_{k, t})[\vths_{k, t}-\bvths,\nabla \cL(\vths_{k,t})], 
\end{align*}
where $\bs_{k,t}\in (0,1)$. By \Cref{lemma:tpsi whole}, with probability at least $1-\delta/3$, the following holds:
\begin{align}\label{eq:T2_phase}
    \normtwo{\nabla \cL(\vths_{k,t})}\leq \sqrt{2\rho_2}\tPsi(\vths_{k,t})\leq C_1\sqrt{2\rho_2\eta\log\frac{3}{\eta\delta}},  \forall k\in [K], 0\leq t \leq H, 0 \leq s < \Rg.
\end{align}
When \eqref{eq:delta theta1}, \eqref{eq:delta theta2} and \eqref{eq:T2_phase} hold simultaneously,  we have for all $0\leq s < \Rg$,
\begin{align*}
    \normtwo{\cTs_1}&\leq \frac{\eta \nu_2}{K}\sum_{t=0}^{H-1} \normtwo{\vths_{k,t}-\bvths}\normtwo{\nabla \cL(\vths_{k,t})}\\
    & \leq \frac{\alpha \nu_2\sqrt{2\rho_2}C_1 C_2}{K}\eta \log \frac{3}{\eta\delta}.
\end{align*}

Finally, we bound $\normtwo{\sum_{r=0}^{s-1} \cTs[r]_2}$. By \Cref{lemma:concen Z}, the following inequality holds with probability at least $1-\delta/3$:
\begin{align}
     \normtwo{\vZs_{H}}\leq \tC_{12}\eta^{-0.5-0.5\beta}\sqrt{\log \frac{3}{\eta\delta}}, \quad \forall 0 \leq s < \Rg.\label{eq:Z}
\end{align}
When  \eqref{eq:delta theta1}, \eqref{eq:delta theta2} and \eqref{eq:Z}  hold simultaneously, we have
\begin{align*}
    \normtwo{\sum_{r=0}^{s}\cTs[r]_2}=\eta\normtwo{\vZs[s]_H}\leq \tC_{12}\eta^{0.5-0.5\beta}\sqrt{\log \frac{3}{\eta\delta}}, \quad \forall 0\leq s < \Rg
\end{align*}
Combining the bounds for $\normtwo{\cTs_1}$, $\normtwo{\sum_{r=0}^{s}\cTs[r]_2}$ and $\normtwo{\cTs_3}$ and taking union bound, we obtain that for $\delta=\cO(\poly(\eta))$, the following inequality holds with probability at least $1-\delta$:
\begin{align*}
    \normtwo{\vphs-\vphs[0]}\leq C_4\eta^{0.5-0.5\beta}\sqrt{\log \frac{1}{\eta \delta}}, \quad \forall 1 \leq s \leq \Rg.
\end{align*}
where $C_4$ is a constant that can depend on $C_0$.
\end{proof}

\subsection{Summary of the dynamics and Proof of \Cref{thm:closeness,thm:change}}\label{sec:summary of high prob}
Based on the results in \Cref{subsec: phase 1} and \Cref{subsec:phase2}, we summarize the dynamics of Local SGD iterates and then present the proof of \Cref{thm:closeness,thm:change} in this subsection. For convenience, we first introduce the definition of \textbf{global step} and  \bm{$\delta$}\textbf{-good step}.
\begin{definition}[Global step]
Define $\cI$ as the index set $\{(s,t) : s\geq 0, 0\leq t \leq H\}$ with lexicographical order, which means $(s_1, t_1)\preceq (s_2, t_2)$ if and only if $s_1<s_2$ or ($s_1=s_2$ and $t_1\leq t_2$). A global step is indexed by $(s, t)$ corresponding to the $t$-th local step at round $s$. 
\end{definition}

\begin{definition}[$\delta$-good step]
In the training process of Local SGD, we say the global step $(s,t)\preceq (\Rtot, 0)$ is $\delta$-good if the following inequalities hold:
\begin{align*}
     \normtwo{\tvZs[r]_{k,\tau}}&\leq \exp(\alpha \rho_2) \sigmax\sqrt{2H\log\frac{6H \Rtot K }{\delta}}, \quad &\forall k\in[K], (r, \tau)\preceq (s, t), \\
      \normtwo{\scrMs[r]_{k,\tau}}&\leq \sigmax \sqrt{2H \log \frac{6KH\Rtot}{\delta}}, \quad &\forall k\in [K], (r, \tau)\preceq (s, t), \\
    \normtwo{\vZs[r]_H}&\leq \sigmax \nu_2 \sqrt{2H\Rg \log \frac{2\Rtot}{\delta}}, \quad &\forall 0\leq r< s.
\end{align*}
\end{definition}
Applying the concentration properties of $\tvZs[r]_{k,\tau}, \scrMs[r]_{k,\tau}$ and $\vZs[r]_H$ (Lemmas~\ref{lemma:concen Z}, \ref{lemma:concen m} and \ref{lemma: concen tvZ}) yields the following theorem.
\begin{theorem}\label{thm:good}
For $\delta=\cO(\poly(\eta))$, with probability at least $1-\delta$, all global steps $(s,t)\preceq (\Rtot,0)$ are $\delta$-good.
\end{theorem}
In the remainder of this subsection, we use $\cO(\cdot)$ notation to hide constants independent of $\delta$ and $\eta$. 

Below we present a summary of the dynamics of Local SGD when $\bvths[0]$ is initialized such that $\Phi(\bvths[0])\in \Gamma$ and all global steps are $\delta$-good. 
Phase 1 lasts for $s_0+s_1=\cO(\log \frac{1}{\eta})$ rounds. At the end of phase 1, the iterate reaches within $\cO( \sqrt{\eta\log \frac{1}{\eta\delta}})$ from  $\Gamma$, i.e.,  $\normtwo{\bvths[s_0+s_1]-\vphs[s_0+s_1]}=\cO(\sqrt{\eta\log \frac{1}{\eta\delta}})$. The change of the projection on manifold over $s_0+s_1$ rounds, $\normtwo{\vphs[s_1+s_0]-\vphs[0]}$, is bounded by $\cO(\log \tfrac{1}{\eta}\sqrt{\eta\log \tfrac{1}{\eta \delta}})$.

After $s_0+s_1$ rounds, the dynamic enters phase 2 when the iterates stay close to $\Gamma$ with $\bvths\in \Gat,  \forall s_0+s_1\leq s\leq \Rtot$ and $ \vths_{k,t}\in \Gat$, $\forall k\in[K], (s_0+s_1,0)\preceq (s, t)\preceq (\Rtot, 0)$. Furthermore,  $\normtwo{\vxs_{k,t}}$ and $\normtwo{\bvxs_H}$ satisfy the following equations:
\begin{align*}
    \normtwo{\vxs_{k,t}}&=\cO(\sqrt{\eta \log \tfrac{1}{\eta\delta}}), & \forall k\in[K], 0\leq t \leq H, s_0+s_1\leq s <\Rtot,\\
    \normtwo{\bvxs_{H}}&=\cO(\sqrt{\eta \log \tfrac{1}{\eta\delta}}), &\forall s_0+s_1 \leq s < \Rtot.
\end{align*}
 Moreover, for $s_0 + s_1 \leq s \leq \Rtot - \Rg$, the change of the manifold projection within $\Rg$ rounds can be bounded as follows:
 \begin{align*}
     \normtwo{\vphs[s+r]-\vphs[s]}=\cO(\eta^{0.5-0.5\beta}\sqrt{\log \frac{1}{\eta\delta}}), \quad \forall 1\leq r \leq \Rg.
 \end{align*}
 After combing through the dynamics of Local SGD iterates during the approaching and drift phase, we are ready to present the proof of \Cref{thm:closeness,thm:change}, which are direct consequences of the lemmas in Appendix~\ref{subsec: phase 1} and \ref{subsec:phase2}.


 \begin{proof}[Proof of \Cref{thm:closeness}]
 By Lemmas~\ref{lemma:s0}, \ref{lemma:bound x} and
 \Cref{coro:tpsi greater}, for $\delta=\cO(\poly(\eta))$, when all global steps are $\delta$-good, 
$\bvths\in \Gat,  \forall s_0+s_1\leq s\leq \Rtot$ and $ \vths_{k,t}\in \Gat$, $\forall k\in[K], (s_0+s_1,0)\preceq (s, t)\preceq (\Rtot, 0)$ and $\normtwo{\vxs_{k,t}}$, $\normtwo{\bvxs_H}$ satisfy the following equations:
\begin{align*}
  \normtwo{\vxs_{k,t}}&=\cO(\sqrt{\eta \log \tfrac{1}{\eta\delta}}), & \forall k\in[K], 0\leq t \leq H, s_0+s_1\leq s <\Rtot,\\
   \normtwo{\bvxs_{H}}&=\cO(\sqrt{\eta \log \tfrac{1}{\eta\delta}}), &\forall s_0+s_1 \leq s < \Rtot.
\end{align*}
 Hence $\normtwo{\bvxs[\Rtot]_0}=\cO(\tPsi(\bvths[\Rtot]))=
 \cO(\normtwo{\bvxs[\Rtot-1]_H})=\cO(\sqrt{\eta\log \tfrac{1}{\eta\delta}})$ by smoothness of $\cL$ and \Cref{lemma:pre equiv}. According to \Cref{thm:good}, with probability at least $1-\delta$, all global steps are $\delta$-good, thus completing the proof. 
 \end{proof}

 \begin{proof}[Proof of \Cref{thm:change}]
By \Cref{lemma:delta phi bound}, for $\delta=\cO(\poly(\eta))$, when all global steps are $\delta$-good,  then
$ \forall s_0+s_1\leq s\leq \Rtot-\Rg$, 
\begin{align*}
    \normtwo{\vphs[s+r]-\vphs[s]}=\ctO(\eta^{0.5-0.5\beta}), \quad \forall 0\leq r \leq \Rg.
\end{align*}
Also, by \Cref{lemma:dvphs R0}, when all global steps are $\delta$-good, the change of projection on manifold over $s_0+s_1$ rounds (i.e., Phase 1), $\normtwo{\vphs[s_0+s_1]-\vphs[0]}$ is bounded by $\ctO(\sqrt{\eta})$.  According to \Cref{thm:good}, with probability at least $1-\delta$, all global steps are $\delta$-good, thus completing the proof. 
\end{proof}

\subsection{Proof of \Cref{coro:alpha}}\label{sec:alpha proof}

In this subsection, we explicitly derive the  dependency of the approximation error on $\alpha$. The proofs are quite similar to those in \Cref{subsec: phase 1} and hence we only state the key proof idea for brevity. With the same method as the proofs in \Cref{sec: subphase 1}, we can show that with high probability, $\normtwo{\bvths-\vphs}\leq \frac{1}{2}\sqrt{\frac{\mu}{\rho_2}}$ after $s_0'=\cO(1)$ rounds. Below we focus on the dynamics of Local SGD thereafter. We first remind the readers of the definition of $\{\tvZ^{s}_{k,t}\}$:
\begin{align*}
\tvZs_{k, t}:=\sum_{\tau=0}^{t-1}\left(\prod_{l=\tau+1}^{t-1}(\mI-\eta \nabla^2 \cL(\tvus_{l}))\right)\vzs_{k, \tau},\qquad \tvZs_{k, 0}=\vzero.
\end{align*}
We have the following lemma that controls the norm of the matrix product $\prod_{l=\tau+1}^{t-1}(\mI-\eta \nabla^2 \cL(\tvus_{l}))$.
\begin{lemma}\label{lemma:prod}
Given $\bvths \in \Gaz$, then there exists a positive constant $C_3'$ independent of $\alpha$ such that for all $0\leq \tau < t \leq H$,
$$\lnormtwo{\prod_{l=\tau+1}^{t-1}(\mI-\eta \nabla^2 \cL(\tvus_{l}))}\leq C_3'.$$
\end{lemma}
\begin{proof}
Since $\bvths\in \Gaz$, then $\tvus_t \in \Gao$ for all $0\leq t \leq H$. We first bound the minimum eigenvalue of $\nabla^2\cL(\tvus_{t})$. Due to the PL condition, by \Cref{lemma:gd2}, for $\eta \leq \frac{1}{\rho_2}$, 
\begin{align*}
    \cL(\tvus_t)-\cL^*& \leq (1-\mu \eta)^{t}\left(\cL(\bvths)-\cL^*\right)\leq \exp(-\mu t\eta)(\cL(\bvths)-\cL^*), \quad \forall 0 \leq t \leq H.
\end{align*}
Therefore, 
\begin{align*}
    \tPsi(\tvus_t)\leq \exp(-\mu t \eta/2)\tPsi(\bvths).
\end{align*}
Let $C_1'=\rho_3\sqrt{\frac{\rho_2}{\mu}}$. By Weyl's inequality, 
\begin{align*}
    \abs{\lambda_{\min}(\nabla^2 \cL(\tvus_{t}))}&= \abs{\lambda_{\min}(\nabla^2 \cL(\tvus_{t}))-\lambda_{\min}(\nabla^2 \cL(\Phi(\tvus_t))}\\
    &\leq \rho_3\normtwo{\nabla^2 \cL(\tvus_{t})-\nabla^2 \cL(\Phi(\tvus_t))}\\
    & \leq \rho_3\normtwo{\tvus_t-\Phi(\tvus_t)}\\
    & \leq \rho_3\sqrt{\frac{2}{\mu}}\exp(-\mu t \eta/2)\tPsi(\bvths)\\
    & \leq C'_1\exp(-\mu t \eta/2)\epsilon_0, 
\end{align*}
where the last two inequalities use Lemmas~\ref{lemma:pre equiv} and \ref{lemma: tpsi lipschitz} respectively. Therefore, for all $0\leq t \leq H$ and $0\leq \tau\leq t-1$, 
\begin{align}
    \normtwo{\prod_{l=\tau+1}^{t-1}(\mI-\eta \nabla^2 \cL(\tvus_{l}))}&\leq \prod_{l=\tau+1}^{t-1} (1+\eta \abs{\lambda_{\min}\nabla^2 \cL(\tvus_l)})\notag\\
    &\leq  \prod_{l=0}^{\infty}(1+\eta \abs{\lambda_{\min}\nabla^2 \cL(\tvus_l)})\notag\\
    & \leq \exp(\eta \epsilon_0 C_1'\sum_{l=0}^{\infty}\exp(-\mu l \eta/2)).\label{eq:i-eta}
\end{align}
For sufficiently small $\eta$, there exists a constant $C'_2$ such that
\begin{align}
\sum_{l=0}^{\infty}\exp(-\mu l \eta/2))=\frac{1}{1-\exp(-\mu \eta / 2)}\leq \frac{C'_2}{\eta}.\label{eq:sum}
\end{align}
Substituting \eqref{eq:sum} into \eqref{eq:i-eta}, we obtain the lemma.
\end{proof}
Based on \Cref{lemma:prod}, we obtain the following lemma about the concentration property of $\tvZs_{k,t}$, which can be derived in the same way as \Cref{lemma: concen tvZ}.
\begin{lemma}\label{lemma: concen tvZ2}
 Given $\bvths\in\Gaz$ , then with probability at least $1-\delta$, 
 $$\normtwo{\tvZs[s]_{k, t}} \leq C_3'\sigmax\sqrt{\frac{2\alpha}{\eta}\log \frac{2  \alpha K}{\eta\delta}}, \qquad \forall  0\leq t\leq H, k\in[K],$$
 where $C_3' $ is defined in \Cref{lemma:prod}.
\end{lemma}
The following lemma can be derived analogously to \Cref{lemma:tdelta} but the error bound is tighter in terms of its dependency on $\alpha$.
\begin{lemma}
Given $\bvths \in \Gao$,  then  for $\delta = \cO(\poly(\eta))$, with probability at least $1-\delta$, there exists a constant $C'_4$ independent of $\alpha$ such that
\begin{align*}
    \normtwo{\vths_{k, t}-\tvus_t}\leq C'_4\sqrt{\alpha\eta \log \frac{\alpha}{\eta  \delta}},\quad \forall 0 \leq t \leq H, k\in [K],
\end{align*}
and 
\begin{align*}
    \normtwo{\bvths[s+1]-\tvus_{H}}\leq C'_4\sqrt{\alpha \eta \log \frac{\alpha}{\eta \delta}}.
\end{align*}
\end{lemma}
Then, similar to \Cref{lemma:tpsi whole}, we can show that for $\delta=\cO(\poly(\eta))$ and simultaneously all $s\geq s_0'+s_1'$ where $s_1'=\cO(\frac{1}{\alpha}\log \frac{1}{\eta})$, it holds with probability at least $1-\delta$ that $\normtwo{\bvths-\vphs}=\cO(\sqrt{\alpha\eta\log \frac{\alpha}{\eta\delta}})$. Note that to eliminate the dependency of the second term's denominator on $\alpha$ in \eqref{eq:tpsi individual}, we can discuss the cases of $\alpha >c_0$ and $\alpha<c_0$ respectively 
where $c_0$ can be an arbitrary positive constant independent of $\alpha$. For the case of $\alpha<c_0$ 
group $\ceil{\frac{c_0}{\alpha}}$ rounds together and repeat the arguments in this subsection to analyze the closeness between Local SGD and GD iterates as well as the evolution of loss.
\subsection{Computing the Moments for One ``Giant Step''}\label{sec:moments phase2}
In this subsection, we compute the first and second moments for the change of manifold projection every $\Rg$ rounds of Local SGD. Since the randomness in training might drive the iterate out of the working zone, making the dynamic intractable, we analyze a more well-behaved sequence $\{\hvths_{k, t}: (s,t)\preceq (\Rtot, 0),  k\in [K]\}$ which is equal to $\{\vths_{k,t}\}$ with high probability. 
Specifically, 
$\hvths_{k,t}$ equal to $\vths_{k,t}$ if the global step $(s,t)$ is $\eta^{100}$-good and is set as a point $\vphinull\in \Gamma$ otherwise. The formal definition is as follows.
\begin{definition}[Well-behaved sequence]\label{def:htheta}
    Denote by $\cEs_t$ the event $\stgood$. Define a well-behaved sequence 
$\hvths_{k, t}:=\vths_{k,t}\onec_{\cEs_t}+\vphinull\onec_{\bcEs_t}$
, which  satisfies the following update rule:
\begin{align}
     \hvths_{k, t+1}
    & =\vths_{k,t+1} \onec_{\cEs_{t+1}}+\vphinull\onec_{\bcEs_{t+1}}\\
    & = \hvths_{k, t}-\eta\nabla \cL(\hvths_{k,t})-\eta \vzs_{k, t}\underbrace{-\onec_{\bcEs_{t+1}}(\hvths_{k, t}-\eta\nabla \cL(\hvths_{k,t})-\eta \vzs_{k, t})+\onec_{\bcEs_{t+1}}\vphinull}_{:= \hves_{k, t}}.\label{eq:htheta}
\end{align}
\end{definition}

By \Cref{thm:good}, with probability at least $1-\eta^{100}$, $\hvths_{k, t}=\vths_{k, t}$,  $\forall k\in [K], (s, t)\preceq (\Rtot, 0)$. 
Similar to $\{\vths_{k,t}\}$, we define the following variables with respect to $\{\hvths_{k,t}\}$:
\begin{align*}
    \bhvths[s+1]&:=\frac{1}{K}\sum_{k\in[K]}\hvths_{k,H}, \quad \hvphs:=\Phi(\bhvths),
   \\
    \hvxs_{k,t}&:=\hvths_{k,t}-\hvphs,\quad \bhvxs{0}:=\bhvths-\hvphs, \quad\bhvxs{H}:=\frac{1}{K}\sum_{k\in[K]}\hvxs_{k,H}.
\end{align*}
Notice that $\hvxs_{k,0}=\bhvxs{0}$ for all $k\in[K]$. Finally, we introduce the following mapping $\mPsi(\vtheta):\Gamma \to \R^{d\times d}$, which is closely related to $ \widehat{\mPsi}$ defined in \Cref{main thm: flow}. 
\begin{figure}[t]
\begin{center}
    \vspace{-0.2in}
    \includegraphics[width=0.29\textwidth]{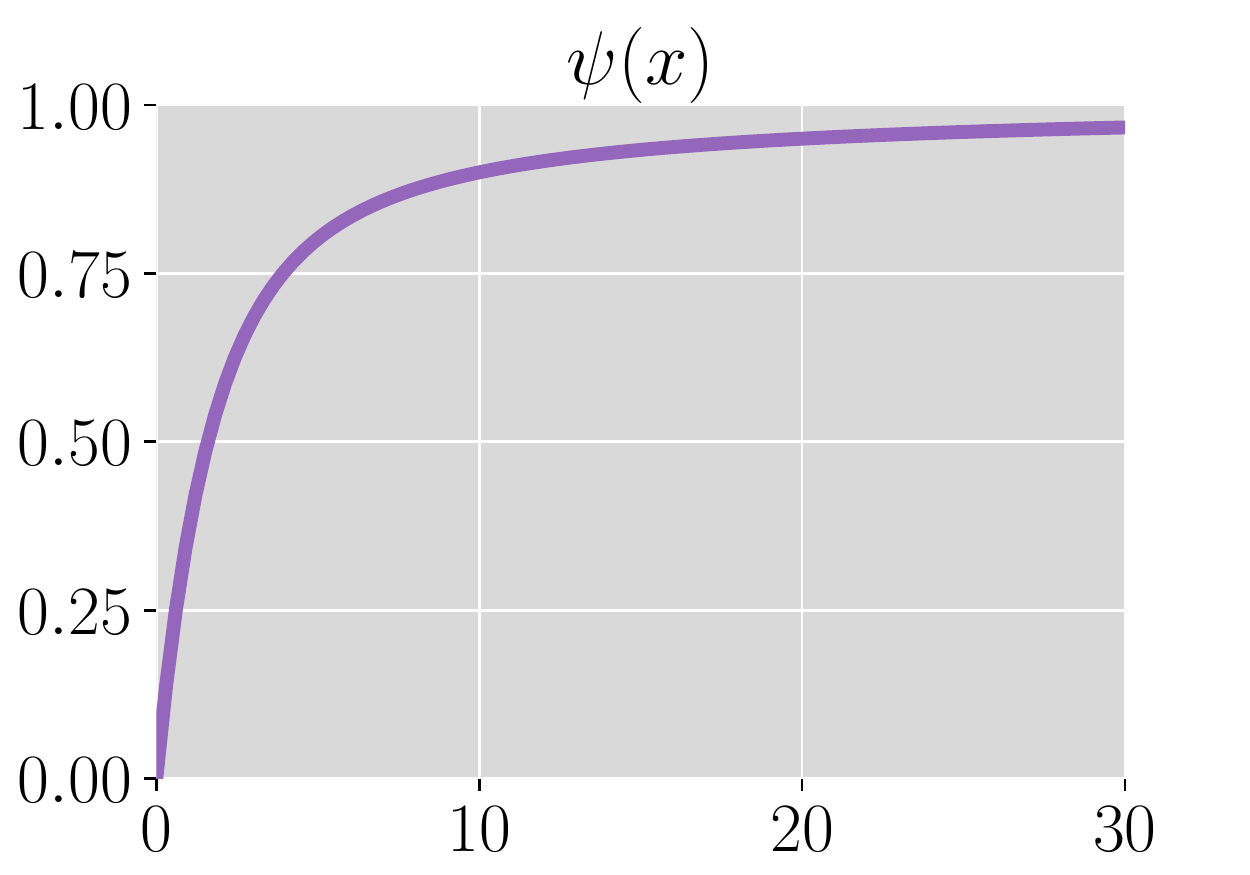}
    \caption{A plot of $\psi(x)$.}
    \label{fig:psi}
\end{center}
\end{figure}

\begin{definition}\label{def:mPsi}
For $\vtheta \in \Gamma$, we define the mapping $\mPsi(\vtheta):\Gamma \to \R^{d\times d}$:
    \begin{align*}
      \mPsi(\vtheta) = \sum_{i, j \in [d]} \psi(\eta H(\lambda_i + \lambda_j)) \left\langle\mSig(\vtheta), \vv_i \vv_j^\top \right\rangle \vv_i \vv_j^\top,
\end{align*}
where $\lambda_i, \vv_i$ are the $i$-th eigenvalue and eigenvector of $\nabla^2 \cL(\vtheta)$ and $\vv_i$'s form an orthonormal basis
of $\R^d$. 
Additionally, $\psi(x):=\frac{e^{-x}-1+x}{x}$ and $\psi(0) = 0$; see \Cref{fig:psi} for a plot. 
\end{definition}

\begin{remark}
Intuitively, $\mPsi(\vtheta)$ rescales the entries of $\mSig(\vtheta)$ in the eigenbasis of $\nabla^2 \cL(\vtheta)$. When $\nabla^2\cL(\vtheta)=\diag(\lambda_1, \cdots, \lambda_d) \in \R^{d \times d}$, where $\lambda_i = 0$ for all $m < i \le d$, $ \Psi(\mSig_0)_{i, j} =  \psi(\eta H (\lambda_i + \lambda_j)) \Sigma_{0, i, j}$.
Note that $\mPsi(\vtheta)$ can also be written as
\begin{align*}
    \vect(\mPsi(\vtheta))=\psi(\eta H(\nabla^2 \cL(\vtheta)\oplus \nabla^2 \cL(\vtheta) ))\vect(\mSig(\vtheta)),
\end{align*}
where $\oplus$ denotes the Kronecker sum $\mA \oplus \mB = \mA \otimes \mI_d +\mI_d\otimes \mB $, $\vect(\cdot)$ is the vectorization operator of a matrix and $\psi(\cdot)$ is interpreted as a matrix function.
\end{remark}

Now we are ready to present the result about the moments of $\hvphs[s+\Rg]-\hvphs$.
\begin{theorem}\label{thm: one step moment new}
For $s_0+s_1\leq s \leq \Rtot-\Rg$ and $0<\beta <0.5$, the first and second moments of $\hvphs[s+\Rg]-\hvphs$ are as follows:
\begin{align}
\begin{aligned}
    \E[\hvphs[s+\Rg] -\hvphs[s] \mid \hvphs, \cEs_{ 0}]&=\frac{\eta^{1-\beta}}{2B}\partial^2\Phi(\hvphs)[\mSig(\hvphs)+(K-1)\Psi(
    \hvphs)]\\
    &\quad +\ctO(\eta^{1.5-2\beta})+\ctO(\eta),\label{eq:s+grp1}
\end{aligned}
\end{align}
\begin{align}
    &\E[(\hvphs[s+\Rg] -\hvphs[s])(\hvphs[s+\Rg] -\hvphs[s])^{\top}\mid \hvphs, \cEs_{0}]=\frac{\eta^{1-\beta}}{B}\mSig_{\parallel}(\hvphs)+\ctO(\eta^{1.5-2\beta})+\ctO(\eta),\label{eq:s+grp2}
\end{align}
where $\ctO(\cdot)$ hides log terms and constants independent of $\eta$.
\end{theorem}
\begin{remark}\label{remark: good}
By \Cref{thm:good} and the definition of $\hvths_{k,t}$, \eqref{eq:s+grp1} and \eqref{eq:s+grp2} still hold when we replace $\hvphs$ with $\vphs$ and replace $\hvphs[s+\Rg]$ with $\vphs[s+\Rg]$.
\end{remark}
We shall have \Cref{thm: one step moment new} if we prove the following theorem, which directly gives \Cref{thm: one step moment new} with a simple shift of index. For brevity, denote by $\hDvphs:=\hvphs-\hvphs[0]$, $\mSig_0:=\mSig(\hvphs[0])$, $\mSigzpara:=\mSig_{\parallel}(\hvphs[0])$.
\begin{theorem}\label{thm: moments}
Given $\normtwo{\bhvths[0]-\hvphs[0]}=\cO(\sqrt{\eta \log \frac{1}{\eta}})$, for $0<\beta < 0.5$,  the first and second moments of $\hDvphs[\Rg]$ are as follows:
\begin{align*}
     \E[\hDvphs[\Rg]]&=\frac{\eta^{1-\beta}}{2B}\partial^2\Phi(\hvphs[0])[\mSig_0+(K-1)\mPsi(
    \hvphs[0])]+\ctO(\eta^{1.5-2\beta})+\ctO(\eta),\\
    \E[\hDvphs[\Rg]\hDvphsT[\Rg] ]&=\frac{\eta^{1-\beta}}{B}\mSigzpara+\ctO(\eta^{1.5-1.5\beta})+\ctO(\eta).
\end{align*}
\end{theorem}
We will prove \Cref{thm: moments} in the remainder of this subsection. For convenience, 
we introduce more notations that will be used throughout the proof. Let $\mH_0: = \nabla^2\cL(\hvphs[0])$. 
By Assumption~\ref{a:Gamma}, $\rank(\mH_0)=m$. WLOG, 
assume $\mH_0=\diag(\lambda_1, \cdots, \lambda_d) \in \R^{d \times d}$, where $\lambda_i = 0$ for all $m < i \le d$ 
and $\lambda_1\geq \lambda_2 \cdots \geq \lambda_m$. 
By \Cref{lemma:mani projection},
 $\partial \Phi(\hvphs[0])$ is the projection matrix onto the tangent space $T_{\hvphs[0]}(\Gamma)$ 
 (i.e. the null space of $\nabla^2 \cL(\hvphs[0])$) and therefore, $\partial\Phi(\hvphs[0])=\begin{bmatrix}\vzero & \vzero \\ \vzero & \mI_{d-m}\end{bmatrix}$.
 Let $\mPpara:=\partial \Phi(\hvphs[0])$ and $\mPperp:=\mI_d-\mPpara$.

 Let $\bhmAs:=\E[\bhvxs{H} \bhvxsT{H}]$, $\hvqs_{t}:=\E[\hvxs_{k,t}]$
 and $\hmBs_t:=\E[\hvxs_{k,t}\hDvphsT]$. The latter two notations are
 independent of $k$ since $\hvths_{1,t}, \dots, \hvths_{K,t}$ are identically distributed. 
 The following lemma computes the first and second moments of the change of manifold projection every round.
\begin{lemma}\label{lemma:recursion vphi}
    Given $\normtwo{\bhvths[0]-\hvphs[0]}=\cO(\sqrt{\eta \log \frac{1}{\eta}})$, 
    for $0\leq s < \Rg$, 
    the first and second moments of $\hvphs[s+1]-\hvphs$ are as follows:
    \begin{align}
        &\E[\hvphs[s+1]-\hvphs]=\mPpara \hvqs_{H} +\partial^2 \Phi(\hvphs[0])[ \hmBs_H]
        +\frac{1}{2} \partial^2 \Phi(\hvphs[0])[\bhmAs]+\ctO(\eta^{1.5-\beta}),\label{eq: one step first}\\
        &\E[(\hvphs[s+1]-\hvphs)(\hvphs[s+1]-\hvphs)^{\top}]
        =\mPpara \bhmAs \mPpara
        +\ctO(\eta^{1.5-0.5\beta}). \label{eq:one step second}
    \end{align}
    \end{lemma}
    \begin{proof} 
    By Taylor expansion, we have
    \begin{align*}
        \hvphs[s+1]
        &= \Phi\left( \hvphs + \bhvxs{H} \right)\\
        & = \hvphs + \partial \Phi(\hvphs)\bhvxs{H}+\frac{1}{2}\partial^2\Phi(\hvphs)[\bhvxs{H} \bhvxsT{H}]+\cO(\normtwo{\bhvxs{H}}^3)\\
        &=\hvphs[s]+\partial \Phi(\hvphs[0]+\hDvphs)\bhvxs{H}+\frac{1}{2}\partial^2 \Phi(\hvphs[0]+\hDvphs)[\bhvxs{H}\bhvxsT{H}]\\
        &\quad +\cO(\normtwo{\bhvxs{H}}^3)\\
        &=\hvphs[s]+ \mPpara \bhvxs{H} 
        +\partial^2 \Phi(\hvphs[0])[\bhvxs{H}\hDvphsT]
        +\frac{1}{2} \partial^2 \Phi(\hvphs[0])[\bhvxs{H}\bhvxsT{H}] \\
        &\quad + \cO(\normtwo{\hDvphs}^2 \normtwo{\bhvxs{H}} + \normtwo{\hDvphs} \normtwo{\bhvxs{H}}^2 + \normtwo{\bhvxs{H}}^3).
    \end{align*}
    Rearrange the terms and we obtain:
    \begin{align}\label{eq:recursion vphi}
        \begin{aligned}
           \hvphs[s+1] - \hvphs[s] &= \mPpara \bhvxs{H} 
           +\partial^2 \Phi(\hvphs[0])[\bhvxs{H}\hDvphsT]+
           \frac{1}{2} \partial^2 \Phi(\hvphs[0])[\bhvxs{H}\bhvxsT{H}] \\
            &\quad + \cO(\normtwo{\hDvphs}^2 \normtwo{\bhvxs{H}} + \normtwo{\hDvphs} \normtwo{\bhvxs{H}}^2 + \normtwo{\bhvxs{H}}^3).
        \end{aligned}
        \end{align}
    Moreover, 
    \begin{align}\label{eq:recursion second moment}
        \begin{aligned}
            (\hvphs[s+1]-\hvphs)(\hvphs[s+1]-\hvphs)^{\top}=\mPpara \bhvxs{H}\bhvxsT{H}\mPpara
            +\cO(\normtwo{\hDvphs}\normtwo{\bhvxs{H}}^2).
        \end{aligned}
    \end{align}
    Noticing that $\hvxs_{k,H}\hDvphsT$ are identically distributed for all $k\in [K]$,
    we have $\E[\bhvxs{H}\hDvphsT]=\frac{1}{K}\sum_{k\in [K]}\E[\hvxs_{k,H}\hDvphsT]=\hmBs_{H}$. 
Then taking expectation of both sides of~\eqref{eq:recursion vphi}
     gives
    \begin{align*}
        \E[\hvphs[s+1] - \hvphs[s]] 
        &= \mPpara \hvqs_{H} +\partial^2 \Phi(\hvphs[0])[ \hmBs_H]
        +\frac{1}{2} \partial^2 \Phi(\hvphs[0])[\bhmAs] \\
            &\quad + \cO(\E[\normtwo{\hDvphs}^2 \normtwo{\bhvxs{H}}] + \E[\normtwo{\hDvphs} \normtwo{\bhvxs{H}}^2]
             + \E[\normtwo{\bhvxs{H}}^3]).
    \end{align*}
    Again taking expectation of both sides of ~\eqref{eq:recursion second moment} yields
    \begin{align*}
        \E[(\hvphs[s+1]-\hvphs)(\hvphs[s+1]-\hDvphsT)]=\mPpara \bhmAs \mPpara
        +\cO(\E[\normtwo{\hDvphs}\normtwo{\bhvxs{H}}^2]).
    \end{align*}
    By Lemmas~\ref{lemma:bound x} and \ref{lemma:delta phi bound}, 
    the following holds simultaneously with probability at least $1-\eta^{100}$:
    \begin{align*}
        \normtwo{\hDvphs}=\ctO(\eta^{0.5-0.5\beta}), \quad 
        \normtwo{\bhvxs{H}}=\ctO(\eta^{0.5}).
    \end{align*}
    Furthermore, since for all $k\in [K]$
    and $(s, t)\preceq (\Rtot, 0)$, $\hvths_{k, t}$ stays in $\Gat$ which
    is a bounded set,  $\normtwo{\hDvphs}$ and $\normtwo{\bhvxs{H}}$
    are also bounded. Therefore, we have
    \begin{align}
        \E[\normtwo{\hDvphs}^2\normtwo{\bhvxs{H}}]&=\ctO(\eta^{1.5-\beta}),\label{eq:error1}\\
        \E[\normtwo{\hDvphs}\normtwo{\bhvxs{H}}^2]&=\ctO(\eta^{1.5-0.5\beta}),\\
        \E[\normtwo{\bhvxs{H}}^3]&=\ctO(\eta^{1.5}),\label{eq:error3}
    \end{align}
    which concludes the proof.
    \end{proof}
We compute $\bhmAs$, $\hvqs_t$ and $\hmBs_t$ by solving
a set of recursions, which is formulated in the 
following lemma. Additionally, define $\hmAs_t:=\E[\hvxs_{k,t}\hvxsT_{k,t}]$ and $\hmMs_t:=\E[\hvxs_{k,t}\hvxs_{k,l}], (k\neq l)$. 

\begin{lemma}\label{lemma:recursion}
    Given $\normtwo{\bhvths[0]-\hvphs[0]}=\cO(\sqrt{\eta \log \frac{1}{\eta}})$, 
    for $0\leq s < \Rg$ and $0\leq t < H$, we have the following recursions.
    \begin{align}
        \hvqs_{t+1}&=\hvqs_t-\eta \mH_0 \hvqs_t  -\eta \nabla^3 \cL(\vphs[0])[\hmBs_t]-\frac{\eta}{2}\nabla^3\cL(\vphs[0])[\hmAs_t]+\ctO(\eta^{2.5-\beta}),\label{eq:recur q}\\
         \hmAs_{t+1} &=\hmAs_{t} -\eta \mH_0 \hmAs_{t}  -\eta\hmAs_{t}  \mH_0+\frac{\eta^2}{\Bloc}\mSig_0+\ctO(\eta^{2.5-0.5\beta}),\label{eq:recur A}\\
        \hmMs_{t+1}&= \hmMs_{t}-\eta \mH_0  \hmMs_{t}-\eta  \hmMs_{t}\mH_0 +\ctO(\eta^{2.5-0.5\beta}),\label{eq:recur M}\\
         \hmBs_{t+1}&= (\mI-\eta \mH_0) \hmBs_t+\ctO(\eta^{2. 5-\beta}).\label{eq:recur B}
    \end{align}
    Moreover, 
    \begin{align}
        \bhmAs&=\frac{1}{K}\hmAs_H + (1-\frac{1}{K})\hmMs_H,\label{eq:recur A tran}\\
        {\hmMs[s+1]}_0 &= {\hmAs[s+1]}_0 =  \mPperp \bhmAs \mPperp + \cO(\eta^{1.5 -0.5 \beta}),\label{eq:recur M tran}\\
    \hvqs[s+1]_0&=\mPperp \hvqs_H -\partial^2 \Phi(\vphs[0])[\hmBs_H]-\frac{1}{2} \partial^2 \Phi(\vphs[0])[\bhmAs]+\ctO(\eta^{1.5-\beta}), \label{eq:recur q tran}\\
    \hmBs[s+1]_0& = \mPperp \hmBs_H + \mPperp \bhmAs \mPpara+\ctO(\eta^{1.5-\beta}). \label{eq:recur B tran}
    \end{align}
\end{lemma}

\begin{proof}
    We first derive the recursion for  $\hvqs_t$. Recall the update rule for $\hvths_{k,t}$:
    \begin{align*}
        \hvths_{k, t+1}&=\hvths_{k, t}-\eta \nabla \cL(\hvths_{k, t})-\eta \hvzs_{k, t}+\hves_{k, t}.
    \end{align*}
    Subtracting $\hvphs$ from both sides gives
    \begin{align}
        \hvxs_{k, t+1}&=\hvxs_{k, t}-\eta \nabla \cL(\hvths_{k, t})-\eta \hvzs_{k, t}+\cO(\normtwo{\hves_{k, t}})\notag\\
    &=\hvxs_{k, t}-\eta \left(\nabla^2 \cL(\hvphs)\hvxs_{k, t}+\frac 12\nabla^3 \cL(\hvphs)[\hvxs_{k, t}\hvxsT_{k, t}]+\cO(\normtwo{\hvxs_{k,t}}^3)\right)\notag\\
    &\quad -\eta \hvzs_{k, t}+\cO(\normtwo{\hves_{k, t}})\notag\\
        & = \hvxs_{k, t}-\eta \left(\nabla^2 \cL(\hvphs[0])+\nabla^3 \cL(\hvphs[0])\hDvphs+\cO(\|\hDvphs\|^2)\right)\hvxs_{k, t}\notag\\
        & \quad-\frac{\eta}{2}\left(\nabla^3 \cL(\hvphs[0])+\cO(\normtwo{\hDvphs})\right)[\hvxs_{k, t}\hvxsT_{kt}]-\eta \hvzs_{k, t}+\cO(\eta\normtwo{\hvxs_{k, t}}^3+\normtwo{\hves_{k, t}})\notag\\
        & =\hvxs_{k, t} - \eta \mH_0 \hvxs_{k, t}-\eta \nabla^3 \cL(\hvphs[0])[\hvxs_{k, t}\hDvphsT] - \frac{\eta}{2}\nabla^3 \cL(\hvphs[0])[\hvxs_{k, t}\hvxsT_{k, t}] -\eta \hvzs_{k, t}\notag\\
        & \quad +\cO(\eta\normtwo{\hvxs_{k, t}}^3+\eta \normtwo{\hDvphs}\normtwo{\hvxs_{k, t}}^2+\eta \normtwo{\hDvphs}^2\normtwo{\hvxs_{k, t}}+\normtwo{\hves_{k, t}}),\label{eq:recur x}
    \end{align}
where the second and third equality perform Taylor expansion. 
Taking expectation on both sides gives
\begin{align*}
    \hvqs_{t+1} &= (\mI-\eta\mH_0)\hvqs_{t}-\eta \nabla^3 \cL(\hvphs[0])[\hvqs_t]-\frac{\eta}{2}\nabla^3 \cL(\hvphs[0])[\hmAs_t]\\
    &\quad +\cO\left(\eta\E[\normtwo{\hvxs_{k,t}}^3] + \eta \E[\normtwo{\hDvphs}\normtwo{\hvxs_{k,t}}^2]+\eta \E[\normtwo{\hDvphs}^2 \normtwo{\hvxs_{k,t}}]+\E[\normtwo{\hves_{k, t}}]\right).
\end{align*}
By \Cref{thm:good}, with probability at least $1-\eta^{100}$, $\hves_{k,t}=\vzero$,  $\forall k\in [K], (s, t)\preceq (\Rg, 0)$. Also notice that both $\hvths_{k, t}$ and $\vphinull$ belong to the bounded set $\Gat$. Therefore, $\normtwo{\hves_{k, t}}$ is bounded and we have $\E[\normtwo{\hves_{k,t}}]=\cO(\eta^{100})$. Combining this with \eqref{eq:error1} to \eqref{eq:error3} yields \eqref{eq:recur q}.

Secondly, we derive the recursion for $\hmBs_t$. Multiplying both sides of~\eqref{eq:recur x} 
by $\hDvphsT$ and taking expectation, we have
\begin{align*}
    \hmBs_{t+1} &=  (\mI-\eta \mH_0)\hmBs_{t}+\cO(\eta \E[\normtwo{\hDvphs}\normtwo{\hvxs_{k, t}}^2 +\normtwo{\hDvphs}^2\normtwo{\hvxs_{k, t}}+\normtwo{\hves_{k,t}}]).
\end{align*}
Still by \Cref{thm:good} and~\eqref{eq:error1} to \eqref{eq:error3}, we have~\eqref{eq:recur B}. 

Thirdly, we derive the recursion for $\hmAs_t$. By \eqref{eq:recur x}, we have
\begin{align*}
    \hmAs_{t+1}&=\hmAs_t - \eta \mH_0 \hmAs_t - \eta\hmAs_t\mH_0 +\frac{\eta^2}{\Bloc}\mSig_0+\cO(\eta^2\E[\normtwo{\hDvphs}+\normtwo{\hvxs_{k,t}}])\\
    &\quad +\cO(\eta\E[\normtwo{\hvxs_{k,t}}^3+\normtwo{\hvxs_{k,t}}^2 \normtwo{\hDvphs}+\normtwo{\hves_{k,t}}])\\
    &=(\mI - \eta \mH_0) \hmAs_t + \frac{\eta^2}{\Bloc}\mSig_0+\ctO(\eta^{2.5-0.5\beta}),
\end{align*}
which establishes \eqref{eq:recur A}. 

Fourthly, we derive the recursion for $\hmMs_{t}$. Multiplying both sides of ~\eqref{eq:recur x} by $\hvxs_{l,t+1}$ and taking expectation, $l\neq k$, we obtain
\begin{align*}
    \hmMs_{t+1}&=\hmMs_t-\eta \mH_0\hmMs_t-\eta\hmMs_t\mH_0+\cO(\eta\E[\normtwo{\hvxs_{k,t}}\normtwo{\hvxs_{l,t}}\normtwo{\hDvphs}])\\
    &\quad + \cO(\eta \E[\normtwo{\hvxs_{k,t}}^2\normtwo{\hvxs_{l,t}}+\normtwo{\hves_{k, t}}]).
\end{align*}
By a similar argument to the proof of \Cref{lemma:recursion vphi}, we have
\begin{align*}
    \E[\normtwo{\hvxs_{k,t}}^2\normtwo{\hvxs_{l,t}}]&=\ctO(\eta^{1.5}),\\
    \E[\normtwo{\hvxs_{k,t}}\normtwo{\hvxs_{l,t}}\normtwo{\hDvphs}]&=\ctO(\eta^{1.5-0.5\beta}),
\end{align*}
which yields~\eqref{eq:recur M}.

Now we proceed to prove~\eqref{eq:recur A tran} to~\eqref{eq:recur B tran}. By definition of $\bhmAs$, 
\begin{align*}
    \bhmAs &= \frac{1}{K^2}\E[(\sum_{k\in [K]}\hvxs_{k,H})(\sum_{k\in [K]}\hvxs_{k,H})^{\top}]\\
    &=\frac{1}{K^2}\sum_{k\in[K]}\E[\hvxs_{k,H}\hvxsT_{k,H}]+\frac{1}{K^2}\sum_{k,l\in[K], k\neq l}\E[\hvxs_{k,H}\hvxsT_{l,H}]\\
    &=\frac{1}{K}\hmAs_{H}+(1-\frac{1}{K})\hmMs_H,
\end{align*}
which demonstrates \eqref{eq:recur A tran}. Then we derive \eqref{eq:recur M tran}. By definition of $\bhvxs[s+1]{0}$, 
\begin{align}
    \bhvxs[s+1]{0} &= \hvphs+\bhvxs{H} - \Phi(\hvphs+\bhvxs{H})\notag\\
    &=\hvphs+\bhvxs{H} - \left(\hvphs+\partial\Phi(\hvphs)\bhvxs{H}+\cO(\normtwo{\bhvxs{H}}^2)\right)\notag\\
    &=\bhvxs{H} - \left(\mPpara+\cO(\normtwo{\hDvphs})\right)\bhvxs{H}+\cO(\normtwo{\bhvxs{H}}^2)\notag\\
    &=\mPperp \bhvxs{H}+\cO(\normtwo{\bhvxs{H}}^2+\normtwo{\bhvxs{H}}\normtwo{\hDvphs}).\label{eq:x avg}
\end{align}
Hence, 
\begin{align*}
    \hmMs[s+1]_0&=\hmAs[s+1]_0=\E[\bhvxs{0}\bhvxsT{0}]\\
    &=\mPperp \bhmAs \mPperp + \cO(\E[\normtwo{\bhvxs{H}}^3+\normtwo{\bhvxs{H}}^2\normtwo{\hDvphs}]).
\end{align*}
By~\eqref{eq:error1} and~\eqref{eq:error3}, we obtain~\eqref{eq:recur M tran}. By~\eqref{eq:recursion vphi},
\begin{align}
    \hvphs[s+1]-\hvphs&=\mPpara\bhvxs{H}+\cO(\normtwo{\bhvxs{H}}\normtwo{\hDvphs}+\normtwo{\bhvxs{H}}^2).\label{eq:vphs upd 2}
\end{align}
Combining~\eqref{eq:x avg} and~\eqref{eq:vphs upd 2} gives
\begin{align*}
\E[\bhvxs{0}(\hvphs[s+1]-\hvphs)^{\top}]& = \mPperp \bhmAs\mPpara+\ctO(\eta^{1.5-0.5\beta}).
\end{align*}
Therefore, 
\begin{align*}
    \hmBs[s+1]_{0}&=\E[\bhvxs[s+1]{0}\hDvphsT[s+1]]=\E[\bhvxs[s+1]{0}(\hDvphs+\hvphs[s+1]-\hvphs)^{\top}]\\
    &=\mPperp\hmBs_H+\mPperp\bhmAs\mPpara+\ctO(\eta^{1.5-\beta}).
\end{align*}
Finally, we apply \Cref{lemma:recursion vphi} to derive \eqref{eq:recur q tran}.
\begin{align*}
    \hvqs[s+1]_0&=\E[\bhvxs[s+1]{0}]=\E[\bhvxs{H}-(\hvphs[s+1]-\hvphs)]\\
    &=\hvqs_H-\mPpara \hvqs_H-\partial^2 \Phi(\hvphs[0])[ \hmBs_H]
    -\frac{1}{2} \partial^2 \Phi(\hvphs[0])[\bhmAs]+\ctO(\eta^{1.5-\beta})\\
    &=\mPperp\hvqs_H-\partial^2 \Phi(\hvphs[0])[ \hmBs_H]
    -\frac{1}{2} \partial^2 \Phi(\hvphs[0])[\bhmAs]+\ctO(\eta^{1.5-\beta}),
\end{align*}
which concludes the proof.
\end{proof}

With the assumption that the hessian at $\hvphs[0]$ is diagonal, we have the following corollary that formulates the recursions for each matrix element.
\begin{corollary}
    Given $\normtwo{\bhvths[0]-\hvphs[0]}=\cO(\sqrt{\eta \log \frac{1}{\eta}})$, 
    for $0\leq s < \Rg$ and $0\leq t < H$, we have the following elementwise recursions.
    \begin{align}
        \hAs_{t+1,i,j}&= \left(1-( \lambda_i + \lambda_j\right)\eta )\hAs_{t,i,j}+\frac{\eta^2}{\Bloc} \Sig_{0, i, j} + \ctO(\eta^{2.5-0.5\beta}), \label{eq:A}\\
        \hMs_{t+1, i, j}&= \left(1-( \lambda_i + \lambda_j\right)\eta )\hMs_{t, i, j}+ \ctO(\eta^{2.5-0.5\beta})\label{eq:M},\\
        \hBs_{t+1, i, j}&=(1-\lambda_i\eta)\hBs_{t, i, j}+\ctO(\eta^{2.5-\beta}),\label{eq:B}\\
        \bhAs{i}{j}&=\frac{1}{K} (\hAs_{H, i, j}-\hMs_{H, i, j})+\hMs_{H, i, j},\label{eq:bA}\\
        \hMs[s+1]_{0, i, j}&=\hAs[s+1]_{0, i, j}=\begin{cases}
        \bhAs{i}{j}+\ctO(\eta^{1.5-0.5\beta}), &1\leq i \leq m, 1\leq j \leq m,\\
        \ctO(\eta^{1.5-0.5\beta}), &\text{otherwise}.
        \end{cases}\label{eq:A0M0}\\
        \hBs[s+1]_{0, i, j}&=
        \begin{cases}
            \hBs[s]_{H, i, j}+\bhAs_{i, j}+\ctO(\eta^{1.5-\beta}), & 1\leq i \leq m, m<j\leq d,\\
            \hBs[s]_{H, i, j}+\ctO(\eta^{1.5-\beta}),& 1\leq i \leq m, 1\leq j\leq m,\\
            \ctO(\eta^{1.5-\beta}),& m< i \leq d.\label{eq:B0A0}
        \end{cases}
     \end{align}
\end{corollary}

Having formulated the recursions, we are ready to solve out the explicit expressions. 
We will split each matrix into four parts and them one by on.
Specifically, a matrix $\mM$ can be split into $\mPpara\mM\mPpara$ in the tangent space of $\Gamma$ at $\hvphs[0]$, $\mPperp\mM\mPperp$ in the normal space, along with
$\mPpara\mM\mPperp$ and $\mPperp\mM\mPpara$ across both spaces. 

 We first compute the elements of $\mPperp\hmAs_t\mPperp$ and $\mPperp\bhmAs\mPperp$.

 \begin{lemma}[General formula for $\mPperp \hmAs_t \mPperp$ and $\mPperp \bhmAs \mPperp$]\label{lemma:perpAperp}
    Let $R_0 := \ceil{\frac{10}{\lambda_{m}\alpha}\log \frac{1}{\eta}}$. Then 
    for $1\leq i \leq m, 1 \leq j \leq m$ and
    $R_0\leq s <\Rg$, 
    \begin{align*}
        \bhAs{i}{j}&=\frac{1}{(\lambda_i+\lambda_j)K\Bloc}\eta\Sig_{0, i, j} + \ctO(\eta^{1.5-0.5\beta}), \\
        \hAs_{t, i, j}
        & =-\left( 1-\frac{1}{K}\right)\frac{(1-(\lambda_i+\lambda_j)\eta)^t}{(\lambda_i + \lambda_j)\Bloc} \eta\Sigma_{0, i, j}+\frac{\eta}{(\lambda_i+\lambda_j)\Bloc}\Sigma_{0, i, j}+\ctO(\eta^{1.5-0.5\beta} ).
    \end{align*}
    For $s<R_0$, $\hAs_{t,i,j}=\ctO(\eta)$ and $\bhAs_{i,j}=\ctO(\eta)$.
    \end{lemma}
    \begin{proof}
    For $1\leq i \leq m, 1 \leq j \leq m$, $\lambda_i>0, \lambda_j>0$. By \eqref{eq:A}, 
    \begin{align*}
        \hAs_{t, i, j}& =(1-(\lambda_i+\lambda_j)\eta)^t\hAs_{0, i, j} +\sum_{\tau=0}^{t-1}(1-(\lambda_i+\lambda_j)\eta)^{\tau}\frac{\eta^2}{\Bloc} \Sig_{0, i, j}\\
        &\quad +\ctO(\sum_{\tau=0}^{t-1}(1-(\lambda_i+\lambda_j)\eta)^{\tau}\eta^{2.5-0.5\beta})\\
        &=(1-(\lambda_i+\lambda_j)\eta)^t\hAs_{0, i, j}+\frac{1-(1-(\lambda_i+\lambda_j)\eta)^{t}}{(\lambda_i+\lambda_j)\Bloc}\eta\Sig_{0, i, j}+\ctO(\eta^{1.5-0.5\beta}),
    \end{align*}
    where the second inequality uses $\sum_{\tau=0}^{t-1}(1-(\lambda_i+\lambda_j)\eta)^{\tau}=\frac{1-(1-(\lambda_i+\lambda_j)\eta)^{t}}{(\lambda_i+\lambda_j)\eta}\leq\frac{1}{(\lambda_i+\lambda_j)\eta}$. By \eqref{eq:M},
    \begin{align*}
        \hMs_{t,i, j }&=(1-(\lambda_i+\lambda_j)\eta)^t\hMs_{0, i, j}+\ctO(\sum_{\tau=0}^{t-1}(1-(\lambda_i+\lambda_j)\eta)^{\tau}\eta^{2.5-0.5\beta})\\
        &=(1-(\lambda_i+\lambda_j)\eta)^t\hAs_{0, i, j}+\ctO(\eta^{1.5-0.5\beta}),
    \end{align*}
    where the second equality uses $\mMs[s+1]_0=\mAs[s+1]_0$. By \eqref{eq:bA} and \eqref{eq:A0M0}, 
    \begin{align*}
        \bhAs{i}{j}
        & = \frac{1-(1-(\lambda_i+\lambda_j)\eta)^{H}}{(\lambda_i+\lambda_j)K\Bloc}\eta\Sig_{0, i, j}+(1-(\lambda_i+\lambda_j)\eta)^H\hAs_{0, i, j}+\ctO(\eta^{1.5-0.5\beta}),\\
        \hAs[s+1]_{0, i, j} &=\bhAs{i}{j}+\ctO(\eta^{2.5-0.5\beta})\\
        &=\frac{1-(1-(\lambda_i+\lambda_j)\eta)^{H}}{(\lambda_i+\lambda_j)K\Bloc}\eta\Sig_{0, i, j}+(1-(\lambda_i+\lambda_j)\eta)^H\hAs_{0, i, j}+\ctO(\eta^{1.5-0.5\beta}).
    \end{align*}
    Then we obtain
    \begin{align*}
     \hAs[s]_{0, i, j}&=(1-(\lambda_i+\lambda_j)\eta)^{s H}\hAs[0]_{0, i, j}+\frac{1-(1-(\lambda_i+\lambda_j)\eta)^H}{(\lambda_i+\lambda_j)K\Bloc}\eta\Sig_{0, i, j}\sum_{r=0}^{s-1}(1-(\lambda_i+\lambda_j)\eta)^{rH}\\
        &\quad + \ctO(\eta^{1.5-0.5\beta}\sum_{r=R_0}^{s-1}(1-(\lambda_i+\lambda_j)\eta)^{rH}).
    \end{align*}
    Notice that $\abs{1-(\lambda_i+\lambda_j)\eta}<1$ and
    \begin{align}
        (1-(\lambda_i+\lambda_j)\eta)^H&\leq \exp(-(\lambda_i+\lambda_j) \eta H)= \exp(-(\lambda_i+\lambda_j)\alpha).\label{eq:lambda}
    \end{align}
    Therefore, 
    \begin{align*}
    &\sum_{r=0}^{s-1}(1-(\lambda_i+\lambda_j)\eta)^{rH}=\frac{1-(1-(\lambda_i+\lambda_j)\eta)^{rH}}{1-(1-(\lambda_i+\lambda_j)\eta)^H}\leq\frac{1}{1-\exp(-(\lambda_i+\lambda_j)\alpha)}.
\end{align*}
Then we have
\begin{align*}
\hAs_{0, i, j}=(1-(\lambda_i+\lambda_j)\eta)^{sH}\hAs[0]_{0, i, j}+\frac{1-(1-(\lambda_i+\lambda_j)\eta)^{s H}}{(\lambda_i+\lambda_j)K\Bloc}\eta\Sig_{0, i, j} + \ctO(\eta^{1.5-0.5\beta}).
\end{align*}
Finally, we demonstrate that for $s\geq R_0$,  $\hAs_{0, i, j}$ and $\bhAs{i}{j}$ is approximately equal to $\frac{\eta}{(\lambda_i+\lambda_j)K\Bloc} \Sigma_{0, i, j}$. By \eqref{eq:lambda}, when $s\geq R_0 $, $(1-(\lambda_i+\lambda_j)\eta)^{s H}=\cO(\eta^{10})$, which gives
\begin{align*}
    \bhAs{i}{j}&=\frac{1}{(\lambda_i+\lambda_j)K\Bloc}\eta\Sig_{0, i, j} + \ctO(\eta^{1.5-0.5\beta}), \\
    \As_{t, i, j}
    & =-\left( 1-\frac{1}{K}\right)\frac{(1-(\lambda_i+\lambda_j)\eta)^t}{(\lambda_i + \lambda_j)\Bloc} \eta\Sigma_{0, i, j}+\frac{\eta}{(\lambda_i+\lambda_j)\Bloc}\Sigma_{0, i, j}+\ctO(\eta^{1.5-0.5\beta} ).    
\end{align*}
For $s<R_0$, since $\hmAs[0]_0=\bhvxs{0}\bhvxsT{0}$=$\ctO(\eta)$, we have $\bhAs_{i, j}=\ctO(\eta)$ and $\hAs_{t,i,j}=\ctO(\eta)$.
    \end{proof}
Secondly, we compute $\mPpara \hmAs_{t} \mPperp$ and  $\mPpara \bhmAs \mPperp$. 
\begin{lemma}[General formula for $\mPperp \hmAs_{t} \mPpara$ and  $\mPperp \bhmAs \mPpara$]\label{lemma:perpApara}
For $1\leq i \leq m, m<j\leq d$, 
\begin{align*}
     \hAs_{t, i, j}
     &=\frac{1-(1-\lambda_i\eta)^t}{\lambda_i\Bloc}\eta \Sig_{0,i, j}+\ctO(\eta^{1.5-0.5\beta}),\\
    \bhAs{i}{j}&=\frac{1-(1-\lambda_i\eta)^H}{\lambda_i K \Bloc}\eta \Sig_{0, i, j}+\ctO(\eta^{1.5-0.5\beta}).
 \end{align*}
\end{lemma}
\begin{proof}
Note that for $1\leq i \leq m, m<j\leq d$ and $\lambda_i>0, \lambda_j=0$. By \eqref{eq:A} and \eqref{eq:A0M0}, 
 \begin{align*}
     \hAs_{t, i, j}&=(1-\lambda_i\eta)^t\hAs_{0, i, j}+\frac{1-(1-\lambda_i\eta)^t}{\lambda_i\Bloc}\eta \Sig_{0,i, j}+\ctO(\eta^{1.5-0.5\beta})\\
     &=\frac{1-(1-\lambda_i\eta)^t}{\lambda_i\Bloc}\eta \Sig_{0,i, j}+\ctO(\eta^{1.5-\beta}).
 \end{align*}
 By \eqref{eq:M} and \eqref{eq:A0M0}, $\hMs_{t, i, j}=\ctO(\eta^{1.5-0.5\beta})$. Then, 
 \begin{align*}
     \bhAs{i}{j}&=\frac{1-(1-\lambda_i\eta)^H}{\lambda_i K\Bloc}\eta \Sig_{0, i, j}+\ctO(\eta^{1.5-0.5\beta}).
 \end{align*}
\end{proof}
Similar to \Cref{lemma:perpApara}, we have the following lemma for the general formula of $\mPpara \hmAs_t\mPperp$ and $\mPpara \bhmAs\mPperp$.
\begin{lemma}[General formula for $\mPpara \hmAs_{t} \mPperp$ and  $\mPpara \bhmAs \mPperp$]\label{lemma:paraAperp}
    For $m< i \leq d$ and $1\leq j\leq m$, 
    \begin{align*}
         \hAs_{t, i, j}
         &=\frac{1-(1-\lambda_j\eta)^t}{\lambda_j\Bloc}\eta \Sig_{0,i, j}+\ctO(\eta^{1.5-0.5\beta}),\\
        \bhAs{i}{j}&=\frac{1-(1-\lambda_j\eta)^H}{\lambda_j K\Bloc}\eta \Sig_{0, i, j}+\ctO(\eta^{1.5-0.5\beta}).
     \end{align*}
    \end{lemma}
Finally, we derive the general formula for $\mPpara\hmAs_t\mPpara$ and $\mPpara\bhmAs\mPpara$.
\begin{lemma}[General formula for  $\mPpara \hmAs_t \mPpara$ and  $\mPpara \bhmAs \mPpara$] \label{lemma:paraApara}
    For $m<i\leq d$ and $m<j\leq d$,
    \begin{align*}
         \bhAs{i}{j}&=\frac{H\eta^2}{K\Bloc}\Sigma_{0, i, j}+\ctO(\eta^{1.5-0.5\beta}),\\
         \hAs_{t,i,j}&= \hAs_{0,i,j}+\frac{t\eta^2}{\Bloc} \Sig_{0, i, j}+\ctO(\eta^{1.5-0.5\beta}).
    \end{align*}
    \end{lemma}
    \begin{proof}
    Note that for $m<i\leq d$, $m<j\leq d$ and $\lambda_i=\lambda_j=0$. \eqref{eq:A}  is then simplified as
    \begin{align*}
         \hAs_{t+1,i,j}&= \hAs_{t,i,j}+\frac{\eta^2}{\Bloc} \Sig_{0, i, j}+\ctO(\eta^{2.5-0.5\beta}).
    \end{align*}
    Therefore, 
    \begin{align}
       \hAs_{t,i,j}&= \hAs_{0,i,j}+\frac{t\eta^2}{\Bloc} \Sig_{0, i, j}+\ctO(\eta^{1.5-0.5\beta}).\label{eq:hAt}
    \end{align}
    According to~\eqref{eq:M}, $\hMs_{t, i, j} = \ctO(\eta^{1.5-0.5\beta})$ for $m<i \leq d$ and $m<j \leq d$. Combining~\eqref{eq:M}, \eqref{eq:A0M0} and  \eqref{eq:hAt} yields
    \begin{align*}
       \bhAs{i}{j}&=\frac{H\eta^2}{K\Bloc}\Sigma_{0, i, j}+\ctO(\eta^{1.5-0.5\beta}).
    \end{align*}
    \end{proof}
Now, we move on to compute the general formula for $\hmBs_t$. 
\begin{lemma}[The general formula for $\mPperp \hmBs_t\mPpara$]\label{lemma:perpBpara}
    Note that for $1\leq i \leq m$ and $m<j\leq d$, when $ R_0:=\ceil{\frac{10}{\lambda_{m}\alpha}\log \frac{1}{\eta}}\leq s < \Rg$,
    \begin{align*}
         \hBs_{t, i, j}&=\frac{(1-\lambda_i\eta)^t}{\lambda_i K \Bloc}\eta\Sig_{0, i, j}+\ctO(\eta^{1.5-\beta}).
    \end{align*}
    For $s<R_0$, $\hBs_{t,i,j}=\ctO(\eta)$.
    \end{lemma}

    \begin{proof}
    Note that for $1\leq i \leq m$, $\lambda_i>0$. By \eqref{eq:B}, 
    \begin{align*}
        \hBs_{t+1, i, j}&=(1-\lambda_i\eta)\hBs_{t, i, j}+\ctO(\eta^{2.5-\beta}).
    \end{align*}
    Hence, 
    \begin{align*}
      \hBs_{t, i, j}&=(1-\lambda_i\eta)^t\hBs_{0, i, j}+\ctO(\eta^{1.5-\beta}).
    \end{align*}
    According to~\eqref{eq:B0A0}, 
    \begin{align*}
        \hBs[s+1]_{0, i, j}&=\hBs[s]_{H, i, j}+\bhAs_{i, j}+\ctO(\eta^{2.5-\beta})\\
        &=(1-\lambda_i\eta)^H\hBs[s]_{0, i, j}+\bhAs{i}{ j}+\ctO(\eta^{1.5-\beta}).
    \end{align*}
    Then we have
    \begin{align*}
        \hBs_{0, i, j}&=(1-\lambda_i\eta)^{sH}\hBs[0]_{0, i, j}+\bhAs{i}{j}\sum_{r=0}^{s-1} (1-\lambda_i\eta)^{r H}+\ctO(\sum_{r=0}^{s-1} (1-\lambda_i\eta)^{r H}\eta^{1.5-\beta})\\
        &=(1-\lambda_i\eta)^{s H}\hBs[0]_{0, i, j} + \frac{1-(1-\lambda_i\eta)^{s H}}{1-(1-\lambda_i\eta)^H}\bhAs_{i, j}+\ctO(\eta^{1.5-\beta})\\
        &=\frac{1-(1-\lambda_i\eta)^{s H}}{1-(1-\lambda_i\eta)^H}\bhAs_{i, j}+\ctO(\eta^{1.5-\beta}).
    \end{align*}
    where the second equality uses \eqref{eq:lambda} and the last inequality
    uses  $\hmBs[0]_0=\bhvxs[0]{0}\hDvphs[0]=\vzero$.
     For $s\geq R_0$,   $\bhAs{i}{j}=\frac{1-(1-\lambda_i\eta)^H}{\lambda_i K \Bloc}\eta \Sig_{0, i, j}+\ctO(\eta^{1.5-0.5\beta})$, which gives
    \begin{align*}
        \hBs_{0, i, j}&=\frac{\eta}{\lambda_i K \Bloc} \Sig_{0, i, j}+\ctO(\eta^{1.5-\beta}).
    \end{align*}
    Therefore, 
    \begin{align*}
        \hBs_{t, i, j}&=\frac{(1-\lambda_i\eta)^t}{\lambda_i K \Bloc}\eta\Sig_{0, i, j}+\ctO(\eta^{1.5-\beta}).
    \end{align*}
    For $s<R_0$, $\bhAs_{i,j}=\ctO(\eta)$ and therefore, $\hBs_{t,i,j}=\ctO(\eta)$.
    \end{proof}

    \begin{lemma}[General formula for the elements of  $\mPperp \hmBs_t \mPperp $ ] \label{lemma:perpBperp}
        For $1\leq i \leq m$ and $1\leq j \leq m$, , $\hBs_{t, i, j}=\ctO(\eta^{1.5-\beta})$.
        \end{lemma}
\begin{proof}
Note that for $1\leq i \leq m$, $\lambda_i > 0$.
By \eqref{eq:B},
\begin{align*}
    \hBs_{t+1, i, j}&=(1-\lambda_i \eta)\hBs_{t, i, j}+\ctO(\eta^{2.5-\beta}).
\end{align*}
Hence, 
\begin{align*}
\hBs_{t, i, j}&=(1-\lambda_i\eta)^t\hBs_{0, i, j}+\ctO(\eta^{1.5-\beta}).
\end{align*}
By \eqref{eq:B0A0},
\begin{align*}
    \hBs[s+1]_{0, i, j}&= \hBs_{H, i, j}+\ctO(\eta^{2.5-\beta})\\
    & = (1-\lambda_i\eta)^H\hBs_{0, i, j}+\ctO(\eta^{1.5-\beta})\\
    & = (1-\lambda_i\eta)^{sH}\hBs[0]_{0, i, j}+\ctO(\sum_{r=0}^{s-1}(1-\lambda_i \eta)^{ rH}\eta^{1.5-\beta})\\
    & = (1-\lambda_i\eta)^{sH}\hBs[0]_{0, i, j}+\ctO(\eta^{1.5-\beta})\\
    &=\ctO(\eta^{1.5-\beta}),
\end{align*}
where the last inequality uses $\hmBs[0]_0=\vzero$.
\end{proof}

\begin{lemma}[General formula for $\mPpara \hmBs_t$]\label{lemma: paraB}
For $m < i \leq d$, $\hBs_{t,i,j}=\ctO(\eta^{1.5-\beta})$.
\end{lemma}
\begin{proof}
Note that $\lambda_i=0$ for $m<i\leq d$. By \eqref{eq:B} and \eqref{eq:B0A0}, 
\begin{align*}
    \hBs_{t+1}&=  \hBs_t+\ctO(\eta^{2. 5-\beta}),
\quad \hBs[s]_0 = \ctO(\eta^{2.5-\beta}).  
\end{align*}
Therefore, 
\begin{align*}
\hBs_t = t\ctO(\eta^{2.5-\beta})+\hBs_0=\ctO(\eta^{1.5-\beta}).
\end{align*}
\end{proof}

Having obtained the expressions for $\hmBs_t$, $\hmAs_t$ and $\bhmAs$, we now provide explicit expressions for the first and second moments of the change of manifold projection every round in the following two lemmas.
\begin{lemma}\label{lemma:one step momoents expelicit 1}
The expectation of the change of manifold projection every round is
\begin{align}\label{eq: one step moment 1}
  \E[\hvphs[s+1]-\hvphs]=\begin{cases}
      \frac{H\eta^2}{2B}\partial^2\Phi(\hvphs[0])[\mSig_0+\mPsi(\hvphs[0])]+\ctO(\eta^{1.5-\beta}),& R_0 < s <\Rg\\
      \ctO(\eta), & s\leq R_0
        \end{cases},
\end{align}
where $R_0:=\ceil{\frac{10}{\lambda_m\alpha}\log \frac{1}{\eta}}$.
\end{lemma}
\begin{proof}
We first compute  $\E[\hvphs[s+1]-\hvphs[s]]$. By~\eqref{eq: one step first},  we only need to compute $\mPpara\hvqs_H$ by relating it to these matrices. Multiplying both sides of \eqref{eq:recur q} by $\mPpara$ gives
\begin{align}
    \mPpara\hvqs_{t+1}&=\mPpara\hvqs_t -\eta\mPpara \nabla^3 \cL(\hvphs[0])[\hmBs_t]-\frac{\eta}{2}\mPpara\nabla^3\cL(\hvphs[0])[\hmAs_t]+\ctO(\eta^{2.5-\beta}). \label{eq:qt}
\end{align}
Similarly, according to \eqref{eq:recur q tran}, we have
\begin{align}
    \mPpara\hvqs[s+1]_0&=-\mPpara \partial^2 \Phi(\hvphs[0])[\hmBs_H]-\frac{1}{2} \mPpara\partial^2 \Phi(\hvphs[0])[\bhmAs]+\ctO(\eta^{1.5-\beta}).\label{eq:q0}
\end{align}
Combining~\eqref{eq:qt} and~\eqref{eq:q0} yields
\begin{align}\label{eq:para q}
    \begin{aligned}
       \mPpara\hvqs_{H}&=-\frac{1}{2} \mPpara\partial^2 \Phi(\hvphs[0])[\bhmAs[s-1]] -\frac{\eta}{2}\mPpara\nabla^3\cL(\hvphs[0])[\sum_{t=0}^{H-1}\hmAs_t]\\
       & \quad -\eta\mPpara \nabla^3 \cL(\hvphs[0])[\sum_{t=0}^{H-1}\hmBs_t] -\mPpara \partial^2 \Phi(\hvphs[0])[\hmBs[s-1]_H] +\ctO(\eta^{1.5-\beta}).
    \end{aligned}
    \end{align}
By Lemmas~\ref{lemma:perpAperp}, \ref{lemma:paraApara} and \ref{lemma:perpApara},
for $s\leq R_0=\floor{\frac{10}{\lambda\alpha}\log \frac{1}{\eta}}$, $\hmAs_t=\ctO(\eta)$, $\bhmAs=\ctO(\eta)$ and $\hmBs_t=\ctO(\eta)$. Therefore, $\E[\hvphs[s+1]-\hvphs]=\ctO(\eta)$.
For $s>R_0$, 
$\bhmAs[s-1]=\bhmAs+\ctO(\eta^{1.5-0.5\beta})$.
Substituting \eqref{eq:para q} into \eqref{eq: one step first} gives
\begin{align*}
    \E[\hvphs[s+1]-\hvphs]&=\underbrace{\frac{1}{2}\mPperp\partial^2\Phi(\hvphs[0])[\bhmAs]+\mPperp\partial^2\Phi(\hvphs[0])[\hmBs_H]}_{\cT_1}\\
    &\quad \overbrace{-\eta \mPpara \nabla^3\cL(\hvphs[0])[\underbrace{\frac{1}{2}\sum_{t=0}^{H-1}\hmAs_t+\sum_{t=0}^{H-1}\hmBs_t}_{\cT_3}]}^{\cT_2}+\ctO(\eta^{1.5-\beta}).
\end{align*}
Below we compute $\cT_1$ and $\cT_2$ for $s>R_0$ respectively. 
By \Cref{lemma:mani inner para perp}, 
\begin{align*}
 \mPperp \partial^2\Phi(\hvphs[0])[\mPperp \bhmAs \mPpara]&=\mPperp \partial^2\Phi(\hvphs[0])[\mPpara \bhmAs \mPperp]=\vzero,\\
\mPperp \partial^2\Phi(\hvphs[0])[\mPpara \bhmAs \mPpara]
&=\partial^2\Phi(\hvphs[0])[\mPpara \bhmAs \mPpara].
\end{align*}
By \Cref{lemma:mani inner perp perp},
\begin{align*}
    \mPperp \partial^2\Phi(\hvphs[0])[\mPperp \bhmAs \mPperp] = \vzero.
\end{align*}
Therefore, for $s>R_0$, 
\begin{align*}
\mPperp\partial^2\Phi(\hvphs[0])[\bhmAs]
&=\frac{H\eta^2}{2K\Bloc}\partial^2\Phi(\hvphs[0])\Phi[\mSigzpara]+\ctO(\eta^{1.5-0.5\beta}),
\end{align*}
where we apply \Cref{lemma:paraApara}.
Similarly, for $s>R_0$, 
\begin{align*}
    \mPperp\partial^2\Phi(\hvphs[0])[\hmBs_H]&=\partial^2\Phi(\hvphs[0])[\mPpara \hmBs_H \mPpara]=\ctO(\eta^{1.5-\beta}),
\end{align*}
where we apply \Cref{lemma: paraB}. Hence, 
\begin{align}\label{eq:cT1}
    \cT_1=\frac{H\eta^2}{2B}\partial^2\Phi(\hvphs[0])[\mSigzpara]+\ctO(\eta^{1.5-\beta}).
\end{align}
We move on to show that
\begin{align}\label{eq:T2}
    \cT_2 = \frac{H\eta^2}{2B}\partial^2\Phi(\hvphs[0])[\mSig_0-\mSigzpara+(K-1)\mPsi(\hvphs[0])].
\end{align}
Similar to the way we compute $\hmAs_t$, $\bhmAs$ and $\hmBs_t$, we compute $\cT_2$ by splitting $\cT_3$ into four matrices and then substituting them into the linear operator $-\eta \mPpara \nabla^3 \cL(\hvphs[0])[\cdot]$ one by one. 
First, we show that
\begin{align}\label{eq:T3perp}
    \begin{aligned}
    -\eta \mPpara\nabla^3\cL(\hvphs[0])[\mPperp\cT_3\mPperp] 
    &=\frac{H\eta^2}{2B}\partial^2 \Phi(\hvphs[0])[\mSigzperp+(K-1)\psi(\mSigzperp)]\\
    &\quad+\ctO(\eta^{1.5-\beta}),
    \end{aligned}
\end{align}
where $\psi(\cdot)$ is interpreted as an \emph{elementwise} matrix function here.
By Lemmas~\ref{lemma:perpAperp} and~\ref{lemma:perpBperp}, for $1\leq i \leq m$, $1\leq j \leq m$ and $s> R_0 $,
\begin{align*}
    \hAs_{t, i, j}
    & =-\left( 1-\frac{1}{K}\right)\frac{(1-(\lambda_i+\lambda_j)\eta)^t}{(\lambda_i + \lambda_j)\Bloc} \eta\Sigma_{0, i, j}+\frac{\eta}{(\lambda_i+\lambda_j)\Bloc}\Sigma_{0, i, j}+\ctO(\eta^{1.5-0.5\beta} ),\\
    \hBs_{t,i,j} &=\ctO(\eta^{1.5-\beta}).
\end{align*}
Therefore, 
\begin{align*}
    \sum_{t=0}^{H-1} \hAs_{t, i, j}
    &=-\left(1-\frac{1}{K}\right)\frac{1-(1-(\lambda_i+\lambda_j)\eta)^H}{(\lambda_i+\lambda_j)^2\Bloc}\Sig_{0,i,j}+\frac{H\eta}{(\lambda_i+\lambda_j)\Bloc}\Sig_{0.,i,j}+\ctO(\eta^{0.5-\beta})\\
    &=\frac{H\eta}{K(\lambda_i+\lambda_j)\Bloc}\Sig_{0.,i,j}\\
    &\quad+\left(1-\frac{1}{K}\right)\frac{H\eta}{(\lambda_i+\lambda_j)\Bloc}\underbrace{\left[1-\frac{1-(1-(\lambda_i+\lambda_j)\eta)^H}{H\eta(\lambda_i+\lambda_j)}\right]}_{\cT_4}\Sig_{0,i,j}
  +\ctO(\eta^{0.5-\beta}).\\
    \sum_{t=0}^{H-1}\hBs_{t,i,j}&=\ctO(\eta^{0.5-\beta}),
\end{align*}
Then we simplify $\cT_4$. Notice that 
\begin{align*}
    (1-(\lambda_i+\lambda_i)\eta)^H&=\exp(-H(\lambda_i+\lambda_j)\eta)[1+\cO(H\eta^2)]\\
    &=\exp(-H(\lambda_i+\lambda_j)\eta)+\cO(\eta).
\end{align*}
Therefore, 
\begin{align*}
 \cT_4&=\psi((\lambda_i+\lambda_j)H\eta)+\cO(\eta).
\end{align*}
Substituting $\cT_4$ back into the expression for $\sum_{t=0}^{H-1}\hAs_{t,i,j}$ gives
\begin{align*}
    \sum_{t=0}^{H-1}\hAs_{t,i,j}=\frac{H\eta}{K(\lambda_i+\lambda_j)\Bloc}\Sig_{0.,i,j}+\left(1-\frac{1}{K}\right)\frac{H\eta\psi((\lambda_i+\lambda_j)H\eta)}{(\lambda_i+\lambda_j)\Bloc}\Sig_{0,i,j}
  +\ctO(\eta^{0.5-\beta}).
\end{align*}
Combining the elementwise results, we obtain the following matrix form expression:
\begin{align*}
    -\eta \mPpara \nabla^3 \cL(\hvphs[0])[\mPperp\cT_3 \mPperp]&=-\frac{H\eta^2}{2B} \mPpara \nabla^3 \cL(\hvphs[0])[\cV_{\mH_0}(\mSigzperp+(K-1)\psi(\mSigzperp))]\\
    &\quad+\ctO(\eta^{1.5-\beta}).
\end{align*}
By \Cref{lemma:mani inner perp perp}, we have \eqref{eq:T3perp}. 

Secondly, we show that for $s>R_0$, 
\begin{align}\label{eq:perpTpara}
\begin{aligned}
    &\quad -\eta \mPpara \nabla^3 \cL(\hvphs[0])[\mPperp\cT_3\mPpara+\mPpara\cT_3\mPperp]\\
   &= \frac{H\eta^2}{B} \partial^2\Phi(\hvphs[0])[\mSigzea+
   (K-1)\psi(\mSigzea)]+\ctO(\eta^{1.5-\beta}),
\end{aligned}
\end{align}
where $\psi(\cdot)$ is interpreted as an \emph{elementwise} matrix function here.
By symmetry of $\hmAs_t$'s and $\nabla^3\cL(\hvphs[0])$,
\begin{align*}
    &\frac{1}{2}\nabla^3\cL(\hvphs[0])\left[\sum_{t=0}^{H-1}\mPperp\hmAs_t\mPpara+\sum_{t=0}^{H-1}\mPpara\hmAs_t\mPperp\right]=\nabla^3\cL(\hvphs[0])\left[\sum_{t=0}^{H-1}\mPperp\hmAs_t\mPpara\right].
\end{align*}
Therefore, we only have to evaluate
\begin{align*}
    \nabla^3\cL(\hvphs[0])\left[\sum_{t=0}^{H-1}\mPperp(\hmAs_t+\hmBs_t)\mPpara+\sum_{t=0}^{H-1}\mPpara\hmBs_t\mPperp\right].
\end{align*}
To compute the elements of $\sum_{t=0}^{H-1}\mPperp(\hmAs_t+\hmBs_t)\mPpara$, we combine Lemmas~\ref{lemma:perpApara} and \ref{lemma:perpBpara} to obtain that for $1\leq i \leq m$ and $m<j\leq d$, 
\begin{align*}
    \sum_{t=0}^{H-1}\hAs_{t,i,j}&=\sum_{t=0}^{H-1}\frac{1-(1-\lambda_i\eta)^t}{\lambda_i\Bloc}\eta \Sig_{0,i, j}+\ctO(\eta^{0.5-\beta})\\
    &=\frac{H\eta}{\lambda_i\Bloc}\Sig_{0,i,j}-\frac{1-(1-\lambda_i\eta)^H}{\lambda_i^2\Bloc}\Sig_{0,i,j}+\ctO(\eta^{0.5-\beta})\\
    &=\frac{H\eta}{\lambda_i\Bloc}\left(1-\frac{1-(1-\lambda_i\eta)^H}{\lambda_i H \eta}\right)\Sig_{0,i,j}+\ctO(\eta^{0.5-\beta})\\
    &=\frac{H\eta}{\lambda_i\Bloc}\psi(\lambda_i H \eta)\Sig_{0,i,j}+\ctO(\eta^{0.5-\beta}), 
    \end{align*}
and 
\begin{align*}
    \sum_{t=0}^{H-1}\hBs_{t, i, j}&=\sum_{t=0}^{H-1}\frac{(1-\lambda_i\eta)^t}{\lambda_iK \Bloc}\eta\Sig_{0, i, j}+\ctO(\eta^{1.5-\beta}),\\
   &=\frac{1-(1-\lambda_i\eta)^H}{\lambda_i^2 K \Bloc}\Sig_{0,i,j}+\ctO(\eta^{0.5-\beta})\\
   &=\frac{H\eta}{\lambda_i K \Bloc}\Sig_{0,i,j}-\frac{H\eta}{\lambda_i K \Bloc}\left(1-\frac{1-(1-\lambda_i\eta)^H}{\lambda_i H \eta}\right)\Sig_{0,i,j}+\ctO(\eta^{0.5-\beta})\\
   &=\frac{H\eta}{\lambda_i K \Bloc}\Sig_{0,i,j}-\frac{H\eta}{\lambda_i K \Bloc}\psi(\lambda_i H \eta)\Sig_{0,i,j}+\ctO(\eta^{0.5-\beta}).
\end{align*}
Therefore, the matrix form of $\sum_{t=0}^{H-1}\mPperp(\hmAs_t+\hmBs_t)\mPpara$ is
\begin{align*}
    \sum_{t=0}^{H-1}\mPperp(\hmAs_t+\hmBs_t)\mPpara&=\frac{H\eta}{B}\cV_{\mH_0}\left(\mSigzea+(K-1)\psi(\mSigzea)\right)+\ctO(\eta^{0.5-\beta}),
\end{align*}
where $\psi(\cdot)$ is interpreted as an \emph{elementwise} matrix function here.
Furthermore, by \Cref{lemma: paraB},  $\sum_{t=0}^{H-1}\hmBs_{t}=\ctO(\eta^{0.5-\beta})$. Applying \Cref{lemma:mani inner para perp}, we have \eqref{eq:perpTpara}.
Finally, directly applying \Cref{lemma:mani}, we have 
\begin{align}\label{eq:T3para}
    -\eta \mPpara\nabla^3\cL(\hvphs[0])[\mPpara\cT_3\mPpara] 
    &=\vzero. 
\end{align}
Notice that $\psi(\mSigzpara)=\vzero$ where $\psi(\cdot)$ operates on each element.
Combining \eqref{eq:T3perp}, \eqref{eq:perpTpara} and \eqref{eq:T3para}, we obtain \eqref{eq:T2}. By \eqref{eq:cT1} and \eqref{eq:T2}, we have \eqref{eq: one step moment 1}.
\end{proof}

\begin{lemma}\label{lemma:one step momoents expelicit 2}
The second moment of the change of manifold projection every round is
\begin{align*}
    \E[(\hvphs[s+1]-\hvphs)(\hvphs[s+1]-\hvphs)^{\top}]=\begin{cases}
        \frac{H\eta^2}{B}\mSigzpara+\ctO(\eta^{1.5-0.5\beta}), & R_0 \leq s< \Rg\\
        \ctO(\eta), &s<R_0
    \end{cases},
\end{align*}
where $R_0:=\ceil{\frac{10}{\lambda_m \alpha}\log \frac{1}{\eta}}$.
\end{lemma}
\begin{proof}
Directly apply \Cref{lemma:paraApara} and \Cref{lemma:recursion vphi} and we have the lemma.
\end{proof}
With Lemmas~\ref{lemma:one step momoents expelicit 1} and \ref{lemma:one step momoents expelicit 2}, we are ready to prove \Cref{thm: moments}.
\begin{proof}[Proof of \Cref{thm: moments}.]
We first derive $\E[\hDvphs[\Rg]]$. Recall that $\Rg=\floor{\frac{1}{\alpha \eta^{\beta}}}=\frac{1}{H\eta^{1+\beta}}+o(1)$ where $0<\beta<0.5$. By \Cref{lemma:one step momoents expelicit 1},
\begin{align*}
    \E[\hvphs[\Rg]-\hvphs[0]]&=\sum_{s=0}^{R_0} \E[\hvphs[s+1]-\hvphs]+\sum_{s=R_0+1}^{\Rg-1}\E[\hvphs[s+1]-\hvphs]\\
    &=\frac{\eta^{1-\beta}}{2B}\partial^2\Phi(\hvphs[0])[\mSig_0+\mPsi(\hvphs[0])]+\ctO(\eta^{1.5-2\beta})+\ctO(\eta).
\end{align*}
Then we compute $\E[\hDvphs[\Rg]\hDvphsT[\Rg]]$.
\begin{align*}
    &\quad \E\left[\left(\sum_{s=0}^{\Rg-1}(\hvphs[s+1]-\hvphs)\right)\left(\sum_{s=0}^{\Rg-1}(\hvphs[s+1]-\hvphs)\right)^{\top}\right]\\
    &=\sum_{s=0}^{\Rg-1} \E[(\hvphs[s+1]-\hvphs)(\hvphs[s+1]-\hvphs)^{\top}]+\sum_{s\neq s'}\E[(\hvphs[s+1]-\hvphs)]\E[(\hvphs[s'+1]-\hvphs[s'])^{\top}]\\
    &=\frac{\eta^{1-\beta}}{B}\mSigzpara+\ctO(\eta)+\ctO(\eta^{1.5-1.5\beta}),
\end{align*}
where the last inequality uses $\E[(\hvphs[s+1]-\hvphs)]\E[(\hvphs[s'+1]-\hvphs[s'])^{\top}]=\ctO(\eta^2)$.
\end{proof}

\subsection{Proof of Weak Approximation}\label{sec:weak approx}
We are now in a position to utilize the estimate of moments obtained in previous subsections to prove the closeness of the sequence $\{\vphs\}_{s=0}^{\floor{T/(H\eta^2)}}$ and the SDE solution $\{\vzeta: t \in [0, T]\}$ in the sense of weak approximation.
Recall the SDE that we expect the manifold projection $\{\Phi(\bvths)\}_{s=0}^{\floor{T/(H\eta^2)}}$ to track:
\begin{align}\label{eq:zeta app}
    \dd \vzeta(t)&=P_{\vzeta}\Big(\underbrace{\tfrac{1}{\sqrt{B}} \mSig_{\parallel}^{\sfrac{1}{2}}(\vzeta)\dd \vW_t}_{\text{(a)\ diffusion}}
    \underbrace{-\tfrac{1}{2B} \nabla^3 \cL(\vzeta)[\widehat{\mSig}_{\Diamond}(\vzeta)] \dd t}_{\text{(b)\ drift-I}}
    \underbrace{-\tfrac{K-1}{2B} \nabla^3 \cL(\vzeta)[\widehat{\mPsi}(\vzeta)] \dd t}_{\text{(c)\ drift-II}}
    \Big),
\end{align}
According to \Cref{lemma:mani inner para perp} and \Cref{lemma:mani inner perp perp}, the drift term in total can be written as the following form:
\begin{align*}
    \text{(b)} + \text{(c)}=\frac{1}{2B}\partial^2 \Phi(\vzeta)[\mSig(\vzeta)+(K-1)\mPsi(\vzeta)]. 
\end{align*}
Then by definition of $P_{\vzeta}$, \eqref{eq:zeta app} is equivalent to the following SDE: 
\begin{align}\label{eq:zeta2}
    \dd \vzeta(t)=\frac{1}{\sqrt{B}}\partial \Phi(\vzeta)\mSig^{1/2}(\vzeta)\dd \vW_t+\frac{1}{2B}\partial^2 \Phi(\vzeta)\left[\mSig(\vzeta) + (K-1) \mPsi(\vzeta)\right]\dd t.
\end{align}
Therefore, we only have to show that $\vphs$ closely tracks $\{\vzeta(t)\}$ satisfying \Cref{eq:zeta2}. 
By \Cref{lemma:workingzone}, there exists an $\epsilon_3$ neighborhood of $\Gamma$, $\Gath$, where $\Phi(\cdot)$ is $\cC^{\infty}$-smooth. Due to compactness of $\Gamma$, $\Gath$ is bounded and the mappings $\partial^2\Phi(\cdot)$, $\partial \Phi(\cdot)$, $\mSig^{1/2}(\cdot)$, $\mSig(\cdot)$ and $\mPsi(\cdot)$ are all Lipschitz in $\Gath$.  By Kirszbraun theorem, both the drift and diffusion term of \eqref{eq:zeta2} can be extended as Lipschitz functions on $\R^d$. Therefore, the solution to the extended SDE exists and is unique. We further show that the solution, if initialized as a point on $\Gamma$, always stays on the manifold almost surely.

As a preparation, we first show that $\Gamma$ has no boundary.
\begin{lemma}
Under Assumptions~\ref{a:smooth} to \ref{a:compact}, $\Gamma$ has no boundary.
\end{lemma}
\begin{proof}
We prove by contradiction. If $\Gamma$ has boundary $\partial \Gamma$, WLOG, for a point $\vp\in \partial \Gamma$, let the Hessian at $\vp$ be diagonal with the form $\nabla^2 \cL(\vp)=\diag(\lambda_1, \cdots, \lambda_d)$ where $\lambda_i>0$ for $1\leq i \leq m$ and $\lambda_i=0$ for $m<i\leq d$ . 

Denote by $\vx_{i:j}:=(x_i, x_{i+1}, \cdots, x_{j})$ ($i\leq j$) the $(j-i+1)$-dimensional vector formed by the $i$-th to $j$-th coordinates of $\vx$. Since $\frac{\partial(\nabla\cL(\vp))}{\partial \vp_{1:m}}=\diag(\lambda_1, \cdots, \lambda_m)$ is invertible, by the implicit function theorem, there exists an open  neighborhood $V$ of $\vp_{m+1:d}$ such that $\nabla\cL(\vv)=\vzero$, $\forall \vv \in V$. Then, $\cL(\vv)=\cL(\vp)=\min_{\vtheta \in U}\cL(\vtheta)$ and hence $V\subset \Gamma$, which contradicts with $\vp \in \partial \Gamma$.
\end{proof}
Therefore, $\Gamma$ is a closed manifold (i.e., compact and without boundary). Then we have the following lemma stating that $\Gamma$ is invariant for \eqref{eq:zeta2}.
\begin{lemma}
Let $\vzeta(t)$ be the solution to \eqref{eq:zeta2} with $\vzeta(0)\in \Gamma$, then $\vzeta(t)\in \Gamma$ for all $t\geq 0$. In other words, $\Gamma$ is invariant for \eqref{eq:zeta2}.
\end{lemma}
\begin{proof}
According to \citet{filipovic2000invariant} and \citet{du2007invariant}, for a closed manifold $\cM$ to be viable for the SDE $\dd \vX(t)=F(\vX(t))\dd t+\mB(\vX(t))\dd \vW_t$ where $F: \R^d \to \R^d$ and $\mB: \R^d\to \R^d$ are locally Lipschitz, we only have to verify the following Nagumo type consistency condition:
\begin{align*}
    \mu(\vx):=F(\vx)-\frac{1}{2}\sum_j \D [B_j(\vx)]B_j(\vx)\in T_{\vx}(\cM), \quad B_j(\vx)\in T_{\vx}(\cM),
\end{align*}
where $\D[\cdot]$ is the Jacobian operator and $B_j(\vx)$ denotes the $j$-th column of $\mB(\vx)$.

In our context, since for $\vphi\in\Gamma$, $\partial\Phi(\vphi)$ is a projection matrix onto $T_{\vphi}(\Gamma)$, each column of $\partial \Phi(\vphi)\mSig^{1/2}(\vphi)$ belongs to $T_{\vphi}(\Gamma)$, verifying the second condition. Denote by $\mPperp(\vphi):=\mI_d - \partial \Phi(\vphi)$ the projection onto the normal space of $\Gamma$ at $\vphi$. To verify the first condition, it suffices to show that $\mPperp(\vphi)\vmu(\vphi)=\vzero$. We evaluate $\sum_j\mPperp(\vphi)\D [B_j(\vphi)] B_j(\vphi)$ as follows.
\begin{align}
\sum_j\mPperp(\vphi)\D [B_j(\vphi)]B_j(\vphi)&=\frac{1}{B}\sum_j\D[\partial \Phi(\vphi)\mSig^{1/2}_j(\vphi)]\partial \Phi(\vphi)\mSig^{1/2}_j(\vphi)\notag\\
&=\frac{1}{B}\mPperp(\vphi)\sum_j\partial^2 \Phi(\vphi)[\mSig_j^{1/2}(\vphi), \partial \Phi(\vphi)\mSig^{1/2}_j(\vphi)]\notag\\
&=-\frac{1}{B}\nabla^2 \cL(\vphi)^+\nabla^3\cL(\vphi)[\mSig_{\parallel}(\vphi)], \label{eq:DD}
\end{align}
where the last inequality uses \Cref{lemma:mani inner para perp}. Again applying \Cref{lemma:mani inner para perp}, we have
\begin{align}\label{eq:F}
   \mPperp(\vphi)F(\vphi) = -\frac{1}{2B}\nabla^2 \cL(\vphi)^+\nabla^3\cL(\vphi)[\mSig_{\parallel}(\vphi)].
\end{align}
Combining \eqref{eq:DD} and \eqref{eq:F}, we can verify the first condition.
\end{proof}
In order to establish \Cref{main thm: flow}, it suffices to prove the following theorem, which captures the closeness of $\vphs$ and $\vzeta(t)$ every $\Rg$ rounds.
\begin{theorem}\label{thm: approx}
 If $\normtwo{\bvths[0]-\vphs[0]}=\cO(\sqrt{\eta\log \frac{1}{\eta}})$ and $\vzeta(0)=\vphs[0]\in \Gamma$, then for $\Rg=\floor{\frac{1}{\alpha \eta^{0.75}}}$every test function $g\in \cC^3$,
\begin{align*}
    \max_{n=0, \cdots, \floor{T/\eta^{0.75}}}\labs{\E g(\vphs[n\Rg])-\E g(\vzeta(n\eta^{0.75}))}\leq C_g \eta^{0.25}(\log \tfrac{1}{\eta})^b,
\end{align*}
where $C_g>0$ is a constant independent of $\eta$ but can depend on $g(\cdot)$ and $b>0$ is a constant independent of $\eta$ and $g(\cdot)$.
\end{theorem}

\subsubsection{Preliminaries and additional notations}
We first introduce a general formulation for stochastic gradient algorithms (SGAs) and then specify the components of this formulation in our context. 
Consider the following SGA:
\begin{align*}
    \vx_{n+1}=\vx_n+\etae\vh(\vx_n, \vxi_n),
\end{align*}
where $\vx_n \in \R^d$ is the parameter, $\etae$ is the learning rate, $\vh(\cdot, \cdot)$ is the update which depends on $\vx_n$ and a random vector $\vxi_n$ sampled from some distribution $\Xi(\vx_n)$. Also, consider the following Stochastic Differential Equation (SDE).
\begin{align*}
    \dd \vX(t)&=\vb(\vX(t))\dd t + \vsig (\vX(t))\dd \vW_t, 
\end{align*}
where $\vb(\cdot):\R^d\to\R^d$ is the drift function and $\vsig(\cdot):\R^{d\times d} \to \R^{d\times d}$ is the diffusion matrix. 

Denote by $\cPX(\vx, s, t)$ the distribution of $\vX(t)$ with the initial condition $\vX(s)=\vx$.
Define
\begin{align*}
  \tvDelta(\vx, n)&:=\vX_{(n+1)\etae}-\vx,  &\text{where \ }\vX_{(n+1)\etae}\sim \cPX(\vx, n\etae, (n+1)\etae ),
\end{align*}
which characterizes the update in one step.

In our context, we view the change of manifold projection over $\Rg:=\floor{\frac{1}{\alpha \eta^{1-\beta}}} (\beta \in (0,0.5))$ rounds as one ``giant step". Hence the $\vphs[n\Rg]$ corresponds to the discrete time random variable $\vx_n$ corresponds to  and $\vzeta(t)$ corresponds to the continuous time random variable $\vX_t$. According to \Cref{thm: one step moment new}, we set 
\begin{align*}
    \etae=\eta^{1-\beta},\quad  \vb(\vzeta)=\frac{1}{2B}\partial^2 \Phi(\vzeta)\left[\mSig(\vzeta) + (K-1) \mPsi(\vzeta)\right],\quad \vsig(\vzeta)=\frac{1}{\sqrt{B}}\partial \Phi(\vzeta)\mSig^{1/2}(\vzeta).
\end{align*}
Due to compactness of $\Gamma$, $\vb(\cdot)$ and $\vsig(\cdot)$ are Lipschitz on $\Gamma$. 

As for the update in one step,  $\tvDelta(\cdot, \cdot)$ is defined in our context as:
\begin{align*}
    \tvDelta(\vphi, n)&:= \vzeta_{(n+1)\etae}-\vphi,& \text{\ where \ } \vzeta_{(n+1)\etae}\sim \cPz(\vphi, n\etae, (n+1)\etae) \text{\ and \ } \vphi \in \Gamma.
\end{align*}
For convenience, we further define
\begin{align*}
     \vDeltas[n]&:=\hvphs[(n+1)\Rg]-\hvphs[n \Rg],& \tvDeltas[n]&:=\tvDelta(\hvphs[\Rg], n),\\
     \vbn: &= \vb(\hvphs[n \Rg]),&
     \vsign: & = \vsig(\hvphs[n \Rg]).
\end{align*}
 We use $\Cg{i}$ to denote constants that can depend on the test function $g$ and independent of $\etae$. The following lemma relates the moments of $\tvDelta(\vphi, n)$ to $\vb(\vphi)$ and $\vsig(\vphi)$.
\begin{lemma}\label{lemma: bound tdelta}
There exists a positive constant $C_0$ independent of $\etae$ and $g$ such that for all $\vphi \in \Gamma$, 
\begin{align*}
    \abs{\E[\tDelta_i(\vphi, n)]-\etae b_i(\vphi)}&\leq C_{0}\etae^2, &\forall 1\leq i \leq d, \\
      \abs{\E[\tDelta_i(\vphi, n)\tDelta_j(\vx, n)]-\etae \sum_{l=1}^d\sigma_{i,l}(\vphi)\sigma_{l, j}(\vphi)}&\leq C_{0}\etae^2, &\forall 1\leq i,j \leq d, \\
      \E\left[\left|\prod_{s=1}^6\tDelta_{i_s}(\vphi, n)\right|\right]&\leq C_{0}\etae^3, &\forall 1\leq i_1,\cdots, i_6 \leq d.
\end{align*}
The lemma below states that  the expectation of the test function is smooth with respect to the initial value.
\begin{proof}
Noticing that (i) the solution to \eqref{eq:zeta2} always stays on $\Gamma$ almost surely if its initial value $\vzeta(0)$ belongs to $ \Gamma$ , (ii) $\vb(\cdot)$ and $\vsig(\cdot)$ are $\cC^{\infty}$ and (iii) $\Gamma$ is compact, we can directly apply
Lemma B.3 in \cite{malladi2022sdes} and Lemma 26 in \cite{li2019stochastic} to obtain the above lemma.
\end{proof}
\end{lemma}

The following lemma states that the expectation of $g(\vzeta(t))$ for $g\in \cC^3$ is smooth with respect to the initial value of the SDE solution.
\begin{lemma}\label{lemma:u}
 Let $s\in [0, T]$, $\vphi \in \Gamma$ and $g\in\cC^3$. For $t\in [s, T]$, define
\begin{align*}
    u(\vphi, s, t):=\E_{\vzeta_t\sim \cPz(\vphi, s, t)}[g(\vzeta_t)].
\end{align*}
Then $u(\cdot, s, t)\in \cC^3$ uniformly in $s,t$.
\end{lemma}
\begin{proof}
A slight modification of Lemma B.4 in \cite{malladi2022sdes} will give the above lemma.
\end{proof}

\subsubsection{Proof of the approximation in our context}

For $\beta\in(0,0.5)$, define $\gamma_1: = \frac{1.5-2\beta}{1-\beta}, \gamma_2: = \frac{1}{1-\beta},$ and then $1 <\gamma_1 < 1.5$,  $1<\gamma_2  < 2$.  We introduce the following lemma which serves as a key step to control the approximation error. Specifically, 
this lemma bounds the difference in one step change between the discrete process and the continuous one as well as the product of higher orders.
\begin{lemma}\label{lemma: many delta}
If $\normtwo{\bvths[0]-\vphs[0]}=\cO(\sqrt{\eta\log \frac{1}{\eta}})$, then there exist positive constants $C_1$ and $b$ independent of $\etae$ and $g$ such that for all $0 \leq n<\floor{T/\etae} $, 
\begin{enumerate}
\item 
 \begin{align*}
    \abs{\E[\Deltas[n]_i-\tDeltas[n]_i\mid \cEs[n\Rg]_0}&\leq C_1\etae^{\gamma_1}(\log \tfrac{1}{\etae})^b+C_1\etae^{\gamma_2}(\log \tfrac{1}{\etae})^b, &\forall 1\leq i \leq d, \\
      \abs{\E[\Deltas[n]_i\Deltas[ n]_j-\tDeltas[n]_i\tDeltas[n]_j  \mid \cEs[n\Rg]_0}&\leq C_1\etae^{\gamma_1}(\log \tfrac{1}{\etae})^b+C_1\etae^{\gamma_2}(\log \tfrac{1}{\etae})^b,& \forall 1\leq i,j \leq d.
\end{align*}
\item
\begin{align*}
       \E\left[\left|\prod_{s=1}^6\Deltas[n]_{i_s}\right|\mid \cEs[n\Rg]_0\right]&\leq C_1^2\etae^{2\gamma_1}(\log \tfrac{1}{\etae})^{2b}, &\forall 1\leq i_1,\cdots, i_6 \leq d,\\
      \E\left[\left|\prod_{s=1}^6\tDeltas[n]_{i_s}\right|\mid \cEs[n\Rg]_0\right]& \leq C_1^2\etae^{2\gamma_1}(\log \tfrac{1}{\etae})^{2b}, &\forall 1\leq i_1,\cdots, i_6 \leq d.
\end{align*}
\end{enumerate}

\end{lemma}
\begin{proof}
According to \Cref{sec:summary of high prob}, we have
\begin{align*}
   \E\left[\left|\prod_{s=1}^6\Deltas[n]_{i_s}\right|\mid \cEs[n\Rg]_0\right]=\ctO(\eta^{3-3\beta}).
\end{align*}
Since $\gamma_1 <1.5 $ and $\gamma_2< 2$, we can utilize \Cref{thm: moments} and conclude that there exist positive constants $C_2$ and $b$ independent of $\etae$ and $g$ such that
\begin{align}
    \labs{\E[\Deltas[n]_i-\etae\bs[n]_i \mid \cEs[n\Rg]_0]}&\leq C_2\etae^{\gamma_1}(\log \tfrac{1}{\etae})^b+C_2\etae^{\gamma_2}(\log \tfrac{1}{\etae})^b, \forall 1\leq i \leq d, \label{eq:C21}\\
      \labs{\E[\Deltas[n]_i\Deltas[ n]_j-\etae\sum_{l=1}^d \sigs[n]_{i,l} \sigs[n]_{l,j} \mid \cEs[n\Rg]_0]}&\leq C_2\etae^{\gamma_1}(\log \tfrac{1}{\etae})^b+C_2\etae^{\gamma_2}(\log \tfrac{1}{\etae})^b, \forall 1\leq i,j \leq d,\label{eq:C22}\\
      \E\left[\left|\prod_{s=1}^6\Deltas[n]_{i_s}\right|\mid \cEs[n\Rg]_0\right]&\leq C_2^2\etae^{2\gamma_1}(\log \tfrac{1}{\etae})^{2b}, \quad \forall 1\leq i_1,\cdots, i_6 \leq d.\label{eq:C23}
\end{align}
Combining \eqref{eq:C21} - \eqref{eq:C23} with \Cref{lemma: bound tdelta} gives the above lemma.
\end{proof}

\begin{lemma}\label{lemma: u one step bound}
For a test function $g\in \cC^3$, let $u_{l,n}(\vphi):=u(\vphi, l\etae, n\etae)=\E_{\vzeta_t\sim \cPz(\vphi, l\etae, n\etae)}[g(\vzeta_t)]$. If $\normtwo{\bvths[0]-\vphs[0]}=\cO(\sqrt{\eta \log \frac{1}{\eta}})$, then for all $0\leq l \leq n-1$ and $1\leq n \leq \floor{T/\etae}$, 
\begin{align*}
\labs{\E[u_{l+1,n}(\hvphs[l \Rg]+\vDeltas[l])-u_{l+1,n}(\hvphs[l\Rg]+\tvDeltas[l+1])\mid \hvphs[l\Rg]]}  \leq \Cg{1}( \etae^{\gamma_1}+\etae^{\gamma_2})\log (\tfrac{1}{\etae})^b,
\end{align*}
where $\Cg{1}$ is a positive constant independent of $\eta$ and $\hvphs[l\Rg]$ but can depend on $g$.
\end{lemma}
\begin{proof}
By \Cref{lemma:u}, $u_{l,n}(\vphi)\in \cC^3$ for all $l$ and $n$. That is, there exists $K(\cdot)\in G$ such that for all $l,n$, $u_{l,n}(\vphi)$ and its partial derivatives up to the third order are bounded by $K(\vphi)$.

By the law of total expectation and triangle inequality,
\begin{align*}
    &\quad \labs{\E[u_{l+1,n}(\hvphs[l \Rg]+\vDeltas[l])-u_{l+1,n}(\hvphs[l\Rg]+\tvDeltas[l])]\mid \hvphs[l\Rg]}\\
    &\leq\underbrace{\labs{\E[u_{l+1,n}(\hvphs[l\Rg]+\vDeltas[l])-u_{l+1,n}(\hvphs[l\Rg]+\tvDeltas[l]) \mid \hvphs[l\Rg],  \cEs[l\Rg]_0]}}_{\cA_1} \\
    &\quad +\underbrace{\eta^{100} \E[\abs{u_{l+1,n}(\hvphs[l\Rg]+\vDeltas[l])}\mid  \hvphs[l\Rg], \bcEs[l\Rg]_0]}_{\cA_2}\\
    &\quad +\underbrace{ \eta^{100}\E[\abs{u_{l+1,n}(\hvphs[l\Rg]+\tvDeltas[l])}\mid  \hvphs[l\Rg], \bcEs[l \Rg]_0]}_{\cA_3}.
\end{align*}
We first bound $\cA_2$ and $\cA_3$. Since $\hvphs[l\Rg]\in \Gamma$, both $\hvphs[l\Rg]+\vDeltas[l]$ and $\hvphs[l\Rg]+\tvDeltas[l]$ belong to $\Gamma$. Due to compactness of $\Gamma$  and smoothness of $u_{l+1, n}(\cdot)$ on $\Gamma$,  there exist a positive constant $\Cg{2}$ such that $\cA_2+\cA_3 \leq \Cg{2}\eta^{100}$.

We proceed to bound $\cA_1$. Expanding $u_{l+1, n}(\cdot)$ at $\hvphs[l \Rg]$ and by triangle inequality, 
\begin{align*}
    \cAs_1 & \leq \underbrace{\sum_{i=1}^d \labs{\E[\frac{\partial u_{l+1, n}}{\partial \phi_i}( \hvphs[l\Rg])\left(\Deltas[l]_i-\tDeltas[l]_i \right)\mid  \hvphs[l\Rg],\cEs[l\Rg]_0}}_{\cB_1}\\
    & \quad + \underbrace{\frac{1}{2}\sum_{1\leq i, j \leq d}\labs{\E[\frac{\partial^2 u_{l+1, n}}{\partial \phi_i\partial \phi_j}(\hvphs[l\Rg])\left(\Deltas[l]_i\Deltas[l]_j - \tDeltas[l]_i\tDeltas[l]_j\right)\mid  \hvphs[l\Rg],\cEs[l\Rg]_0]}}_{\cB_2}\\
    &\quad +\abs{\cR}+\abs{\tcR},
\end{align*}
where the remainders $\cR$ and $\tcR$ are
\begin{align*}
    \cR&=\frac{1}{6}\sum_{1\leq i,j,p\leq d}\E[\frac{\partial^3 u_{l+1, n}}{\partial \phi_i\partial \phi_j \partial \phi_p}(\hvphs[l\Rg]+\theta \vDeltas[l])\Deltas[l]_i\Deltas[l]_j\mid  \hvphs[l\Rg],\cEs[l\Rg]_0],\\
     \tcR&=\frac{1}{6}\sum_{1\leq i,j,p\leq d}\E[\frac{\partial^3 u_{l+1, n}}{\partial \phi_i\partial \phi_j \partial \phi_p}(\hvphs[l\Rg]+\ttheta \tvDeltas[l])\tDeltas[l]_i\tDeltas[l]_j\tDeltas[l]_p\mid  \hvphs[l\Rg],\cEs[l\Rg]_0],
\end{align*}
for some $\theta, \ttheta \in (0, 1)$. Since $\hvphs[l\Rg] $ belongs to $\Gamma$ which is compact,  there exists a constant $\Cg{3}$ such that for all $1\leq i, j\leq d, 0\leq l \leq n-1, 1\leq n \leq \floor{T/\etae}$, 
\begin{align*}
    \abs{\frac{\partial u_{l+1, n}}{\partial \phi_i}(\hvphs[l\Rg])} \leq \Cg{3},
    \qquad 
      \abs{\frac{\partial^2 u_{l+1, n}}{\partial \phi_i \partial \phi_j}(\hvphs[l\Rg])}\leq \Cg{3}.
\end{align*}
By \Cref{lemma: many delta}, 
\begin{align*}
    \cB_1&\leq d\Cg{3}C_1( \etae^{\gamma_1}+\etae^{\gamma_2})(\log \tfrac{1}{\etae})^b, \qquad \cB_2\leq \frac{d^2}{2} \Cg{3}C_1 ( \etae^{\gamma_1}+\etae^{\gamma_2})(\log \tfrac{1}{\etae})^b.
\end{align*}
Now we bound the remainders. By Cauchy-Schwartz inequality, 
\begin{align*}
 &\quad \labs{\E[\frac{\partial^3 u_{l+1, n}}{\partial \phi_i\partial \phi_j \partial \phi_p}(\hvphs[l\Rg]+\theta \vDeltas[l])\Deltas[l]_i\Deltas[l]_j\Deltas[l]_p\mid  \hvphs[l\Rg],\cEs[l\Rg]_0]}\\
  & \leq\left( \E\left[\left(\frac{\partial^3 u_{l+1, n}}{\partial \phi_i\partial \phi_j \partial \phi_p}(\hvphs[l\Rg]+\theta \vDeltas[l])\right)^2 \mid \hvphs[l\Rg], \cEs[n\Rg]_{0}\right]\right)^{1/2}\times\\
  &\quad \left( \E[(\Deltas[l]_i\Deltas[l]_j\Deltas[l]_p)^2\mid  \hvphs[l\Rg],\cEs[n\Rg]_{0}]\right)^{1/2}.
\end{align*}

Since  $\hvphs[l \Rg]$ and $ \hvphs[l \Rg] + \vDeltas[l] $ both belong to $\Gamma$ which is compact, there exists a constant $\Cg{4}$ such that for all $1\leq i,j,p \leq d$, $0\leq l \leq n-1$ and $1 \leq n \leq \floor{T/\etae}$,
\begin{align*}
    \left(\tfrac{\partial^3 u_{l+1, n}}{\partial \phi_i\partial \phi_j \partial \phi_p}(\hvphs[l\Rg]+\theta \vDeltas[l])\right)^2 \leq \Cg{4}^2.
\end{align*}
  Combining the above inequality with \Cref{lemma: many delta}, we have
\begin{align*}
  \labs{\E[\frac{\partial^3 u_{l+1, n}}{\partial \phi_i\partial \phi_j \partial \phi_p}(\hvphs[l\Rg]+\theta \vDeltas[l])\Deltas[l]_i\Deltas[l]_j\Deltas[l]_p\mid  \hvphs[l\Rg],\cEs[l\Rg]_0]}\leq \Cg{4}C_1 \etae^{\gamma_1}\log (\tfrac{1}{\etae})^b.
\end{align*}
Hence, for all $1 \leq n \leq \floor{T/\etae}, 0\leq l \leq n-1$, 
\begin{align*}
    \abs{\cR}\leq \frac{d^3}{6}\Cg{4}C_1 \etae^{\gamma_1}\log (\tfrac{1}{\etae})^b.
\end{align*}
Similarly, we can show that there exists a constant $\Cg{5}$ such that for all $1 \leq n \leq \floor{T/\etae}, 0\leq l \leq n-1$, 
\begin{align*}
  \abs{\tcR}\leq \frac{d^3}{6}\Cg{5}C_1 \etae^{\gamma_1}\log (\tfrac{1}{\etae})^b.
\end{align*}
Combining the bounds on $\cA_1$ to  $\cA_3$, we have the lemma.
\end{proof}

Finally, we prove \Cref{thm: approx}.
\begin{proof}
For $0\leq l \leq n$, define the random variable $\hvzeta_{l, n}$ which follows the distribution $\cPz(\hvphs[l\Rg ], l, n)$ conditioned on $\hvphs[l \Rg]$. Therefore, $\PP(\hvzeta_{n, n}=\hvphs[n\Rg])=1$ and $\hvzeta_{0, n}\sim \vzeta_{n\etae}$. Denote by $u(\vphi, s, t):=\E_{\vzeta_t\sim \cPz(\vphi, s, t)}[g(\vzeta_t)]$ and $\cT_{l+1, n}:=u_{l+1, n}(\hvphs[l \Rg]+\vDeltas[l], (l+1)\etae, n\etae)-u_{l+1, n}(\hvphs[l\Rg]+\tvDeltas[l] ,(l+1)\etae, n\etae)$. 
\begin{align*}
&\quad\labs{\E[g(\vphs[n\Rg])] -\E[ g(\vzeta(n\etae))]}\\
&\leq \labs{\E[g(\hvzeta_{n, n})- g(\hvzeta_{0, n})\mid \cEs[n\Rg]_0]}+\cO(\eta^{100})\\
&\leq \sum_{l=0}^{n-1}\labs{\E[g(\hvzeta_{l+1, n})- g(\hvzeta_{l, n})\mid \cEs[n\Rg]_0]}+\cO(\eta^{100})\\
& = \sum_{l=0}^{n-1}\labs{\E[u(\hvphs[(l+1)\Rg], (l+1)\etae, n\etae) -u(\hvzeta_{l, l+1} ,(l+1)\etae, n\etae)\mid \cEs[n\Rg]_0]}+\cO(\eta^{100})\\
& = \sum_{l=0}^{n-1}\labs{\E[\cT_{l+1, n}\mid \cEs[n\Rg]_0]}+ \cO(\eta^{100}).
\end{align*}
Noticing that
$\E[\cT_{l+1, n}\mid \cEs[n\Rg]_0]=\E[\E[\cT_{l+1, n}\mid \hvphs[l\Rg], \cEs[l\Rg]_0]\mid \cEs[n\Rg]_0]$, 
 we can apply \Cref{lemma: u one step bound} and obtain that for all $0 \leq n \leq \floor{T/\etae}$,
\begin{align*}
    \labs{\E[g(\vphs[n\Rg])] -\E[ g(\vzeta(n\etae))]} & \leq n \Cg{1}( \etae^{\gamma_1}+\etae^{\gamma_2})(\log \tfrac{1}{\etae})^{b}\\
    & \leq T\Cg{1}( \etae^{\gamma_1-1}+\etae^{\gamma_2-1})(\log \tfrac{1}{\etae})^{b}.
\end{align*}
Notice that $ \etae^{\gamma_1}+\etae^{\gamma_2}=\eta^{0.5-\beta}+\eta^{\beta}$ and $T$, $\Cg{1}$ are both constants that are independent of $\etae$. Let $\beta=0.25$ and we have \Cref{thm: approx}.
\end{proof}
Having established \Cref{thm: approx}, we are thus led to prove \Cref{main thm: flow}.
\begin{proof}[Proof of \Cref{main thm: flow}]
Denote by $\scls=s_0+s_1=\cO(\log \frac{1}{\eta})$, which is the time the global iterate $\bvths$ 
will reach within $\ctO(\eta)$ from $\Gamma$ with high probability. Define $\tvzeta(t)$ to be the solution to the limiting SDE \eqref{eq:zeta2} conditioned on $\cEs[\scls]_0$ and $\tvzeta(0)=\vphs[\scls]$. By \Cref{thm: approx}, we have
\begin{align*}
    \max_{n=0, \cdots, \floor{T/\eta^{0.75}}}\labs{\E [g(\vphs[n\Rg+\scls])- g(\tvzeta(n\eta^{0.75}))\mid \vphs[\scls], \cEs[\scls]_0]}\leq C_g \eta^{0.25}(\log \tfrac{1}{\eta})^b,
\end{align*}
where $\Rg=\floor{\frac{1}{\alpha \eta^{0.75}}}$. Noticing that (i) $g\in \cC^3$ (ii) $\vb, \vsig \in \cC^{\infty}$ and (iii) $\vzeta(t), \tvzeta(t) \in \Gamma, t\in[0, \infty)$ almost surely, we can conclude that given $\cEs[\scls]_0$, $$\normtwo{\vzeta(t)-\tvzeta(t)}=\ctO(\sqrt{\eta}), \quad \forall t \in [0, T].$$
Then  there exists positive constant $b'$ independent of $\eta$ and $g$, and $C_g'$ which is independent of $\eta$ but can depend on $g$ such that
\begin{align*}
    \max_{n=0, \cdots, \floor{T/\eta^{0.75}}}\labs{\E [g(\vphs[n\Rg+\scls])- g(\vzeta(n\eta^{0.75}+\scls H\eta^2))]}\leq C_g'\eta^{0.25}(\log \tfrac{1}{\eta})^{b'}.
\end{align*}
We can view the random variable pairs $\{(\vphs[n\Rg+\scls], \vzeta_{n\eta^{0.75}+\scls\alpha\eta}): n=0, \cdots, \floor{T/\eta^{0.75}}\}$ as reference points and then approximate the value of $g(\vphs)$ and $g(\vzeta(sH\eta^2))$ with the value at the nearest reference points.
By Lemmas~\ref{lemma:dvphs R0} and \ref{lemma:delta phi bound}, for $0\leq r \leq \Rg$ and $ 0 \leq s \leq \Rtot -r$, 
\begin{align*}
    \E[\normtwo{\vphs[s+r] - \vphs}]=\ctO(\eta^{0.375}).
\end{align*}
Since the values of $\vphs$ and $\vzeta$ are restricted to a bounded set, $g(\cdot)$ is Lipschitz on that set. Therefore, we have the theorem.
\end{proof}

\section{Deriving the Slow SDE for Label Noise Regularization}\label{sec:app label noise}
In this section, we formulate how label noise regularization works and provide a detailed derivation of the theoretical results in \Cref{sec:thm label noise}.

Consider training a model for 
$C$-class classification on dataset $\mathcal{D}=\{(\vx_i, y_i)\}_{i=1}^{N}$, where $\vx_i$ denotes the input and $y_i\in [C]$ denotes the label. Denote by  $\Delta_{+}^{C-1}$ the $(C-1)$-open simplex. Let $f(\vtheta;\vx)\in \Delta_{+}^{C-1}$ be the model output on input $\vx$ with parameter $\vtheta$, whose $j$-th coordinate $f_j(\vtheta ; \vx)$ stands for the probability of $\vx$ belonging to class $j$. Let $\ell(\vtheta;\vx, y)$ be the cross entropy loss given input $\vx$ and label $y$, i.e, $\ell(\vtheta; \vx, y)=-\log f_{y}(\vtheta; \vx)$. 

 Adding label noise means replacing the true label $y$ with a fresh noisy label $\hat{y}$ every time we access the sample. Specifically, $\hat{y}$ is set as the true label $y$ with probability $1-p$ and as any other  label with probability $\tfrac{p}{C-1}$, where $p$ is the fixed corruption probability. The training loss  is defined as $\cL(\vtheta)=\frac{1}{N}\sum_{i=1}^N \E[\ell(\vtheta; \vx_i, \hat{y}_i)]$, where the expectation is taken over the stochasticity of $\hat{y}_i$. 
 Notice that given a sample $(\vx, y)$, 
 \begin{align}
 \E[\ell(\vtheta; \vx, \hat{y})]=-(1-p)\log f_y(\vtheta; \vx) -\frac{p}{C-1}\sum_{j\neq y}\log f_j(\vtheta; \vx).\label{eq: sample loss}\end{align}
By the property of cross-entropy loss, \eqref{eq: sample loss} attains its global minimum if and only if $f_j=\frac{p}{C-1}$, for all $j\in [C], j\neq y$ and $f_y=1-p$. Due to the large expressiveness of modern deep learning models, there typically exists a set $\cS^*:=\{\vtheta \mid f_i(\vtheta)=\E[\hat{y}_i], \forall i \in [N]\}$ such that all elements of $\cS^*$ minimizes $\cL(\vtheta)$. Then, the manifold $\Gamma$ is a subset of $\cS^*$. The following lemma relates the noise covariance $\mSig(\vtheta):=\frac{1}{N}\sum_{i\in[N]}\E[\left(\nabla \ell(\vtheta; \vx_i, \hat{y}_i)-\nabla \cL(\vtheta)\right)\left(\nabla \ell(\vtheta; \vx_i, \hat{y}_i)-\nabla \cL(\vtheta)\right)^{\top}]$ to the hessian $\nabla^2\cL(\vtheta)$ for all $\vtheta \in \cS^*$.

\begin{lemma}
    If $f(\vtheta; \vx_i, \hat{y}_i)$ is $\cC^2$-smooth on $\R^d$ given any $i\in [N]$, $\hat{y}_i \in [C]$ and  $\cS^*\neq \varnothing$, then for all $ \vtheta \in \cS^*$,  $\mSig(\vtheta)=\nabla^2 \cL(\vtheta)$.
\end{lemma}
\begin{proof}
Since $\cL(\cdot)$ is $\cC_2$-smooth, $\nabla \cL(\vtheta)=\vzero$ for all $\vtheta \in \cS^*$.
To prove the above lemma, it suffices to show that $\forall i \in [N]$, $\E[\nabla \ell(\vtheta; \vx_i, \hat{y}_i)\nabla\ell(\vtheta; \vx_i, \hat{y}_i)^{\top}]=\nabla^2\cL(\vtheta)$. W.L.O.G, 
let $y=1$ and therefore for all $\vtheta \in S^*$, 
\begin{align*}
    f_1(\vtheta;\vx)&=1-p=:a_1,\\
     f_j(\vtheta;\vx)&=\frac{p}{C-1}=:a_2, \forall j>1, j \in [C].
\end{align*}
Additionally, let $h(x):=-\log (x), x \in \R^+$. The stochastic gradient $\nabla \ell(\vtheta; \vx, \hat{y})$ follows the distribution:
\begin{align*}
    \nabla \ell(\vtheta; \vx, \hat{y})=\begin{cases}
        h'(a_1)\frac{\partial f_1}{\partial \vtheta} &\mathrm{\ w.p.\ } 1-p,\\
        h'(a_2)\frac{\partial f_j}{\partial \vtheta},  &\mathrm{\ w.p.\ } \frac{p}{C-1}, \forall j \in [C], j>1.
    \end{cases}
\end{align*}
Then the covariance of the gradient noise is:
\begin{align*}
    \E[\nabla \ell(\vtheta; \vx, \hat{y})\nabla\ell(\vtheta; \vx, \hat{y})^{\top}]&=(1-p)(h'(a_1))^2\frac{\partial f_1(\vtheta^*)}{\partial \vtheta^*}\left(\frac{\partial f_1(\vtheta^*)}{\partial \vtheta^*}\right)^{\top}\\
    &\quad + \frac{p(h'(a_2))^2}{C-1}\sum_{j>1}\frac{\partial f_j(\vtheta^*)}{\partial \vtheta^*}\left(\frac{\partial f_j(\vtheta^*)}{\partial \vtheta^*}\right)^{\top}.
\end{align*}
And the hessian is:
\begin{align*}
    \nabla^2 \cL(\vtheta)&=\underbrace{(1-p)h'(a_1)\frac{\partial^2 f_1}{\partial\vtheta^2 }+\frac{p h'(a_2)}{C-1}\sum_{j>1}\frac{\partial^2 f_j}{\partial \vtheta^2 }}_{\cT}\\
    &\quad + (1-p)h''(a_1)\frac{\partial f_1}{\partial\vtheta }\left(\frac{\partial f_1}{\partial\vtheta }\right)^{\top}+\frac{ph''(a_2)}{C-1}\sum_{j>1}\frac{\partial f_j}{\partial \vtheta }
    \left(\frac{\partial f_j(\vtheta)}{\partial \vtheta }\right)^{\top}.\\
\end{align*}
Since $\sum_{j\in [C]} f_i=1$, 
\begin{align}\frac{\partial^2 f_1}{\partial\vtheta^2}=-\sum_{j>1}\frac{\partial^2 f_j}{\partial\vtheta^2}. \label{eq:sum fj}
\end{align}
    Also, notice that $h'(x)=-\frac{1}{x}$. Therefore, 
\begin{align}
    (1-p)h'(a_1)=\frac{ph'(a_2)}{C-1}.\label{eq:ratio}
\end{align}
Substituting \eqref{eq:sum fj} and \eqref{eq:ratio} into the expression of $\cT$ gives $\cT=\vzero$, which simplifies $\nabla^2\cL(\vtheta)$ as the following form:
\begin{align*}
    \nabla^2 \cL(\vtheta)&=(1-p)h''(a_1)\frac{\partial f_1}{\partial \vtheta }
    \left(\frac{\partial f_j(\vtheta)}{\partial \vtheta }\right)^{\top}+\frac{ph''(a_2)}{C-1}\sum_{j>1}\frac{\partial f_j}{\partial \vtheta }
    \left(\frac{\partial f_j(\vtheta)}{\partial \vtheta }\right)^{\top}.\\
\end{align*}
Again notice that $h''(x)=h'(x)$ for all $x \in \R^+$. Therefore, $\nabla^2\cL(\vtheta)=\mSig(\vtheta)$.
\end{proof}

With the property $\mSig(\vtheta)=\nabla^2\cL(\vtheta)$, we are ready to prove \Cref{thm:local sgd label noise}. 
\begin{proof}[Proof of \Cref{thm:local sgd label noise}]
    Recall the general form of the slow SDE:
    \begin{align}
        \dd \vzeta(t)=\frac{1}{\sqrt{B}}\partial \Phi(\vzeta)\mSig^{1/2}(\vzeta)\dd \vW(t)+\frac{1}{2B}\partial^2 \Phi(\vzeta)\left[\mSig(\vzeta) + (K-1) \mPsi(\vzeta)\right]\dd t,
        \label{eq:general SDE}
    \end{align}
where $\mPsi$ is defined in \Cref{def:mPsi}. 
Since for $\vzeta \in \Gamma$, $\mSig(\vzeta)=\nabla^2\cL(\vzeta)$, then 
\begin{align}
\partial \Phi(\vzeta)\mSig^{1/2}(\vzeta)=\vzero. \label{eq:diffusion zero}
\end{align}
Now we show that
\begin{align}
    \partial^2 \Phi(\vzeta)[\mSig(\vzeta)]=-\gradGa \tr(\nabla^2 \cL(\vzeta)).\label{eq:grad trace}
\end{align}
Since $\nabla^2\cL(\vzeta)=\mSig(\vzeta)$, $\cV_{\nabla^2\cL(\vzeta)}[\mSig]=\frac{1}{2}\mI$.
By \Cref{lemma:mani inner perp perp},
\begin{align*}
    \partial^2\Phi(\vzeta)[\mSig(\vzeta)]=-\frac{1}{2}\partial \Phi(\vzeta)\nabla^3\cL(\vzeta)[\mI]=-\frac{1}{2}\nabla_{\Gamma}\tr(\nabla^2 \cL(\vzeta)).
\end{align*}
Finally, we show that
\begin{align} 
\partial^2 \Phi(\vzeta)[\mPsi(\vzeta)] =-\gradGa\frac{1}{2H\eta}\tr (F(2H\eta \nabla^2\cL(\vzeta) )). \label{eq:trF}
\end{align}
Define $\hat{\psi}(x):=x\psi(x)=e^{-x}-1+x$.
By definition of $\Psi(\vzeta)$, when $\mSig(\vzeta)=\nabla^2\cL(\vzeta)$, $\Psi(\vzeta)=\hat{\psi}(2\eta H \nabla^2 \cL(\vzeta))$, where $\hat{\psi}(\cdot)$ is interpreted as a matrix function. Since $\psi(2\eta H \nabla^2 \cL(\vzeta))\in \spann\{\vu \vu^{\top}\mid \vu \in T_{\vzeta}^{\perp}(\Gamma)\}$, by \Cref{lemma:mani inner perp perp},
\begin{align*}
    \partial^2 \Phi(\vzeta)[\Psi(\vzeta)]&=-\frac{1}{2}{\partial\Phi(\vzeta)\tr\psi(2\eta H \nabla^2 \cL(\vzeta))}.
\end{align*}
By the chain rule, we have \eqref{eq:trF}. Combining \eqref{eq:diffusion zero},\eqref{eq:grad trace} and \eqref{eq:trF} gives the theorem.
\end{proof}

\newpage
\section{Experimental Details}
In this section, we specify the experimental details that are omitted in the main text. Our experiments are conducted on CIFAR-10 \citep{krizhevsky2009learning} and ImageNet \cite{ILSVRC15}. Our code is available at \url{https://github.com/hmgxr128/Local-SGD}. Our implementation of ResNet-56 \citep{he2016deep} and VGG-16 \citep{vgg} is based on the high-starred repository by Wei Yang\footnote{\href{https://github.com/bearpaw/pytorch-classification}{https://github.com/bearpaw/pytorch-classification}} and we use the implementation of ResNet-50 from torchvision 0.3.1. We run all CIFAR-10 experiments with $\Bloc=128$ on 8 NVIDIA Tesla P100 GPUs while ImageNet experiments are run on 8 NVIDIA A5000 GPUS with $\Bloc=32$. All ImageNet experiments are trained with ResNet-50.


We generally adopt  the following training strategies. We do not add any
 momentum unless otherwise stated. We follow the suggestions by
 \citet{jia2018highly} and do not add weight decay to the bias and learnable
 parameters in the normalization layers. For all models with BatchNorm layers,
 we go through $100$ batches of data with batch size $\Bloc$ to estimate the
 running mean and variance before evaluation. Experiments on both datasets
 follow the standard data augmentation pipeline in \citet{he2016deep} except the
 label noise experiments. Additionally, we use FFCV \citep{leclerc2022ffcv} to
 accelerate data loading for ImageNet training.

 Slightly different from the update rule of Local SGD in \Cref{sec:intro}, we
 use sampling without replacement unless otherwise stated. 
 See \Cref{sec:pseudocode} for implementation details and discussion.

\subsection{Post-local SGD Experiments in \Cref{sec:intro}}\label{sec:intro detail}
\paragraph{CIFAR-10 experiments.} We simulate $32$ clients with $B=4096$. We follow the linear scaling rule and linear learning rate warmup strategy suggested by \citet{goyal2017accurate}.  We first run $250$ epochs of SGD with the learning rate gradually ramping up from $0.1$ to $3.2$ for the first 50 epochs. Resuming from the model obtained at epoch $250$, we run Local SGD with $\eta=0.32$.
Note that  we conduct grid search for the initial learning rate among $\{0.005, 0.01, 0.05, 0.1, 0.15, 0.2\}$ and choose the learning rate with which parallel SGD ($H=1$) achieves the best test accuracy.  We also make sure that the optimal learning rate resides in the middle of the set. The weight decay $\lambda$ is set as $5\times 10^{-4}$.   As for the initialization scheme, we follow \citet{lin2020dont}  and \citet{goyal2017accurate}. Specifically, we use Kaiming Normal \citep{he2015delving} for the weights of convolutional layers and initialize the weights of fully-connected layers by a Gaussian distribution with mean zero and standard deviation $0.01$. The weights for normalization layers are initialized as one. All bias parameters are initialized as zero. We report the mean and standard deviation over 5 runs.

\paragraph{ImageNet experiments.} We simulate $256$ workers with $B=8192$. We
follow the linear scaling rule and linear learning rate warmup strategy
suggested by \citet{goyal2017accurate}. We first run 100 epochs of SGD where the
learning rate linearly ramps up from $0.5$ to $16$ for the first $5$ epochs and
then decays by a factor of $0.1$ at epoch $50$. Resuming from epoch $100$, we
run Local SGD with $\eta=0.16$. Note that  we conduct grid search for the
initial learning rate among $\{0.05, 0.1, 0.5, 1\}$ and choose the learning rate
with which parallel SGD ($H=1$) achieves the best test accuracy. We also make
sure that the optimal learning rate resides in the middle of the set. The weight
decay $\lambda$ is set as $1\times 10^{-4}$ and we do not add any momentum. The
initialization scheme follows the implementation of torchvision 0.3.1.  We
report the mean and standard deviation over 3 runs. \subsection{Experimental
Details for Figures~\ref{fig:effect} and~\ref{fig:add effect}}\label{sec:detail
regime} \paragraph{CIFAR-10 experiments.} We use ResNet-56 for all CIFAR-10
experiments in the two figures.  
We simulate $32$ workers with $B=4096$ and set the weight decay as $5\times 10^{-4}$. For Figures~\ref{fig:effect-a} and \ref{fig:effect-b},  we set $\eta=0.32$, which is the same as the learning rate after decay in \Cref{fig:cifar_intro}. For \Cref{fig:effect-a}, we adopt the same initialization scheme introduced in the corresponding paragraph in \Cref{sec:intro detail}. For Figures \ref{fig:effect-b}, \ref{fig:effect-e} and ~\ref{fig:add effect-c}, we use the model at epoch 250 in \Cref{fig:cifar_intro} as the pre-trained model. Additionally, we  use a training budget of $250$ epochs for \Cref{fig:effect-e}.  In \Cref{fig:add effect-e}, we use Local SGD with momentum $0.9$, where the momentum buffer is kept locally and never averaged. We run SGD with momentum $0.9$ for $150$ epochs to obtain the pre-trained model, where the learning rate ramps up from $0.05$ to $1.6$ linearly in the first $150$ epochs. Note that we conduct grid search for the initial learning rate among $\{0.01, 0.05, 0.1, 0.15, 0.2\}$ and choose the learning rate with which parallel SGD ($H=1$) achieves the highest test accuracy.  We also make sure that the optimal learning rate resides in the middle of the set. Resuming from epoch $150$, we run Local SGD $H=1$ (i.e., SGD) and $24$ with $\eta=0.16$ and decay $\eta$ by $0.1$ at epoch $226$. For Local SGD $H=900$, we resume from the model at epoch $226$ of $H=24$ with $\eta=0.016$. 
We report the mean and standard deviation over $3$ runs for Figures~\ref{fig:effect-a}, \ref{fig:effect-b} and \ref{fig:add effect-c}, and over $5$ runs for  \Cref{fig:effect-e}. 
\paragraph{ImageNet experiments.} We simulate $256$ clients with $B=8192$ and set the weight decay as $1\times 10^{-4}$. In \Cref{fig:effect-d} , both Local SGD and SGD start from the same random initialization. We warm up the learning rate from $0.1$ to $3.2$ in the first $5$ epochs and decay the learning rate by a factor of $0.1$ at epochs $50$ and $100$. For Figures~\ref{fig:effect-c},  \ref{fig:effect-f} and \ref{fig:add effect-d}, we use the model at epoch 100 in \Cref{fig:img_intro} as the pre-trained model. In \Cref{fig:effect-c}, we set the learning rate as $0.16$, which is the same as the learning rate after epoch 100 in \Cref{fig:img_intro}. 
Finally, in Figures~\ref{fig:effect-c}, \ref{fig:effect-f}, \ref{fig:add effect-b} and \ref{fig:add effect-d}, we report the mean and average over $3$ runs.

\subsection{Details for Experiments in \Cref{fig:add exp}}\label{sec:add exp detail}
For all experiments in \Cref{fig:add exp}, we train a ResNet-56 model on CIFAR-10. We report mean test accuracy over three runs and the shaded area reflects the standard deviation. For \Cref{fig:add exp-a}, we use the same setup as Figures~\ref{fig:effect-a} and \ref{fig:effect-b}  for training from random initialization and from a pre-trained model respectively except the learning rate. For \Cref{fig:add exp-b}, we resume from the model obtained at epoch $250$ in \Cref{fig:cifar_intro} and train for another $250$ epochs. For \Cref{fig:add exp-c}, we follow the same procedure as \Cref{fig:cifar_intro} except that we use sampling with replacement. We also ensure that the total numbers of iterations in Figures~\ref{fig:cifar_intro} and \ref{fig:add exp-c} are the same.

\subsection{Details for Experiments on the Effect of the Diffusion Term}\label{sec: diffusion details} 
\begin{figure}[t]
    \vspace{-0.5in}
\begin{center}
 \subfigure[CIFAR-10, start from \#$250$.]{
        \includegraphics[width=0.4\textwidth]{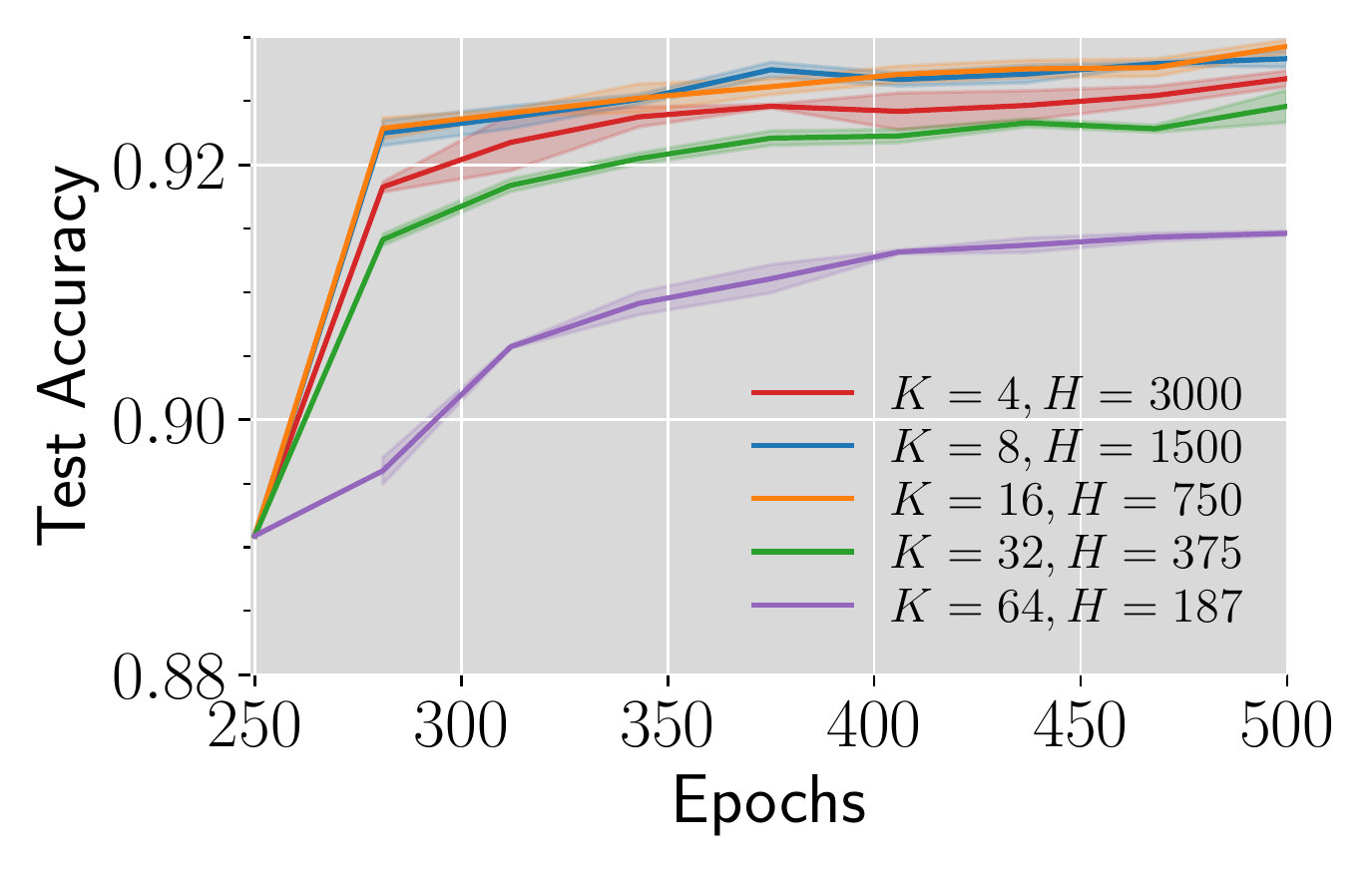}\label{fig:cifar lsr curve}   }
    \hspace{0.1in}
         \subfigure[ImageNet, start from \#$100$.]{
        \includegraphics[width=0.4\textwidth]{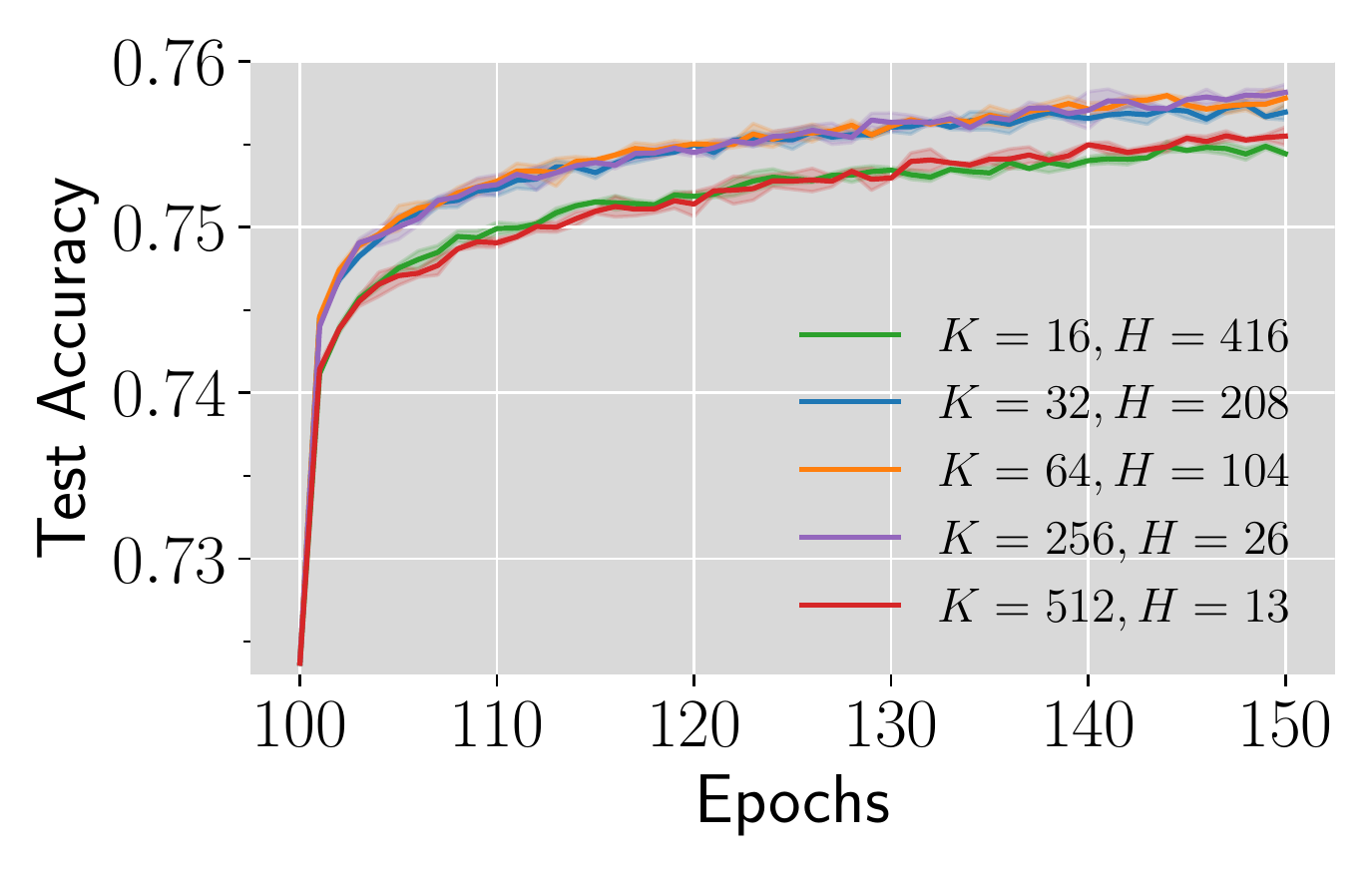}\label{fig:img lsr curve}
    }
\subfigure[CIFAR-10, start from \#$250$, optimal $H$.]{
        \includegraphics[width=0.4\textwidth]{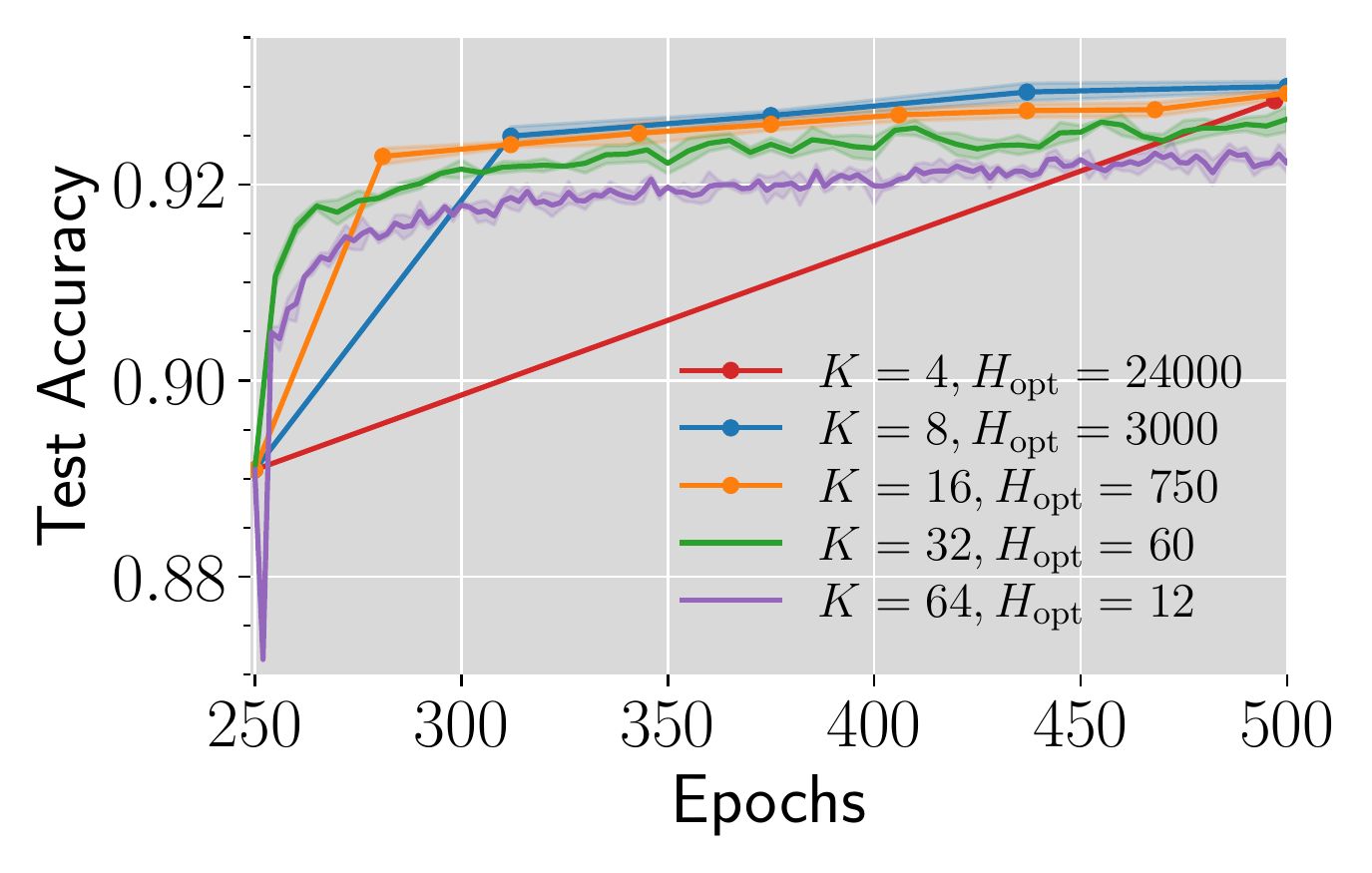}\label{fig:cifar lsr curve opt}   }
            \hspace{0.1in}
         \subfigure[ImageNet, start from \#$100$, optimal $H$.]{
        \includegraphics[width=0.4\textwidth]{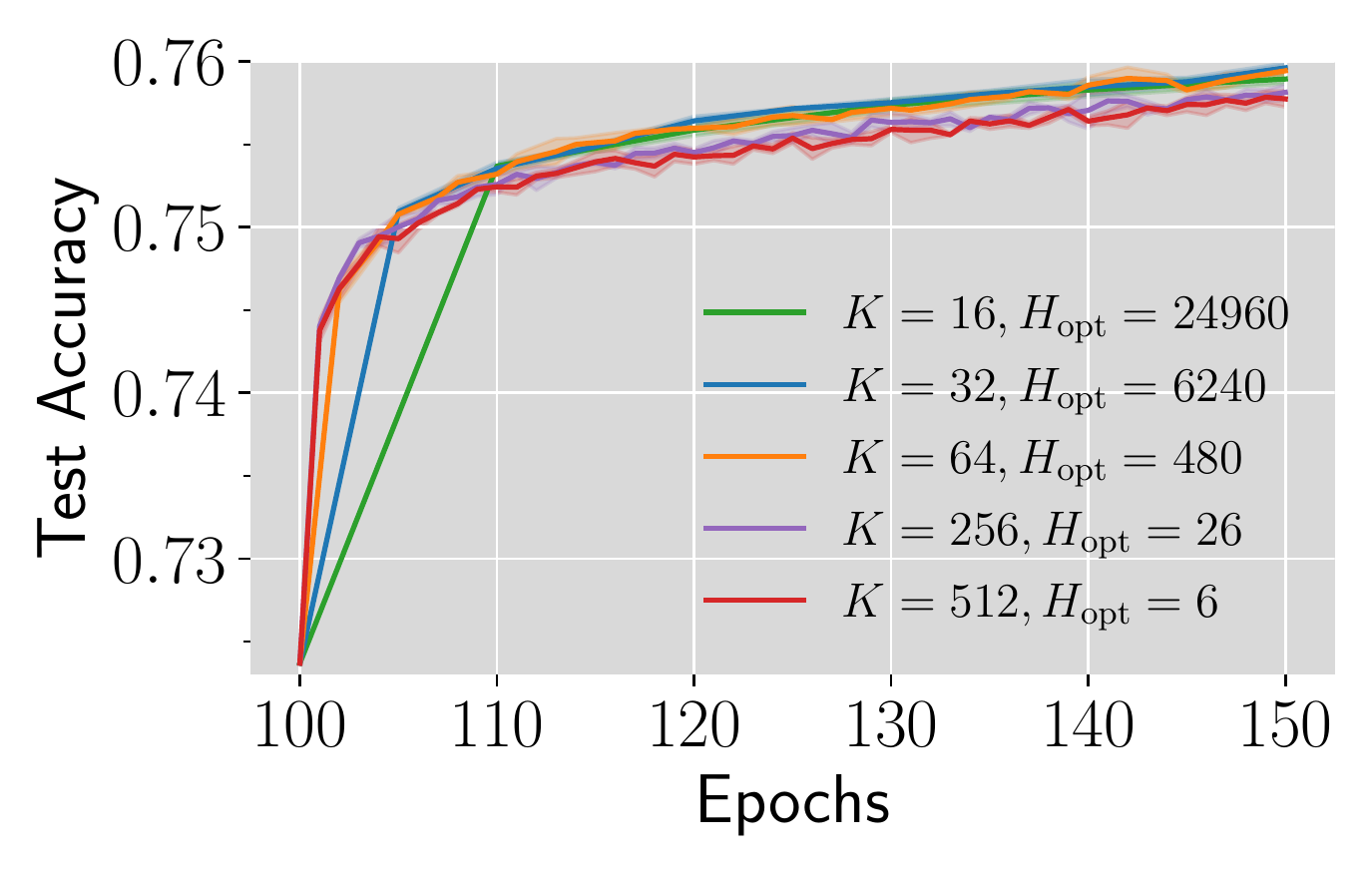}\label{fig:img lsr curve opt}
    }
    \vspace{-0.1in}
    \caption{
   The learning curves for experiments in \Cref{fig:lsr}. }
  \label{fig:lsr-detail}
\end{center}
    \vspace{-0.15in}
\end{figure}

\paragraph{CIFAR-10 experiments.} The model we use is ResNet-56. For \Cref{fig:cifar diffusion}, we first run SGD with batch size $128$ and learning rate $\eta=0.5$ for $250$ epochs to obtain the  pre-trained model. The initialization scheme is the same as the corresponding paragraph in \Cref{sec:intro detail}. Resuming from epoch $250$ with $\eta=0.05$, we run Local SGD with $K=16$ until epoch $6000$ and run all other setups for the same number of iterations. We report the mean and standard deviation over $3$ runs.

\paragraph{ImageNet experiments.} For Figures~\ref{fig:img diffusion} and \ref{fig:img lsr}, we start from the model obtained at epoch $100$ in \Cref{fig:img_intro}. In \Cref{fig:img diffusion}, we run Local SGD with $K=256$ for another $150$ epochs with $\eta=0.032$. We run all other setups for the same number of iterations with the same learning rate.

\subsection{Details for Experiments on the Effect of Global Batch Size}\label{sec: bs details} 
\paragraph{CIFAR-10 experiments.} The model we use is ResNet-56.
We resume from the model obtained in \Cref{fig:cifar_intro} at epoch $250$ and train for another 250 epochs. The local batch size for all runs is $\Bloc=128$. We first make grid search of $\eta$ for SGD with $K=16$ among $\{0.04, 0.08, 0.16, 0.32, 0.64\}$ and find that the final test accuracy varies little across different learning rates (within $0.1\%$). Then we choose $\eta=0.32$.  For the green curve in \Cref{fig:cifar lsr}, we search for the optimal $H$ for $K=16$ and keep $\alpha$ fixed when scaling $\eta$ with $K$. For the red curve in \Cref{fig:cifar lsr}, we search for the optimal $H$  for each $K$ among $\{6, 12, 60, 120, 300, 750, 1500, 3000, 6000, 12000, 24000\}$ and also make sure that $H$ does not exceed the total number of iterations for 250 epochs. The learning curves for constant and optimal $\alpha$ are visualized in \Cref{fig:cifar lsr curve,fig:cifar lsr curve opt} respectively. We report the mean and standard deviation over three runs. 

\paragraph{ImageNet experiments.}   We start from the model obtained at epoch $100$ in \Cref{fig:img_intro} and train for another $50$ epochs. The local batch size for all runs is $\Bloc=32$. We first make grid search among $\{0.032, 0.064, 0.16, 0.32\}$ for $H=1$ to achieve the best test accuracy and choose $H=0.064$. For the orange curve in \Cref{fig:img lsr}, we search $H$ among   $\{2, 4, 6, 13, 26, 52, 78, 156\}$ for $K=256$ to achieve the optimal test accuracy and the keep $\alpha$ constant as we scale $\eta$ with $K$. To obtain the optimal $H$ for each $K$, we search among $\{6240, 7800, 10400, 12480, 15600, 20800, 24960, 31200\}$ for $K=16$,  $\{1600, 3120, 4160, 5200, 6240, 7800, 10400\}$ for $K=32$, $\{312, 480, 520, 624, 800, 975, 1040, \\ 1248, 1560, 1950\}$ for $K=64$, and $\{1, 2, 3, 6, 13\}$ for $K=512$. The learning curves for constant and optimal $\alpha$ are visualized in \Cref{fig:img lsr curve,fig:img lsr curve opt} respectively. We report the mean and standard deviation over three runs.

\subsection{Details for Experiments on Label Noise Regularization}\label{sec:label noise detail}

For all label noise experiments,  
we do not use data augmentation, use sampling with replacement, and set the corruption probability as $0.1$. 
We simulate $32$ workers with $B=4096$ in \Cref{fig:label_noise} and $4$ workers with $B=512$ in \Cref{fig:K=4 label noise}. We use ResNet-56 with GroupNorm with the number of groups $8$ for \Cref{fig:label noise resnet} and VGG-16 without normalization for Figures~\ref{fig:label noise vgg} and \ref{fig:K=4 label noise}. 
Below we list the training details for ResNet-56 and VGG-16  respectively.

\paragraph{ResNet-56.} As for the model architecture, we replace the batch normalization layer in Yang's implementation with group normalization such that the training loss is independent of the sampling order. We also use Swish activation \citep{ramachandran2017searching} in place of ReLU to ensure the smoothness of the loss function.  We generate the pre-trained model by running label noise SGD with  corruption probability $p=0.1$ for $500$ epochs ($6,000$ iterations). We initialize the model by the same strategy introduced in the first paragraph of \Cref{sec:intro detail}. Applying the linear warmup scheme proposed by \citet{goyal2017accurate}, we gradually ramp up the learning rate $\eta$ from $0.1$ to $3.2$ for the first $20$ epochs and multiply the learning rate by $0.1$ at epoch $250$. All subsequent experiments in \Cref{fig:label noise resnet} (a) use learning rate $0.1$. The weight decay $\lambda$ is set as $5\times 10^{-4}$ . Note that adding weight decay in the presence of normalization accelerates the limiting dynamics and will not affect the implicit regularization on the original loss function \citep{li2022fast}. 

\paragraph{VGG-16.} We follow Yang's implementation of the model architecture except that we replace maximum pooling with average pooling and use Swish activation \citep{ramachandran2017searching} to make the training loss smooth. We initialize all weight parameters by Kaiming Normal and all bias parameters as zero. The pre-trained model is obtained by running label noise SGD with total batch size $4096$ and corruption probability $p=0.1$ for $6000$ iterations. We use a linear learning rate warmup from $0.1$ to $0.5$ in the first $500$ iterations. All runs in Figures~\ref{fig:label noise vgg} and \ref{fig:K=4 label noise} resume from
the model obtained by SGD with label noise. In \Cref{fig:label noise vgg},  we use learning rate $\eta=0.1$. In \Cref{fig:K=4 label noise}, we set $\eta=0.005$ for $H=97,000$ and $\eta=0.01$ for SGD ($H=1$). The weight decay $\lambda$ is set as zero. 





 
\end{document}